\documentclass[twoside,11pt]{article}

\usepackage{jmlr2e}
\usepackage{natbib}
\usepackage{bm}
\usepackage{graphicx}
\usepackage{amsmath}
\usepackage{amsfonts}
\usepackage{subfig}
\usepackage{tikz}
\usepackage{dsfont}
\usepackage{url}
\usepackage{xcolor}
\usepackage{comment}
\usepackage{rotating}
\usepackage[colorlinks=true,urlcolor=blue!50!black,linkcolor=blue!50!black,citecolor=blue!50!black]{hyperref}
\usepackage{algorithm}
\usepackage{algorithmic}
\usepackage{longtable}
\usepackage{ifthen}
\usepackage{mymacros}
\usetikzlibrary{arrows}

\jmlrheading{?}{????}{?-?}{12/14}{??/??}{Joris M.~Mooij, Jonas~Peters, Dominik~Janzing, Jakob~Zscheischler and Bernhard~Sch\"olkopf}

\ShortHeadings{Distinguishing cause from effect}{Mooij, Peters, Janzing, Zscheischler and Sch\"olkopf}
\firstpageno{1}
\title{Distinguishing cause from effect using observational data: methods and benchmarks}
\author{\name Joris M.~Mooij\thanks{Part of this work was done while JMM, JP and JZ were with the MPI T\"ubingen.} \email \url{j.m.mooij@uva.nl}\\
  \addr Institute for Informatics, University of Amsterdam\\
  Postbox 94323, 1090 GH Amsterdam, The Netherlands\\
  \AND
  \name Jonas Peters \email \url{jonas.peters@tuebingen.mpg.de}\\
  \addr 
  Max Planck Institute for Intelligent Systems\\
  Spemannstra{\ss}e 38, 72076 T\"ubingen, Germany, and\\
  Seminar for Statistics, ETH Z\"urich\\
  R\"amistrasse 101, 8092 Z\"urich, Switzerland\\
  \AND
  \name Dominik Janzing \email \url{janzing@tuebingen.mpg.de}\\
  \addr Max Planck Institute for Intelligent Systems\\
  Spemannstra{\ss}e 38, 72076 T\"ubingen, Germany\\
  \AND
  \name Jakob Zscheischler \email \url{jakob.zscheischler@env.ethz.ch}\\
  \addr Institute for Atmospheric and Climate Science, ETH Z\"urich\\
  Universit\"atstrasse 16, 8092 Z\"urich, Switzerland\\
\AND
  \name Bernhard Sch\"olkopf \email \url{bs@tuebingen.mpg.de}\\
  \addr Max Planck Institute for Intelligent Systems\\
  Spemannstra{\ss}e 38, 72076 T\"ubingen, Germany
}
\editor{??}

\graphicspath{{fig/}}

\tikzstyle{var}=[circle,draw=black,fill=black!25,thick,minimum size=24pt] %or 12pt
\tikzstyle{varh}=[circle,draw=black,fill=white,thick,minimum size=24pt] %or 12pt
\tikzstyle{arr}=[->,>=stealth',draw=black,fill=black,thick]
\tikzstyle{arrh}=[->,>=stealth',draw=black,fill=black,thick,dashed]

\begin{document}
\maketitle

\begin{abstract}%
The discovery of causal relationships from purely observational data is a
fundamental problem in science. The most elementary form of such a
causal discovery problem is to decide whether $X$ causes $Y$ or, alternatively,
$Y$ causes $X$, given joint observations of two variables $X,Y$. An example
is to decide whether altitude causes temperature, or vice versa, given
only joint measurements of both variables.
Even under the simplifying assumptions of no confounding, no feedback loops, and no selection bias, 
such bivariate causal discovery problems are challenging. Nevertheless, several approaches for addressing 
those problems have been proposed in recent years. We review two families of such methods:
Additive Noise Methods (ANM) and Information Geometric Causal Inference (IGCI). We present
the benchmark \CEP\ that consists of data for \nrpairs\ different cause-effect pairs selected from \nrdatasets\ datasets from various domains 
(e.g., meteorology, biology, medicine, engineering, economy, etc.) and motivate our decisions regarding
the ``ground truth'' causal directions of all pairs.
We evaluate the performance of several bivariate causal discovery methods on these real-world benchmark
data and in addition on artificially simulated data. 
Our empirical results on real-world data indicate that certain methods are indeed able to distinguish 
cause from effect using only purely observational data, although more benchmark data would be needed 
to obtain statistically significant conclusions.
One of the best performing methods overall is the additive-noise method originally proposed
by \citet{HoyerJanzingMooijPetersSchoelkopf_NIPS_08}, which obtains an accuracy of 63 $\pm$ 10 \% and 
an AUC of 0.74 $\pm$ 0.05 on the real-world benchmark. As the main theoretical contribution of this work
we prove the consistency of that method.
\end{abstract}


%\tableofcontents

\section{Introduction}\label{sec:introduction} %%%%%%%%%%%%%%%%%%%%%%%%%%%%%%%%%%%%%%%%%%%%%%%%%%%%%%%%%%%%%%%%%%%%%%%%%%%%%%%%%%%%%%%%%%%%%%%%%%%%%%%%%%%%%%%%%%%%%%%%%%%%%%%%%%%%%%%%%%%%%%%%%%%%%%%%

An advantage of having knowledge about causal relationships rather than 
statistical associations is that the former enables prediction of the effects 
of actions that perturb the observed system. % \citep{Pearl2000,SpirtesGlymourScheines2000}.
Knowledge of cause and effect can also have implications on the applicability of semi-supervised learning and covariate shift adaptation \citep{ScholkopfJPSZMJ2012}.
While the gold standard for identifying causal relationships is controlled experimentation, in many
cases, the required experiments are too expensive, unethical, 
or technically impossible to perform. The development of methods to identify
causal relationships from purely observational data therefore constitutes an
important field of research.

An observed statistical
dependence between two variables $X$, $Y$ can be explained by a causal
influence from $X$ to $Y$, a causal influence from $Y$ to $X$, a possibly
unobserved common cause that influences both $X$ and $Y$ \citep[``confounding'', see e.g.,][]{Pearl2000}, a possibly
unobserved common effect of $X$ and $Y$ that is conditioned
upon in data acquisition \citep[``selection bias'', see e.g.,][]{Pearl2000}, or combinations of these.
% (see also Figure~\ref{fig:bivariate_causal_relations}).
%\Jonas{Shall we add a reference to the markov condition? maybe sth. like ``(This statement assumes that there is an underlying causal graph with possibly more than two variables such that the observed distribution is Markov wrt the graph.)''?}
%\Joris{You suggested it before, but the nice aspect of the previous sentence is that
%it does not make that assumption. For example, it also works in case of feedback. So I don't see exactly where you would want to make that remark, and what it would add.}
Most state-of-the-art causal discovery algorithms that attempt to distinguish
these cases based on observational data require that 
$X$ and $Y$ are part of a larger set of observed random variables influencing
each other. For example, in that case, and under a genericity condition called ``faithfulness'',
conditional independences between subsets of observed variables allow one
to draw partial conclusions regarding their causal relationships \citep{SpirtesGlymourScheines2000,Pearl2000,RichardsonSpirtes2002,Zhang2008}.

In this article, we focus on the \emph{bivariate} case, assuming that only two
variables, say $X$ and $Y$, have been observed. We simplify the causal discovery problem considerably by
assuming no confounding, no selection bias and no feedback. We study how to distinguish $X$ causing $Y$
from $Y$ causing $X$ using only purely observational data, i.e., a finite
i.i.d.\ sample drawn from the joint distribution $\Prb_{X,Y}$.\footnote{We denote probability
distributions by $\Prb$ and probability densities (typically with respect to Lebesgue measure on $\RN^d$) by $p$.}
%%For a long time, this was considered to be impossible. 
%Some have considered this task to be impossible. 
%For example, \citet[][Remark 17.16]{Wasserman2004} %\Jonas{this dangerous? Leider habe ich das Buch nicht hier. sollte man doppelt checken.}
%writes: ``We could try to learn the correct causal graph from data but this is dangerous. In
%fact it is impossible with two variables.''
As an example, consider the data visualized in Figure~\ref{fig:example_task}. The question is: does
$X$ cause $Y$, or does $Y$ cause $X$? The true answer is ``$X$ causes $Y$'', as here 
$X$ is the altitude of weather stations and $Y$ is the mean temperature measured at these weather 
stations (both in arbitrary units). In the absence of knowledge about the measurement procedures that the variables correspond
with, one can try to exploit the subtle statistical patterns in the data in order to find the causal 
direction. This challenge of distinguishing cause from effect using only
observational data has attracted increasing interest recently \citep{MooijJanzing_JMLR_10,NIPSCausalityChallenge2008,Guyon++Challenges}.
Approaches to causal discovery based on conditional independences do not work here, as $X$ and
$Y$ are typically dependent, and there are no other observed variables to condition on. 

\begin{figure}
  \centerline{\includegraphics[width=0.25\textwidth]{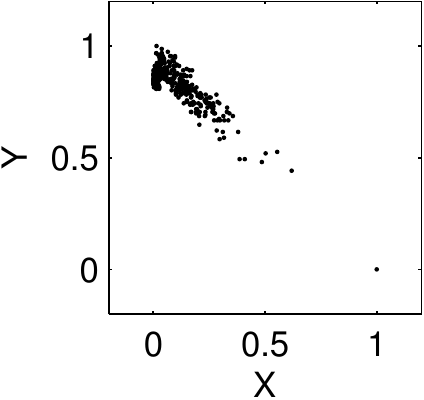}}
  \caption{\label{fig:example_task}Example of a bivariate causal discovery task: decide whether $X$ causes $Y$, or $Y$ causes $X$, using only the observed data (visualized here as a scatter plot).}
\end{figure}

A variety of causal discovery methods has been proposed in recent
years \citep{FriedmanNachman2000,KanoShimizu2003,ShimizuHoyerHyvarinenKerminen2006,SunJanzingSchoelkopf2006,
SunJanzingSchoelkopf2008,HoyerJanzingMooijPetersSchoelkopf_NIPS_08,MooijJanzingPetersSchoelkopf_ICML_09,
ZhangHyvarinen2009,JanzingHoyerSchoelkopf2010,Mooij_et_al_NIPS_10,Daniusis_et_al_UAI_10,
Mooij_et_al_NIPS_11,Shimizu++2011,Janzing_et_al_AI_12,HyvarinenSmith2013,Peters2014biom,Kpotufe++2014,Nowzohour2015,Sgouritsa++2015}
that were claimed to be able to solve this task under certain assumptions.
%Smith et al: mostly consider multivariate case and time-series data. Only Patel's kappa and tau might be interesting, but these authors themselves say that one should be careful with interpreting these measures as signifying causation.
All these approaches exploit the \emph{complexity} of the marginal and conditional probability distributions, in one way or the other.
On an intuitive level, the idea is that the factorization of the joint density
$p_{C,E}(c,e)$ of cause $C$ and effect $E$ into
%\footnote{Measure-theoretically, one should consider \emph{regular conditional distributions} $\Prb_{E\given C}$, see \citep[e.g.,][p.\ 229--231]{Durrett1996}, a rigorous generalization of the concept of conditional probability density that can also be applied to variables that do not have a density.} 
$p_{C}(c) p_{E \given C}(e\given c)$
%$\Prb_{C,E}$ of cause $C$ and effect $E$ into $\Prb_{C} \Prb_{E \given C}$ (i.e., the marginal distribution of the cause multiplied with the regular conditional distribution
%of effect given cause) 
typically yields models of lower total complexity than the alternative factorization into
%$\Prb_{E} \Prb_{C \given E}$ 
$p_{E}(e) p_{C \given E}(c \given e)$. 
%(the marginal distribution of the effect multiplied with the (regular) conditional distribution of cause given effect). %\Jonas{Ich weiss, dass das ein muehsamer Kommentar ist, aber ich bin mir nicht sicher, dass diese Schreibweise mathematisch korrekt ist. das $\otimes$ sieht nach einem Produktmass aus, aber $\Prb_{E \given C}$ ist m.E. nicht einfach ein normales Mass. Oder hast du diese Schreibweise irgendwo in einem math. Buch gefunden?}
Although this idea is intuitively appealing, it is not clear how to define complexity. 
If ``complexity'' and ``information'' are measured by Kolmogorov complexity and algorithmic information, respectively, as in \citep{JanzingSchoelkopf2010,Lemeire2013}, one can show that the statement
``$p_C$ contains no information about $p_{E\given C}$'' implies that
the sum of the complexities of $p_C$ and $p_{E\given C}$ cannot be greater than the sum of the complexities of $p_E$ and $p_{C\given E}$. 
Some approaches, instead,
define certain classes of ``simple'' conditionals,
e.g., Additive Noise Models \citep{HoyerJanzingMooijPetersSchoelkopf_NIPS_08} and second-order exponential models \citep{SunJanzingSchoelkopf2006,JanzingSunSchoelkopf}, and 
infer $X$ to be the cause of $Y$ whenever
$\Prb_{Y\given X}$ is from this class (and $\Prb_{X\given Y}$ is not). Another approach 
that employs complexity in a more implicit way postulates that 
$\Prb_{C}$ contains no information about $\Prb_{E \given C}$ \citep{Janzing_et_al_AI_12}.

%Indeed, each of these methods effectively uses its own measure of complexity.
%\Joris{@Dominik: could you elaborate a bit on this? I agree with reviewer \#2 that this sounds a bit vague.}

Despite the large number of methods for bivariate causal discovery that has
been proposed over the last few years, their practical performance has not been
studied very systematically (although domain-specific studies have been
performed, see \citep{Smith++2011,Statnikov++2012}). The present work attempts to address
this by presenting benchmark data and reporting extensive empirical results on the 
performance of various bivariate causal discovery methods. Our main contributions are fourfold:
\begin{itemize}
  \item We review two families of bivariate causal discovery methods, \emph{Additive Noise Methods (ANM)} \citep[originally proposed by][]{HoyerJanzingMooijPetersSchoelkopf_NIPS_08}, and \emph{Information Geometric Causal Inference (IGCI)} \citep[originally proposed by][]{Daniusis_et_al_UAI_10}.
  \item We present a detailed description of the benchmark \CEP\ that we collected over the years
  for the purpose of evaluating bivariate causal discovery methods. It currently consists of data for \nrpairs\ 
different cause-effect pairs selected from \nrdatasets\ datasets from various domains 
(e.g., meteorology, biology, medicine, engineering, economy, etc.).
  \item We report the results of extensive empirical evaluations of the performance of several members of the ANM and IGCI families,
    both on artificially simulated data as well as on the \CEP\ benchmark.
  \item We prove the consistency of the original implementation of ANM that was proposed by \citet{HoyerJanzingMooijPetersSchoelkopf_NIPS_08}.
\end{itemize}
The \CEP\ benchmark data are provided on our website \citep{MooijJanzingSchoelkopf2014}.
%\url{http://webdav.tuebingen.mpg.de/cause-effect/}.
In addition, all the code (including the code to run the experiments and create the figures) is provided on the first author's 
homepage\footnote{\url{http://www.jorismooij.nl/}} under an open source license to allow others to reproduce and build on our work.

%The main contribution of this work is to provide extensive empirical
%results on how well two families of such methods perform:
%\emph{Additive Noise Methods (ANM)} \citep[originally proposed by][]{HoyerJanzingMooijPetersSchoelkopf_NIPS_08},
%and \emph{Information Geometric Causal Inference (IGCI)} \citep[originally proposed by][]{Daniusis_et_al_UAI_10}.
%Other contributions are a proof of the consistency of the original implementation of
%ANM \citep{HoyerJanzingMooijPetersSchoelkopf_NIPS_08} and a detailed description of the \CEP\ 
%benchmark data that we collected over the years for the purpose of evaluating bivariate causal discovery methods.

%Our conclusions are twofold:
%\begin{enumerate}
%\item Certain implementations of ANM work significantly better than chance on both real-world and simulated data, 
%and their performance is robust against small perturbations of the data. 
%\item IGCI, which was designed for the deterministic (i.e., noise-free) case, does generally not perform 
%significantly better than chance on noisy data, and performance varies greatly depending on implementation 
%details and small perturbations of the data.
%\end{enumerate}
%The former conclusion is in line with earlier reports, but the latter conclusion is surprising, 
%considering that good performance of certain IGCI methods has been reported on several occasions in earlier work.

The structure of this article is somewhat unconventional, as it partially consists of a review of existing
methods, but it also contains new theoretical and empirical results.
We will start in the next subsection by giving a more rigorous definition of the causal discovery task we consider in 
this article. In Section~\ref{sec:ANM} we give a review of ANM, an approach based on the assumed additivity of
the noise, and describe various ways of implementing this idea for bivariate causal discovery.
In Appendix~\ref{sec:consistency} we provide a proof for the consistency of the original ANM implementation that
was proposed by \citet{HoyerJanzingMooijPetersSchoelkopf_NIPS_08}.
In Section~\ref{sec:IGCI}, we review IGCI, a method that exploits the independence of the distribution
of the cause and the functional relationship between cause and effect. This method is designed for the
deterministic (noise-free) case, but has been reported to work on noisy data as well. Section~\ref{sec:experiments}
gives more details on the experiments that we have performed, the results of which are
reported in Section~\ref{sec:results}. Appendix~\ref{sec:dataset} describes the \CEP\ benchmark data set that
we used for assessing the accuracy of various methods. We conclude in Section~\ref{sec:discussion}.
%, our main
%conclusion being that our empirical results suggest that certain bivariate causal discovery methods are 
%indeed able to distinguish cause from effect using only purely observational data with an accuracy of 63-69\%
%and an AUC of 0.70-0.78, although the benchmark data set is too small to draw statistically significant conclusions.

\subsection{Problem setting}
In this subsection, we formulate the problem of interest central to this work. We tried to make this section
as self-contained as possible and hope that it also appeals to readers who are not familiar with
the terminology in the field of causality. For more details, we refer the reader to \citep{Pearl2000}.

Suppose that $X,Y$ are two random variables with joint distribution
$\Prb_{X,Y}$.
%\footnote{In this work, ``density'' will typically refer to the Radon-Nikodym derivative
%with respect to Lebesgue measure on $\RN^d$, but if the possible outcomes are restricted to a finite
%set of values in $\RN^d$, it refers to the Radon-Nikodym derivative with respect to the counting 
%measure on the set of possible values (often called a ``probability mass function'').}
This observational distribution corresponds to measurements of $X$
and $Y$ in an experiment in which $X$ and $Y$ are both (passively) observed.
If an external intervention (i.e., from outside the system under consideration) changes 
some aspect of the system, then in general, this may lead to a change in the joint 
distribution of $X$ and $Y$. In particular, we will consider a perfect intervention\footnote{Different types of ``imperfect'' interventions can be considered as well, see e.g., \citet{Eberhardt2007,EatonMurphy07,MooijHeskes_UAI_13}. In this paper we only consider perfect interventions.} 
``$\intervene{x}$'' (or more explicitly: ``$\intervene{X=x}$'') that forces the variable $X$ to have the value $x$, and leaves 
the rest of the system untouched. We denote the resulting interventional
distribution of $Y$ as $\Prb_{Y \given \intervene{x}}$,
a notation inspired by \citet{Pearl2000}. This interventional distribution
corresponds to the distribution of $Y$ in an experiment in which $X$ has 
been set to the value $x$ by the experimenter, after which $Y$ is measured.
%Following \citet[][sec.~3]{Pearl2000}, we denote the joint distribution of $X,Y$ under a
%\emph{perfect intervention} $\intervene(X=x)$, i.e., an external intervention from outside
%the system that forces the variable $X$ to have the value $x$ (thereby overriding the value $X$ would get in the absence of the
%intervention) as $p(X,Y \given \intervene(X=x))$ \Jonas{Ich finde die Definition unklar (setzt man $X$ auf $x$ NACHDEM man die anderen Variablen gemessen hat?). Alternative: Dichte annehmen - stetig, diskret oder gemischt- und dann Markov factorization}.
Similarly, we may consider a
perfect intervention $\intervene{y}$ that forces $Y$ to have the value $y$,
leading to the interventional distribution $\Prb_{X \given \intervene{y}}$ of $X$.

For example, $X$ and $Y$ could be binary variables corresponding to whether the 
battery of a car is empty, and whether the start engine of the car is broken.
Measuring these variables in many cars, we get an estimate of the joint distribution $\Prb_{X,Y}$.
The marginal distribution $\Prb_X$, which only considers the distribution of $X$, can be obtained
by integrating the joint distribution over $Y$. The conditional distribution $\Prb_{X \given Y=0}$ 
corresponds with the distribution of
$X$ for the cars with a broken start engine (i.e., those cars for which we observe
that $Y = 0$). The interventional distribution $\Prb_{X \given \intervene{Y=0}}$, on the 
other hand, corresponds with the distribution of $X$ after destroying the start engines 
of all cars (i.e., after actively setting $Y=0$).
Note that the distributions $\Prb_X, \Prb_{X \given Y=0}, \Prb_{X \given \intervene{Y=0}}$ may 
all be different.

%In particular, the marginal distribution can be different from the interventional distribution.
%If, for example, $\Prb_{Y}$ is different from $\Prb_{Y \given \intervene{x}}$ for some values of $x$,
%then in the absence of selection bias this means that $X$ must cause $Y$:
In the absence of selection bias, we define:\footnote{In the presence of selection bias, one has to be
careful when linking causal relations to interventional distributions. Indeed, if one would (incorrectly) apply
Definition~\ref{def:causes} to the conditional interventional distributions $\Prb_{Y \given \intervene{X=x},S=s}$ 
instead of to the unconditional interventional distributions $\Prb_{Y \given \intervene{X=x}}$ (e.g., because one
is not aware of the fact that the data has been conditioned on $S$), one may 
obtain incorrect conclusions regarding causal relations.}
\begin{definition}\label{def:causes}
%  We say that $X$ \emph{causes} $Y$ if\footnote{More precisely, 
%  %the set $\{y \,:\,p(y \given \intervene{x}) \ne p(y \given \intervene{x'})\}$ is not a Lebesgue null set for some $x,x'$.
%the corresponding cumulative distribution functions differ for some $x,x'$ (densities are only defined up to null sets).} $p(y \given \intervene{x}) \ne p(y \given \intervene{x'})$ for some $x, x'$. %\todo{Jonas, is this the correct definition, also for densities w.r.t.\ Lebesgue measure? Or do we need to assume continuity of the density in that case? Or maybe use the cdf instead?}
  We say that $X$ \textbf{causes} $Y$ if\ \,$\Prb_{Y \given \intervene{x}} \ne \Prb_{Y \given \intervene{x'}}$ for some $x, x'$. %\todo{Jonas, is this the correct definition, also for densities w.r.t.\ Lebesgue measure? Or do we need to assume continuity of the density in that case? Or maybe use the cdf instead?}
\end{definition}
Causal relations can be \emph{cyclic}, i.e., $X$ causes $Y$ and $Y$ also causes $X$. 
For example, an increase of the global temperature causes sea ice to melt, which causes the temperature to rise further (because ice reflects more sun light).

In the context of multiple variables $X_1, \dots, X_p$ with $p \ge 2$, 
we define \emph{direct} causation in the absence of selection bias as follows:
\begin{definition}
  $X_i$ is a \textbf{direct cause} of $X_j$ with respect to $X_1,\dots,X_p$ if 
  $$\Prb_{X_j \given \intervene{X_i=x, \B{X}_{\setminus ij}=\B{c}}} \ne \Prb_{X_j \given \intervene{X_i=x', \B{X}_{\setminus ij} = \B{c}}}$$
  for some $x, x'$ and some $\B{c}$, where $\B{X}_{\setminus ij} := X_{\{1, \dots, p\} \setminus \{i,j\}}$ are all other variables besides $X_i,X_j$.
\end{definition}
In words: $X$ is a direct cause of $Y$ with respect to a set of variables under consideration
if $Y$ depends on the value we force $X$ to have in a perfect intervention, while fixing all 
other variables. The intuition is that a direct causal relation of $X$ on $Y$ is not mediated via the
other variables. The more variables one considers, the harder it becomes experimentally to distinguish
direct from indirect causation, as one has to keep more variables fixed.\footnote{For the special case
$p=2$ that is of interest in this work, we do not need to distinguish indirect from direct causality,
as they are equivalent in that special case. However, we introduce this concept in order to define causal
graphs on more than two variables, which we use to explain the concepts of confounding and selection bias.}
%The main reason for defining direct causality nonetheless is that we need it to define causal graphs. 

We may visualize direct causal relations in a \emph{causal graph}:
\begin{definition}
  The \textbf{causal graph} $\C{G}$ has variables $X_1, \dots, X_p$ as nodes, and a directed edge
  from $X_i$ to $X_j$ if and only if $X_i$ is a direct cause of $X_j$ with respect to $X_1, \dots, X_p$.
\end{definition}
Note that this definition allows for cyclic causal relations. In contrast with the typical assumption
in the causal discovery literature, we do not assume here that the causal graph is necessarily a Directed Acyclic 
Graph (DAG).

If $X$ causes $Y$, we generically have that $\Prb_{Y \given \intervene{x}} \ne \Prb_Y$.
Figure~\ref{fig:bivariate_causal_relations} illustrates how various causal relationships 
between $X$ and $Y$ (and at most one other variable) generically give rise to different (in)equalities between 
marginal, conditional, and interventional distributions involving $X$ and $Y$.\footnote{Note 
that the list of possibilities in Figure~\ref{fig:bivariate_causal_relations} 
is not exhaustive, as (i) feedback relationships with a
latent variable were not considered; (ii) combinations of the cases shown are
possible as well, e.g., (d) can be considered to be the combination of (a) and
(b), and both (e) and (f) can be combined with all other cases; (iii) more than
one latent variable could be present.}

Returning to the example of the empty batteries ($X$) and broken start engines
($Y$), it seems reasonable to assume that these two variables are not causally
related and case (c) in Figure~\ref{fig:bivariate_causal_relations} would apply,
and therefore $X$ and $Y$ must be
statistically independent. 

In order to illustrate case (f), let us introduce a
third binary variable, $S$, which measures whether the car starts or not. If
the data acquisition is done by a car mechanic who only considers cars that
do not start ($S=0$), then we are in case (f): conditioning on the common effect
$S$ of $X$ and $Y$ leads to selection bias, i.e., $X$ and $Y$ are statistically
dependent when conditioning on $S$ (even though they are not directly causally
related). Indeed, if we know that a car doesn't start, then learning that the
battery is not empty makes it much more likely that the start engine is broken.

Another way in which two variables that are not directly causally related can
still be statistically dependent is case (e), i.e., if they have a common cause. As
an example, take for $X$ the number of stork breeding pairs (per year) and for
$Y$ the number of human births (per year) in a country. Data has been 
collected for different countries and shows a significant correlation between 
$X$ and $Y$ \citep{Matthews2000}.
Few people nowadays believe that storks deliver babies, or the other way around, 
and therefore it seems reasonable to assume that $X$ and $Y$ are not directly causally related. 
One obvious confounder ($Z$ in Figure~\ref{fig:bivariate_causal_relations}(e)) 
that may explain the observed dependence between $X$ and $Y$ is land area.

When data from all (observational and interventional) distributions are available, 
it becomes straightforward in principle to distinguish the six cases in Figure~\ref{fig:bivariate_causal_relations}
simply by checking which (in)equalities in 
Figure~\ref{fig:bivariate_causal_relations} hold.
In practice, however, we often only have data from the observational distribution $\Prb_{X,Y}$
(for example, because intervening on stork population or human birth rate is impractical). 
Can we then still infer the causal relationship between $X$ and $Y$? If, under certain assumptions, we can
decide upon the causal direction, we say that the causal direction is \emph{identifiable}
from the observational distribution (and our assumptions). 

In this work, we will simplify matters considerably
by considering only (a) and (b) in Figure~\ref{fig:bivariate_causal_relations} as possibilities. In other words, we
assume that $X$ and $Y$ are dependent (i.e., $\Prb_{X,Y} \ne \Prb_X \Prb_Y$), there is 
no confounding (common cause of $X$ and $Y$), no selection bias (common effect of 
$X$ and $Y$ that is implicitly conditioned on), and no feedback between $X$ and $Y$ 
(a two-way causal relationship between $X$ and $Y$). 
Inferring the causal direction between $X$ and $Y$, i.e., deciding which of the two
cases (a) and (b) holds, using \emph{only the observational distribution $\Prb_{X,Y}$} 
is the challenging task that we consider in this work.\footnote{Note that this is a different 
question from the one often faced in problems in 
epidemiology, economics and other disciplines where causal considerations play an important role.
There, the causal direction is often known \emph{a priori}, i.e.,
one can exclude case (b), but the challenge is to distinguish case (a) from
case (e) or a combination of both. Even though our empirical results indicate
that some methods for distinguishing case (a) from case (b) still perform reasonably
well when their assumption of no confounding is violated by adding a latent
confounder as in (e), we do not claim that these methods can be used to
distinguish case (e) from case (a).}

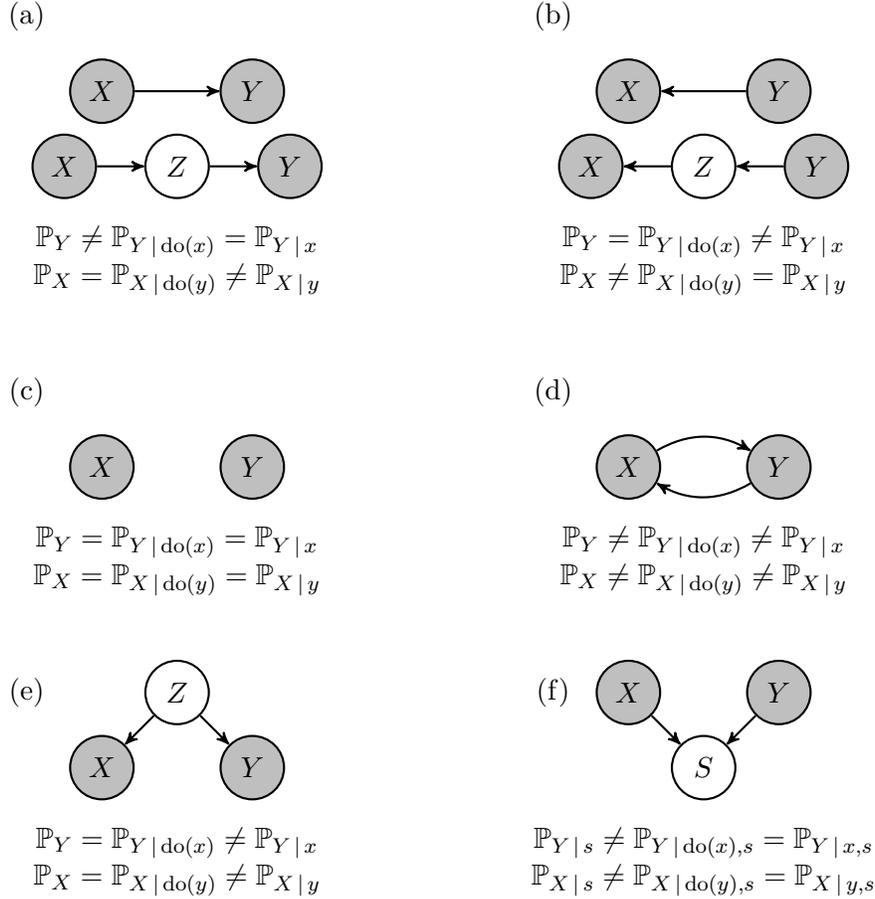
\begin{figure}[t]
\centerline{\begin{tikzpicture}
  \begin{scope}
    \draw (-2.0,1.0) node (X) {(a)};
    \draw (-1.0,0.0) node (X) [var] {$X$};
    \draw ( 1.0,0.0) node (Y) [var] {$Y$};
    \draw [arr] (X) -- (Y);
    \draw (-1.5,-1.0) node (X2) [var] {$X$};
    \draw ( 1.5,-1.0) node (Y2) [var] {$Y$};
    \draw ( 0.0,-1.0) node (Z2) [varh] {$Z$};
    \draw [arr] (X2) -- (Z2);
    \draw [arr] (Z2) -- (Y2);
    \draw ( 0.0,-2.0) node {$\Prb_Y \ne \Prb_{Y \given \intervene{x}} = \Prb_{Y \given x}$};
    \draw ( 0.0,-2.5) node {$\Prb_X = \Prb_{X \given \intervene{y}} \ne \Prb_{X \given y}$};
  \end{scope}
  \begin{scope}[xshift=7cm]
    \draw (-2.0,1.0) node (X) {(b)};
    \draw (-1.0,0.0) node (X) [var] {$X$};
    \draw ( 1.0,0.0) node (Y) [var] {$Y$};
    \draw [arr] (Y) -- (X);
    \draw (-1.5,-1.0) node (X2) [var] {$X$};
    \draw ( 0.0,-1.0) node (Z2) [varh] {$Z$};
    \draw ( 1.5,-1.0) node (Y2) [var] {$Y$};
    \draw [arr] (Y2) -- (Z2);
    \draw [arr] (Z2) -- (X2);
    \draw ( 0.0,-2.0) node {$\Prb_Y = \Prb_{Y \given \intervene{x}} \ne \Prb_{Y \given x}$};
    \draw ( 0.0,-2.5) node {$\Prb_{X} \ne \Prb_{X \given \intervene{y}} = \Prb_{X \given y}$};
  \end{scope}
  \begin{scope}[yshift=-5cm]
    \draw (-2.0,1.0) node (X) {(c)};
    \draw (-1.0,0.0) node (X) [var] {$X$};
    \draw ( 1.0,0.0) node (Y) [var] {$Y$};
    \draw ( 0.0,-1.0) node {$\Prb_Y = \Prb_{Y \given \intervene{x}} = \Prb_{Y \given x}$};
    \draw ( 0.0,-1.5) node {$\Prb_X = \Prb_{X \given \intervene{y}} = \Prb_{X \given y}$};
  \end{scope}
  \begin{scope}[xshift=7cm,yshift=-5cm]
    \draw (-2.0,1.0) node (X) {(d)};
    \draw (-1.0,0.0) node (X) [var] {$X$};
    \draw ( 1.0,0.0) node (Y) [var] {$Y$};
    \draw [arr,bend left] (Y) edge (X);
    \draw [arr,bend left] (X) edge (Y);
    \draw ( 0.0,-1.0) node {$\Prb_Y \ne \Prb_{Y \given \intervene{x}} \ne \Prb_{Y \given x}$};
    \draw ( 0.0,-1.5) node {$\Prb_X \ne \Prb_{X \given \intervene{y}} \ne \Prb_{X \given y}$};
  \end{scope}
  \begin{scope}[yshift=-9cm]
    \draw (-2.0,1.0) node (X) {(e)};
    \draw (-1.0,0.0) node (X) [var] {$X$};
    \draw ( 1.0,0.0) node (Y) [var] {$Y$};
    \draw ( 0.0,1.0) node (Z) [varh] {$Z$};
    \draw [arr] (Z) -- (X);
    \draw [arr] (Z) -- (Y);
    \draw ( 0.0,-1.0) node {$\Prb_Y = \Prb_{Y \given \intervene{x}} \ne \Prb_{Y \given x}$};
    \draw ( 0.0,-1.5) node {$\Prb_X = \Prb_{X \given \intervene{y}} \ne \Prb_{X \given y}$};
  \end{scope}
  \begin{scope}[xshift=7cm,yshift=-9cm]
    \draw (-2.0,1.0) node (X) {(f)};
    \draw ( 0.0,0.0) node (S) [varh] {$S$};
    \draw (-1.0,1.0) node (X) [var] {$X$};
    \draw ( 1.0,1.0) node (Y) [var] {$Y$};
    \draw [arr] (X) -- (S);
    \draw [arr] (Y) -- (S);
    \draw ( 0.0,-1.0) node {$\Prb_{Y \given s} \ne \Prb_{Y \given \intervene{x}, s} = \Prb_{Y \given x, s}$};
    \draw ( 0.0,-1.5) node {$\Prb_{X \given s} \ne \Prb_{X \given \intervene{y}, s} = \Prb_{X \given y, s}$};
  \end{scope}
\end{tikzpicture}}
\caption{\label{fig:bivariate_causal_relations}Several possible causal relationships between two observed variables $X,Y$ and a single latent variable: (a) $X$ causes $Y$; (b) $Y$ causes $X$; (c) $X,Y$ are not causally related; (d) feedback relationship, i.e., $X$ causes $Y$ and $Y$ causes $X$; (e) a hidden confounder $Z$ explains the observed dependence; (f) conditioning on a hidden selection variable $S$ explains the observed dependence.\protect\footnotemark\ We used shorthand notation regarding quantifiers: equalities are generally valid, inequalities not necessarily. For example, $\Prb_X = \Prb_{X \given y}$ means that $\forall y:\, \Prb_X = \Prb_{X \given y}$, whereas $\Prb_X \ne \Prb_{X \given y}$ means $\exists y:\, \Prb_X \ne \Prb_{X \given y}$. In all situations except (c), $X$ and $Y$ are (generically) dependent, i.e., $\Prb_{X,Y} \ne \Prb_X \Prb_Y$. The basic task we consider in this article is deciding between (a) and (b), using only data from $\Prb_{X,Y}$.}
% All equalities are valid for all $x, y$, but the inequalities are only (generically) valid for some $x,y$. Note that all inequalities here are only \emph{generic}, i.e., they do not necessarily hold, although they hold \emph{typically}. 
%\Jonas{Die Figure ist sehr hilfreich! In der ersten Zeile von (a) (z.B.) muss man allerdings aufpassen, dass $X$ nicht zufaellig auf die Marginalverteilung von $X$ interveniert wird (etwas allgemeinere Definition von do), sonst bleibt $p(Y)$ gleich, oder? Klarheit haette man z.B., wenn wir z.B. Dichten voraussetzen...}
\end{figure}
\footnotetext{Here, we assume that the intervention is performed \emph{before} the conditioning. Since conditioning and intervening do not commute in general, one has to be careful when modeling causal processes in the presence of selection bias to take into account the actual ordering of these events.}

\section{Additive Noise Models}\label{sec:ANM} %%%%%%%%%%%%%%%%%%%%%%%%%%%%%%%%%%%%%%%%%%%%%%%%%%%%%%%%%%%%%%%%%%%%%%%%%%%%%%%%%%%%%%%%%%%%%%%%%%%%%%%%%%%%%%%%%%%%%%%%%%%%%%%%%%%%%%%%%%%%%%%%%%%%%%%%

In this section, we review a family of causal discovery methods that exploits \emph{additivity} of the noise.
We only consider the bivariate case here. More details and extensions to the multivariate case can be found in \citep{HoyerJanzingMooijPetersSchoelkopf_NIPS_08,PetersMooijJanzingSchoelkopf_JMLR_14}.

\subsection{Theory}

%\Jonas{I would keep the theory (i.e., identifiability) parts very short for three reasons: (1) The paper's main focus is on the methodological and empirical side (consistency result fits well), (2) The paper is rather long already, (3) To me, it looks a lot like recycling (in fact, until now, some parts are copy-paste). Is a short reference enough?}

There is an extensive body of literature on causal modeling
and causal discovery that assumes that effects are linear functions of their
causes plus independent, Gaussian noise. These models are known as
\emph{Structural Equation Models} (SEM) \citep{Wright1921,Bollen1989} and are popular in
econometry, sociology, psychology and other fields. Although the assumptions of
linearity and Gaussianity are mathematically convenient, they are not always
realistic. More generally, one can define \emph{Functional Models} (also known
as \emph{Structural Causal Models} (SCM) or \emph{Non-Parametric Structural Equation
Models} (NP-SEM)) \citep{Pearl2000} in which effects are modeled as (possibly
nonlinear) functions of their causes and latent noise variables.

\subsubsection{Bivariate Structural Causal Models}

In general, if $Y \in \RN$ is a direct effect of a cause $X \in \RN$ and $m$ latent causes $\B{U} = (U_1,\dots,U_m) \in \RN^m$, 
then it is intuitively reasonable to model this relationship as follows:
\begin{equation}\label{eq:structural_eqn}
  \left\{\begin{array}{l}
    Y = f(X,U_1,\dots,U_m), \\
    X \indep \B{U}, \quad X \sim p_{X}(x), \quad \B{U} \sim p_{\B{U}}(u_1,\dots,u_m)
  \end{array}\right.
\end{equation}
where $f: \RN \times \RN^m \to \RN$ is a possibly nonlinear function (measurable with respect to the Borel
sets of $\RN \times \RN^m$ and $\RN$), and $p_{X}(x)$ and $p_{\B{U}}(u_1,\dots,u_m)$
are the joint densities of the observed cause $X$ and latent causes $\B{U}$ (with respect to Lebesgue measure
on $\RN$ and $\RN^m$, respectively). The assumption 
that $X$ and $\B{U}$ are independent (``$X \indep \B{U}$'') is justified by the assumption that there is no confounding, no selection bias, 
and no feedback between $X$ and $Y$.\footnote{Another assumption that we have made here is that there is no \emph{measurement noise}, i.e., 
noise added by the measurement apparatus. Measurement noise would mean that instead of measuring $X$ itself, we observe a noisy version $\tilde X$, but $Y$ is still a function of $X$, the (latent) variable $X$ that is not corrupted by measurement noise.}
We will denote the observational distribution corresponding to \eref{eq:structural_eqn} by $\Prb_{X,Y}^{\eref{eq:structural_eqn}}$.
By making use of the semantics of SCMs \citep{Pearl2000}, \eref{eq:structural_eqn} also induces interventional distributions 
$\Prb_{X \given \intervene y}^{\eref{eq:structural_eqn}} = \Prb_X^{\eref{eq:structural_eqn}}$ and $\Prb_{Y \given \intervene x}^{\eref{eq:structural_eqn}} = \Prb_{Y \given x}^{\eref{eq:structural_eqn}}$.

\begin{figure}[t]
\centerline{\begin{tikzpicture}
  \begin{scope}
    \draw (-2.0,1.0) node (X) {(a)};
    \draw (-1.0,0.0) node (X) [var] {$X$};
    \draw ( 1.0,0.0) node (Y) [var] {$Y$};
    \draw [arr] (X) -- (Y);
    \draw (0.0,1.5) node (U1) [varh] {$U_1$};
    \draw (1.0,2.5) node (Udots) [varh] {$\dots$};
    \draw (2.0,1.5) node (Um) [varh] {$U_m$};
    \draw [arr] (U1) -- (Y);
    \draw [arr] (Udots) -- (Y);
    \draw [arr] (Um) -- (Y);
    \draw [dotted] (U1) -- (Udots);
    \draw [dotted] (Udots) -- (Um);
    \draw [dotted] (Um) -- (U1);
%    \draw ( 0.0,-2.0) node {$\Prb_Y \ne \Prb_{Y \given \intervene{x}} = \Prb_{Y \given x}$};
%    \draw ( 0.0,-2.5) node {$\Prb_X = \Prb_{X \given \intervene{y}} \ne \Prb_{X \given y}$};
  \end{scope}
  \begin{scope}[xshift=5cm]
    \draw (-2.0,1.0) node (X) {(b)};
    \draw (-1.0,0.0) node (X) [var] {$X$};
    \draw ( 1.0,0.0) node (Y) [var] {$Y$};
    \draw [arr] (X) -- (Y);
    \draw (1.0,1.5) node (EY) [varh] {$E_Y$};
    \draw [arr] (EY) -- (Y);
%    \draw ( 0.0,-2.0) node {$\Prb_Y \ne \Prb_{Y \given \intervene{x}} = \Prb_{Y \given x}$};
%    \draw ( 0.0,-2.5) node {$\Prb_X = \Prb_{X \given \intervene{y}} \ne \Prb_{X \given y}$};
  \end{scope}
  \begin{scope}[xshift=10cm]
    \draw (-2.0,1.0) node (X) {(c)};
    \draw (-1.0,0.0) node (X) [var] {$X$};
    \draw ( 1.0,0.0) node (Y) [var] {$Y$};
    \draw [arr] (Y) -- (X);
    \draw (-1.0,1.5) node (EX) [varh] {$E_X$};
    \draw [arr] (EX) -- (X);
%    \draw ( 0.0,-2.0) node {$\Prb_Y \ne \Prb_{Y \given \intervene{x}} = \Prb_{Y \given x}$};
%    \draw ( 0.0,-2.5) node {$\Prb_X = \Prb_{X \given \intervene{y}} \ne \Prb_{X \given y}$};
  \end{scope}
\end{tikzpicture}}
\caption{\label{fig:simple_SCM}Causal graphs of Structural Causal Models \eref{eq:structural_eqn}, \eref{eq:structural_eqn_reduced_forwards} and \eref{eq:structural_eqn_reduced_backwards}, respectively. (a) and (b) are interventionally equivalent, (b) and (c) are only observationally equivalent in general.}
\end{figure}
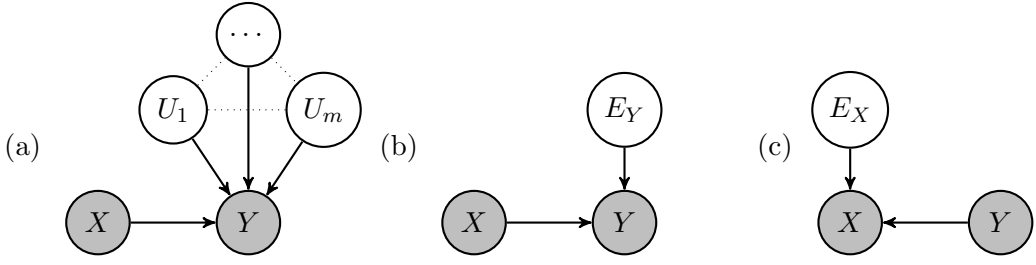

As the latent causes $\B{U}$ are unobserved anyway, we can summarize their influence by a single ``effective'' \emph{noise}
variable $E_Y \in \RN$ (also known as ``disturbance term''):
\begin{equation}\label{eq:structural_eqn_reduced_forwards}
  \left\{\begin{array}{l}
  Y = f_Y(X,E_Y) \\
  X \indep E_Y, \quad X \sim p_{X}(x), \quad E_Y \sim p_{E_Y}(e_Y).
  \end{array}\right.
\end{equation}
This simpler model can be constructed in such a way that it induces the same (observational and interventional) distributions as \eref{eq:structural_eqn}:
\begin{proposition}
  Given a model of the form \eref{eq:structural_eqn} for which the observational distribution 
  has a positive density with respect to Lebesgue measure, there exists a model of the form 
  \eref{eq:structural_eqn_reduced_forwards} that is \textbf{interventionally equivalent}, i.e., 
  it induces the same observational distribution $\Prb_{X,Y}$ and the same interventional
  distributions $\Prb_{X \given \intervene y}$, $\Prb_{Y \given \intervene x}$.
\end{proposition}
\begin{proof}
Denote by $\Prb_{X,Y}^{\eref{eq:structural_eqn}}$ the observational distribution induced by model \eref{eq:structural_eqn}.
One possible way to construct $E_Y$ and $f_Y$ is to define the conditional cumulative density function 
$F_{Y \given x}(y) := \Prb^{\eref{eq:structural_eqn}}(Y \le y \given X = x)$ and its inverse with respect to $y$ 
for fixed $x$, $F^{-1}_{Y \given x}$. Then, one can define $E_Y$ as the random 
variable 
$$E_Y := F_{Y \given X}(Y),$$
(where now the fixed value $x$ is substituted with the random variable $X$)
%\Jonas{why not define $E_Y$ as an independent uniform variable??} \Joris{That's a matter of taste, I guess. I like
%the fact that $E_Y$ can be explicitly constructed from $X$ and $Y$. This definition implies that $E_Y$ it is uniform and
%independent. The other way around would be less trivial, I guess. If that's trivial as well, then I think I would be
%equally happy with that alternative definition. But maybe in the end you're right that all that matters is $\Prb_{X,Y,E_Y}$
%and not the r.v.'s themselves (cf. ``a.s.'' vs. ``in distribution'').} 
and the function $f_Y$ by\footnote{Note that we denote probability densities with the symbol $p$, so we can safely use the symbol $f$ for a function without risking any confusion.}
$$f_Y(x,e) := F^{-1}_{Y \given x}(e).$$
%Assuming that these quantities are well-defined, it is easy to check %\citep[][Prop. 2.6]{PetersThesis} 
%\citep[e.g.,][Theorem 1]{HyvarinenPajunen1999} that \eref{eq:structural_eqn_reduced_forwards} holds with $E_Y$ uniformly distributed on $[0,1]$.
Now consider the change-of-variables $(X,Y) \mapsto (X,E_Y)$. The corresponding joint densities transform as:
$$p^{\eref{eq:structural_eqn}}_{X,Y}(x,y) = p_{X,E_Y}\big(x,F_{Y\given x}(y)\big) \left|\dadb{F_{Y\given x}}{y}(x,y)\right| = p_{X,E_Y}\big(x,F_{Y\given x}(y)\big) p^{\eref{eq:structural_eqn}}_{Y \given X}(y \given x),$$
and therefore:
$$p^{\eref{eq:structural_eqn}}_{X}(x) = p_{X,E_Y}(x,F_{Y\given x}(y))$$
%= p^{\eref{eq:structural_eqn}}_{X}(x) p_{E_Y \given X}\big(F_{Y \given x}(y) \given x\big)
for all $x, y$. This implies that $E_Y \indep X$ and that $p_{E_Y} = \id_{(0,1)}$.

This establishes that $\Prb_{X,Y}^{\eref{eq:structural_eqn_reduced_forwards}} = \Prb_{X,Y}^{\eref{eq:structural_eqn}}$. The identity of the
interventional distributions follows directly, because:
$$\Prb_{X \given \intervene y}^{\eref{eq:structural_eqn}} = \Prb_X^{\eref{eq:structural_eqn}} = \Prb_X^{\eref{eq:structural_eqn_reduced_forwards}} = \Prb_{X \given \intervene y}^{\eref{eq:structural_eqn_reduced_forwards}}$$
and
$$\Prb_{Y \given \intervene x}^{\eref{eq:structural_eqn}} = \Prb_{Y \given x}^{\eref{eq:structural_eqn}} = \Prb_{Y \given x}^{\eref{eq:structural_eqn_reduced_forwards}} = \Prb_{Y \given \intervene x}^{\eref{eq:structural_eqn_reduced_forwards}}.$$
\end{proof}

%Let us now consider the bivariate case, %\Jonas{Vllt. auch personal taste, aber ich wuerde gleich mit dem bivariaten anfangen. Den multivariaten Fall benoetigen wir ja eigentlich nicht...}
%where $X$ is one-dimensional, i.e., $p=1$. 
%\todo{Then} \eref{eq:structural_eqn_reduced} simplifies to:
%\begin{equation}\label{eq:structural_eqn_reduced_forwards}
%  \left\{\begin{array}{l}
%  Y = f_Y(X,E_Y) \\
%  X \indep E_Y, \quad X \sim p(x), \quad E_Y \sim p(e_Y).
%  \end{array}\right.
%\end{equation}
A similar construction of an effective noise variable can be performed in the other direction as well, at
least to obtain a model that induces the same observational distribution. More precisely, we can construct
a function $f_X$ and a random variable $E_X$ such that:
\begin{equation}\label{eq:structural_eqn_reduced_backwards}
  \left\{\begin{array}{l}
  X = f_X(Y,E_X) \\
  Y \indep E_X, \quad Y \sim p_Y(y), \quad E_X \sim p_{E_X}(e_X)
  \end{array}\right.
\end{equation}
induces the same observational distribution $\Prb_{X,Y}^{\eref{eq:structural_eqn_reduced_backwards}} = \Prb_{X,Y}^{\eref{eq:structural_eqn_reduced_forwards}}$  as \eref{eq:structural_eqn_reduced_forwards} and the original \eref{eq:structural_eqn}.
A well-known example is the linear-Gaussian case:
\begin{example}\label{ex:linear_Gauss_ANM}
Suppose that 
\begin{equation*}
  \left\{\begin{array}{ll}
  Y = \alpha X + E_Y & X \sim \C{N}(\mu_X,\sigma_X^2) \\
      E_Y \indep X & E_Y \sim \C{N}(\mu_{E_Y},\sigma_{E_Y}^2).
  \end{array}\right.
\end{equation*}
Then: 
\begin{equation*}
  \left\{\begin{array}{ll}
  X = \beta Y + E_X & Y \sim \C{N}(\mu_Y,\sigma_Y^2) \\
      E_X \indep Y & E_X \sim \C{N}(\mu_{E_X},\sigma_{E_X}^2),
  \end{array}\right.
\end{equation*}
with
\begin{align*}
  \beta & = \frac{\alpha \sigma_X^2}{\alpha^2 \sigma_X^2 + \sigma_{E_Y}^2},\\
  \mu_Y & = \alpha \mu_X + \mu_{E_Y}, \quad \sigma_Y^2 = \alpha^2 \sigma_X^2 + \sigma_{E_Y}^2, \\
  \mu_{E_X} & = (1-\alpha\beta)\mu_X  - \beta\mu_{E_Y}, \quad \sigma_{E_X}^2 = (1-\alpha\beta)^2\sigma_X^2 + \beta^2\sigma_{E_Y}^2.
\end{align*}
induces the same joint distribution on $X,Y$.
%Figure~\ref{fig:ANM_nonidentifiable} (left) illustrates the joint density and the forward and backward functions $f$ and $g$ for such a linear-Gaussian ANM.
\end{example}
However, in general the interventional distributions induced by \eref{eq:structural_eqn_reduced_backwards} will be different from those of \eref{eq:structural_eqn_reduced_forwards} and the original model \eref{eq:structural_eqn}. For example, in general
$$\Prb^{\eref{eq:structural_eqn_reduced_backwards}}_{X \given \intervene y} = 
  \Prb^{\eref{eq:structural_eqn_reduced_backwards}}_{X \given y} =
  \Prb^{\eref{eq:structural_eqn_reduced_forwards}}_{X \given y} \ne 
  \Prb^{\eref{eq:structural_eqn_reduced_forwards}}_{X} =
  \Prb^{\eref{eq:structural_eqn_reduced_forwards}}_{X \given \intervene y}.$$

This means that whenever we can model an observational distribution $\Prb_{X,Y}$ with a model of the form \eref{eq:structural_eqn_reduced_backwards}, we can also model it using \eref{eq:structural_eqn_reduced_forwards}, and therefore the causal relationship between $X$ and $Y$ is not identifiable from the observational distribution without making additional assumptions. In other words: \eref{eq:structural_eqn} and \eref{eq:structural_eqn_reduced_forwards} are interventionally equivalent, but \eref{eq:structural_eqn_reduced_forwards} and \eref{eq:structural_eqn_reduced_backwards} are only observationally equivalent.
Without having access to the interventional distributions, this symmetry prevents us from drawing any
conclusions regarding the direction of the causal relationship between $X$ and $Y$ if we only have access to the observational distribution $\Prb_{X,Y}$.

%Model \eref{eq:structural_eqn_reduced_forwards} does not yield any asymmetry between $X$ and $Y$, as the same construction of an effective noise variable can be
%performed in the other direction. That gives another model for the joint density $p_{X,Y}$, where we could now interpret $Y$ as the
%cause and $X$ as the effect:
%\begin{equation}\label{eq:structural_eqn_reduced_backwards}
%  \left\{\begin{array}{l}
%  X = f_X(Y,E_X) \\
%  Y \indep E_X, \quad Y \sim p(y), \quad E_X \sim p(e_X).
%  \end{array}\right.
%\end{equation}

\subsubsection{Breaking the symmetry}

By \emph{restricting} the models \eref{eq:structural_eqn_reduced_forwards} and \eref{eq:structural_eqn_reduced_backwards}
to have lower complexity, asymmetries can be introduced. The work of \citep{KanoShimizu2003,ShimizuHoyerHyvarinenKerminen2006} showed that 
for \emph{linear} models (i.e., where the functions $f_X$ and $f_Y$ are restricted to be linear), 
\emph{non-Gaussianity} of the input and noise distributions actually allows one to distinguish the directionality
of such functional models. 
\citet{Peters2014biom} recently proved that for linear models, Gaussian noise variables with \emph{equal variances} also lead to identifiability.
For high-dimensional variables, the structure of the covariance matrices can be exploited to achieve asymmetries \citep{JanzingHoyerSchoelkopf2010,Zscheischler2011}.
%Indeed: 
%\begin{prop}\label{prop:lingam}
%Let $X$ and $Y$ be two random variables, for which
%\begin{align*}
%Y&=\alpha X + \epsilon, \quad \epsilon \independent X, \; \alpha \neq 0
%\end{align*}
%holds.
%Then we can reverse the process, i.e., there exists $\beta \in \RN$ and a noise variable $\eta$, such that
%$$X= \beta Y + \eta, \quad \eta \independent Y,$$
%if and {\emph only if} $X$ and $\epsilon$ are normally distributed.
%\end{prop}
%\citet{ShimizuHoyerHyvarinenKerminen2006} were the first to report this result. They prove a more general result (for more than two variables) using Independent Component Analysis (ICA) \citep[][Theorem~11]{Comon1994}, which itself is proved using the Darmois-Skitovi\v{c} theorem \citep{Darmois1953,Skitovich1954,Skitovich1962}. Alternatively, Proposition~\ref{prop:lingam} can be proved directly using the Darmois-Skitovi\v{c} theorem \citep[e.g.,][Theorem~2.10]{Peters2008}. 
%The authors call this model (and its multivariate generalization) a Linear Non-Gaussian Acyclic Model (LiNGAM). Various practical methods that can be applied to a finite sample of the observational distribution have been proposed \citep{ShimizuHoyerHyvarinenKerminen2006,Shimizu++2011,HyvarinenSmith2013}.

\citet{HoyerJanzingMooijPetersSchoelkopf_NIPS_08} 
%and \citet{Peters_et_al_UAI_11,PetersMooijJanzingSchoelkopf_JMLR_14} 
showed that also \emph{nonlinearity} of the functional relationships aids in 
identifying the causal direction, as long as the influence of the noise is additive. More
precisely, they consider the following class of models:
\begin{definition}
A tuple $(p_X, p_{E_Y}, f_Y)$ consisting of a density $p_X$, a density $p_{E_Y}$ with finite mean, and a Borel-measurable function $f_Y : \RN \to \RN$, defines a
\textbf{bivariate additive noise model (ANM) $X\to Y$} by:
\begin{equation}\label{eq:anm}
  \left\{\begin{array}{l}
  Y = f_Y(X) + E_Y \\
  X \indep E_Y, \quad X \sim p_X, \quad E_Y \sim p_{E_Y}.
  \end{array}\right.
\end{equation}
If the induced distribution $\Prb_{X,Y}$ has a density with respect to Lebesgue measure, the
induced density $p(x,y)$ is said to satisfy an additive noise model $X \to Y$.
\end{definition}
Note that an ANM is a special case of model \eref{eq:structural_eqn_reduced_forwards} where
the influence of the noise on $Y$ is restricted to be additive. 
%If there are many independent latent variables $U_1, \dots, U_m$,
%and their effect on $Y$ is small, then the additive noise model \eref{eq:anm} can be a good
%approximation to \eref{eq:structural_eqn}. 

We are especially interested in cases for which the 
additivity requirement introduces an asymmetry between $X$ and $Y$:
\begin{definition}
  If the joint density $p(x,y)$ satisfies an additive noise model $X \to Y$, but does not satisfy any additive noise model $Y \to X$, then we call the ANM $X \to Y$ \textbf{identifiable} (from the observational distribution).
\end{definition}
\begin{figure}[t]
  \centerline{\includegraphics[width=0.8\textwidth]{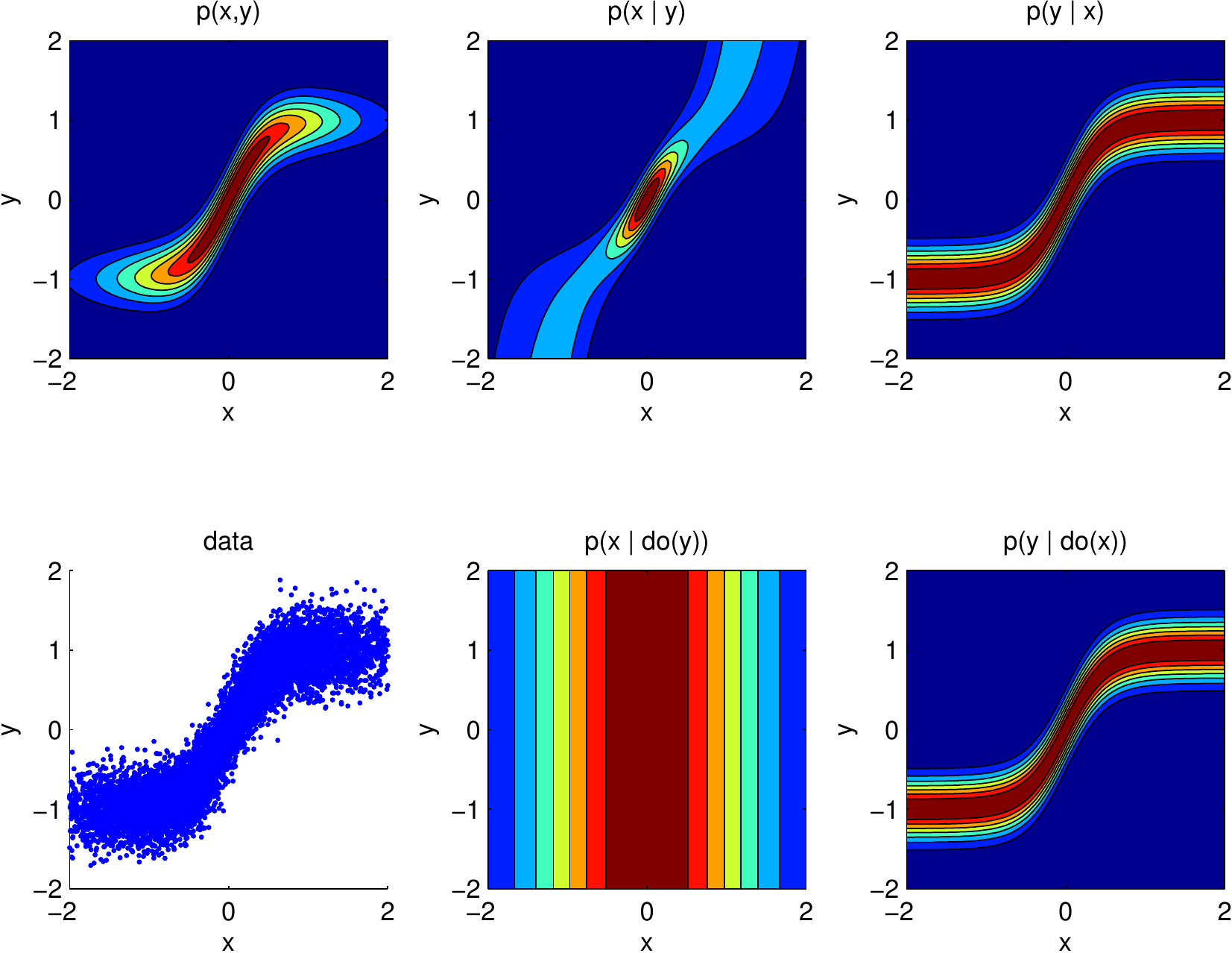}}
  \caption{\label{fig:anm_identifiability}Identifiable ANM with $Y = \tanh(X) + E$, where $X \sim \C{N}(0,1)$ and $E \sim \C{N}(0,0.5^2)$. Shown are contours of the joint and conditional distributions, and a scatter plot of data sampled from the model distribution. Note that the contour lines of $p(y\given x)$ only shift as $x$ changes. On the other hand, $p(x\given y)$ differs by more than just its mean for different values of $y$.}
\end{figure}

\citet{HoyerJanzingMooijPetersSchoelkopf_NIPS_08} proved that additive noise models are generically identifiable.
The intuition behind this result is that if $p(x,y)$ satisfies an additive noise model $X \to Y$, then $p(y \given x)$ depends
on $x$ only through its mean, and all other aspects of this conditional distribution do not depend on $x$. On the other hand,
$p(x \given y)$ will typically depend in a more
complicated way on $y$ (see also Figure~\ref{fig:anm_identifiability}). Only for very specific choices of the parameters
of an ANM one obtains a non-identifiable ANM.
We have already seen an example of such a non-identifiable ANM: the linear-Gaussian case (Example \ref{ex:linear_Gauss_ANM}). 
A more exotic example with non-Gaussian distributions was given in \citep[][Example 25]{PetersMooijJanzingSchoelkopf_JMLR_14}.
\citet{ZhangHyvarinen2009} proved that non-identifiable ANMs necessarily fall into one out of five classes.
In particular, their result implies something that we might expect intuitively: if $f$ is not injective\footnote{A mapping is said to be \emph{injective} if it does not map distinct elements of its domain to the same element of its codomain.}, the ANM is identifiable. 
%\begin{example}\label{ex:nonidentifiable_ANM}
%Let $Y = f(X) + \epsilon$ and $X \indep \epsilon$, with
%$$f(x) = a x + b$$
%for $a \ne 0$, and log-densities
%$$\log p_{X}(x) = c_1 \exp(c_2 x) + c_3 x + c_4 $$
%and
%$$\log p_{\epsilon}(n) = \gamma_1 \exp(\gamma_2 n) + \gamma_3 n + \gamma_4 \,.  $$
%If and only if $c_2 = -a \gamma_2$ and $c_3 \neq a \gamma_3$ we obtain a valid backward ANM model of the form $X = g(Y) + \eta$ with $\eta \indep Y$ and
%$$\log p_{\eta}(n) = -e^{-\zeta_Y} e^{c_2 n} + (c_3 - \gamma_3 a) n + \zeta_{\eta}$$
%$$g(y) = -\frac{1}{c_2} \ln (-c_1 - \gamma_1 e^{\gamma_2(y-b)}) + \frac{\zeta_Y}{c_2}\,.$$
%$$\log p_Y(y) = \left(-\frac{1}{c_2} \ln (-c_1 - \gamma_1 e^{\gamma_2(y-b)}) + \frac{\zeta_Y}{c_2}\right) (c_3 - \gamma_3 a) + \gamma_3 y + c_4 - \gamma_3 b + \gamma_4 - \zeta_{\eta}\,.$$
%where $\zeta_{\eta}$ and $\zeta_Y$ have to be chosen such that the densities of $\eta$ and $Y$ are properly normalized.
%Figure~\ref{fig:ANM_nonidentifiable} (right) illustrates the joint distribution over $X$ and $Y$ and forward and backward functions $f$ and $g$ for a choice of the parameters that leads to non-identifability.
%\end{example}
%This example shows how parameters of function, input and noise distribution may have to be ``fine-tuned'' to yield non-identifiability \citep{JanzingSteudel2010}.

\citet{Mooij_et_al_NIPS_11} showed that bivariate identifiability even holds generically when feedback is allowed (i.e., if both $X \rightarrow Y$ \emph{and} $Y \rightarrow X$), at least when assuming noise and input distributions to be Gaussian. \citet{PetersJanzingSchoelkopf2011} provide an extension of the acyclic model for discrete variables.
\citet{ZhangHyvarinen2009} give an extension of the identifiability results allowing for an additional bijective\footnote{A mapping is said to be \emph{surjective} if every element in its codomain is mapped to by at least one element of its domain. It is called \emph{bijective} if it is surjective and injective.} transformation of the data, i.e., using a functional model of the
form $Y = \phi\big(f_Y(X) + E_Y\big)$, with $E_Y \indep X$, and $\phi : \RN \to \RN$ bijective, which they call the Post-NonLinear (PNL) model.
%They obtain a similar differential equation as~\eref{eq:nonlingam_sufficient_ODE} and based on that give a classification of all non-identifiable PNL models.
The results on identifiability of additive noise models can be extended to the multivariate case \citep{PetersMooijJanzingSchoelkopf_JMLR_14} if there are no hidden variables and no feedback loops.
This extension can be applied to nonlinear ANMs \citep{HoyerJanzingMooijPetersSchoelkopf_NIPS_08,BuehlmannPetersErnest2013}, linear non-Gaussian models \citep{Shimizu++2011}, the model of equal error variances \citep{Peters2014biom} or to the case of discrete variables \citep{PetersJanzingSchoelkopf2011}. Full identifiability in the presence of hidden variables for the acyclic case has only been established for linear non-Gaussian models \citep{Hoyer2008hidden}.

\subsubsection{Additive Noise Principle}

%Assuming that the joint distribution $p_{XY}$ is induced by an additive noise model $X \to Y$,
%then Theorem \ref{theo:nonlingam_sufficient} states that only for very special choices
%of the model parameters will $p_{XY}$ be induced by a reverse additive noise model.
Following \citet{HoyerJanzingMooijPetersSchoelkopf_NIPS_08}, we hypothesize that:
\begin{principle}
  Suppose we are given a joint density $p(x,y)$ and we know that the causal structure is either that of (a)
  or (b) in Figure \ref{fig:bivariate_causal_relations}. If $p(x,y)$ satisfies an \emph{identifiable} additive noise model 
  $X \to Y$, then it is likely %\Jonas{was heisst highly likely? Brauchen wir dieses Postulat ueberhaupt? Ich glaube, ich wuerde kein Postulat-environment verwenden (machen wir fuer die anderen methoden ja auch nicht.)} 
  that we are in case (a), i.e., $X$ causes $Y$.
\end{principle}
This principle should not be regarded as a rigorous statement, but rather as an empirical assumption: 
we cannot exactly quantify \emph{how likely} the conclusion that $X$ causes $Y$ is, as there is always a possibility
that $Y$ causes $X$ while $p_{X,Y}$ happens to satisfy an identifiable additive noise model $X \to Y$. 
In general, that would require a special choice of the distribution of $X$ and the conditional distribution of $Y$
given $X$, which is unlikely. In this sence, we can regard this principle as a special case of Occam's Razor.

In the next subsection, we will discuss various ways of operationalizing this principle.
In Section~\ref{sec:experiments}, we provide empirical evidence supporting this principle. 

%The basic idea of the recent method by \citep{HoyerJanzingMooijPetersSchoelkopf_NIPS_08} 
%is to assume that the effect can be written as some (not necessarily linear) function of
%the cause, plus additive noise, which is independent of the cause.
%In \citep{HoyerJanzingMooijPetersSchoelkopf_NIPS_08} it was shown that the linear--non-Gaussian
%causal discovery framework of \citet{ShimizuHoyerHyvarinenKerminen2006} can be generalized to admit \emph{nonlinear} functional
%dependencies, as long as the noise on the variables remains additive. Although there exist special cases which
%admit reverse models, in the generic case the model is identifiable.
%When modeling each variable in this way, we obtain the class of what we call \emph{Additive
%Noise Models} \citep{PetersMooijJanzingSchoelkopf_JMLR_14}.
%Note that this model includes the special case when all functions $f_i$ are linear
%and all noise variables $\epsilon_i$ have Gaussian distributions, yielding the standard
%linear--Gaussian model family \citep{SpirtesGlymourScheines2000,Geiger1994,Bollen1989}. When the functions
%are linear but the noise variables may have non-Gaussian distributions, we obtain the
%linear--non-Gaussian (LiNGAM) models described in \citep{ShimizuHoyerHyvarinenKerminen2006}.
 
\subsection{Estimation methods}

The following Lemma is helpful to test whether a density satisfies a bivariate additive noise model:
\begin{lemma}\label{lem:ANMtest}
Given a joint density $p(x,y)$ of two random variables $X,Y$ such that the conditional expectation $\Exp(Y \given X=x)$
is well-defined for all $x$ and measurable. Then, $p(x,y)$ satisfies a bivariate additive noise model $X \to Y$ if and only if 
$E_Y := Y - \Exp(Y \given X)$ has finite mean and is independent of $X$.
\end{lemma}
\begin{proof}
Suppose that $p(x,y)$ is induced by $(p_X, p_U, f)$, say
$Y = f(X) + U$ with $X \indep U$, $X \sim p_X$, $U \sim p_U$.
Then $\Exp(Y \given X = x) = f(x) + \nu$, with $\nu = \Exp (U)$.
Therefore, $E_Y = Y - \Exp(Y \given X) = Y - (f(X) + \nu) = U - \nu$ is independent of $X$. 
Conversely, if $E_Y$ is independent of $X$, $p(x,y)$ is induced by the bivariate additive noise model 
%$\langle p(X), p(Y - \Exp(Y \given X) - \Exp(Y - \Exp(Y \given X))), \Exp(Y \given X) + \Exp(Y - \Exp(Y \given X)) \rangle$.
$(p_X, p_{E_Y}, x \mapsto \Exp(Y \given X = x))$.
\end{proof}
In practice, we usually do not have the density $p(x,y)$, but rather a finite sample of it. In that case, we can use the same idea
for testing whether this sample comes from a density that satisfies an additive noise model: we estimate the conditional expectation
$\Exp(Y \given X)$ by regression, and then test the independence of the residuals $Y - \Exp(Y \given X)$ and $X$.

Suppose we have two data sets, a \emph{training} data set $\C{D}_N := \{(x_n,y_n)\}_{n=1}^N$ (for estimating the function) and a \emph{test} data set 
$\C{D}_N' := \{(x_n',y_n')\}_{n=1}^N$ (for testing independence of residuals), both consisting of i.i.d.\ samples distributed according to 
$p(x,y)$. We will write $\B{x} = (x_1,\dots,x_N)$, $\B{y} = (y_1,\dots,y_N)$,
$\B{x}' = (x_1',\dots,x_N')$ and $\B{y}' = (y_1',\dots,y_N')$. We will consider two scenarios: the ``data splitting'' scenario where
training and test set are independent (typically achieved by splitting a bigger data set into two parts), and the ``data recycling'' 
scenario in which the training and test data are identical (where we use the same data twice for different purposes: regression and
independence testing).\footnote{\cite{Kpotufe++2014} refer to these scenarios as ``decoupled estimation'' and ``coupled estimation'',
respectively.}

\citet{HoyerJanzingMooijPetersSchoelkopf_NIPS_08} suggested the following procedure
to test whether the data come from a density that satisfies an additive noise model.\footnote{They only considered the
data recycling scenario, but the same idea can be applied to the data splitting scenario.}
By regressing $Y$ on $X$ using the training data $\C{D}_N$, an estimate $\hat f_Y$ for the regression function 
$x \mapsto \Exp(Y \given X=x)$ is obtained. Then, an independence test is used to estimate whether the predicted 
residuals are independent of the input, i.e., whether $(Y - \hat f_Y(X)) \indep X$, using test data $(\B{x}',\B{y}')$.
%This is then used to estimate $\B{\hat e}_Y' := \B{y}' - \hat f_Y(\B{x}')$, the residuals of $Y$ on the test data $\C{D}_N'$.
%Then, one tests whether the predicted residuals $\B{\hat e}_Y'$ are independent of $\B{x}'$ using an independence test. 
%\Jonas{This is probably fine this way. Personally, I think of a statistical test as testing the independence of random variables (rather than samples). For example, we could write that we construct a test for $H_0$: $(Y - E(Y|X)) \independent X$. For me, this would clarify what is meant by a finite sample version of Lemma5} \Joris{Good point. But how does your proposed notation express that this test is done on test data? By the way, to me it is very natural to test independence of a sample (maybe that's my Bayesian glasses: in practice, the only thing we have is samples, not r.v.'s). Actually, thinking through your proposal, wouldn't it be even better to say: ``we construct a test for $H_0$: $Y - E(Y|X) \indep X$ and apply this on test data $(\B{x}',\B{y}')$ that is assumed to be an i.i.d.\ sample of $\Prb_{X,Y}$''.} 
If the null hypothesis of independence is not rejected, one concludes
that $p(x,y)$ satisfies an additive noise model $X \to Y$. % This is basically a finite-sample version of Lemma~\ref{lem:ANMtest}.
The regression procedure and the independence test can be freely chosen.

There is a caveat, however: under the null hypothesis that $p(x,y)$ indeed satisfies an ANM, the error in the estimated residuals 
may introduce a dependence between the predicted residuals $\B{\hat e}_Y' := \B{y}' - \hat f_Y(\B{x}')$ and $\B{x}'$ even if the true residuals $\B{y}' - \Exp(Y \given X=\B{x}')$ are independent of $\B{x}'$. 
Therefore, the threshold for the independence test statistic has to be chosen with care: the standard threshold that
would ensure consistency of the independence test on its own may be too tight. 
%\Jonas{asymptotically correct?}
As far as we know, there are no theoretical results
on the choice of that threshold that would lead to a consistent way to test whether $p(x,y)$ satisfies an ANM $X \to Y$.

We circumvent this problem by assuming \emph{a priori} that $p(x,y)$ either satisfies an ANM $X \to Y$, or an ANM $Y \to X$,
but not both. In that sense, the test statistics of the independence test can be directly compared, and no threshold needs to be chosen.
This leads us to Algorithm~\ref{alg:bivariateANM} as a general scheme for identifying the direction of the ANM.
In order to decide whether $p(x,y)$ satisfies an additive noise model $X \to Y$, or an
additive noise model $Y \to X$, we simply estimate the regression functions in both directions,
calculate the corresponding residuals, estimate the dependence of the residuals with respect to
the input by some dependence measure $\hat C$, and output the direction that has the lowest dependence.

\begin{algorithm}[t]
\caption{\label{alg:bivariateANM}General procedure to decide whether $p(x,y)$ satisfies an additive noise model $X \to Y$ or $Y \to X$.}
\textbf{Input}: 
\begin{enumerate}
  \item I.i.d.\ sample $\C{D}_N := \{(x_i,y_i)\}_{i=1}^N$ of $X$ and $Y$ (``training data'');
  \item I.i.d.\ sample $\C{D}_N' := \{(x_i',y_i')\}_{i=1}^N$ of $X$ and $Y$ (``test data'');
  \item Regression method;
  \item Score estimator $\hat C : \RN^N \times \RN^N \to \RN$.
\end{enumerate}
\textbf{Output}: $\hat C_{X\to Y}$, $\hat C_{Y\to X}$, \texttt{dir}.
\begin{enumerate}
  \item Use the regression method to obtain estimates:
  \begin{enumerate}
    \item $\hat f_Y$ of the regression function $x \mapsto \Exp(Y \given X=x)$,
    \item $\hat f_X$ of the regression function $y \mapsto \Exp(X \given Y=y)$
  \end{enumerate}
  using the training data $\C{D}_N$;
  \item Use the estimated regression functions to predict residuals:
  \begin{enumerate}
    \item %predict residuals $\hat E_Y := Y - \hat f_Y(X; \C{D}_N)$ on the test data $\C{D}_N'$; %as $\hat \epsilon_n' := y_n' - \hat f_Y(x_n'; \C{D}_N)$; 
      $\B{\hat e}_Y' := \B{y}' - \hat f_Y(\B{x}')$
    \item %predict residuals $\hat E_X := X - \hat f_X(Y; \C{D}_N)$ on the test data $\C{D}_N'$; %as $\hat \eta_n' := x_n' - \hat f_X(y_n'; \C{D}_N)$; 
      $\B{\hat e}_X' := \B{x}' - \hat f_X(\B{y}')$
  \end{enumerate}
  from the test data $\C{D}_N'$.
  \item Calculate the scores to measure dependence of inputs and estimated residuals on the test data $\C{D}_N'$:
  \begin{enumerate}
    \item $\hat C_{X\to Y} := \hat C(\B{x}',\B{\hat e}_Y')$;
    \item $\hat C_{Y\to X} := \hat C(\B{y}',\B{\hat e}_X')$;
  \end{enumerate}
  \item Output $\hat C_{X\to Y}, \hat C_{Y\to X}$ and: 
    $$\mathtt{dir} := \begin{cases}
        X \to Y & \text{ if } \hat C_{X\to Y} < \hat C_{Y\to X}, \\
        Y \to X & \text{ if } \hat C_{X\to Y} > \hat C_{Y\to X}, \\
        ?       & \text{ if } \hat C_{X\to Y} = \hat C_{Y\to X}.
    \end{cases}$$
\end{enumerate}
\end{algorithm}

In principle, any consistent regression method can be used in Algorithm~\ref{alg:bivariateANM}.
Likewise, in principle any consistent measure of dependence can be used in Algorithm~\ref{alg:bivariateANM}
as score function. 
In the next subsections, we will consider in more detail some possible choices 
for the score function. Originally,
\citet{HoyerJanzingMooijPetersSchoelkopf_NIPS_08} proposed to use the $p$-value
of the Hilbert Schmidt Independence Criterion (HSIC), a kernel-based
non-parametric independence test. Alternatively, one can also use the HSIC
statistic itself as a score, and we will show that this leads to a consistent
procedure. Other dependence measures could be used instead, e.g., the measure proposed by \citet{Reshef++2011}.
\citet{Kpotufe++2014,Nowzohour2015} proposed to use as a score the sum of
the estimated differential entropies of inputs and residuals and proved consistency
of that procedure. For the Gaussian case, that is equivalent to the score considered in a high-dimensional context that was
shown to be consistent by \citet{BuehlmannPetersErnest2013}. This Gaussian score is also 
strongly related to an empirical-Bayes score originally proposed by \citet{FriedmanNachman2000}.
Finally, we will briefly discuss a Minimum Message Length score
that was considered by \citet{Mooij_et_al_NIPS_10} and another idea (based on minimizing
a dependence measure directly) proposed by \citet{MooijJanzingPetersSchoelkopf_ICML_09}.

\subsubsection{HSIC-based scores}

One possibility, first considered by \citet{HoyerJanzingMooijPetersSchoelkopf_NIPS_08}, is to use
the Hilbert-Schmidt Independence Criterion (HSIC) \citep{GrettonBousquetSmolaSchoelkopf2005} 
for testing the independence of the estimated residuals with the inputs. 
See Appendix~\ref{sec:HSIC} for a definition and basic properties of the HSIC independence test.

As proposed by \citet{HoyerJanzingMooijPetersSchoelkopf_NIPS_08}, one can use
the $p$-value of the HSIC statistic under the null hypothesis of independence.
This amounts to the following score function for measuring dependence:
\begin{equation}\label{eq:ANM_score_pHSIC}
  \hat C(\B{u},\B{v}) := -\log \pempHSIC{\kappa_{\hat\ell(\B{u})},\kappa_{\hat\ell(\B{v})}}(\B{u},\B{v}).
\end{equation}
Here, $\kappa_{\ell}$ is a kernel with parameters $\ell$, that are estimated from the data.
$\B{u}$ and $\B{v}$ are either inputs or estimated residuals (see also Algorithm \ref{alg:bivariateANM}).
A low HSIC $p$-value indicates that we should reject the null hypothesis of independence.
Another possibility is to use the HSIC value itself (instead of its $p$-value):
\begin{equation}\label{eq:ANM_score_HSIC}
  \hat C(\B{u},\B{v}) := \empHSIC{\kappa_{\hat\ell(\B{u})},\kappa_{\hat\ell(\B{v})}}(\B{u},\B{v}).
\end{equation}
An even simpler option is to use a fixed kernel $k$:
\begin{equation}\label{eq:ANM_score_HSIC_fixed}
\hat C(\B{u},\B{v}) := \empHSIC_{k,k}(\B{u},\B{v}).
\end{equation}
In Appendix \ref{sec:consistency}, we prove that under certain technical assumptions, Algorithm
\ref{alg:bivariateANM} with score function \eref{eq:ANM_score_HSIC_fixed} is a consistent procedure for inferring the direction of the ANM.
In particular, the product kernel $k \cdot k$ should be characteristic in order for HSIC to detect all possible independencies,
and the regression method should satisfy the following condition:
\begin{definition}\label{def:suitable_main}
Let $X, Y$ be two real-valued random variables with joint distribution $\Prb_{X,Y}$.
Suppose we are given sequences of training data sets $\C{D}_N = \{X_1,X_2,\dots,X_N\}$ and test data sets $\C{D}_N' = \{X_1',X_2',\dots,X_N'\}$ (in either the data splitting or the data recycling scenario).
We call a regression method \textbf{suitable} for regressing $Y$ on $X$ if the mean squared error between true and estimated regression function, evaluated on the test data, vanishes asymptotically in expectation:
\begin{equation}\label{eq:regression_suitable_main}
  \lim_{N\to\infty} \Exp_{\C{D}_N,\C{D}_N'} \left( \frac{1}{N} \sum_{n=1}^N \abs{\hat f_Y(X_n'; \C{D}_N) - \Exp(Y \given X = X_n')}^2 \right) = 0.
\end{equation}
Here, the expectation is taken over both training data $\C{D}_N$ and test data $\C{D}_N'$.
\end{definition}
The consistency result then reads as follows:
\begin{theorem}\label{theo:bivariate_ANM_HSIC_consistency}
Let $X, Y$ be two real-valued random variables with joint distribution $\Prb_{X,Y}$ that either satisfies an additive noise model $X \to Y$, or $Y \to X$, but not both.
Suppose we are given sequences of training data sets $\C{D}_N$ and test data sets $\C{D}_N'$ (in either the data splitting or the data recycling scenario).
Let $k:\RN\times\RN\to\RN$ be a bounded non-negative Lipschitz-continuous kernel such that the product $k \cdot k$ is characteristic. If the regression procedure
used in Algorithm \ref{alg:bivariateANM} is suitable for both $\Prb_{X,Y}$ and $\Prb_{Y,X}$, then Algorithm \ref{alg:bivariateANM} with score \eref{eq:ANM_score_HSIC_fixed} is a consistent procedure for
estimating the direction of the additive noise model.
%\jonas{we might need finite variances...}\Joris{Could you remember me why we need that?}
\end{theorem}
\begin{proof}
See Appendix~\ref{sec:consistency} (where a slightly more general result is shown, allowing for two different kernels $k, l$ to be used).
The main technical difficulty consists of the fact that the error in the estimated regression function introduces a dependency
between the cause and the estimated residuals. We overcome this difficulty by showing that the dependence is so weak that
its influence on the test statistic vanishes asymptotically.
\end{proof}
In the data splitting case, weakly universally consistent regression methods \citep{Gyorfi++2002} are suitable. In the data
recycling scenario, any regression method that satisfies \eref{eq:regression_suitable_main} is suitable. An example of a
kernel $k$ that satisfies the conditions of Theorem~\ref{theo:bivariate_ANM_HSIC_consistency} is the Gaussian kernel.

\subsubsection{Entropy-based scores}

Instead of explicitly testing for independence of residuals and inputs, one can use the sum of
their differential entropies as a score function \citep[e.g.,][]{Kpotufe++2014,Nowzohour2015}. This can easily be seen using 
Lemma 1 of \citet{Kpotufe++2014}, which we reproduce here because it is very instructive: 
\begin{lemma}\label{lem:Kpotufe}
Consider a joint distribution of $X,Y$ with density $p(x,y)$. For arbitrary functions $f, g : \RN \to \RN$ we have:
$$H(X) + H(Y - f(X)) = H(Y) + H(X - g(Y)) - \big(I(X - g(Y), Y) - I(Y - f(X), X)\big).$$
where $H(\cdot)$ denotes differential Shannon entropy, and $I(\cdot,\cdot)$ denotes differential mutual information \citep{CoverThomas2006}.\hfill\BlackBox%
\end{lemma}
%\begin{proof}
%By the chain rule of differential entropy:
%$$H(X, Y) = H(X) + H(Y \given X) = H(X) + H(Y - f(X) \given X) = H(X) + H(Y - f(X)) - I(Y - f(X), X).$$
%Similarly:
%$$H(X, Y) = H(Y) + H(X \given Y) = H(Y) + H(X - g(Y) \given Y) = H(Y) + H(X - g(Y)) - I(X - g(Y), Y).$$
%The result follows from a rearrangement of terms.
%\end{proof}
The proof is a simple application of the chain rule of differential entropy.
If $p(x,y)$ satisfies an identifiable additive noise model $X \to Y$, then there exists a function $f$ with
$I(Y - f(X), X) = 0$ (e.g., the regression function $x \mapsto \Exp(Y \given X=x)$), but $I(X - g(Y), Y) > 0$ for any function $g$.
%One can show that $\hat f(x) = \Exp(Y | X=x)$ minimizes \Jonas{better: minimizes?} $H(Y - f(X))$. \Jonas{Ist das wirklich so?? Ich bin gerade noch nicht ueberzeugt. Hast du eine quelle fuer das ``one can show''? Ich dachte eigentlich, dass $x \mapsto \Exp(Y | X=x)$ nur die varianz von $Y-f(X)$ minimiert, die aussage also nur(?) im gauss-fall stimmt? weiss gerade nicht, was das fuer konsequenzen haette, aber vllt. uebersehe ich ja eh was.} 
Therefore, one can use Algorithm~\ref{alg:bivariateANM} with score function
\begin{equation}\label{eq:ANM_score_entropy}
\hat C(\B{u},\B{v}) := \hat H(\B{u}) + \hat H(\B{v})
\end{equation}
in order to estimate the causal direction, using any estimator $\hat H(\cdot)$ of the differential Shannon entropy.
\citet{Kpotufe++2014,Nowzohour2015} prove that this approach to estimating the direction of additive noise models is consistent under certain technical assumptions.

\citet{Kpotufe++2014} note that the advantage of score \eref{eq:ANM_score_entropy} (based on marginal entropies) over score \eref{eq:ANM_score_HSIC} (based on dependence)
is that marginal entropies are cheaper to estimate than dependence (or mutual information). This is certainly true
when considering computation time. However, as we will see later, a disadvantage
of relying on differential entropy estimators is that these can be quite sensitive to discretization effects.

\subsubsection{Gaussian score}\label{sec:Gaussian_score}

%Note Theorem 8.6.6 in Cover \& Thomas:
%\begin{theorem}
%For any random variable $X$ and estimator $\hat X$:
%$$\Exp((X - \hat X)^2) \ge \Var(X) \ge \frac{1}{2\pi e} e^{2 H(X)}$$
%\end{theorem}
%and its Corollary:
%\begin{theorem}
%Given a random variable $X$, side information $Y$ and an estimator $\hat X(Y)$, it follows that:
%$$\Exp((X - \hat X(Y))^2) \ge \Var(X\given Y) \ge \frac{1}{2\pi e} e^{2 H(X|Y)}$$
%\end{theorem}

The differential entropy of a random variable $X$ can be upper bounded in terms of its variance \citep[see e.g.,][Theorem 8.6.6]{CoverThomas2006}: 
\begin{equation}\label{eq:entropy_Gaussian}
  H(X) \le \frac{1}{2} \log (2\pi e) + \frac{1}{2} \log \Var(X),
\end{equation}
where identity holds in case $X$ has a Gaussian distribution.
Assuming that $p(x,y)$ satisfies an identifiable Gaussian additive noise model $X \to Y$ with Gaussian input and Gaussian noise distributions, we therefore conclude from Lemma \ref{lem:Kpotufe}:
\begin{equation*}\begin{split}
  \log \Var (X) + \log \Var (Y - \hat f(X)) & = 2 H(X) + 2 H(Y - \hat f(X)) - 2 \log(2\pi e) \\
  & < 2 H(Y) + 2 H(X - \hat g(Y)) - 2 \log(2\pi e) \\
  & \le \log \Var Y + \log \Var(X - \hat g(Y))
\end{split}\end{equation*}
for any function $g$. In that case, we can therefore use Algorithm~\ref{alg:bivariateANM} with score function
\begin{equation}\label{eq:ANM_score_Gauss}
\hat C(\B{u},\B{v}) := \log \widehat{\Var}(\B{u}) + \log \widehat{\Var}(\B{v}).
\end{equation}
This score was also considered recently by \citet{BuehlmannPetersErnest2013} and shown to lead to a consistent estimation procedure under certain assumptions.

\subsubsection{Empirical-Bayes scores}\label{sec:Bayesian_score}

Deciding the direction of the ANM can also be done by applying model selection using empirical Bayes.
As an example, for the ANM $X \to Y$, one can consider a generative model that models $X$ as a Gaussian, and 
$Y$ as a Gaussian Process \citep{RasmussenWilliams2006} conditional on $X$. For the ANM $Y \to X$, one considers a similar model with
the roles of $X$ and $Y$ reversed.
Empirical-Bayes model selection is performed by calculating the maximum evidences (marginal likelihoods) 
of these two models when optimizing over the hyperparameters, and preferring the model with larger maximum
evidence. This is actually a special case (the
bivariate case) of an approach proposed by \citet{FriedmanNachman2000}.\footnote{\citet{FriedmanNachman2000} 
even hint at using this method for inferring causal relationships, although it seems that they only
thought of cases where the functional dependence of the effect on the cause was not injective.}
Considering the negative log marginal likelihoods leads to the following score for the ANM $X \to Y$:
\begin{equation}\label{eq:ANM_score_FriedmanNachman}
  \hat C_{X \to Y}(\B{x},\B{y}) := \min_{\mu, \tau^2, \B{\theta}, \sigma^2} \left( -\log \C{N}(\B{x} \given \mu \B{1}, \tau^2 \B{I}) - \log \C{N}(\B{y} \given \B{0}, \B{K}_{\B{\theta}}(\B{x}) + \sigma^2 \B{I}) \right),
\end{equation}
%\Jonas{Ich finde, die Notation faellt etwas raus. Wie waere es mit
%$$
%C_{X \to Y} := \min_{\mu, \tau^2, \B{\theta}, \sigma^2} -\log p_{\C{N}}(\B{x} \given \mu \B{1}, \tau^2 \B{I}) - \log p_{\C{N}}(\B{y} \given \B{0}, \B{K}_{\B{\theta}}(\B{x}) + \sigma^2 \B{I}),
%$$}
%\Joris{This is a standard notation in machine learning. Your proposal conflicts with $p_{X}(x)$ being the density of the r.v. $X$ evaluated at $x$. But we could of course add a footnote explaining what $\C{N}(x \given \mu,\sigma^2)$ means.}
and a similar expression for $\hat C_{Y \to X}$, the score of the ANM $Y \to X$. Here, $\B{K}_{\B{\theta}}(\B{x})$
is the $N \times N$ kernel matrix $K_{ij} = k_{\B{\theta}}(x_i,x_j)$ for a kernel with parameters $\B{\theta}$ and $\C{N}(\cdot \given \B{\mu}, \B{\Sigma})$ denotes the density of a multivariate normal distribution with mean $\B{\mu}$ and covariance matrix $\B{\Sigma}$. If one would 
put a prior distribution on the hyperparameters and integrate them out, this would correspond to Bayesian model selection.
In practice, one typically uses ``empirical Bayes'', which means that the hyperparameters $(\mu,\tau,\B{\theta},\sigma)$ are
optimized over instead for computational reasons.
Note that this method skips the explicit regression step, instead it (implicitly) integrates over all possible regression functions
\citep{RasmussenWilliams2006}.
Also, it does not distinguish the data splitting and data recycling scenarios, instead it uses the data directly to calculate
the (maximum) marginal likelihood. Therefore, the structure of the algorithm is slightly different, see Algorithm~\ref{alg:bivariateANM_MML}.
In Appendix~\ref{sec:gp} we show that this score is actually closely related to the Gaussian score considered in Section~\ref{sec:Gaussian_score}.
%We expect that this method can also be shown to be consistent using standard results for the consistency of Gaussian Process Regression
%when the true cause has a Gaussian distribution, and the true effect is a function of the cause plus (additive) Gaussian noise.

\begin{algorithm}[t]
\caption{\label{alg:bivariateANM_MML}Procedure to decide whether $p(x,y)$ satisfies an additive noise model $X \to Y$ or $Y \to X$
suitable for empirical-Bayes or MML model selection.}
\textbf{Input}: 
\begin{enumerate}
  \item I.i.d.\ sample $\C{D}_N := \{(x_i,y_i)\}_{i=1}^N$ of $X$ and $Y$ (``data'');
  \item Score function $\hat C : \RN^N \times \RN^N \to \RN$ for measuring model fit and model complexity.
\end{enumerate}
\textbf{Output}: $\hat C_{X\to Y}$, $\hat C_{Y\to X}$, \texttt{dir}.
\begin{enumerate}
  \item \begin{enumerate}
    \item calculate $\hat C_{X\to Y} = \hat C(\B{x},\B{y})$
%      $$C_{X \to Y} := \min_{\mu, \tau^2, \B{\theta}, \sigma^2} \left( -\log \C{N}(\B{x} \given \mu \B{1}, \tau^2 \B{I}) - \log \C{N}(\B{y} \given 0, \B{K}_{\theta}(\B{x}) + \sigma^2 \B{I}) \right)$$
    \item calculate $\hat C_{Y\to X} = \hat C(\B{y},\B{x})$
%      $$C_{Y \to X} := \min_{\mu, \tau^2, \B{\theta}, \sigma^2} \left( -\log \C{N}(\B{y} \given \mu \B{1}, \tau^2 \B{I}) - \log \C{N}(\B{x} \given 0, \B{K}_{\theta}(\B{y}) + \sigma^2 \B{I}) \right)$$
  \end{enumerate}
  \item Output $\hat C_{X\to Y}, \hat C_{Y\to X}$ and:
    $$\mathtt{dir} := \begin{cases}
        X \to Y & \text{ if } \hat C_{X\to Y} < \hat C_{Y\to X}, \\
        Y \to X & \text{ if } \hat C_{X\to Y} > \hat C_{Y\to X}, \\
        ?       & \text{ if } \hat C_{X\to Y} = \hat C_{Y\to X}.
    \end{cases}$$
\end{enumerate}
\end{algorithm}

\subsubsection{Minimum Message Length scores}

In a similar vein as (empirical) Bayesian marginal likelihoods can be interpreted as
measuring likelihood in combination with a complexity penalty, Minimum
Message Length (MML) techniques can be used to construct scores that
incorporate a trade-off between model fit (likelihood) and model complexity
\citep{Grunwald2007}. Asymptotically, as the number of data points tends to
infinity, one would expect the model fit to outweigh the model complexity, and
therefore by Lemma~\ref{lem:Kpotufe}, simple comparison of MML scores should
be enough to identify the direction of an identifiable additive noise model.

A particular MML score was considered by \citet{Mooij_et_al_NIPS_10}. This is a special case (referred 
to in \citet{Mooij_et_al_NIPS_10} as ``AN-MML'') of their more general framework 
that allows for non-additive noise. Like \eref{eq:ANM_score_FriedmanNachman}, 
the score is a sum of two terms, one corresponding to the marginal density
$p(x)$ and the other to the conditional density $p(y \given x)$:
\begin{equation}\label{eq:ANM_score_MML}
  \hat C_{X \to Y}(\B{x},\B{y}) := \C{L}(\B{x}) + \min_{\B{\theta}, \sigma^2} \left( - \log \C{N}(\B{y} \given 0, \B{K}_{\theta}(\B{x}) + \sigma^2 \B{I}) \right).
\end{equation}
The second term is an MML score for the conditional density $p(y \given x)$, and
is identical to the conditional density term in \eref{eq:ANM_score_FriedmanNachman}.
The MML score $\C{L}(\B{x})$ for the marginal density $p(x)$ is derived as an 
asymptotic expansion based on the Minimum Message Length principle for 
a mixture-of-Gaussians model \citep{FigueiredoJain2002}: 
%For the complexity measure of the marginal density $p(x)$, \citet{Mooij_et_al_NIPS_10}
%proposed to use the following formula, that can be obtained as an asymptotic expansion based on the 
%Minimum Message Length principle \citep{FigueiredoJain2002}:
  \begin{equation}\label{eq:opt_MML}
    \C{L}(\B{x}) = \min_{\B{\eta}} \left( \sum_{j=1}^k \log\left(\frac{N \alpha_j}{12}\right) + \frac{k}{2}\log\frac{N}{12} + \frac{3k}{2} - \log p(\B{x} \given \B{\eta}) \right),
  \end{equation}
where $p(\B{x} \given \B{\eta})$ is a Gaussian mixture model:
$p(x_i \given \B{\eta}) = \sum_{j=1}^k \alpha_j \C{N}(x_i \given \mu_j, \sigma_j^2)$
with $\B{\eta} = (\alpha_i,\mu_i,\sigma^2_i)_{i=1}^{k}$.
The optimization problem \eref{eq:opt_MML} is solved numerically by means of the algorithm proposed by 
\cite{FigueiredoJain2002}, using a small but nonzero value ($10^{-4}$) of the regularization parameter.
%with hyperparameters $\bm{\theta}_X = (k,\alpha_1,\dots,\alpha_k,\mu_1,\dots,\mu_k,\sigma_1,\dots,\sigma_k)$.
%Following \citep{FigueiredoJain2002}, they put an improper Dirichlet prior (with parameters $(-1,-1,\dots,-1)$) 
%on the component weights $\B{\alpha}$ and flat priors on the component parameters $\B{\mu}$, $\B{\sigma}$.

Comparing this score with the empirical-Bayes score \eref{eq:ANM_score_FriedmanNachman}, 
the main conceptual difference is that the former uses a more complicated mixture-of-Gaussians 
model for the marginal density, whereas \eref{eq:ANM_score_FriedmanNachman} uses a simple Gaussian model.
We can use \eref{eq:ANM_score_MML} in combination with Algorithm~\ref{alg:bivariateANM_MML}
in order to estimate the direction of an identifiable additive noise model.

\subsubsection{Minimizing HSIC directly}

One can try to apply the idea of combining regression and independence testing
into a single procedure (as achieved with the empirical-Bayes score described in
Section~\ref{sec:Bayesian_score}, for example) more generally.
Indeed, a score that measures the dependence between the residuals $\B{y}' -
f_Y(\B{x}')$ and the inputs $\B{x}'$ can be minimized with respect to the
function $f_Y$. \citet{MooijJanzingPetersSchoelkopf_ICML_09}
proposed to minimize $\empHSIC\big(\B{x},\B{y}-f(\B{x})\big)$ with
respect to the function $f$. However, the optimization problem with respect to
$f$ turns out to be a challenging non-convex optimization problem with multiple local
minima, and there are no guarantees to find the global minimum. In addition, the
performance depends strongly on the selection of suitable kernel bandwidths, for 
which no suitable automatic procedure is known in this context. Finally, proving consistency of such a 
method might be challenging, as the minimization may introduce strong dependences
between the residuals. Therefore, we do not discuss or evaluate this method in more detail here.
%We chose Gaussian RBF kernels for both $X$ and $Y$, i.e., $k_X(x,x') =
%\exp(-\norm{x-x'}^2 \sigma_X^{-2})$.

\section{Information-Geometric Causal Inference}\label{sec:IGCI} %%%%%%%%%%%%%%%%%%%%%%%%%%%%%%%%%%%%%%%%%%%%%%%%%%%%%%%%%%%%%%%%%%%%%%%%%%%%%%%%%%%%%%%%%%%%%%%%%%%%%%%%%%%%%%%%%%%%%%%%%%%%%%%%%%%%%%

In this section, we review a class of causal discovery methods that exploits independence of the
distribution of the cause and the conditional distribution of the effect given the cause.
It nicely complements causal inference based on additive noise 
by employing asymmetries between cause and effect that have nothing to do with noise. 

\subsection{Theory}

Information-Geometric Causal Inference (IGCI) is an approach
that builds upon the assumption that for $X\rightarrow Y$ the marginal distribution $\Prb_{X}$
contains no information about the conditional\footnote{Note that $\Prb_{Y\given X}$ 
represents the whole family of distributions
$x \mapsto \Prb_{Y\given X=x}$.}
 $\Prb_{Y\given X}$ and vice versa, since they represent independent mechanisms. As 
\citet{JanzingSchoelkopf2010} illustrated for several toy examples,
the conditional and marginal distributions
$\Prb_Y,\Prb_{X\given Y}$ may then contain information about each other, but it is hard to formalize in what sense this is the case for scenarios that go beyond 
simple toy models. IGCI is based on the strong assumption that
$X$ and $Y$ are deterministically related by a bijective function $f$, that is, $Y=f(X)$ and $X=f^{-1}(Y)$. Although  
its practical applicability is limited to causal relations
with sufficiently small noise and sufficiently high non-linearity, IGCI provides a  setting in which the independence
of $\Prb_X$ and $\Prb_{Y\given X}$ provably implies well-defined dependences 
between $\Prb_{Y}$ and $\Prb_{X\given Y}$ in a sense 
described below. 

To introduce IGCI,
note that the deterministic relation $Y=f(X)$ implies that the conditional $\Prb_{Y\given X}$
has no density $p(y\given x)$, but it can be
represented using $f$ via
\[
\Prb (Y=y\given X=x)= \left\{ \begin{array}{cc} 1 & \hbox{ if }  y=f(x) \\
                                                        0 & \hbox{ otherwise. }  \end{array}\right.
                                                        \]
The fact that $\Prb_{X}$ and $\Prb_{Y\given X}$ contain no information about each other then
translates into the statement that $\Prb_X$ and  $f$  contain no information about each other.

Before sketching a more general formulation of IGCI \citep{Daniusis_et_al_UAI_10,Janzing_et_al_AI_12}, we begin with 
the most intuitive case where $f$ is  a strictly monotonically increasing
differentiable bijection of $[0,1]$. We then assume
that the following equality is approximately satisfied:
\begin{equation}\label{eq:orth}
\int_0^1 \log f'(x) p(x) \,dx = \int_0^1 \log f'(x) \,dx,
\end{equation}
where $f'$ is the derivative of $f$.
To see why \eref{eq:orth} is an independence between function $f$ and input density $p_X$,
we interpret $x\mapsto \log f'(x)$ and $x \mapsto p(x)$ as random variables\footnote{Note that random variables are formally defined as maps from a probability space to the real numbers.} on the probability space $[0,1]$.
Then the difference between the two sides of \eref{eq:orth} is the covariance
of these two random variables with respect to the uniform distribution on $[0,1]$:
\begin{equation*}\begin{split}
  \Cov(\log f',p_X)&=\Exp(\log f' \cdot p_X)-\Exp(\log f')\Exp(p_X)\\
                   &= \int \log f'(x) \cdot p(x) dx -\int \log f'(x) dx \int p(x) dx\,.
\end{split}\end{equation*}
As shown in Section~2 in \citep{Daniusis_et_al_UAI_10}, $p_Y$ is then related to the inverse function $f^{-1}$
in the sense that
\begin{equation*}
\int_0^1 \log (f^{-1})'(y) \cdot p(y) \,dy \geq \int_0^1 \log (f^{-1})'(y) \,dy\,,
\end{equation*}
with equality if and only if $f'$ is constant.
Hence, $\log (f^{-1})'$ and $p_Y$ are positively correlated due to
\[
\Exp\big(\log (f^{-1})' \cdot p_Y\big) - \Exp\big(\log (f^{-1})'\big) \Exp(p_Y) >0\,.
\]
 Intuitively, this is because the density  $p_Y$ tends to be high in regions where $f$ is flat, or equivalently, $f^{-1}$ is steep
(see also Figure~\ref{fig:igci}). 
Hence, we have shown that $\Prb_Y$ contains information about $f^{-1}$ and hence about $\Prb_{X\given Y}$
whenever $\Prb_X$ does not contain information about $\Prb_{Y \given X}$ (in the sense that \eref{eq:orth}
is satisfied), except for the trivial case where $f$ is linear.

\begin{figure}[t]
  \centering
  \includegraphics[width=0.5\textwidth]{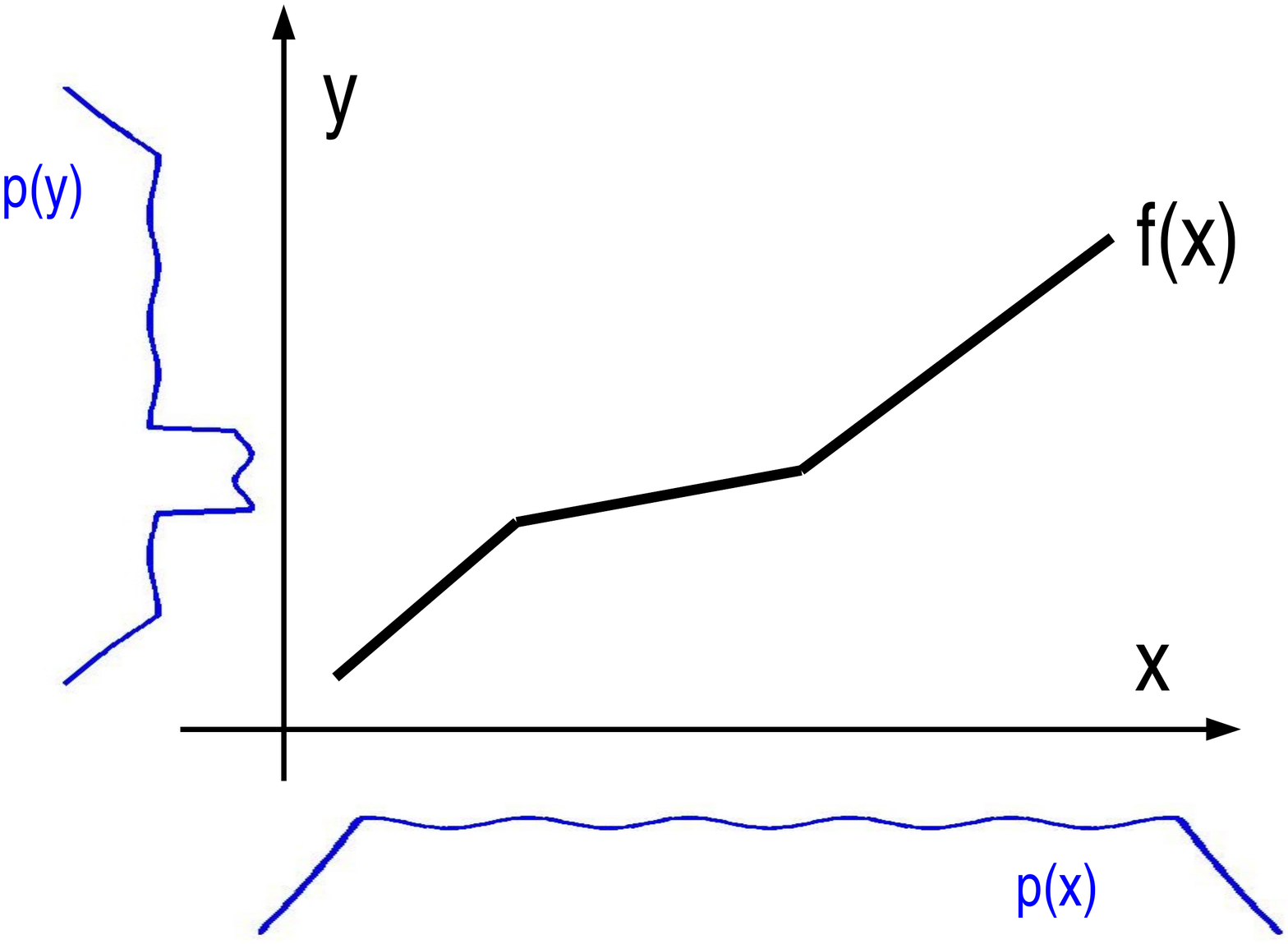}
  \caption{\label{fig:igci}Illustration of the basic intuition behind IGCI. If the density $p_X$ of the cause
  $X$ is not correlated with the slope of $f$, then the density $p_Y$ tends to be high in regions where $f$ is flat (and $f^{-1}$ is steep). Source: \cite{Janzing_et_al_AI_12}.} 
\end{figure}

To employ this asymmetry, 
\citet{Daniusis_et_al_UAI_10} introduce the expressions
\begin{eqnarray}\label{cdefx}
C_{X\to Y} &:=&
\int_0^1 \log f'(x) p(x) dx \\ \label{cdefy}
C_{Y\to X} &:=& \int_0^1 \log (f^{-1})'(y) p(y) dy =-C_{X \to Y}\,.
\end{eqnarray}
Since the right hand side of \eref{eq:orth} is smaller than zero due to $\int_0^1 \log f'(x) dx \leq \log \int_0^1 f'(x) dx=0$  by
concavity of the logarithm (exactly zero only for
constant $f$), IGCI infers $X\to Y$ whenever $C_{X\to Y}$ is negative. 
Section~3.5 in \citep{Daniusis_et_al_UAI_10} also shows that
\[
C_{X\to Y} =H(Y)-H(X)\,,
\]
i.e., the decision rule considers the variable with  lower differential entropy as the effect.
The idea is that the function introduces new irregularities
to a distribution rather than smoothing the irregularities of the distribution of the cause.

\vspace{0.5cm}
\noindent
{\it Generalization to other reference measures:} In the above version of IGCI 
the uniform distribution on $[0,1]$ plays a special role because it is the distribution with respect to which uncorrelatedness between $p_X$ and $\log f'$ is defined. 
The idea can be generalized to other reference distributions.
How to choose the right one for a particular inference problem is a difficult question which
goes beyond the scope of this article.
From a high-level perspective, it is comparable to the question of choosing the right kernel for kernel-based machine learning algorithms; it also is an \emph{a priori} structure of the range of $X$ and $Y$ without which
the inference problem is not well-defined.

Let $u_X$ and $u_Y$ be densities of $X$ and $Y$, respectively, that we call ``reference densities''. 
For example, uniform or Gaussian distributions would be reasonable choices.
Let  
$u_f$ be the image of $u_X$ under $f$ and  $u_{f^{-1}}$ be the image of $u_Y$ under $f^{-1}$. 
Then we hypothesize the following generalization of \eref{eq:orth}:
\begin{principle}
If $X$ causes $Y$  via  a deterministic bijective function $f$
such that $u_{f^{-1}}$ has a density with respect to Lebesgue measure, then
\begin{equation}\label{eq:genorth}
\int \log \frac{u_{f^{-1}}(x)}{u_X(x)} p(x) \,dx \approx
\int \log \frac{u_{f^{-1}}(x)}{u_X(x)} u_X(x) \,dx\,.
\end{equation}
\end{principle}

In analogy to the remarks above, this can also be interpreted as uncorrelatedness
of the functions $\log (u_{f^{-1}}/u_X)$  and $p_X / u_X$ 
with respect to the measure given by the density of $u_X$ with respect to the Lebesgue measure. 
Again, we hypothesize this because the former expression is a property of the function $f$
alone (and the reference densities) and should thus be unrelated to the marginal density $p_X$.
The special case \eref{eq:orth} can be obtained by taking the uniform distribution on $[0,1]$ for $u_X$ and $u_Y$.

As generalization of (\ref{cdefx},\ref{cdefy}) we define\footnote{Note that the formulation in Section~2.3 in \citep{Daniusis_et_al_UAI_10} is more general because it uses {\it manifolds} of reference densities instead of a single density.} 
\begin{eqnarray}
C_{X\to Y} &:=& \int \log \frac{u_{f^{-1}}(x)}{u_X(x)} p(x) dx\\
C_{Y\to X} &:=& \int \log \frac{u_f (y)}{u_Y(y)} p(y) dy =
\int \log \frac{u_X(x)}{u_{f^{-1}}(x)} p(x) dx =
- C_{Y\to X} \label{minusRel}\,,
\end{eqnarray}
where the second equality in (\ref{minusRel}) follows by substitution of variables.
Again, the hypothesized independence implies 
$C_{X\to Y}\leq 0$ since the right hand side of \eref{eq:genorth}
coincides with $-D(u_X\|u_{f^{-1}})$ where
$D(\cdot\|\cdot)$ denotes Kullback-Leibler divergence.
Hence, we also infer $X\to Y$ whenever $C_{X\to Y}<0$. 
Note also that 
\[
C_{X\to Y}= D(p_X\|u_X)-D(p_X\|u_{f^{-1}}) =  D(p_X\|u_X)-D(p_Y\|u_Y)\,,
\]
where we have only used the fact that relative entropy is preserved under bijections.
Hence, our decision rule amounts to inferring that the density of the cause is closer to its reference density.  
%It can be shown that \ref{eq:genorth} implies
%\begin{equation}\label{KL}
%D(p_Y\|u_Y)=D(p_X\|u_X)+D(\overrightarrow{p}_Y\|u_Y)\,,
%\end{equation}
%where $D(\cdot\|\cdot)$ denotes relative entropy distance. 
%Interpreting the distance to the reference density as complexity of a density,
%(\ref{KL}) states that the complexity of the output density is a sum of the input complexity and the complexity of an output induced by an input that is distributed according to the reference density. 
%Then we infer $X\rightarrow Y$ whenever 
This decision rule gets quite simple, for instance, if
 $u_X$ and $u_Y$ are Gaussians with the same mean and variance as $p_X$ and $p_Y$, respectively. Then it again amounts to inferring $X\rightarrow Y$ whenever $X$ has larger entropy than $Y$
 after rescaling both $X$ and $Y$ to have the same variance. 

\subsection{Estimation methods}\label{subsec:igciimpl}

The specification of the reference measure is essential for IGCI.
We describe the implementation for two different choices:
\begin{enumerate}
\item {\it Uniform distribution:} scale and shift $X$ and $Y$ such that extrema are mapped
onto $0$ and $1$. 
\item {\it Gaussian distribution:} scale $X$ and $Y$ to variance 1. 
\end{enumerate}
Given this preprocessing step, there are different options for estimating $C_{X\to Y}$ and
$C_{Y\to X}$ from empirical data (see Section 3.5 in \citep{Daniusis_et_al_UAI_10}):
\begin{enumerate}
\item {\it Slope-based estimator:}
\begin{equation}\label{eq:IGCI_score_slope}
\hat{C}_{X\rightarrow Y}:=\frac{1}{N-1}\sum_{j=1}^{N-1} \log \frac{|y_{j+1}-y_j|}{x_{j+1}-x_j}\,, 
\end{equation}
where we assumed the pairs $\{(x_i,y_i)\}$ to be ordered ascendingly according to $x_i$. Since empirical data are noisy,
the $y$-values need not be in the same order.
$\hat{C}_{Y\to X}$ is given by exchanging the roles of $X$ and $Y$. 

\item {\it Entropy-based estimator:} 
\begin{equation}\label{eq:IGCI_score_entropy}
\hat{C}_{X\to Y} := \hat{H}(Y)-\hat{H}(X)\,,
\end{equation}
where $\hat{H}(\cdot)$ denotes some differential entropy estimator. 
\end{enumerate}
The theoretical equivalence between these estimators breaks down on empirical data
not only due to finite sample effects but also because of noise. For the slope based estimator, we even have 
\[
\hat{C}_{X\to Y} \neq -\hat{C}_{Y\to X}\,,
\]
and thus need to compute both terms separately.

\begin{algorithm}[t]
\caption{\label{alg:IGCI} General procedure to decide whether $\Prb_{X,Y}$ is generated by
a deterministic monotonic bijective function from $X$ to $Y$ or from $Y$ to $X$.}
\textbf{Input}: 
\begin{enumerate}
  \item I.i.d.\ sample $\C{D}_N := \{(x_i,y_i)\}_{i=1}^N$ of $X$ and $Y$ (``data'');
  \item Normalization procedure $\nu : \RN^N \to \RN^N$;
  \item IGCI score estimator $\hat C : \RN^N \times \RN^N \to \RN$.
\end{enumerate}
\textbf{Output}: $\hat C_{X\to Y}$, $\hat C_{Y\to X}$, \texttt{dir}.
\begin{enumerate}
  \item Normalization: \begin{enumerate}
    \item calculate $\tilde{\B{x}} = \nu(\B{x})$
    \item calculate $\tilde{\B{y}} = \nu(\B{y})$
  \end{enumerate}
  \item Estimation of scores: 
  \begin{enumerate}
    \item calculate $\hat{C}_{X\to Y} = \hat C(\tilde{\B{x}},\tilde{\B{y}})$
    \item calculate $\hat{C}_{Y\to X} = \hat C(\tilde{\B{y}},\tilde{\B{x}})$
  \end{enumerate}
  \item Output $\hat C_{X\to Y}, \hat C_{Y\to X}$ and:
    $$\mathtt{dir} := \begin{cases}
      X \to Y & \text{ if } \hat{C}_{X\to Y} < \hat{C}_{Y\to X}, \\
      Y \to X & \text{ if } \hat{C}_{X\to Y} > \hat{C}_{Y\to X}, \\
      ?       & \text{ if } \hat{C}_{X\to Y} = \hat{C}_{Y\to X}.
    \end{cases}$$
\end{enumerate}
\end{algorithm}

Note that the IGCI implementations discussed here make sense only for continuous variables with
a density with respect to Lebesgue measure.
This is because the difference quotients are undefined if a value occurs twice. 
In many empirical data sets, however, the discretization (e.g., due to rounding to some number
of digits) is not fine enough to guarantee this.
A very preliminary heuristic that was employed in earlier work \citep{Daniusis_et_al_UAI_10} removes repeated occurrences by removing data points, but
a conceptually cleaner solution would be, for instance, the following procedure:
Let ${\tilde x}_j$ with $1 \leq j\leq \tilde{N}$ be the ordered values after removing repetitions and let $\tilde{y}_j$ 
denote
the corresponding $y$-values. 
Then we replace \eref{eq:IGCI_score_slope} with
\begin{equation}\label{eq:IGCI_score_slope++}
  \hat{C}_{X\rightarrow Y} := \frac{1}{\sum_{j=1}^{\tilde{N}-1} n_j} \sum_{j=1}^{\tilde{N}-1} n_j \log \frac{|\tilde{y}_{j+1}-\tilde{y}_j|}{\tilde{x}_{j+1}-\tilde{x}_j}\,,
\end{equation}
where $n_j$ denotes the number of occurrences of $\tilde{x}_j$ in the original data set. 
Here we have ignored the problem of  repetitions of $y$-values since they are less likely,
because they are not ordered if the relation between $X$ and $Y$ is noisy (and for bijective deterministic relations, they only occur together with repetitions of $x$ anyway).

Finally, let us mention one simple case where IGCI with estimator \eref{eq:IGCI_score_slope} provably works asymptotically, even though its assumptions are violated. This happens if the effect is a non-injective function of the cause. More precisely, assume $Y=f(X)$ where $f:[0,1]\rightarrow [0,1]$ is continuously differentiable and non-injective, and moreover, that $p_X$ is strictly positive and bounded away from zero. 
To argue that 
$\hat{C}_{Y\to X}> \hat{C}_{X\to Y}$ asymptotically for $N\to\infty$ we first observe that
the mean value theorem implies 
\begin{equation}\label{eq:meanv}
\frac{|f(x_2)-f(x_1)|}{|x_2-x_1|} \leq \max_{x\in [0,1]} |f'(x)|=: s_{\max}
\end{equation}
 for any pair $x_1,x_2$. 
Thus, for any sample size $N$ we have $\hat{C}_{X\to Y} \leq \log s_{\max}$.
On the other hand, $\hat{C}_{Y\to X} \to \infty$ for $N\to\infty$.
To see this, note that all terms in the sum \eqref{eq:IGCI_score_slope} are bounded from below
by $-\log s_{\max}$ due to \eqref{eq:meanv}, while there is no {\it upper} bound for the summands
because adjacent $y$-values may be from different branches of the non-injective function $f$
and then the corresponding $x$-values may not be close. 
Indeed, this will happen for a constant fraction of adjacent pairs. For those, the gaps between
the $y$-values decrease with $\mathcal{O}(1/N)$ while the distances of the corresponding $x$-values remain of $\mathcal{O}(1)$. Thus, the overall sum \eqref{eq:IGCI_score_slope} diverges. 
It should be emphasized, however, that one can have opposite effects for any finite $N$. To see this, 
consider the function $x \mapsto 2|x-1/2|$ and modify it locally around $x=1/2$ to obtain a continuously differentiable function $f$. Assume that the probability density in $[1/2,1]$ is so low that almost all points are contained in $[0,1/2]$. 
Then, $\hat{C}_{X\to Y}\approx \log 2$ while $\hat{C}_{Y\to X} \approx -\log 2$ and IGCI with estimator \eref{eq:IGCI_score_slope} decides
(incorrectly) on $Y\to X$. For sufficiently large $N$, however, a constant (though possibly very small) fraction of
$y$-values come from different branches of $f$ and thus $\hat{C}_{Y\to X}$ diverges (while $\hat{C}_{X \to Y}$ remains bounded from above).

\section{Experiments}\label{sec:experiments} %%%%%%%%%%%%%%%%%%%%%%%%%%%%%%%%%%%%%%%%%%%%%%%%%%%%%%%%%%%%%%%%%%%%%%%%%%%%%%%%%%%%%%%%%%%%%%%%%%%%%%%%%%%%%%%%%%%%%%%%%%%%%%%%%%%%%%%%%%%%%%%%%%%%%%%%%%

In this section we describe the data that we used for evaluation, implementation details for various
methods, and our evaluation criteria. The results of the empirical study will be presented in 
Section~\ref{sec:results}.

\subsection{Implementation details}

The complete source code to reproduce our experiments has been made available online as open source
under the FreeBSD license on the homepage of the first author\footnote{\url{http://www.jorismooij.nl/}}.
We used \texttt{MatLab} on a \texttt{Linux} platform, and made use of external libraries \texttt{GPML} v3.5 (2014-12-08)
\citep{RasmussenNickisch2010} for GP regression and \texttt{ITE} v0.61 \citep{Szabo2014} for entropy estimation.
For parallelization, we used the convenient command line tool \texttt{GNU parallel} \citep{Tange2011}.

\subsubsection{Regression}

For regression, we used standard Gaussian Process (GP) Regression
\citep{RasmussenWilliams2006}, using the \texttt{GPML} implementation
\citep{RasmussenNickisch2010}. We used a squared exponential covariance
function, constant mean function, and an additive Gaussian noise likelihood. We
used the FITC approximation \citep{QuinoneroRasmussen2005} as an approximation 
for exact GP regression in order to reduce computation time. We found that
100 FITC points distributed on a linearly spaced grid greatly reduce
computation time without introducing a noticeable approximation error.
Therefore, we used this setting as a default for the GP regression. 
The computation time of this method scales as $\mathcal{O}(N m^2 T)$, where
$N$ is the number of data points, $m = 100$ is the number of FITC points, and
$T$ is the number of iterations necessary to optimize the marginal likelihood
with respect to the hyperparameters. In practice, this yields considerable 
speedups compared with exact GP inference, which scales as $\mathcal{O}(N^3 T)$.
%To optimize the hyperparameters, we used the \texttt{LBFGSB} minimization code that is distributed 
%along with the \texttt{GPML} toolbox, because it turned out to be faster than the default
%numerical minimization procedure \texttt{minimize.m} that is offered by the toolbox.\footnote{We noticed that in a few cases, the two routines did not obtain the same minimum, but we did not investigate this in detail.}

\subsubsection{Entropy estimation}

We tried many different empirical entropy estimators, see Table~\ref{tab:ITE}. The first
method, \texttt{1sp}, uses a so-called ``1-spacing'' estimate \citep[e.g.,][]{Kraskov++2004}:
\begin{equation}\label{eq:ent_estimator}
  \hat H(\B{x}) := \psi(N) - \psi(1) + \frac{1}{N-1} \sum_{i=1}^{N-1} \log \abs{x_{i+1}-x_{i}}\,,
\end{equation}
where the $x$-values should be ordered ascendingly, i.e., $x_i \le x_{i+1}$, and
$\psi$ is the digamma function (i.e., the logarithmic derivative of the gamma function: $\psi(x) = d/dx\, \log \Gamma(x)$,
which behaves as $\log x$ asymptotically for $x \to \infty$). As this estimator would
become $-\infty$ if a value occurs more than once, we first remove duplicate values from
the data before applying \eref{eq:ent_estimator}. There should be better ways of dealing with 
discretization effects, but we nevertheless include this particular estimator for comparison,
as it was also used in previous implementations of the entropy-based IGCI method \citep{Daniusis_et_al_UAI_10, Janzing_et_al_AI_12}.
These estimators can be implemented in $\mathcal{O}(N \ln N)$ complexity, as they only need
to sort the data and then calculate a sum over data points.

\begin{table}[t]
  \centering
  \caption{\label{tab:ITE}Entropy estimation methods. ``ITE'' refers to the Information Theoretical Estimators Toolbox \citep{Szabo2014}. The first group of entropy estimators is nonparametric, the second group makes additional parametric assumptions on the distribution of the data.} 
  \begin{tabular}{lll}
    Name & Implementation & References \\
    \hline
    1sp & based on \eref{eq:ent_estimator} & \citep{Kraskov++2004} \\
    3NN & ITE: Shannon\_kNN\_k & \citep{KozachenkoLeonenko1987} \\
    sp1 & ITE: Shannon\_spacing\_V & \citep{Vasicek1976} \\
    sp2 & ITE: Shannon\_spacing\_Vb & \Citep{VanEs1992} \\
    sp3 & ITE: Shannon\_spacing\_Vpconst & \citep{EbrahimiPflughoeftSoofi1994} \\
    sp4 & ITE: Shannon\_spacing\_Vplin & \citep{EbrahimiPflughoeftSoofi1994} \\
    sp5 & ITE: Shannon\_spacing\_Vplin2 & \citep{EbrahimiPflughoeftSoofi1994} \\
    sp6 & ITE: Shannon\_spacing\_VKDE & \citep{NoughabiNoughabi2013} \\
    KDP & ITE: Shannon\_KDP & \citep{StowellPlumbley2009} \\ 
    PSD & ITE: Shannon\_PSD\_SzegoT & \citep{RamirezViaSantamariaCrespo2009,Gray2006} \\
    & & \citep{GrenanderSzego1958} \\
    EdE & ITE: Shannon\_Edgeworth & \citep{VanHulle2005} \\
    \hline
    Gau & based on \eref{eq:entropy_Gaussian} & \\
    ME1 & ITE: Shannon\_MaxEnt1 & \citep{Hyvarinen1997} \\
    ME2 & ITE: Shannon\_MaxEnt2 & \citep{Hyvarinen1997} \\
%    exF & ITE: Shannon\_expF & \citep{NielsenNock2012}  \\ = Gauss
    \hline
  \end{tabular}
\end{table}

We also made use of various entropy estimators implemented in the Information
Theoretical Estimators (ITE) Toolbox \citep{Szabo2014}. The method \texttt{3NN} is
based on $k$-nearest neighbors with $k=3$, all \texttt{sp*} methods use Vasicek's
spacing method with various corrections, \texttt{KDP} uses k-d partitioning,
\texttt{PSD} uses the power spectral density representation and Szego's
theorem, \texttt{ME1} and \texttt{ME2} use the maximum entropy
distribution method, and \texttt{EdE} uses the Edgeworth expansion.
%and \texttt{exF} uses maximum likelihood estimation using an exponential family model. 
For more details, see the documentation of the \text{ITE} toolbox \citep{Szabo2014}.

\subsubsection{Independence testing: HSIC}

As covariance function for HSIC, we use the popular Gaussian kernel:
$$\kappa_{\ell} : (x,x') \mapsto \exp \left(-\frac{(x-x')^2}{\ell^2}\right),$$
with bandwidths selected by the median heuristic \citep{SchoelkopfSmola02}, i.e.,
we take 
$$\hat\ell(\B{u}) := \mathrm{median}\{ \norm{u_i - u_j} : 1 \le i < j \le N, \norm{u_i - u_j} \ne 0 \},$$
and similarly for $\hat \ell(\B{v})$. We also compare with a fixed bandwidth of 0.5.
As the product of two Gaussian kernels is characteristic, HSIC with such kernels
will detect any dependence asymptotically (see also Lemma~\ref{lemm:HSICzero} in Appendix~\ref{sec:consistency}),
at least when the bandwidths are fixed.

The $p$-value can either be estimated by using permutation, or can be approximated by a Gamma
approximation, as the mean and variance of the HSIC value under the null
hypothesis can also be estimated in closed form \citep{Gretton++2008}. In this work, we use the
Gamma approximation for the HSIC $p$-value.

The computation time of our na\"ive implementation of HSIC scales as $\C{O}(N^2)$. Using
incomplete Cholesky decompositions, one can obtain an accurate approximation in only
$\C{O}(N)$ \citep{JegelkaGretton2007}. However, the na\"ive implementation was fast enough
for our purpose.

\subsection{Data sets}

We will use both real-world and simulated data in order to evaluate the methods.
Here we give short descriptions and refer the reader to Appendix~\ref{sec:sim} and Appendix~\ref{sec:dataset}
for details.

\subsubsection{Real-world benchmark data}

The \texttt{CauseEffectPairs} (\texttt{CEP}) benchmark data set that we propose
in this work consists of different ``cause-effect pairs'', each one consisting 
of samples of a pair of statistically dependent random variables, where one 
variable is known to cause the other one.
It is an extension of the collection of the eight data sets that formed
the ``CauseEffectPairs'' task in the \emph{Causality Challenge \#2: Pot-Luck}
competition \citep{MooijJanzing_JMLR_10} which was performed as part of the NIPS
2008 Workshop on Causality \citep{NIPSCausalityChallenge2008}. 
Version \version\ of the {\tt CauseEffectPairs} collection that we present here consists of \nrpairs\ pairs, 
taken from \nrdatasets\ different data sets from various domains.
%\citet{Thoemmes2015} recently computed several features of the variable pairs including (among others) correlation, skewness, kurtosis, normality and whether their relationship is nonlinear.
The \texttt{CEP} data are publicly available at \citep{MooijJanzingSchoelkopf2014}. 
Appendix~\ref{sec:dataset} contains a detailed description of each cause-effect
pair and a justification of what we believe to be the ground truth causal relations. Scatter plots
of the pairs are shown in Figure~\ref{fig:allpairs_CEP}.
In our experiments, we only considered the 95 out of 100 pairs that have one-dimensional variables,
i.e., we left out pairs 52--55 and 71.

\begin{figure}[t]
\centerline{\includegraphics[width=0.9\textwidth]{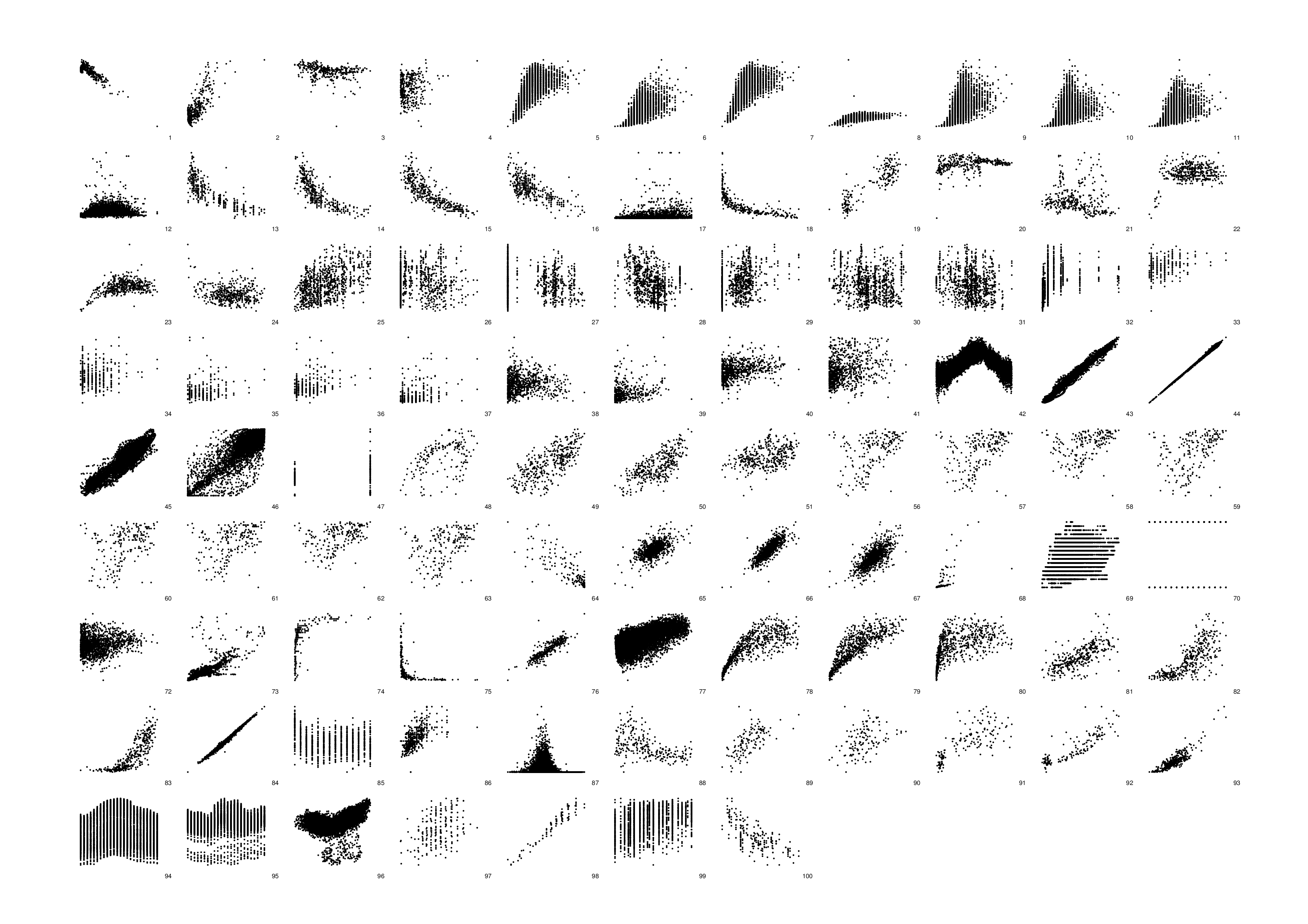}}
\caption{\label{fig:allpairs_CEP}Scatter plots of the cause-effect pairs in the \texttt{CauseEffectPairs} benchmark data. We only show the pairs for which both variables are one-dimensional.}
\end{figure}

\subsubsection{Simulated data}

As collecting real-world benchmark data is a tedious process (mostly because
the ground truths are unknown, and acquiring the necessary understanding of the
data-generating process in order to decide about the ground truth is not
straightforward), we also studied the performance of methods on simulated data
where we can control the data-generating process, and therefore can be certain
about the ground truth.

Simulating data can be done in many ways. It is not straightforward to simulate
data in a ``realistic'' way, e.g., in such a way that scatter plots of
simulated data look similar to those of the real-world data (see
Figure~\ref{fig:allpairs_CEP}). For reproducibility, we describe in Appendix~\ref{sec:sim}
in detail how the simulations were done. Here, we will just sketch
the main ideas.

We sample data from the following structural equation models. If we do not want to 
model a confounder, we use:
\begin{align*}
  & E_X \sim p_{E_X},\, E_Y \sim p_{E_Y} \\
  & X = f_X(E_X) \\
  & Y = f_Y(X,E_Y),
\end{align*}
and if we do want to include a confounder $Z$, we use:
\begin{align*}
  & E_X \sim p_{E_X},\, E_Y \sim p_{E_Y},\, E_Z \sim p_{E_Z} \\
  & Z = f_Z(E_Z) \\
  & X = f_X(E_X,E_Z) \\
  & Y = f_Y(X,E_Y,E_Z).
\end{align*}
Here, the noise distributions $p_{E_X}, p_{E_Y}, p_{E_Z}$ are randomly generated
distributions, and the causal mechanisms $f_Z, f_X, f_Y$ are randomly generated functions.
Sampling the random distributions for a noise variable $E_X$ (and similarly for $E_Y$ and $E_Z$) 
is done by mapping a standard-normal distribution through a random function, which we sample 
from a Gaussian Process. The causal mechanism $f_X$ (and similarly $f_Y$ and $f_Z$) is
drawn from a Gaussian Process as well. After sampling the noise distributions and the functional
relations, we generate data for $X,Y,Z$. Finally, Gaussian measurement noise is added to both 
$X$ and $Y$.

By controlling various hyperparameters, we can control certain aspects of the
data generation process. We considered four different scenarios. \texttt{SIM}
is the default scenario without confounders. \texttt{SIM-c} includes a
one-dimensional confounder, whose influences on $X$ and $Y$ are typically
equally strong as the influence of $X$ on $Y$. The setting \texttt{SIM-ln}
has low noise levels, and we would expect IGCI to work well in this scenario.
Finally, \texttt{SIM-G} has approximate Gaussian distributions for the cause
$X$ and approximately additive Gaussian noise (on top of a nonlinear
relationship between cause and effect); we expect that methods which make these
Gaussianity assumptions will work well in this scenario. Scatter plots of the
simulated data are shown in Figures~\ref{fig:allpairs_sim}--\ref{fig:allpairs_simgauss}.

\begin{figure}[p]
\centerline{\includegraphics[width=0.9\textwidth]{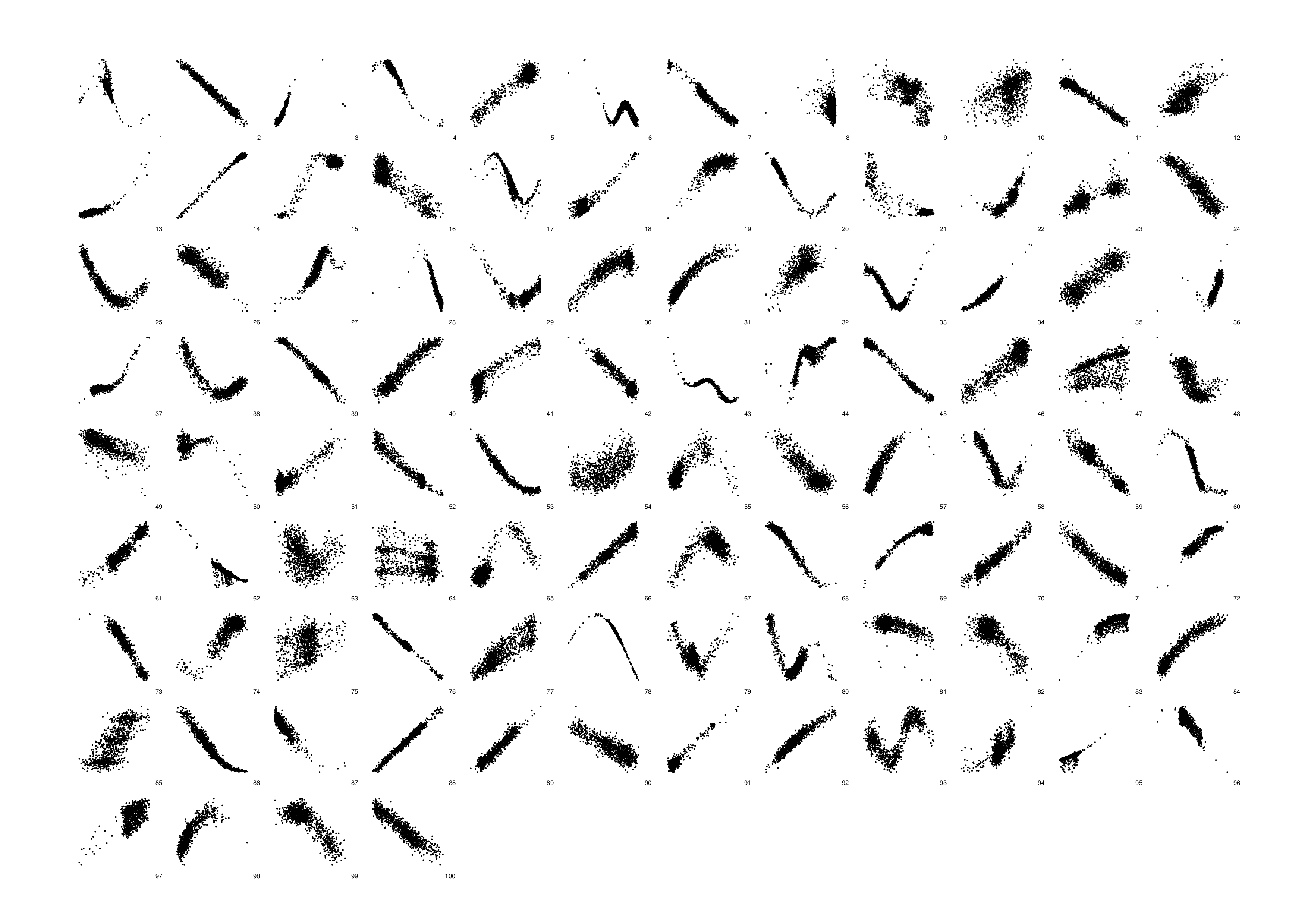}}
\caption{\label{fig:allpairs_sim}Scatter plots of the cause-effect pairs in simulation scenario \texttt{SIM}.}
\end{figure}
\begin{figure}[p]
\centerline{\includegraphics[width=0.9\textwidth]{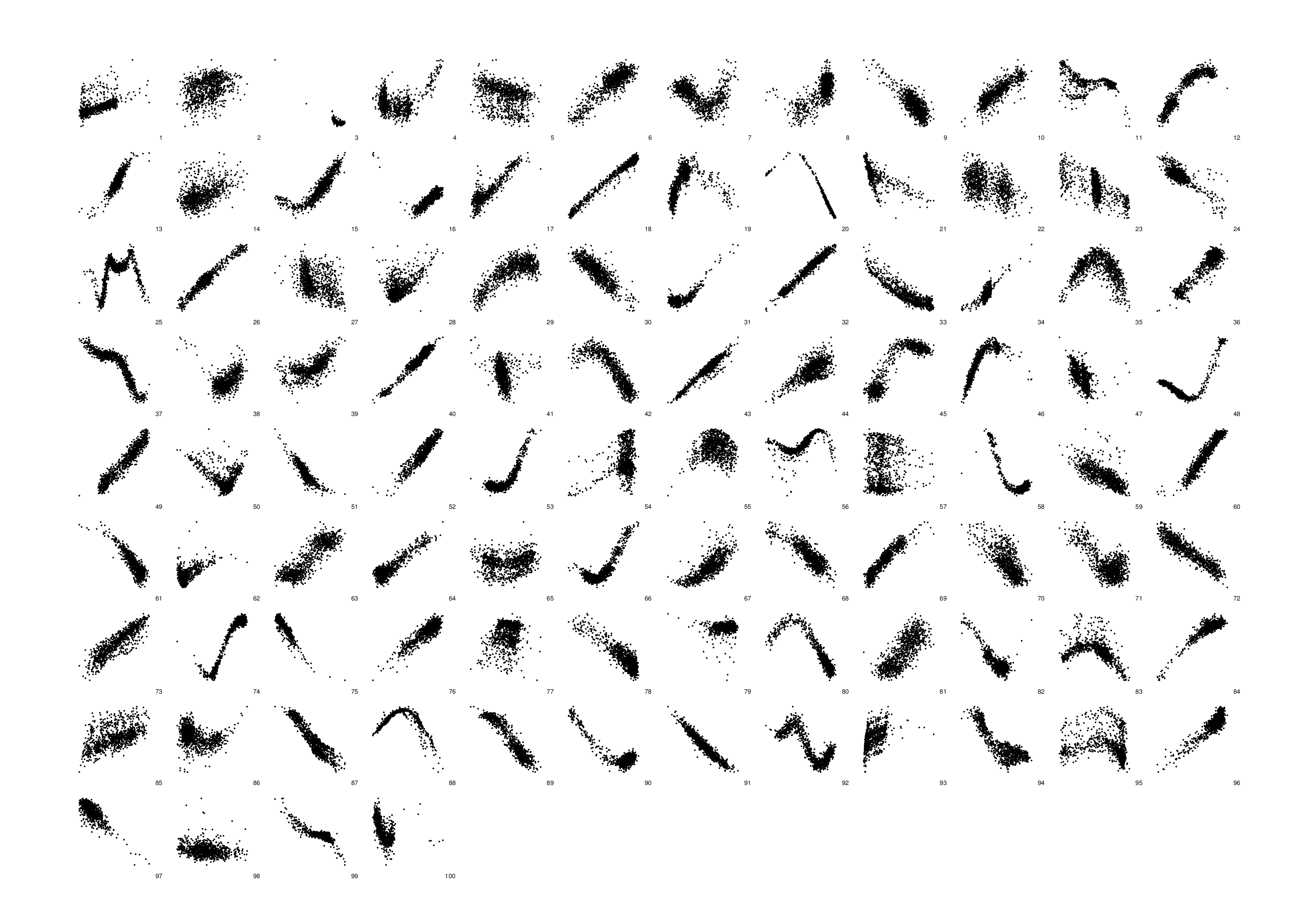}}
\caption{\label{fig:allpairs_simconf}Scatter plots of the cause-effect pairs in simulation scenario \texttt{SIM-c}.}
\end{figure}
\begin{figure}[p]
\centerline{\includegraphics[width=0.9\textwidth]{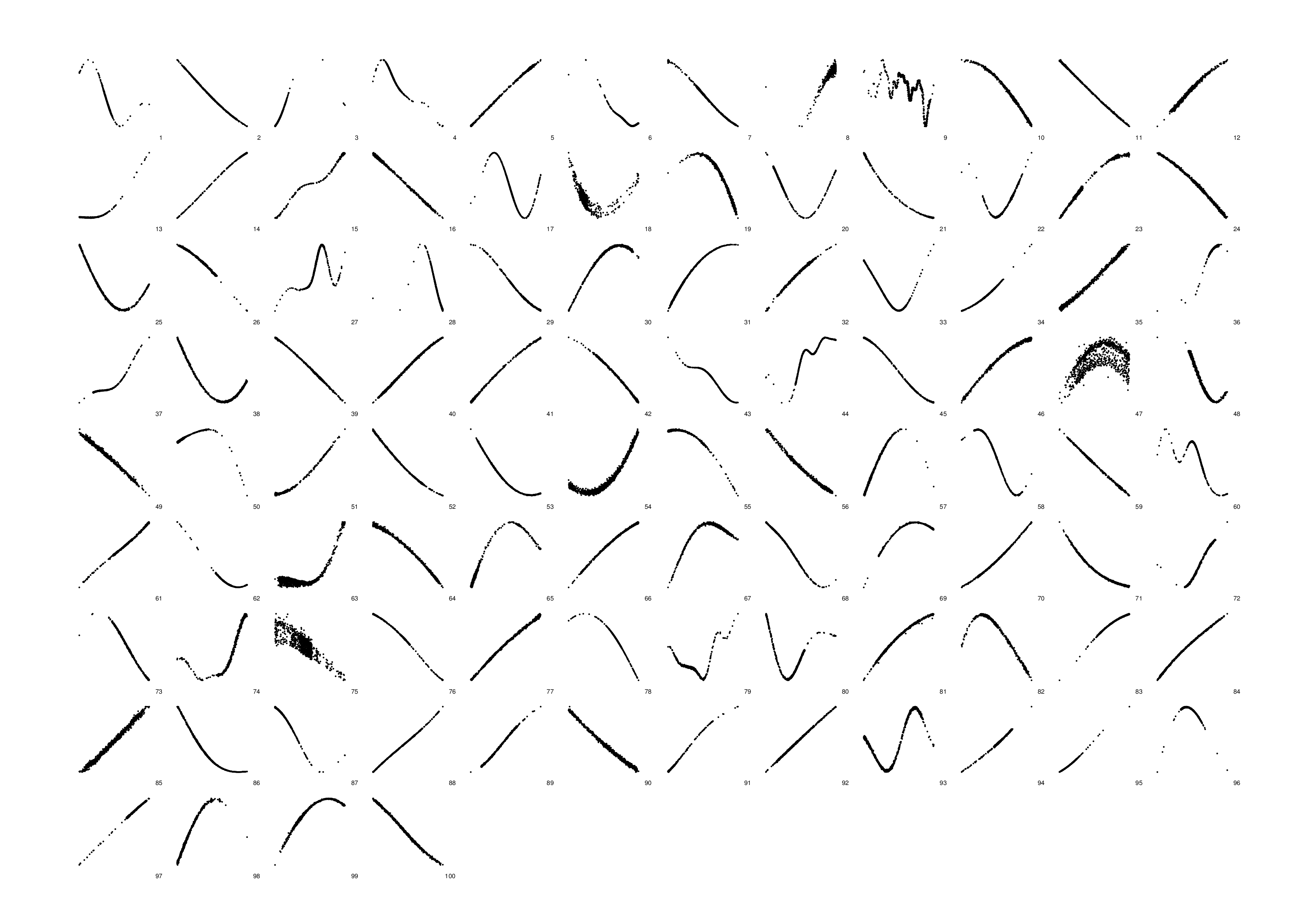}}
\caption{\label{fig:allpairs_simlownoise}Scatter plots of the cause-effect pairs in simulation scenario \texttt{SIM-ln}.}
\end{figure}
\begin{figure}[p]
\centerline{\includegraphics[width=0.9\textwidth]{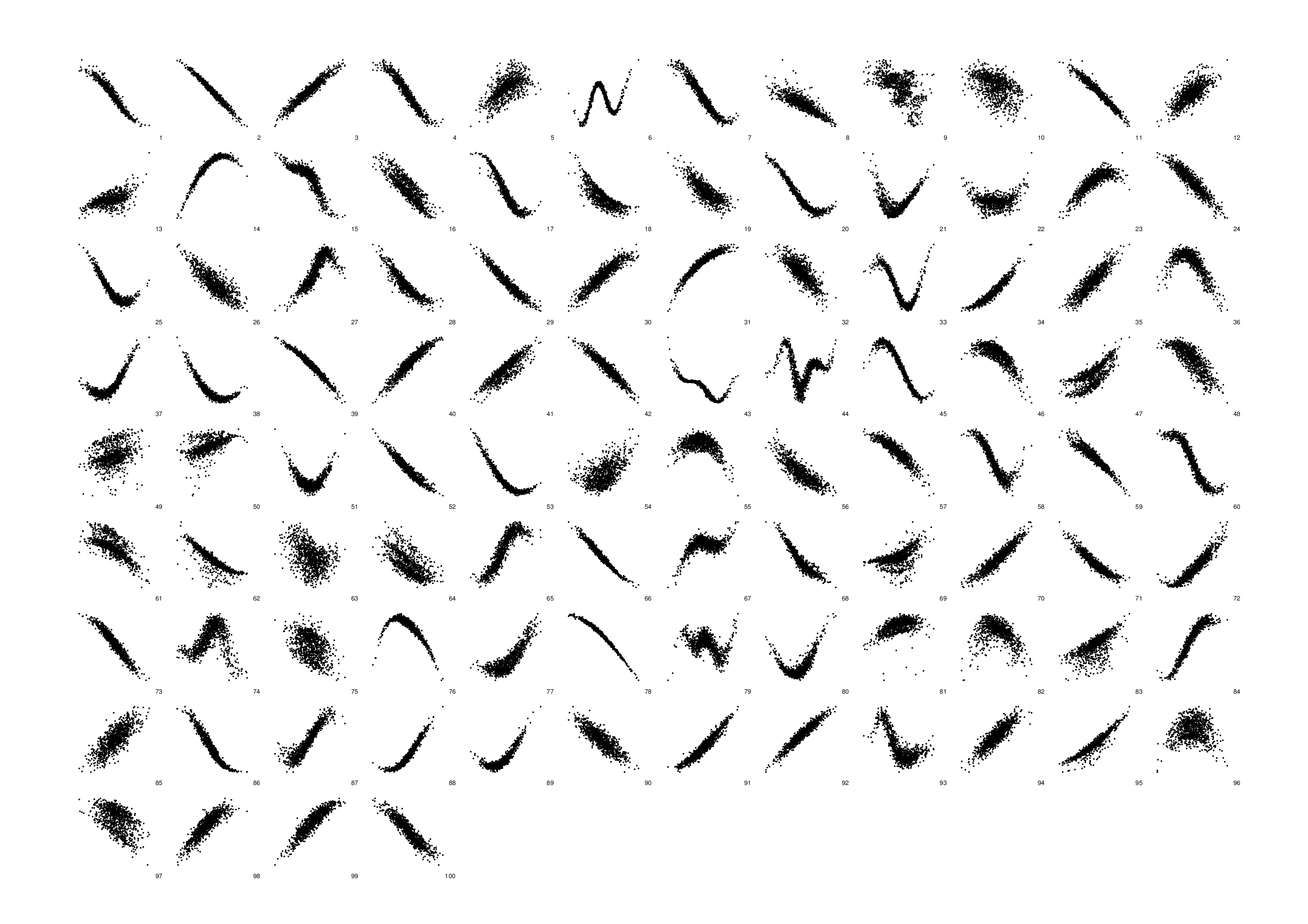}}
\caption{\label{fig:allpairs_simgauss}Scatter plots of the cause-effect pairs in simulation scenario \texttt{SIM-G}.}
\end{figure}
%\begin{figure}
%\centerline{\includegraphics[width=0.9\textwidth]{allpairs_simeleni.pdf}}
%\caption{\label{fig:allpairs_simeleni}\textsf{simeleni}: All cause-effect pairs (simulated, smooth input distributions).}
%\end{figure}

\subsection{Preprocessing and Perturbations}

The following preprocessing was applied to each pair $(X,Y)$.
Both variables $X$ and $Y$ were standardized (i.e., an affine transformation is
applied on both variables such that their empirical mean becomes 0, and their
empirical standard deviation becomes 1). 
%Then a random subsample of $N_{max}$ samples was taken.
In order to study the effect of discretization and other small perturbations of the data,
one of these four perturbations was applied:
\begin{description}
  \item[unperturbed]: No perturbation is applied.
  \item[discretized]: Discretize the variable that has the most unique values such that after discretization, it has as many unique values as the other variable.
    The discretization procedure repeatedly merges those values for which the sum of the absolute error that would be caused by the merge is minimized.
  \item[undiscretized]: ``Undiscretize'' both variables $X$ and $Y$. The undiscretization procedure adds noise to each data point $z$, drawn uniformly from the
  interval $[0,z'-z]$, where $z'$ is the smallest value $z' > z$ that occurs in the data.
  \item[small noise]: Add tiny independent Gaussian noise to both $X$ and $Y$ (with mean 0 and standard deviation $10^{-9}$).
\end{description}
Ideally, a causal discovery method should be robust against these and other small perturbations of the data.

\subsection{Evaluation Measures}

We evaluate the performance of the methods in two different ways:
\begin{description}
  \item[forced-decision]: given a sample of a pair $(X,Y)$ the methods have to decide either $X \to Y$ or $Y \to X$; in this setting we evaluate the accuracy of these decisions;
  \item[ranked-decision]: we used the scores $\hat C_{X\to Y}$ and $\hat C_{Y \to X}$ to construct heuristic confidence estimates that are used to rank the decisions; we then produced receiver-operating characteristic (ROC) curves and used the area under the curve (AUC) as performance measure.
\end{description}
Some methods have an advantage in the second setting, as the scores on which their decisions are based yield a reasonably accurate ranking of the decisions. By only taking the most confident (highest ranked) decisions, the accuracy of these decisions increases, and this leads to a higher AUC than for random confidence estimates. Which of the two evaluation measures (accuracy or AUC) is the most relevant depends on the application.\footnote{In earlier work, we have reported accuracy-decision rate curves instead of ROC curves. However, it is easy to visually overinterpret the significance of such a curve in the low decision-rate region. In addition, AUC was used as the evaluation measure in \citet{Guyon++Challenges}. A slight disadvantage of ROC curves is that they introduce an asymmetry between ``positives'' and ``negatives'', whereas for our task, there is no such asymmetry: we can easily transform a positive into a negative and vice versa by swapping the variables $X$ and $Y$. Therefore, ``accuracy'' is a more natural measure than ``precision'' in our setting. We mitigate this problem by balancing the class labels by swapping $X$ and $Y$ variables for a subset of the pairs.}

\subsubsection{Weights}

For the \texttt{CEP} data, we cannot always consider pairs that come from the
same data set as independent.  For example, in the case of the \texttt{Abalone}
data set \citep{UCI_ML_repository,NashSellersTalbotCawthornFord94}, the variables ``whole weight'', ``shucked weight'', ``viscera
weight'', ``shell weight'' are strongly correlated. Considering the four pairs
(age, whole weight), (age, shucked weight), etc., as independent could
introduce a bias. We (conservatively) correct for that bias by downweighting these
pairs. In general, we chose the weights
such that the weights of all pairs from the same data set are equal and sum to one. 
For the real-world cause-effect pairs, the weights are specified in Table~\ref{tab:CEP_pairs}.
For the simulated pairs, we do not use weighting.

\subsubsection{Forced-decision: evaluation of accuracy}

In the ``forced-decision'' setting, we calculate the weighted accuracy of a method in the following way:
\begin{equation}\label{eq:accuracy}
  \mathrm{accuracy} = \frac{\sum_{m=1}^M w_m \delta_{\hat d_m, d_m}}{\sum_{m=1}^M w_m},
\end{equation}
where $d_m$ is the true causal direction for the $m$'th pair (either ``$\leftarrow$'' or ``$\rightarrow$''), $\hat d_m$ is the estimated direction
(one of ``$\leftarrow$'', ``$\rightarrow$'', and ``?''), and $w_m$ is the \emph{weight} of the pair.
Note that we are only awarding correct decisions, i.e., if no estimate is given ($d_m = \text{``?''}$), this will negatively affect the accuracy.
We calculate confidence intervals assuming a binomial distribution using the method by \citet{ClopperPearson1934}.

\subsubsection{Ranked-decision: evaluation of AUC}

To construct an ROC curve, we need to rank the decisions based on some heuristic estimate of confidence.
For most methods we simply use:
\begin{equation}\label{eq:heuristic_default}
  \hat S := -\hat C_{X \to Y} + \hat C_{Y \to X}.
\end{equation}
The interpretation is that the higher $\hat S$, the more likely $X \to Y$, and the lower $\hat S$, the more likely $Y \to X$.
For \texttt{ANM-pHSIC}, we use a different heuristic:
\begin{equation}\label{eq:heuristic_pHSIC}
  \hat S := \begin{cases}
    \frac{-1}{\min\{\hat C_{X \to Y}, \hat C_{Y \to X}\}} & \text{ if }\hat C_{X \to Y} < \hat C_{Y \to X} \\
    \frac{1}{\min\{\hat C_{X \to Y}, \hat C_{Y \to X}\}} & \text{ if }\hat C_{X \to Y} > \hat C_{Y \to X}, \\
  \end{cases}
\end{equation}
and for \texttt{ANM-HSIC}, we use:
\begin{equation}\label{eq:heuristic_HSIC}
  \hat S := \begin{cases}
    \frac{-1}{1 - \min\{\hat C_{X \to Y}, \hat C_{Y \to X}\}} & \text{ if }\hat C_{X \to Y} < \hat C_{Y \to X} \\
    \frac{1}{1 - \min\{\hat C_{X \to Y}, \hat C_{Y \to X}\}} & \text{ if }\hat C_{X \to Y} > \hat C_{Y \to X}. \\
  \end{cases}
\end{equation}

In the ``ranked-decision'' setting, we also use weights to calculate weighted recall (depending on a threshold $\theta$)
\begin{equation}\label{eq:recall}
  \mathrm{recall}(\theta) = \frac{\sum_{m=1}^M w_m \id_{\hat S_m > \theta} \delta_{d_m, \rightarrow}}{\sum_{m=1}^M w_m \delta_{d_m, \rightarrow}},
\end{equation}
where $\hat S_m$ is the heuristic score of the $m$'th pair (high values indicating high likelihood that $d_m = {}\rightarrow$, low values indicating high likelihood that $d_m = {}\leftarrow$), and the weighted precision (also depending on $\theta$)
\begin{equation}\label{eq:precision}
  \mathrm{precision}(\theta) = \frac{\sum_{m=1}^M w_m \id_{\hat S_m > \theta} \delta_{d_m, \rightarrow}}{\sum_{m=1}^M w_m \id_{\hat S_m > \theta}}.
\end{equation}
We use the \texttt{MatLab} routine \texttt{perfcurve} to produce (weighted) ROC curves and to estimate weighted AUC and confidence intervals for the weighted AUC by bootstrapping.\footnote{We used the ``percentile method'' (\texttt{BootType} = \texttt{'per'}) as the default method (``bias corrected and accelerated percentile method'') sometimes yielded an estimated AUC that fell outside the estimated 95\% confidence interval of the AUC.}

\section{Results}\label{sec:results} %%%%%%%%%%%%%%%%%%%%%%%%%%%%%%%%%%%%%%%%%%%%%%%%%%%%%%%%%%%%%%%%%%%%%%%%%%%%%%%%%%%%%%%%%%%%%%%%%%%%%%%%%%%%%%%%%%%%%%%%%%%%%%%%%%%%%%%%%%%%%%%%%%%%%%%%%%%%%%%%%%

In this section, we report the results of the experiments that we carried out in order to evaluate the performance of various methods. We plot the accuracies
and AUCs as box plots, indicating the estimated (weighted) accuracy or AUC, the corresponding $68\%$ confidence interval, and the $95\%$ confidence interval.
If there were pairs for which no decision was taken because of some failure, the number of nondecisions is indicated on the corresponding boxplot.
The methods that we evaluated are listed in Table~\ref{tab:methods}. Computation times are reported in Appendix~\ref{sec:computation_times}.

\begin{table}[t]
  \centering
  \caption{\label{tab:methods}The methods that are evaluated in this work.}
  \small
  \begin{tabular}{lllll}
    Name & Algorithm & Score & Heuristic & Details \\
    \hline
    \texttt{ANM-pHSIC} & \ref{alg:bivariateANM} & \eref{eq:ANM_score_pHSIC} & \eref{eq:heuristic_pHSIC} & Data recycling, adaptive kernel bandwidth \\
    \texttt{ANM-HSIC} & \ref{alg:bivariateANM} & \eref{eq:ANM_score_HSIC} & \eref{eq:heuristic_HSIC} & Data recycling, adaptive kernel bandwidth \\
    \texttt{ANM-HSIC-ds} & \ref{alg:bivariateANM} & \eref{eq:ANM_score_HSIC} & \eref{eq:heuristic_HSIC} & Data splitting, adaptive kernel bandwidth \\
    \texttt{ANM-HSIC-fk} & \ref{alg:bivariateANM} & \eref{eq:ANM_score_HSIC} & \eref{eq:heuristic_HSIC} & Data recycling, fixed kernel bandwidth (0.5) \\
    \texttt{ANM-HSIC-ds-fk} & \ref{alg:bivariateANM} & \eref{eq:ANM_score_HSIC} & \eref{eq:heuristic_HSIC} & Data splitting, fixed kernel bandwidth (0.5) \\
    \texttt{ANM-ent-\dots} & \ref{alg:bivariateANM} & \eref{eq:ANM_score_entropy} & \eref{eq:heuristic_default} & Data recycling, entropy estimators from Table~\ref{tab:ITE} \\
    \texttt{ANM-Gauss} & \ref{alg:bivariateANM} & \eref{eq:ANM_score_Gauss} & \eref{eq:heuristic_default} & Data recycling \\
    \texttt{ANM-FN} & \ref{alg:bivariateANM_MML} & \eref{eq:ANM_score_FriedmanNachman} & \eref{eq:heuristic_default}  & \\
    \texttt{ANM-MML} & \ref{alg:bivariateANM_MML} & \eref{eq:ANM_score_MML} & \eref{eq:heuristic_default}  & \\
    \hline
    \texttt{IGCI-slope} & \ref{alg:IGCI} & \eref{eq:IGCI_score_slope} & \eref{eq:heuristic_default} & \\
    \texttt{IGCI-slope++} & \ref{alg:IGCI} & \eref{eq:IGCI_score_slope++} & \eref{eq:heuristic_default} & \\
    %IGCI-org_entropy = IGCI-ent-1sp
    \texttt{IGCI-ent-\dots} & \ref{alg:IGCI} & \eref{eq:IGCI_score_entropy} & \eref{eq:heuristic_default} & Entropy estimators from Table~\ref{tab:ITE} \\
    \hline
  \end{tabular}
\end{table}

\subsection{Additive Noise Models}

We start by reporting the results for methods that exploit additivity of the noise. Figure~\ref{fig:acc_ANM_all_data} shows the performance of all ANM
methods on different unperturbed data sets, i.e., the \texttt{CEP} benchmark and various simulated data sets. Figure~\ref{fig:acc_ANM_all_CEP}
shows the performance of the same methods on different perturbations of the \texttt{CEP} benchmark data. 
The six variants \texttt{sp1},\dots,\texttt{sp6} of the spacing estimators perform very similarly, so we 
show only the results for \texttt{ANM-ent-sp1}. For the ``undiscretized'' perturbed version of the \texttt{CEP}
benchmark data, GP regression failed in one case because of a numerical problem, which explains the failures
across all methods in Figure~\ref{fig:acc_ANM_all_CEP} for that case.
%For the six variants \texttt{ANM-ent-sp}{\it i} of 
%the spacing estimator, only the results for \texttt{sp1} are shown, as results for \texttt{sp2},\dots,\texttt{sp6} were similar.

%\subsubsection{FITC approximation}
%FITC works good as an approximation for exact GP regression (in combination with pHSIC as score), and 100 FITC points are enough.
%This holds both for LBFGSB as optimizer and the standard optimizer.
%This greatly reduces computation time (which would otherwise scale cubically with the number of data points).

\begin{figure}[p]
\centering
\includegraphics[scale=0.9]{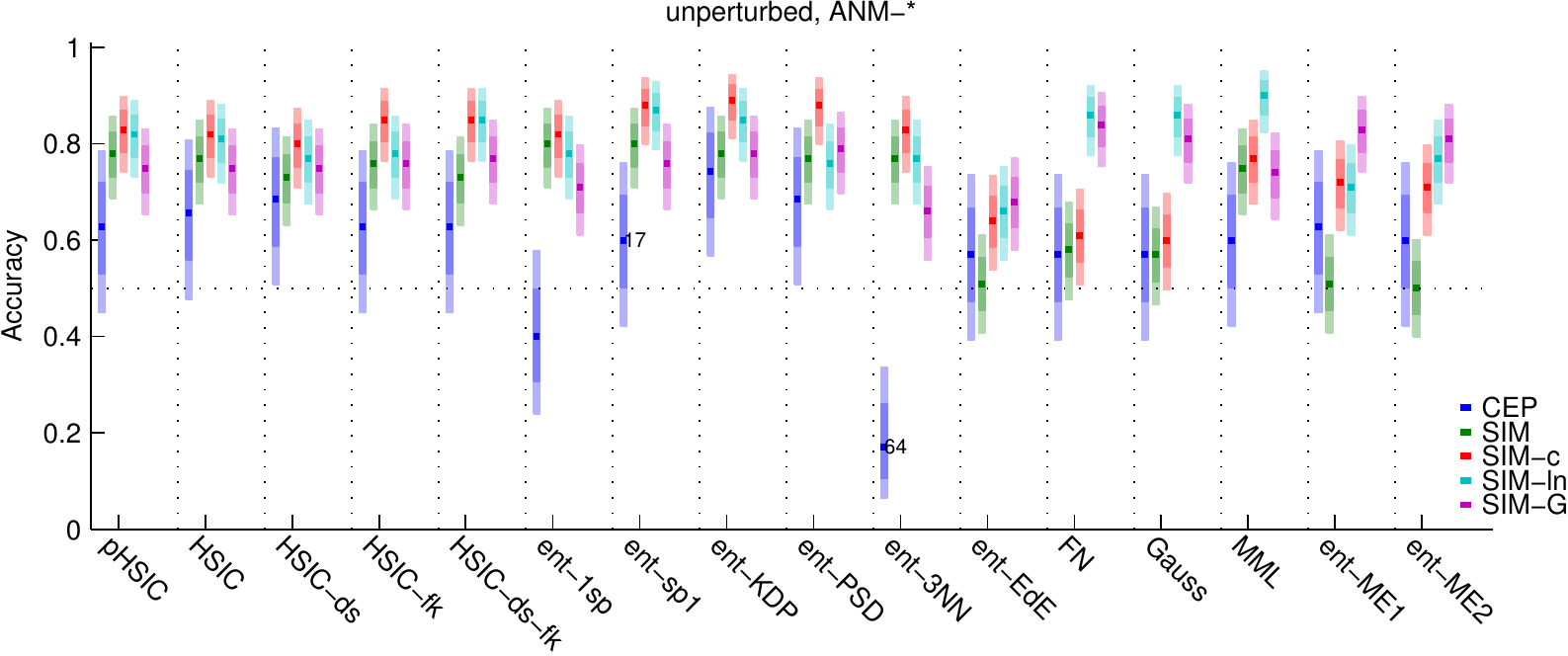}\\[\baselineskip]
\includegraphics[scale=0.9]{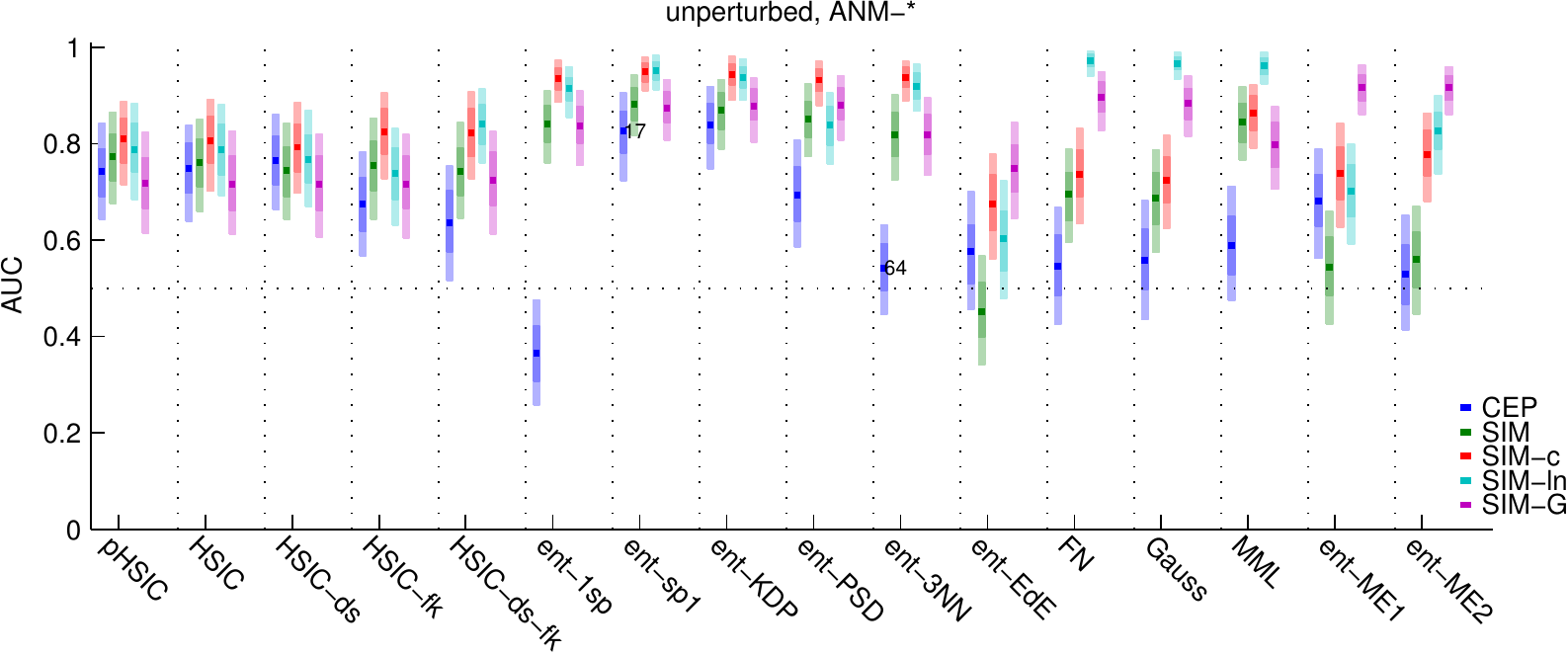}
\caption{\label{fig:acc_ANM_all_data}Accuracies (top) and AUCs (bottom) of various ANM methods on different (unperturbed) data sets. For the variants of the spacing estimator, only the results for \texttt{sp1} are shown, as results for \texttt{sp2},\dots,\texttt{sp6} were similar.}
\end{figure}

\begin{figure}[p]
\centering
\includegraphics[scale=0.9]{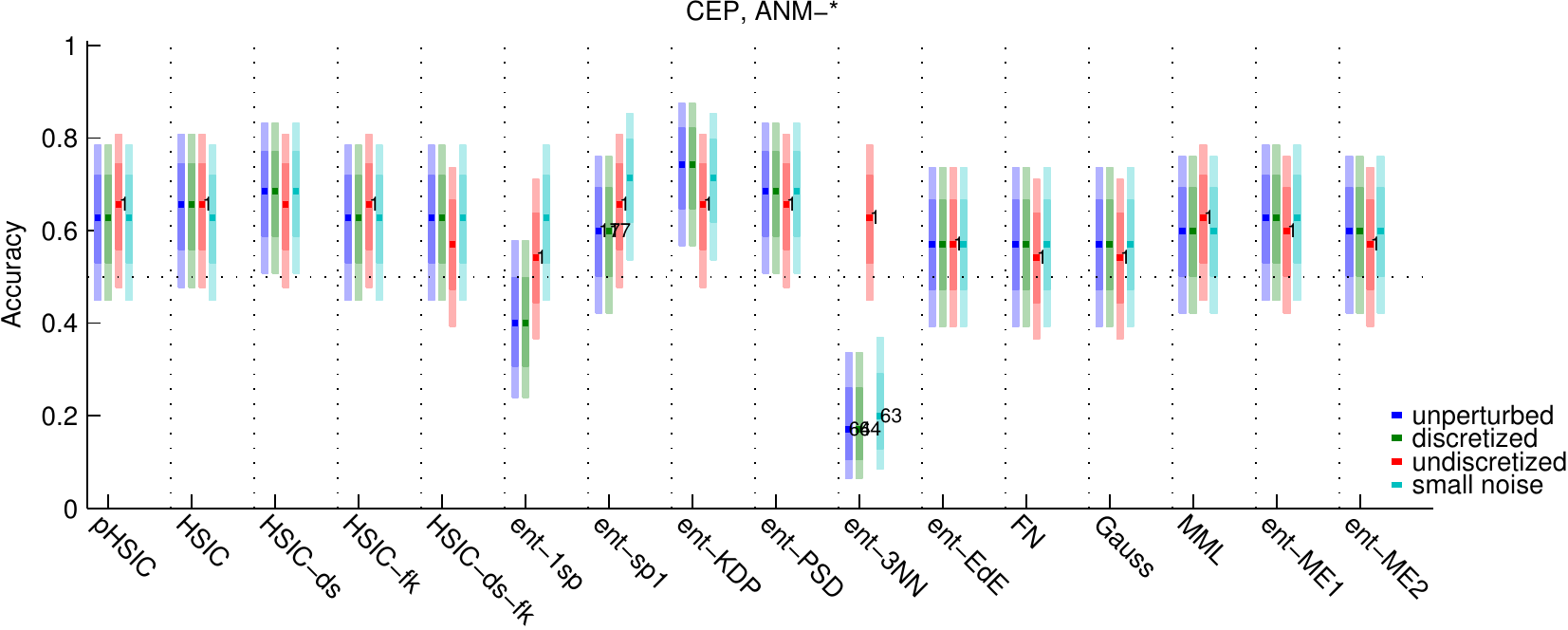}\\[\baselineskip]
\includegraphics[scale=0.9]{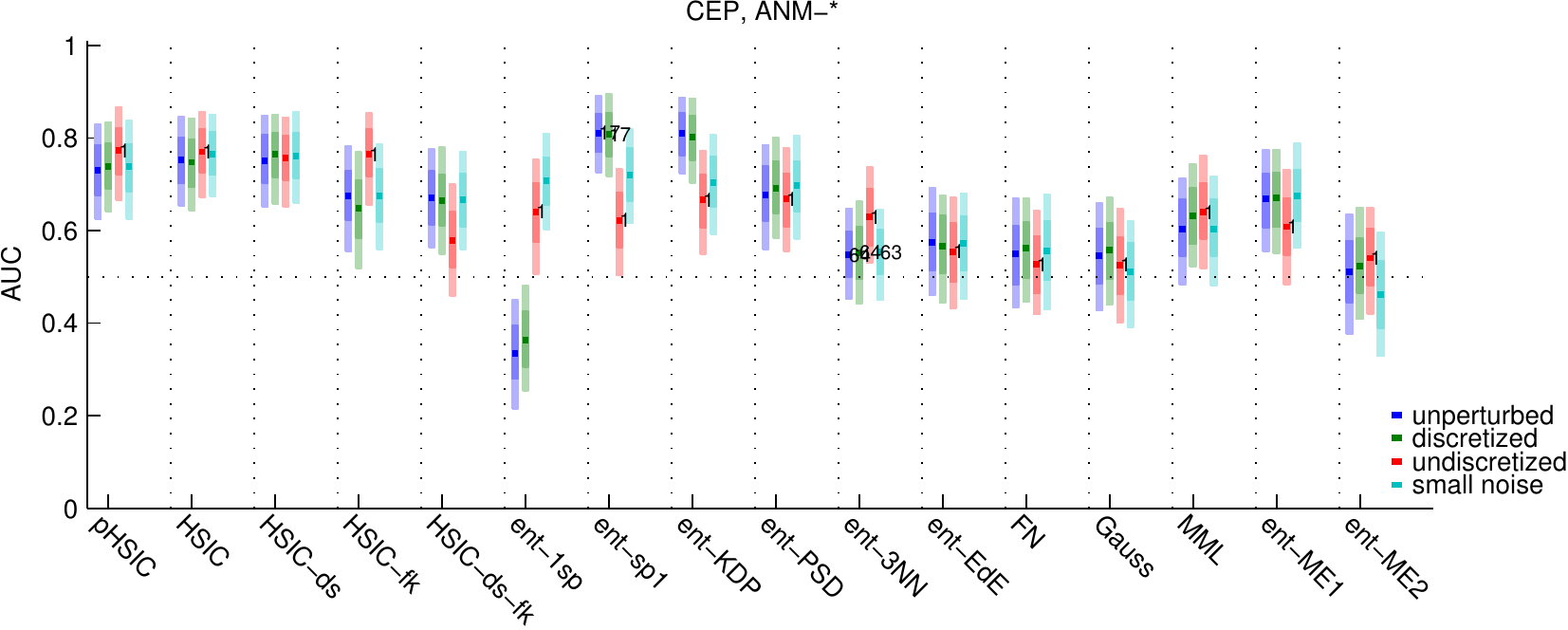}
\caption{\label{fig:acc_ANM_all_CEP}Accuracies (top) and AUCs (bottom) of various ANM methods on different perturbations of the \texttt{CEP} benchmark data. For the variants of the spacing estimator, only the results for \texttt{sp1} are shown, as results for \texttt{sp2},\dots,\texttt{sp6} were similar.}
\end{figure}

\subsubsection{HSIC-based scores}
As we see in Figure~\ref{fig:acc_ANM_all_data} and Figure~\ref{fig:acc_ANM_all_CEP}, 
the ANM methods that use HSIC perform reasonably well on all data sets, obtaining 
accuracies between 63\% and 85\%. Note that the simulated data (and also the
real-world data) deviate in at least three ways from the assumptions made by the additive
noise method: (i) the noise is not additive, (ii) a confounder can be present, 
and (iii) additional measurement noise was added to both cause and effect. 
Moreover, the results turn out to be robust
against small perturbations of the data. This shows that the additive noise
method can perform reasonably well, even in case of model misspecification.
%\footnote{Of course,
%this good performance does not generalize to \emph{all} distributions over cause-effect
%pairs. For example, if all pairs have a joint Gaussian distribution, the accuracy will
%be back to chance level, and if the pairs are simulated in such a way that they satisfy
%an additive noise model in the anti-causal direction, then the accuracy can be made 
%arbitrarily small.}

%The method obtains an accuracy of about 70\% on the \texttt{CauseEffectPairs}
%benchmark set, with a 95\% confidence interval that falls slightly above the $50\%$ accuracy level
%that corresponds to random guessing. 

The results of \texttt{ANM-pHSIC} and \texttt{ANM-HSIC} are very similar. The
influence of various implementation details on performance is small. On the
\texttt{CEP} benchmark, data-splitting (\texttt{ANM-HSIC-ds}) slightly increases 
accuracy, whereas using a fixed kernel (\texttt{ANM-HSIC-fk}, \texttt{ANM-HSIC-ds-fk}) 
slightly lowers AUC. Generally, the differences in performance are small and not statistically
significant. The variant \texttt{ANM-HSIC-ds-fk} is proved to be consistent in Appendix~\ref{sec:consistency}. 
If standard GP regression satisfies the property in \eref{eq:regression_suitable}, then \texttt{ANM-HSIC-fk} 
is also consistent.

\subsubsection{Entropy-based scores}

For the entropy-based score \eref{eq:ANM_score_entropy}, we see in Figure~\ref{fig:acc_ANM_all_data} and Figure~\ref{fig:acc_ANM_all_CEP}
that the results depend strongly on which entropy estimator is used. 

All (nonparametric) entropy estimators (\texttt{1sp}, \texttt{3NN},
\texttt{sp}{\it i}, \texttt{KDP}, \texttt{PSD}) perform well on simulated data,
with the exception of \texttt{EdE}. On the \texttt{CEP} benchmark on
the other hand, the performance varies greatly over estimators. One of the reasons
for this are discretization effects. Indeed, the differential entropy of a variable
that can take only a finite number of values is $-\infty$. The way in which
differential entropy estimators treat values that occur multiple times differs,
and this can have a large influence on the estimated entropy. For example,
\texttt{1sp} simply ignores values that occur more than once, which
leads to a performance that is below chance level on the CEP data. \texttt{3NN}
returns $-\infty$ (for both $\hat C_{X\to Y}$ and $\hat C_{Y\to X}$) in the majority of the pairs in the \texttt{CEP} benchmark
and therefore often cannot decide. The spacing estimators \texttt{sp}{\it i} also return $-\infty$ in quite a few cases. 
The only (nonparametric) entropy-based ANM methods that perform well on both the
\texttt{CEP} benchmark data and the simulated data are \texttt{ANM-ent-KDP} and
\texttt{ANM-ent-PSD}. Of these two methods, \texttt{ANM-ent-PSD} seems more
robust under perturbations than \texttt{ANM-ent-KDP}, and can compete with the
HSIC-based methods.

\subsubsection{Other scores}

Consider now the results for the ``parametric'' entropy estimators 
(\texttt{ANM-Gauss}, \texttt{ANM-ent-ME1}, \texttt{ANM-ent-ME2}), the empirical-Bayes method 
\texttt{ANM-FN}, and the MML method \texttt{ANM-MML}.

First, note that \texttt{ANM-Gauss} and \texttt{ANM-FN} perform very similarly. This means that the difference 
between these two scores (i.e., the complexity measure of the regression function, see also Appendix~\ref{sec:gp}) does not outweigh the common 
part (the likelihood) of these two scores. Both these scores do not perform much better than chance on the \texttt{CEP}
data, probably because the Gaussianity assumption is typically violated in real data. They do obtain 
high accuracies and AUCs for the \texttt{SIM-ln} and \texttt{SIM-G} scenarios. For \texttt{SIM-G} this is to be expected,
as the assumption that the cause has a Gaussian distribution is satisfied in that scenario. For
\texttt{SIM-ln} it is not evident why these scores perform so well---it could be that the noise is close to
additive and Gaussian in that scenario. 

The related score \texttt{ANM-MML}, which employs a more sophisticated
complexity measure for the distribution of the cause, performs better on 
the two simulation settings \texttt{SIM} and \texttt{SIM-c}. However,
\texttt{ANM-MML} performs worse in the \texttt{SIM-G} scenario, which is
probably due to a higher variance of the MML complexity measure compared with
the simple Gaussian entropy measure. This is in line with expectations.
However, performance of \texttt{ANM-MML} is hardly better than chance on
the \texttt{CEP} data. In particular, the AUC of \texttt{ANM-MML} is 
worse than that of \texttt{ANM-pHSIC}.

The parametric entropy estimators \texttt{ME1} and \texttt{ME2} do
not perform very well on the \texttt{SIM} data, although their
performance on the other simulated data sets (in particular \texttt{SIM-G}) is
good. The reasons for this behaviour are not understood; we speculate that the
parametric assumptions made by these estimators match the actual distribution
of the data in these particular simulation settings quite well.
The accuracy and AUC of \texttt{ANM-ent-ME1} and \texttt{ANM-ent-ME2} on the 
\texttt{CEP} data are lower than those of \texttt{ANM-pHSIC}.

\subsection{Information Geometric Causal Inference}

Here we report the results of the evaluation of different IGCI variants. Figure~\ref{fig:acc_IGCI_uniform_all_data} shows the performance of all the IGCI variants on different
(unperturbed) data sets, the \texttt{CEP} benchmark and four different simulation settings, using the uniform base measure. Figure~\ref{fig:acc_IGCI_Gaussian_all_data} shows the same for the Gaussian base measure. Figure~\ref{fig:acc_IGCI_uniform_all_CEP} shows the performance of the
IGCI methods on different perturbations of the \texttt{CEP} benchmark, using the uniform base measure, and Figure~\ref{fig:acc_IGCI_Gaussian_all_CEP} for the Gaussian base measure.
Again, the six variants \texttt{sp1},\dots,\texttt{sp6} of the spacing estimators perform very similarly, so we 
show only the results for \texttt{IGCI-ent-sp1}. 

\begin{figure}
\centering
\includegraphics[scale=0.9]{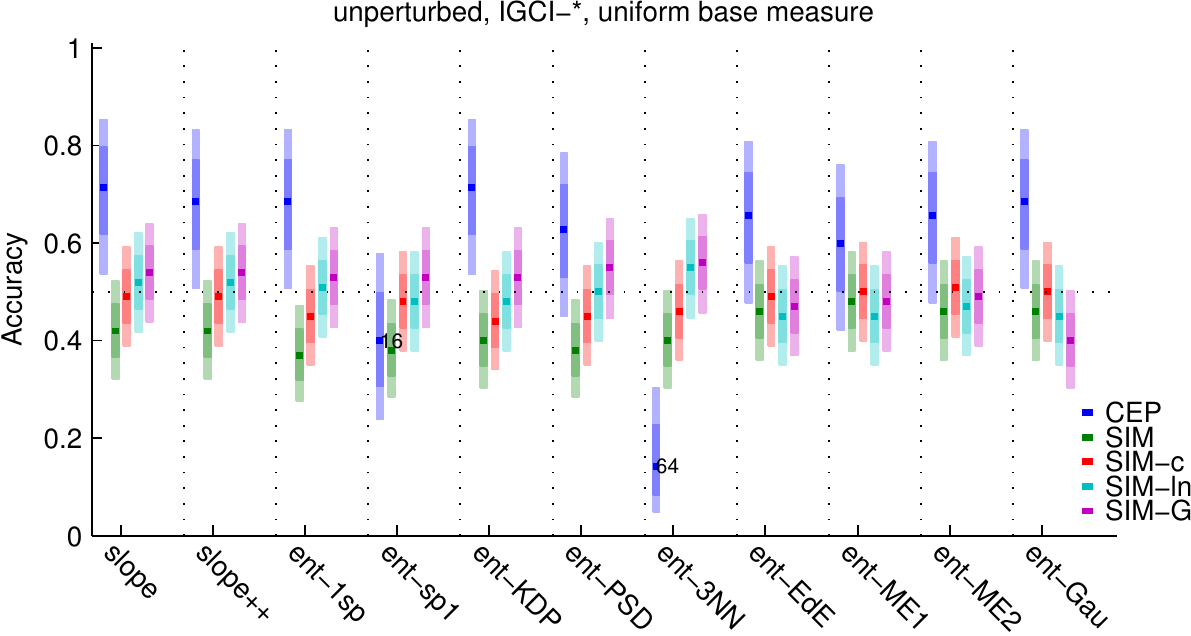}\\[\baselineskip]
\includegraphics[scale=0.9]{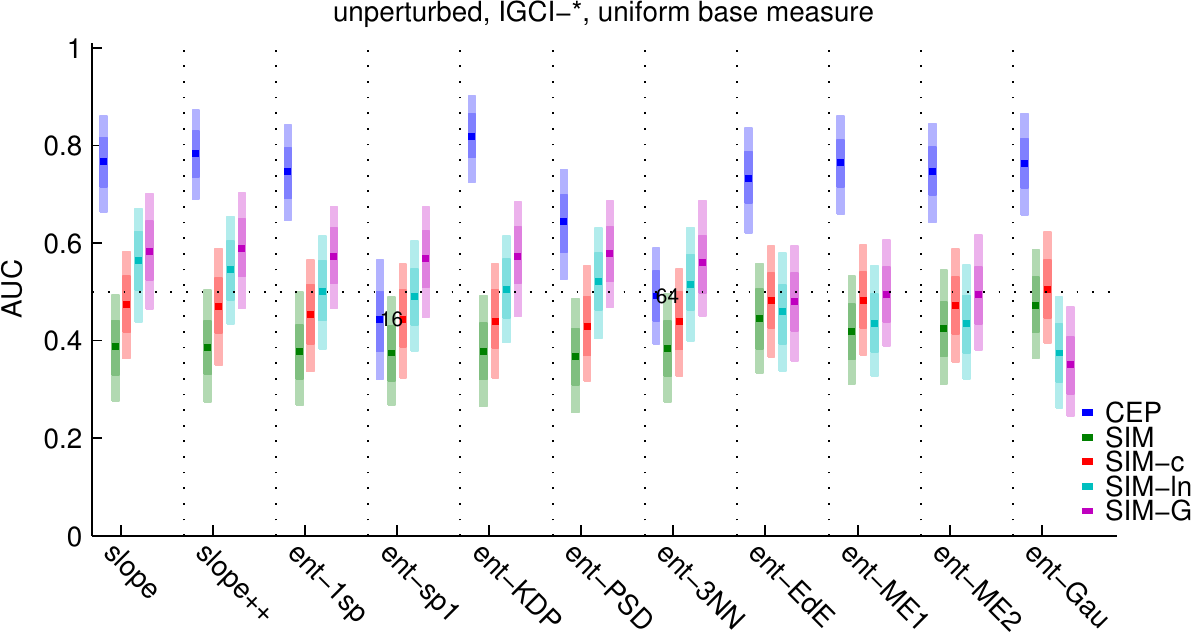}
\caption{\label{fig:acc_IGCI_uniform_all_data}Accuracies (top) and AUCs (bottom) for various IGCI methods using the uniform base measure on different (unperturbed) data sets.}
\end{figure}

\begin{figure}
\centering
\includegraphics[scale=0.9]{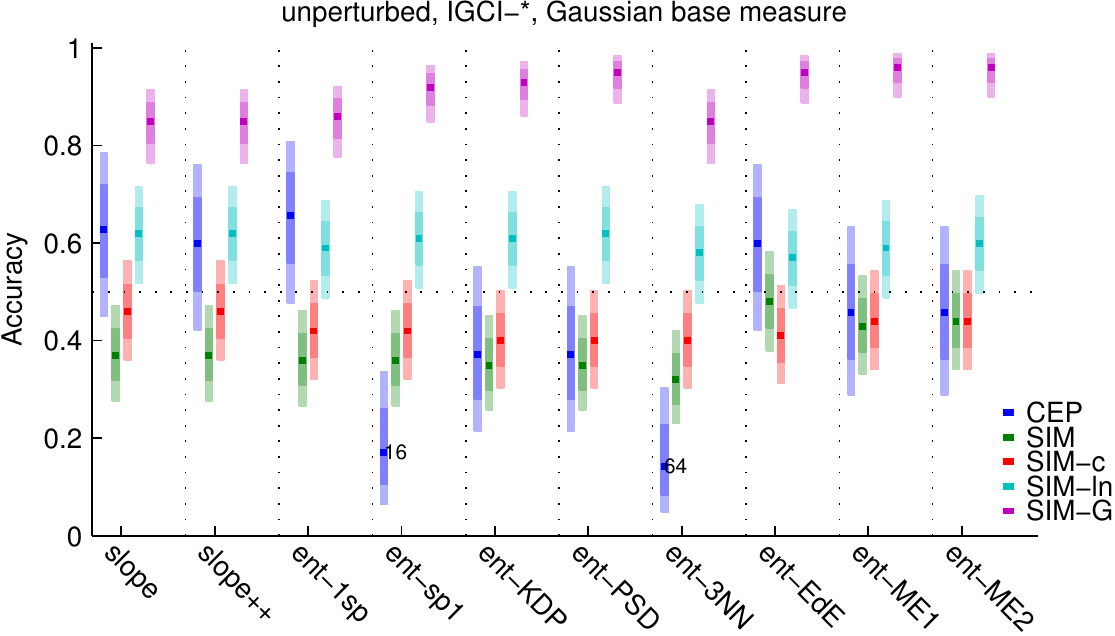}\\[\baselineskip]
\includegraphics[scale=0.9]{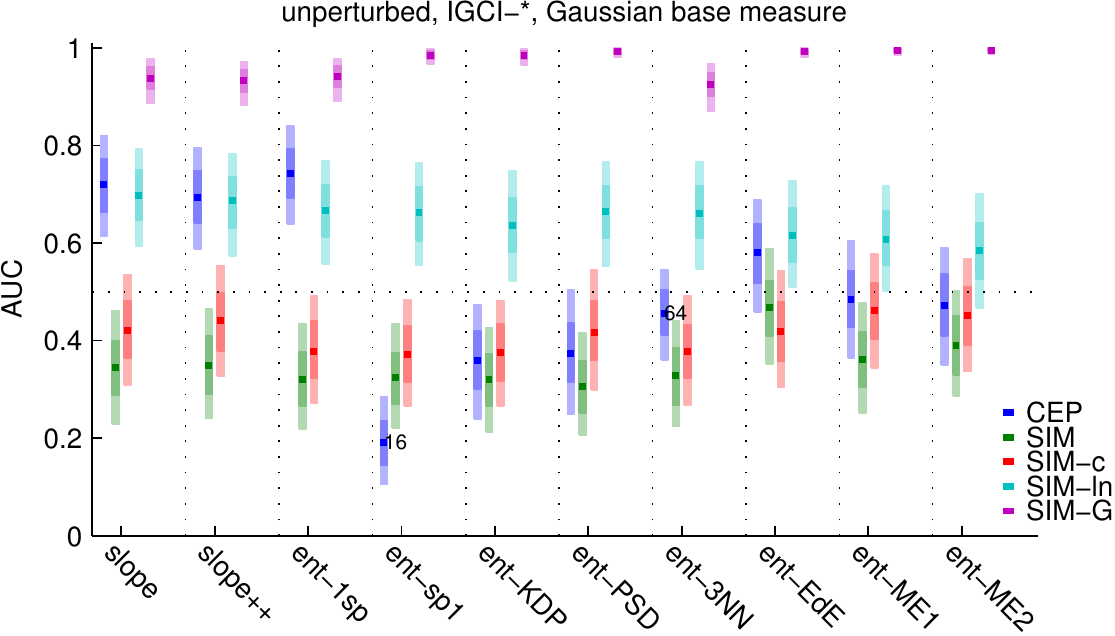}
\caption{\label{fig:acc_IGCI_Gaussian_all_data}Accuracies (top) and AUCs (bottom)  for various IGCI methods using the Gaussian base measure on different (unperturbed) data sets.}
\end{figure}

\begin{figure}
\centering
\includegraphics[scale=0.9]{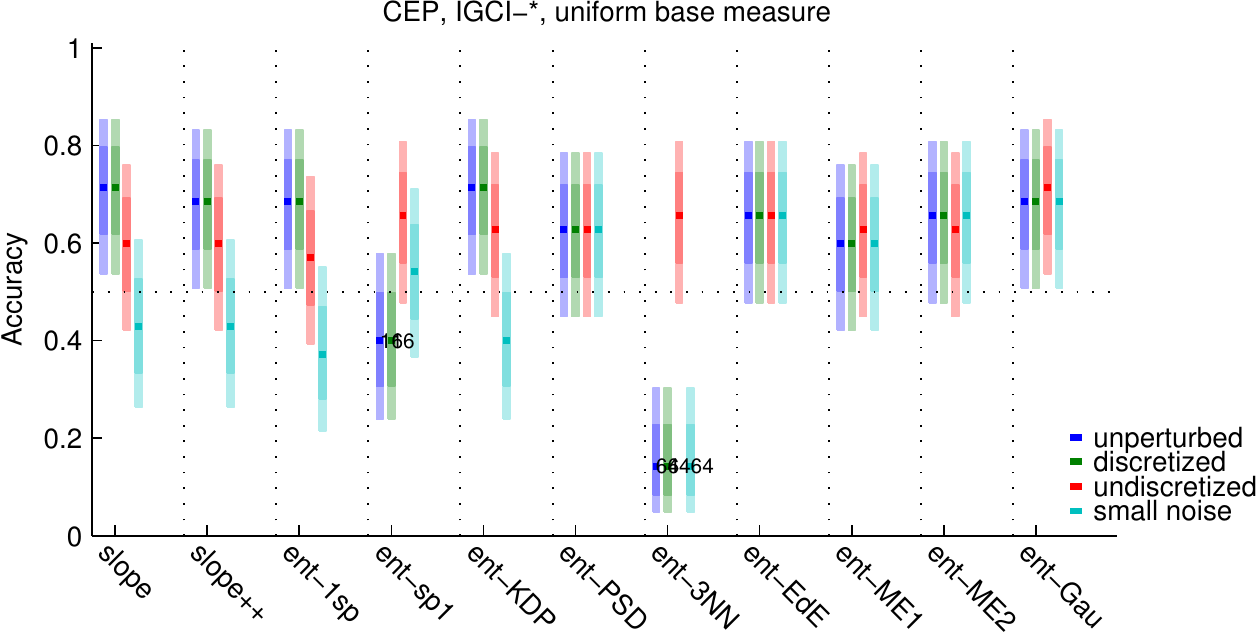}\\[\baselineskip]
\includegraphics[scale=0.9]{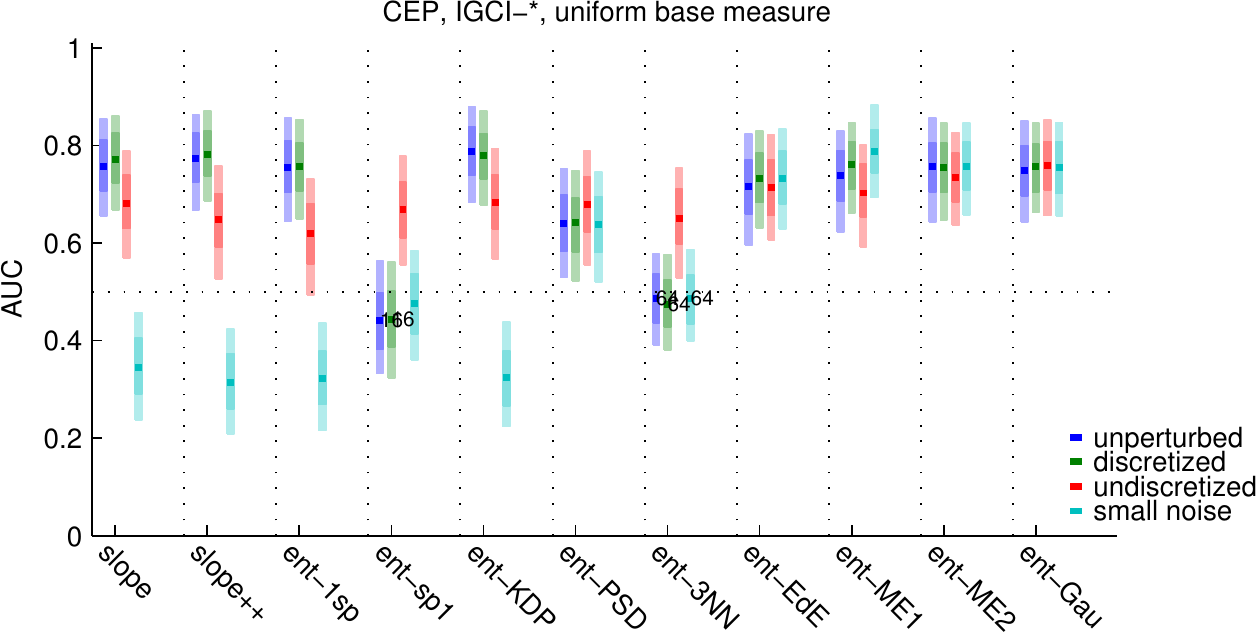}
\caption{\label{fig:acc_IGCI_uniform_all_CEP}Accuracies (top) and AUCs (bottom)  for various IGCI methods using the uniform base measure on different perturbations of the \texttt{CEP} benchmark data.}
\end{figure}

\begin{figure}
\centering
\includegraphics[scale=0.9]{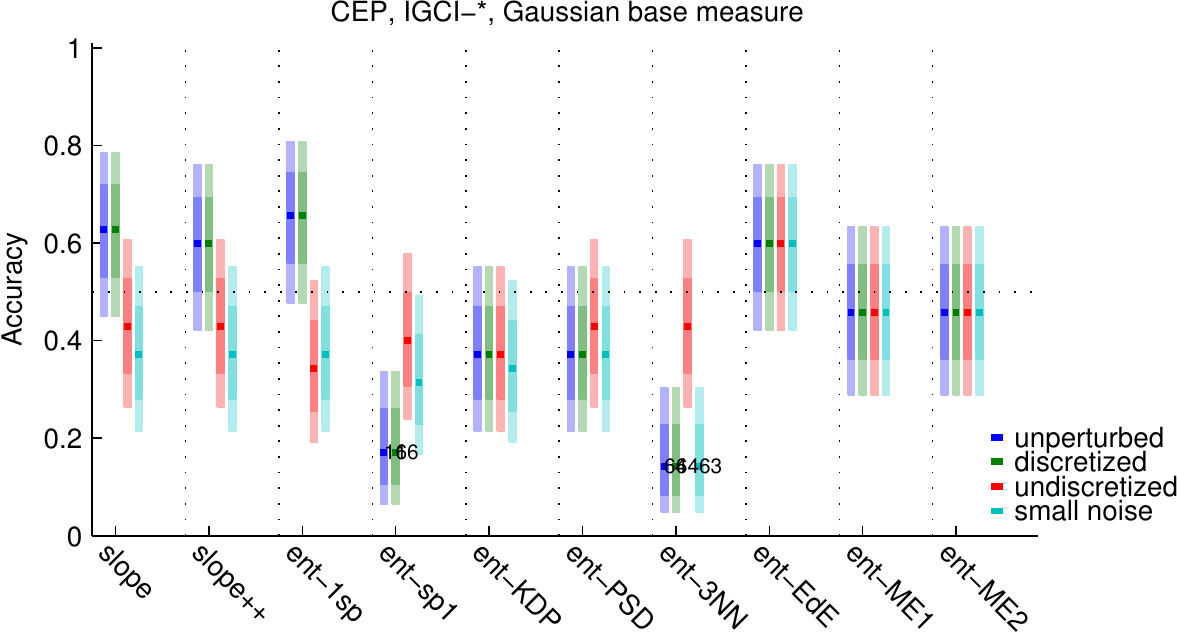}\\[\baselineskip]
\includegraphics[scale=0.9]{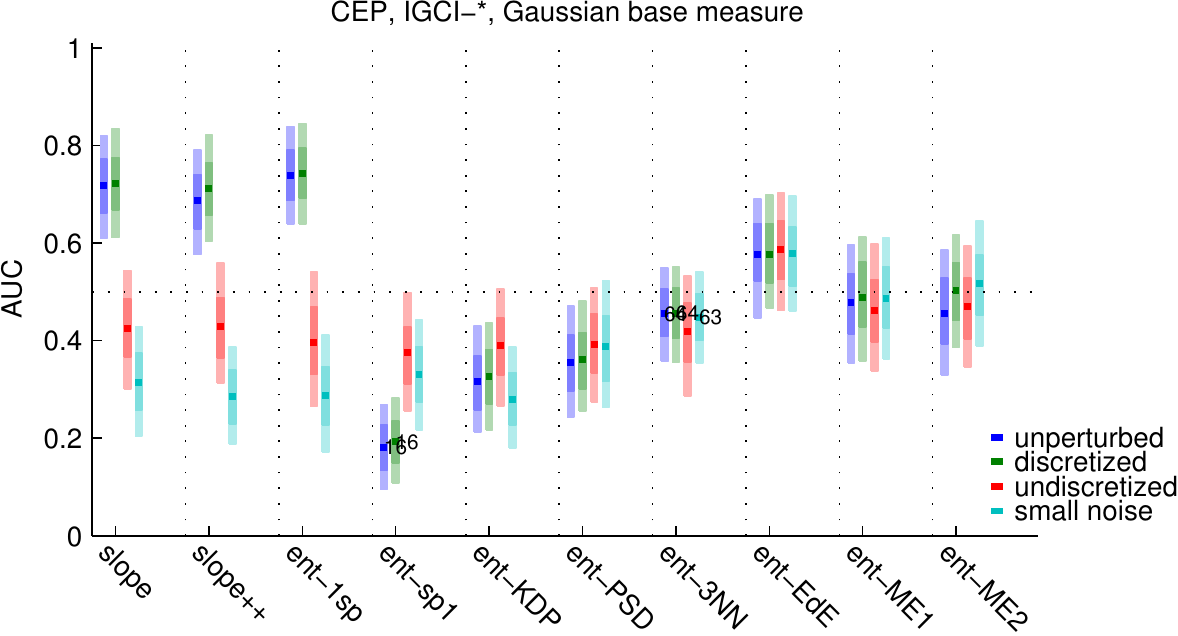}
\caption{\label{fig:acc_IGCI_Gaussian_all_CEP}Accuracies (top) and AUCs (bottom)  for various IGCI methods using the Gaussian base measure on different perturbations of the \texttt{CEP} benchmark data.}
\end{figure}

Let us first look at the performance on simulated data.  Note that none of the
IGCI methods performs well on the simulated data when using the uniform base
measure.  A very different picture emerges when using the Gaussian base
measure: here the performance covers a wide spectrum, from lower than chance
level on the \texttt{SIM} data to accuracies higher than 90\% on
\texttt{SIM-G}. The choice of the base measure clearly has a larger influence
on the performance than the choice of the estimation method.

As IGCI was designed for the bijective deterministic case, one would expect
that IGCI would work best on \texttt{SIM-ln} (without depending too strongly on
the reference measure), because in that scenario the noise is relatively small.
Surprisingly, this does not turn out to be the case.  To understand this
unexpected behavior, we inspect the scatter plots in
Figure~\ref{fig:allpairs_simlownoise}  and observe that the functions in
\texttt{SIM-ln} are either non-injective or relatively close to linear. Both
can spoil the performance despite having low noise (see also the remarks at the
end of Subsection~\ref{subsec:igciimpl} on finite sample effects).

For the more noisy settings, earlier experiments showed that
\texttt{IGCI-slope} and \texttt{IGCI-1sp} can perform surprisingly well on
simulated data \citep{Janzing_et_al_AI_12}. Here, however, we see that the
performance of all IGCI variants on noisy data depends strongly on
characteristics of the data generation process and on the chosen base measure.
IGCI seems to pick up certain features in the data that turn out to be
correlated with the causal direction in some settings, but can be
anticorrelated with the causal direction in other settings. In addition, our
results suggest that if the distribution of the cause is close to the base
measure used in IGCI, then also for noisy data the method may work well (as
in the \texttt{SIM-G} setting).  However, for causal relations that are
not sufficiently non-linear, performance can drop significantly (even below
chance level) in case of a discrepancy between the actual distribution of the
cause and the base measure assumed by IGCI. 

Even though the performance of all IGCI variants with uniform base measure is
close to chance level on the simulated data, most methods perform better than
chance on the \texttt{CEP} data (with the exception of \texttt{IGCI-ent-sp1}
and \texttt{IGCI-ent-3NN}). When using the Gaussian base measure, performance
of IGCI methods on \texttt{CEP} data varies considerably depending on
implementation details. For some IGCI variants the performance on \texttt{CEP}
data is robust to small perturbations (most notably the parametric entropy
estimators), but for most non-parametric entropy estimators and for \texttt{IGCI-slope},
there is a strong dependence and sometimes even an inversion of the accuracy when perturbing
the data slightly. We do not have a good explanation for these observations. 

\subsubsection{Original implementations}

%Figure~\ref{fig:acc_IGCI_org} shows
Let us now take a closer look at the accuracies of the original methods
\texttt{IGCI-slope} and \texttt{IGCI-ent-1sp} that were proposed in
\citep{Daniusis_et_al_UAI_10,Janzing_et_al_AI_12}, and at the newly introduced
\texttt{IGCI-slope++} that is closely related to \texttt{IGCI-slope}.  The IGCI
variants \texttt{slope}, \texttt{slope++} and \texttt{ent-1sp}
perform very similar on all data sets.  For both uniform and Gaussian base
measures, the performance is better than chance level on the \texttt{CEP}
benchmark, but not as much as in previous evaluations on earlier versions of
the benchmark. The discrepancy with the accuracies of around 80\% reported by
\citet{Janzing_et_al_AI_12} could be explained by the fact that here we
evaluate on a larger set of cause-effect pairs, and we chose the weights more
conservatively.

It is also interesting to look at the behavior under perturbations of the
\texttt{CEP} data. When using the uniform base measure, the accuracy of both
\texttt{IGCI-slope} and \texttt{IGCI-ent-1sp} drops back to chance level if
small noise is added, whereas AUC even becomes worse than chance level. For the
Gaussian base measure, both accuracy and AUC become worse than random guessing
on certain perturbations of the \texttt{CEP} data, although discretization does
not affect performance. This observation motivated the introduction of the
slope-based estimator \texttt{IGCI-slope++} that uses \eref{eq:IGCI_score_slope++} 
instead of \eref{eq:IGCI_score_slope} in order to deal better with repetitions of
values. However, as we can see, this estimator does not perform better in practice
than the original estimator \texttt{IGCI-slope}.

\subsubsection{Nonparametric entropy estimators}

It is clear that discretization effects play an important role in the
performance of the nonparametric entropy estimators. For example, the closely
related estimators \texttt{1sp} and \texttt{sp}{\it i}
perform comparably on simulated data, but on the \texttt{CEP} data, the
\texttt{sp}{\it i} estimators perform worse because of nondecisions
due to repeated values. Similarly, the bad performance of \texttt{IGCI-ent-3NN}
on the \texttt{CEP} data is related to discretization effects. This is in line
with our observations on the behavior of these entropy estimators when using
them for entropy-based ANM methods.

Further, note that the performance of \texttt{IGCI-ent-KDP} is qualitatively
similar to that of \texttt{IGCI-ent-PSD}, but in contrast with the \texttt{PSD}
estimator, the results of the \texttt{KDP} estimator are not robust under
perturbations when using the uniform base measure. The only
nonparametric entropy estimators that give results that are robust to small
perturbations of the data (for both base measures) are
\texttt{PSD} and \texttt{EdE}.  The performance of \texttt{IGCI-ent-PSD} on the
\texttt{CEP} benchmark depends on the chosen base measure: for the uniform base
it is better than chance level, for the Gaussian base measure it is worse than
chance level.  Interestingly, the \texttt{EdE} estimator that performed poorly
for ANM gives consistently good results on the \texttt{CEP} benchmark when used
for IGCI: it is the only nonparametric entropy estimator that yields results
that are better than chance for both base measures and irrespectively of
whether the data were perturbed or not.

Apparently, implementation details of entropy estimators can result in huge
differences in performance, often in ways that we do not understand well. 

\subsubsection{Parametric entropy estimators}

Let us finally consider the performance of entropy-based IGCI methods
that use parametric entropy estimators, which make additional assumptions on the
distribution. As expected, these estimators are robust to small perturbations
of the data. 

Interestingly, \texttt{IGCI-ent-Gau} with uniform base measure turns out to be one 
of the best IGCI methods on the \texttt{CEP} benchmark, in the sense that
it obtains good accuracy and AUC and in addition is robust to perturbations. 
Note that the performance of 
\texttt{IGCI-ent-Gau} on the \texttt{CEP} benchmark is comparable with that 
of the original implementation \texttt{IGCI-slope} and the newer version 
\texttt{IGCI-slope++}, but that only \texttt{IGCI-ent-Gau} is robust to small
perturbations of the data.
This estimator simply estimates entropy by assuming a Gaussian distribution. 
In other words, it uses:
$$\hat C_{X\to Y} := \frac{1}{2} \log \widehat{\Var}(\tilde{\B{y}}) - \frac{1}{2} \log \widehat{\Var}(\tilde{\B{x}}) = \log \left( \frac{\sqrt{\widehat{\Var(\B{y})}}}{\max(\B{y})-\min(\B{y})} \Bigg/ \frac{\sqrt{\widehat{\Var(\B{x})}}}{\max(\B{x})-\min(\B{x})} \right).$$
Apparently, the ratio of the size of the support of the distribution and its 
standard deviation is already quite informative on the causal
direction for the \texttt{CEP} data. This might also explain the relatively
good performance on this benchmark of \texttt{IGCI-ent-ME1} and
\texttt{IGCI-ent-ME2} when using the uniform base measure, as these estimate
entropy by fitting a parametric distribution to the data (which includes
Gaussian distributions as a special case). On the other hand, these methods
do not work better than chance on the simulated data. 

Now let us look at the results when using the Gaussian base measure.
\texttt{IGCI-ent-Gau} makes no sense in combination with the Gaussian base
measure. \texttt{IGCI-ent-ME1} and \texttt{IGCI-ent-ME2} on the \texttt{CEP} data do not
perform better than chance.  On simulated data, they only do (extremely) well for the
\texttt{SIM-G} scenario which uses a Gaussian distribution of the cause
measure, i.e., for which the chosen base measure exactly corresponds with the
distribution of the cause.

%\subsubsection{Best methods}

%Figure~\ref{fig:acc_IGCI_best} shows the accuracies of the IGCI methods that
%performed best on the \texttt{CEP} benchmark data, all using the uniform base
%measure, and how the accuracies depend on small perturbations of this data. The
%good accuracy of \texttt{IGCI-ent-KDP} is not stable against such perturbations, but for
%the other estimators, it is. 
%Because of the large number of methods that we have 
%compared, we cannot exclude the possibility that we are looking at a statistical
%artefact. 

\section{Discussion and Conclusion}\label{sec:discussion} %%%%%%%%%%%%%%%%%%%%%%%%%%%%%%%%%%%%%%%%%%%%%%%%%%%%%%%%%%%%%%%%%%%%%%%%%%%%%%%%%%%%%%%%%%%%%%%%%%%%%%%%%%%%%%%%%%%%%%%%%%%%%%%%%%%%%%%%%%%%%

\begin{figure}
  \includegraphics[width=0.48\textwidth]{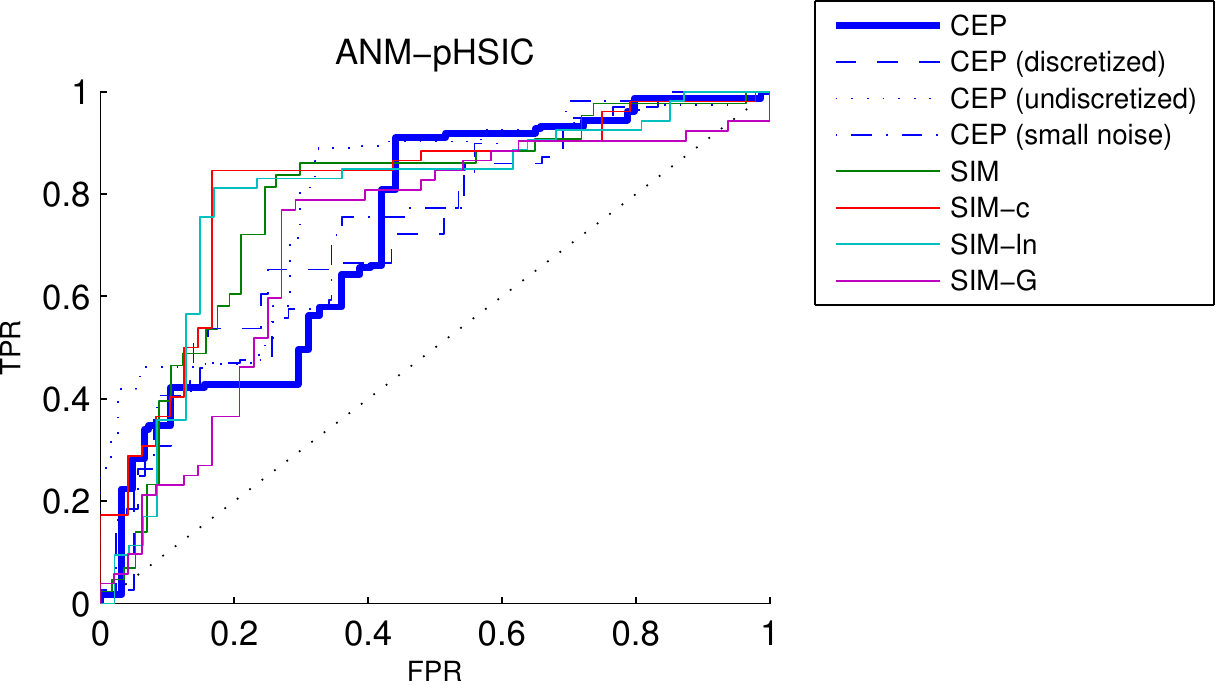}\hfill
  \includegraphics[width=0.48\textwidth]{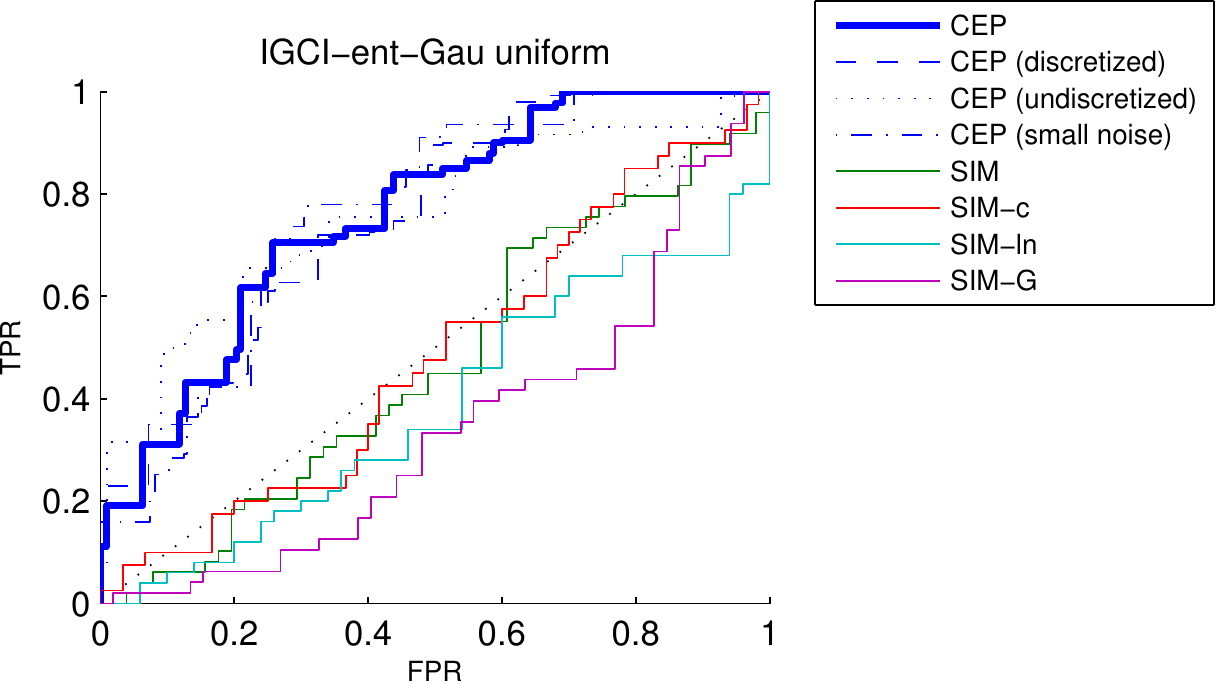}
  \caption{\label{fig:conclusion}ROC curves for two of the best-performing methods (\texttt{ANM-pHSIC} and \texttt{IGCI-ent-Gau}).
  Both methods work well on the \texttt{CEP} benchmark and keep performing well under small perturbations of the data, 
  but only \texttt{ANM-pHSIC} also performs well on the simulated data.}
\end{figure}

In this work, we considered a challenging bivariate causal discovery problem,
where the task is to decide whether $X$ causes $Y$ or vice versa, using
only a sample of purely observational data. We reviewed two families of methods
that can be applied to this task: Additive Noise Methods (ANM) and Information
Geometric Causal Inference (IGCI) methods. We discussed various possible
implementations of these methods and how they are related.

In addition, we have proposed the \texttt{CauseEffectPairs} benchmark data set consisting of
\nrpairs\ real-world cause-effect pairs and we provided our justifications for
the ground truths. We have used this benchmark data in combination with several simulated data sets in order to
evaluate various bivariate causal discovery methods. Our main conclusions (illustrated in Figure~\ref{fig:conclusion}) are twofold:
\begin{enumerate}
  \item The ANM methods that use HSIC perform reasonably well on all data sets (including the
    perturbed versions of the \texttt{CEP} benchmark and all simulation settings), obtaining
    accuracies between 63\% and 85\% (see Figures \ref{fig:acc_ANM_all_data} and \ref{fig:acc_ANM_all_CEP}).
    In particular, the original \texttt{ANM-pHSIC} method
    obtains an accuracy of 63 $\pm$ 10 \% and an AUC of 0.74 $\pm$ 0.05 on the \texttt{CEP} benchmark.
    The only other ANM method that performs well on all data sets is \texttt{ANM-ent-PSD}. It obtains
    a higher accuracy (69 $\pm$ 10\%) than \texttt{ANM-pHSIC} on the \texttt{CEP} benchmark,
    but a lower AUC (0.68 $\pm$ 0.06), but these differences are not statistically significant.
  \item The performance of IGCI-based methods varies greatly depending on implementation details, perturbations of the data and certain characteristics of the data, in ways that we do not fully understand (see Figures \ref{fig:acc_IGCI_uniform_all_data}, \ref{fig:acc_IGCI_Gaussian_all_data}, \ref{fig:acc_IGCI_uniform_all_CEP}, \ref{fig:acc_IGCI_Gaussian_all_CEP}). In many cases, causal relations seem to be too linear for IGCI to work well.  
    None of the IGCI implementations performed well on \emph{all} data sets that we considered,
    and the apparent better-than-chance performance of some of these methods on the \texttt{CEP} benchmark remains somewhat of a mystery.
\end{enumerate}
The former conclusion about the performance of \texttt{ANM-pHSIC} is in line with earlier reports, but the latter conclusion is surprising, 
considering that good performance of \texttt{IGCI-slope} and \texttt{IGCI-ent-1sp} has been reported on several occassions in earlier work
\citep{Daniusis_et_al_UAI_10,Mooij_et_al_NIPS_10,Janzing_et_al_AI_12,Statnikov++2012,Sgouritsa++2015}. 
%We decided to highlight here one method from the ANM family, and one from the IGCI 
%family, that work well on the \texttt{CEP} benchmark and are robust against preprocessing. Of these methods, only \texttt{ANM-pHSIC} 
%also performs well on the simulated data.
%The discrepancies
%in the performance with earlier studies that used a subset of the \CEP\ data are due to the new weighting 
%scheme and the larger number of pairs used here.
One possible explanation that the performance of IGCI on simulated data here differs from earlier
reports is that earlier simulations used considerably smoother distributions of the cause variable.

Ironically, the original ANM method \texttt{ANM-pHSIC} proposed by
\citet{HoyerJanzingMooijPetersSchoelkopf_NIPS_08} turned out to be one of the
best methods overall, despite the recent research efforts aimed at
developing better methods. This observation motivated the consistency proof of
HSIC-based ANM methods, the major theoretical contribution of this work. We
expect that extending this consistency result to the multivariate case
\citep[see also][]{PetersMooijJanzingSchoelkopf_JMLR_14} should be
straightforward.

One reason for the disappointing performance of several methods (in particular, the slope-based
IGCI estimators and methods that make use of certain nonparametric differential entropy estimators) 
is discretization. When dealing with real-world
data on a computer, variables of a continuous nature are usually
discretized because they have to be represented as floating point numbers.
Often, additional rounding is applied, for example because only the most
significant digits are recorded. We found that for many methods, especially for
those that use differential entropy estimators, (coarse) discretization of the
data causes problems. This suggests that performance of several methods can
still be improved, e.g., by using entropy estimators that are more robust to
such discretization effects. The HSIC independence measure (and its $p$-value)
and the \texttt{PSD} entropy estimator were found to be robust against small
perturbations of the data, including discretization.

Since we compared many different implementations (which turned out to
have quite different performance characteristics), we need to use a strong correction for
multiple testing if we would like to conclude that one of these methods
performs significantly better than chance. Although it seems unlikely that the
good performance of \texttt{ANM-pHSIC} on both \texttt{CEP} data \emph{and} all
simulated data is due to chance alone, eventually we are most interested in
the performance on real-world data alone. Unfortunately, the \texttt{CEP} benchmark
turned out to be too small to warrant significant conclusions for any of the
tested methods.

A rough estimate how large the \CEP\ benchmark should have been in order to
obtain significant results can easily be made. Using a standard (conservative)
Bonferroni correction, taking into account that we compared 37 methods, we
would need about 120 (weighted) pairs for an accuracy of 65\% to be considered
significant (with two-sided testing and 5\% significance threshold). This is
about four times as much as the current number of \nrdatasets\ (weighted) pairs in the
\CEP\ benchmark. Therefore, we suggest that 
at this point, the highest priority regarding future work should be to obtain
more validation data, rather than developing additional methods or optimizing
computation time of existing methods. 
We hope that our publication of the \CEP\
benchmark data inspires researchers to collaborate on this important task
and we invite everybody to contribute pairs to the \CEP\ benchmark data.

%Whether the performance of a method is \emph{significantly} better than random
%guessing is not so clear-cut. 
%If one does not correct for multiple testing,
%there are several methods for which this seems to be the case. 
%For example,
%the $p$-value for \texttt{ANM-pHSIC} under the null hypothesis of random guessing
%is $0.068$ \todo{update} on the \texttt{CEP} benchmark, 
%only slightly larger than the popular $\alpha = 5 \%$ threshold. However, because we
%tested many different methods, one should correct for multiple testing. Using a
%conservative Bonferroni correction, the results on a single data set would be far from
%being significant. In other words, from this study we cannot conclude that the 
%\texttt{ANM-pHSIC} method significantly outperforms random guessing on the \CEP\ 
%benchmark. However, the Bonferroni correction would be overly conservative for our 
%purposes, as many of the methods that we compare are small variations of each other, 
%and their results are clearly dependent. In addition, good performance across data
%sets increases the significance of the results. Although a proper quantitative
%evaluation of the significance of our results is a nontrivial exercise, we 
%believe that when combining the results on the \texttt{CEP} benchmark with those on
%the simulated data, the hypothesis that the good performance of the methods 
%\texttt{ANM-pHSIC}, \texttt{ANM-HSIC}, \texttt{ANM-PSD} and \texttt{ANM-MML} 
%is only due to chance is implausible.
% ANM-pHSIC on CEP: 19 correct out of 29, p-value 0.0307

Concluding, our results provide some evidence that distinguishing cause from
effect is indeed possible from purely observational real-world data by
exploiting certain statistical patterns in the data. However, the performance of current
state-of-the-art bivariate causal discovery methods still has to be improved further in
order to enable practical applications, and more validation data are needed in
order to obtain statistically significant conclusions.  Furthermore, it is not
clear at this stage under what assumptions current methods could be extended to
deal with possible confounding variables, an important issue in practice.

%We did not consider the more general PNL models studied by \citet{ZhangHyvarinen2009}
%as they lead to more challenging statistical estimation problems \todo{Shall we include an appendix
%containing results in that direction?}, and we believe that for many practical
%purposes, additive noise models offer a suitable trade-off between model complexity and generality.

%Their apparent better-than-chance performance is
%somewhat of a mystery, and it should be kept in mind that the performance of all these
%methods on simulated data was not more than chance level. 
%This should be compared with the best additive-noise based methods,
%whose performance is good on both the real-world \texttt{CEP} data and the simulated data.

%We would have expected IGCI to work best on the \texttt{SIM-ln} data, as these pairs have relatively low noise levels. However,
%we see that the original IGCI estimators barely outperforms random guessing in that scenario.

%IGCI: The highest variance in performance across estimators is observed for the \texttt{CEP} benchmark, which
%is probably due to discretization effects.

\clearpage
\appendix

\section{Consistency Proof of \texttt{ANM-HSIC}}\label{sec:consistency}

%\todo{\begin{itemize}
%  \item All random variables in capitals consistently? Now, data r.v.'s in the appendix are in small letters.
%\end{itemize}}
In this Appendix, we prove the consistency of Algorithm~\ref{alg:bivariateANM} with score \eref{eq:ANM_score_HSIC_fixed}, which is closely related to the algorithm originally proposed by \citet{HoyerJanzingMooijPetersSchoelkopf_NIPS_08} that uses score \eref{eq:ANM_score_pHSIC}. The main difference is
that the original implementation uses the HSIC $p$-value, whereas here, we use the HSIC value itself as a score. Also, we consider the option of splitting the dataset into one part for regression and another part for independence
testing. Finally, we fix the HSIC kernel instead of letting its bandwidth be chosen by a heuristic that depends on the data. The reason that we make these small modifications is that they lead to an easier proof of consistency of
the method.

We start with recapitulating the definition and basic properties of the Hilbert Schmidt Independence Criterion (HSIC) in Section~\ref{sec:HSIC}. Then, we discuss asymptotic properties
of non-parametric regression methods in Section~\ref{sec:regression}. Finally, we combine these ingredients in Section~\ref{sec:ANM_HSIC_consistency}.

\subsection{Consistency of HSIC}\label{sec:HSIC}
%\Jonas{we could call all theorems that are cited from other works lemmata?}
%\Joris{is that indeed conventional? in that way, you loose information about how important these statements are. however, we could give explicit citations like [Theorem 3, bla].}
%\Jonas{I would still call them Lemmata. This is exactly their use in this paper and it makes it easier to find your main result. IMO, importance is quite relative to the questions you want to answer.}
%\Joris{Ok.}
We recapitulate the definitions and some asymptotic properties of the Hilbert Schmidt Independence Criterion (HSIC), following mostly the notations and terminology in \citep{GrettonBousquetSmolaSchoelkopf2005}.
The HSIC estimator that we use here is the original biased estimator proposed by \citet{GrettonBousquetSmolaSchoelkopf2005}.
\begin{definition}
Given two random variables $X \in \C{X}$ and $Y \in \C{Y}$ with joint distribution $\Prb_{X,Y}$, and bounded kernels $k: \C{X}^2 \to \RN$ and
$l: \C{Y}^2 \to \RN$, we define the \textbf{population HSIC} of $X$ and $Y$ as:
\begin{equation*}\begin{split}
%  \HSIC_{k,l}(X,Y)  := {} & \Exp_{X,X',Y,Y'}\big(k(X,X')l(Y,Y')\big) + \Exp_{X,X'}\big(k(X,X')\big)\Exp_{Y,Y'}\big(l(Y,Y')\big) \\
%                          & - 2\Exp_{X,Y}\Big(\Exp_{X'}\big(k(X,X')\big) \Exp_{Y'}\big(l(Y,Y')\big)\Big).
  \HSIC_{k,l}(X,Y) := & \Exp \big(k(X,X')l(Y,Y')\big) + \Exp \big(k(X,X')\big) \Exp \big(l(Y,Y')\big) \\
                      & - 2 \Exp \Big( \Exp \big( k(X,X') \given X \big) \Exp \big( l(Y,Y') \given Y \big) \Big)
\end{split}\end{equation*}
%\Jonas{alternative SChreibweise:}
%\begin{equation*}\begin{split}
%  \HSIC_{k,l}(X,Y)  := {} & \Exp_{X,Y} \Exp_{X',Y'} \, k(X,X')l(Y,Y') \\
%                          & + \Exp_{X} \Exp_{X'} \Exp_{Y'} \,k(X,X') l(Y,Y') \\
%                          & - 2\Exp_{X,Y} \Exp_{X'} \Exp_{Y'}\, k(X,X') l(Y,Y')
%\end{split}\end{equation*}
Here, $(X,Y)$ and $(X',Y')$ are two independent random variables distributed according to $\Prb_{X,Y}$.
\end{definition}
When $k$ and $l$ are clear from the context, we will typically suppress the dependence of the population HSIC on the choice of the kernels $k$ and $l$, simply writing $\HSIC(X,Y)$ instead.
The justification for the name ``independence criterion'' stems from the following important result \citep[Theorem 3]{Fukumizu++2008}:
\begin{lemma}\label{lemm:HSICzero}
%Whenever $\mathcal{F},\mathcal{G}$ are Reproducible Kernel Hilbert Spaces with characteristic kernels $k$, $l$ (in the sense of \cite{Fukumizu++2008,Sriperumbudur++2008,Sriperumbudur++2010})
Whenever the product kernel $k \cdot l$ is characteristic (in the sense of \citep{Fukumizu++2008,Sriperumbudur++2010}):
$\HSIC_{k,l}(X,Y) = 0$ if and only if $X \indep Y$ (i.e., $X$ and $Y$ are independent).\hfill\BlackBox%
\end{lemma}
A special case of this lemma, assuming that $X$ and $Y$ have compact domain, was proven originally in \citep{GrettonBousquetSmolaSchoelkopf2005}.
Recently, \citet{Gretton2015} showed that a similar result also holds if both kernels $k$ and $l$ are characteristic and satisfy some other conditions as well.
Intuitively, a characteristic kernel leads to an injective embedding %$p \mapsto \mu(p)$ 
of probability measures into the corresponding Reproducible Kernel Hilbert Space (RKHS). The HSIC is the squared RKHS distance between the embedded joint distribution and the embedded product of the marginals. Given that the embedding is injective, this distance is zero if and only if the variables are independent.
Examples of characteristic kernels are Gaussian RBF kernels and Laplace kernels. For more details on the notion of characteristic kernel, see \citep{Sriperumbudur++2010}. 
We will use the following (biased) estimator of the population HSIC \citep{GrettonBousquetSmolaSchoelkopf2005}:
\begin{definition}\label{def:empHSIC}
Given two $N$-tuples (with $N \ge 2$) $\B{x} = (x_1,\dots,x_N) \in \C{X}^N$ and $\B{y} = (y_1,\dots,y_N) \in \C{Y}^N$, and bounded kernels
$k: \C{X}^2 \to \RN$ and $l: \C{Y}^2 \to \RN$, we define
\begin{equation}\label{eq:empHSIC}
\empHSIC_{k,l}(\B{x},\B{y}) := (N - 1)^{-2} \tr (KHLH) = (N-1)^{-2} \sum_{i,j=1}^N \bar K_{ij} L_{ij},
\end{equation}
where $K_{ij} = k(x_i,x_j)$, $L_{ij} = l(y_i,y_j)$ are Gram matrices and $H_{ij} = \delta_{ij} - N^{-1}$ is a centering matrix,
and we write $\bar K := H K H$ for the centered Gram matrix $K$.
Given an i.i.d.\ sample $\C{D}_N = \{(x_n,y_n)\}_{n=1}^N$ from $\Prb_{X,Y}$, we define
the \textbf{empirical HSIC} of $X$ and $Y$ estimated from $\C{D}_N$ as:
$$\empHSIC_{k,l}(X,Y ; \C{D}_N) := \empHSIC_{k,l}(\B{x},\B{y}).$$
\end{definition}
Again, when $k$ and $l$ are clear from the context, we will typically suppress the dependence of the empirical HSIC on the choice of the kernels $k$ and $l$.
Unbiased estimators of the population HSIC were proposed in later work \citep{SongSmolaGrettonBedoBorgwardt2012}, but we will not consider those here.
%Asymptotically, the expected empirical HSIC converges to the population HSIC \cite[Theorem 1]{GrettonBousquetSmolaSchoelkopf2005}:
%\begin{theorem} For all $p_{XY}$ \textbf{(Todo: check if there are any restrictions on $k$ and $l$)}:
%$$\HSIC(X,Y) = \Exp_{\C{D}_N}(\empHSIC(X,Y ; \C{D}_N)) + \mathcal{O}(N^{-1})$$
%or, alternatively,
%$$\lim_{N\to\infty} \Exp_{\C{D}_N}(\empHSIC(X,Y;\C{D}_N)) = \HSIC(X,Y)$$
%In other words, the empirical HSIC is a biased estimator of the population HSIC, but the bias vanishes asymptotically.
%\end{theorem}
%\begin{equation*}\begin{split}
%(n-1)^{-2} \tr KHLH 
%& = (n-1)^{-2} (\tr KL - 2 n^{-1} \tr \onevec^T KL \onevec + n^{-2} \tr K \tr L) \\
%& = (n-1)^{-2} (\sum_i K_{ii} L_{ii} + \sum_{(i,j) \in \ind{n}{2}} K_{ij} L_{ij} - 2 n^{-1} \tr \onevec^T KL \onevec + n^{-2} \tr K \tr L)
%\end{split}\end{equation*}
A large deviation result for this empirical HSIC estimator is given by \citep[Theorem 3]{GrettonBousquetSmolaSchoelkopf2005}:
\begin{lemma}\label{lemm:HSIClargedev}
Assume that kernels $k$ and $l$ are bounded almost everywhere by 1, and are non-negative.
Suppose that the data set $\C{D}_N$ consists of $N$ i.i.d.\ samples from some joint probability distribution $\Prb_{X,Y}$.
Then, for $N \ge 2$ and all $\delta > 0$, with probability at least $1 - \delta$:
$$\abs{\HSIC_{k,l}(X,Y) - \empHSIC_{k,l}(X,Y;\C{D}_N)} \le \sqrt{\frac{\log(6 / \delta)}{\alpha^2 N}} + \frac{c}{N},$$
where $\alpha^2 > 0.24$ and $c$ are constants.\hfill\BlackBox%
\end{lemma}
This directly implies the consistency of the empirical HSIC estimator:\footnote{Let $X_1,X_2,\dots$ be a sequence of random variables and let $X$ be another random variable.
We say that $X_n$ converges to $X$ \textbf{in probability}, written $X_n \xto{P} X$, if 
$$\forall \epsilon>0: \qquad \lim_{n\to\infty} \Prb(\abs{X_n - X} > \epsilon) = 0.$$}
\begin{corollary}\label{coro:HSICconsistency}
Let $(X_1,Y_1), (X_2,Y_2), \dots$ be i.i.d.\ according to $\Prb_{X,Y}$. Defining the sequence of data sets $\C{D}_N = \{(X_n,Y_n)\}_{n=1}^N$ for $N=2,3,\dots$,
we have for non-negative bounded kernels $k,l$ that, as $N\to\infty$:
$$\empHSIC_{k,l}(X,Y;\C{D}_N) \xto{P} \HSIC_{k,l}(X,Y).$$\hfill\BlackBox%
\end{corollary}
We do not know of any results for consistency of HSIC when using adaptive kernel parameters (i.e., when estimating the kernel from the data).
This is why we only present a consistency result for fixed kernels here.
%\begin{proof}
%Lemma \ref{lemm:HSIClargedev} implies $L_1$ convergence, which implies convergence in probability
%(see, e.g., \cite{Wasserman2004}).
%\end{proof}

For the special case that $\C{Y} = \RN$, we will use the following continuity property of the empirical HSIC estimator. 
It shows that for a Lipschitz-continuous kernel $l$, the empirical HSIC is also Lipschitz-continuous in the corresponding argument, 
but with a Lipschitz constant that scales at least as $N^{-1/2}$ for $N\to\infty$.
This novel technical result will be the key to our consistency proof of Algorithm~\ref{alg:bivariateANM} with score \eref{eq:ANM_score_HSIC_fixed}.
%\Jonas{genau genommen hast du nicht gezeigt, dass es keinere bessere gibt, oder?}
\begin{lemma}\label{lemm:HSICbound}
For all $N \ge 2$, for all $\B{x} \in \C{X}^N$, for all $\B{y},\B{y}' \in \RN^N$, for all bounded kernels $k : \C{X}^2 \to \RN$, and for all bounded and Lipschitz-continuous kernels $l : \RN^2 \to \RN$:
$$\abs{\empHSIC(\B{x},\B{y}) - \empHSIC(\B{x},\B{y}')} \le \frac{32 \lambda C}{\sqrt{N}} \norm{\B{y} - \B{y}'},$$
where $\abs{k(\xi,\xi')} \le C$ for all $\xi,\xi' \in \C{X}$ and $\lambda$ is the Lipschitz constant of $l$.
%where $\abs{\bar K_{ij}} \le C$, and $\lambda$ is the Lipschitz constant of $l$.
\end{lemma}
\begin{proof}
From its definition, \eref{eq:empHSIC}:
$$\abs{\empHSIC(\B{x},\B{y}) - \empHSIC(\B{x},\B{y}')} = (N-1)^{-2} \abs{\sum_{i,j=1}^N \bar K_{ij} (L_{ij} - L_{ij}')},$$
where $K_{ij} = k(x_i,x_j)$, $L_{ij} = l(y_i,y_j)$, $L_{ij}' = l(y_i',y_j')$ and $\bar K := H K H$ with $H_{ij} = \delta_{ij} - N^{-1}$.
First, note that $\abs{K_{ij}} \le C$ implies that $\abs{\bar K_{ij}} \le 4C$:
\begin{equation*}\begin{split}
\abs{\bar K_{ij}} 
& = \abs{K_{ij} - \frac{1}{N}\sum_{i'=1}^N K_{i'j} - \frac{1}{N}\sum_{j'=1}^N K_{ij'} + \frac{1}{N^2}\sum_{i',j'=1}^N K_{i'j'}} \\
& \le \abs{K_{ij}} + \frac{1}{N}\sum_{i'=1}^N \abs{K_{i'j}} + \frac{1}{N}\sum_{j'=1}^N \abs{K_{ij'}} + \frac{1}{N^2}\sum_{i',j'=1}^N \abs{K_{i'j'}} \le 4C.
\end{split}\end{equation*}
%\Jonas{maybe you want to change the indices $i$ and $j$!?}
%We first show that:
%\begin{equation*}\begin{split}
%  \abs{\sum_{i,j=1}^N \bar K_{ij} l(y_i',y_j') - \sum_{i,j=1}^N \bar K_{ij} l(y_i, y_j)} 
%%  & \le 2 \lambda (1 + C^2) \sqrt{\frac{1}{N} \sum_{i=1}^N \abs{y_i - y_i'}^2} \\
%%  & \le \frac{2 \lambda (1+C^2)}{\sqrt{N}} \norm{\hat y - y},
%  & \le \sqrt{2} \lambda C N^{3/2} \norm{y' - y},
%\end{split}\end{equation*}
%where $\abs{\bar K_{ij}} \le C$, and $\lambda$ is the Lipschitz constant of $l$.
Now starting from the definition and using the triangle inequality:
\begin{equation*}\begin{split}
  \abs{\sum_{i,j=1}^N \bar K_{ij} (L_{ij} - L_{ij}')} 
%  = \abs{ \sum_{i,j=1}^N \bar K_{ij} \left( l(y_i',y_j') - l(y_i, y_j) \right)} \\ 
  \le & \abs{\sum_{i,j=1}^N \bar K_{ij} \big( l(y_i', y_j') - l(y_i', y_j) \big)} 
      + \abs{\sum_{i,j=1}^N \bar K_{ij} \big( l(y_i', y_j) - l(y_i, y_j) \big)}.
\end{split}\end{equation*}
For the first term, using Cauchy-Schwartz (in $\RN^{N^2}$) and the Lipschitz property of $l$:
\begin{equation*}\begin{split}
  \abs{\sum_{i,j=1}^N \bar K_{ij} \big( l(y_i',y_j') - l(y_i', y_j) \big)}^2
  & \le \left(\sum_{i,j=1}^N \abs{\bar K_{ij}}^2 \right) \left( \sum_{i,j=1}^N \abs{l(y_i', y_j') - l(y_i', y_j)}^2 \right) \\
  & \le 16 N^2 C^2 \cdot \lambda^2 N \sum_{j=1}^N \abs{y_j' - y_j}^2 \\
  & = 16 N^3 C^2 \lambda^2 \norm{\B{y}' - \B{y}}^2.
\end{split}\end{equation*}
%Alternatively, using the $L_1$ norm:
%\begin{equation*}\begin{split}
%  \abs{\sum_{i,j=1}^N \bar K_{ij} \big( l(\hat e_i,\hat e_j) - l(\hat e_i, e_j) \big)}
%  & \le C \sum_{i,j=1}^N \abs{l(\hat e_i, \hat e_j) - l(\hat e_i, e_j)} \\
%  & \le C \lambda N \sum_{j=1}^N \abs{\hat e_j - e_j} \\
%  & = N C \lambda \norm{\hat e - e}_1.
%\end{split}\end{equation*}
The second term is similar.
%\Jonas{apparently with the same constant!? Also: are we 100\% sure that $C>0$!?}
The result now follows (using that $\frac{N}{N-1} \le 2$ for $N \ge 2$).
\end{proof}

\subsection{Consistency of nonparametric regression}\label{sec:regression}

From now on, we will assume that both $X$ and $Y$ take values in $\RN$. \citet{Gyorfi++2002} provide consistency results for several nonparametric regression methods. Here we 
briefly discuss the main property (``weak universal consistency'') that is of particular interest in our
setting.

Given a distribution $\Prb_{X,Y}$, one defines the \textbf{regression function} of $Y$ on $X$ as the conditional expectation
  $$f(x) := \Exp(Y \given X=x).$$
Given an i.i.d.\ sample of data points $\C{D}_N = \{(x_n,y_n)\}_{n=1}^N$ (the ``training'' data), a regression method provides an
estimate of the regression function $\hat f(\,\cdot\,; \C{D}_N)$. The \textbf{mean squared error on the training data} (also called ``training error'') is defined as:
$$\frac{1}{N}\sum_{n=1}^N \abs{f(x_n) - \hat f(x_n;\C{D}_N)}^2.$$
The \textbf{risk} (also called ``generalization error''), i.e., the expected $L_2$ error on an independent test datum, is defined as:
$$\Exp_X \abs{f(X) - \hat f(X;\C{D}_N)}^2 = \int \abs{f(x) - \hat f(x;\C{D}_N)}^2 \,d\Prb_X(x).$$
Note that the risk is a random variable that depends on the training data $\C{D}_N$.

If the expected risk converges to zero as the number of training points increases, the regression method
is called ``weakly consistent''. 
More precisely, following \cite{Gyorfi++2002}:
\begin{definition}
Let $(X_1,Y_1), (X_2,Y_2), \dots$ be i.i.d.\ according to $\Prb_{X,Y}$. Defining training data sets 
$\C{D}_N = \{(X_n,Y_n)\}_{n=1}^N$ for $N=2,3,\dots$, and writing $\Exp_{\C{D}_N}$ for the expectation value 
when averaging over $\C{D}_N$, a sequence of estimated regression functions $\hat f(\cdot;\C{D}_N)$ is called 
\textbf{weakly consistent for a certain distribution $\Prb_{X,Y}$} if
\begin{equation}\label{eq:regression_weak_consistency}
  \lim_{N\to\infty} \Exp_{\C{D}_N} \Big( \Exp_X \abs{f(X) - \hat f(X;\C{D}_N)}^2 \Big) = 0.
\end{equation}
%and called \textbf{strongly consistent for a certain distribution $p(X,Y)$}, if
%$$\lim_{N\to\infty} \Exp_X \abs{f(X) - \hat f(X;\C{D}_N)}^2 = 0 \quad a.s.$$
A regression method is called \textbf{weakly universally consistent} if it is weakly consistent for all distributions
$\Prb_{X,Y}$ with finite second moment of $Y$, i.e., with $\Exp_Y(Y^2) < \infty$.
\end{definition}
Many popular nonparametric regression methods have been shown to be weakly universally consistent, see e.g., \citep{Gyorfi++2002}.
One might expect na\"ively that if the expected risk goes to zero, then also the expected training error should vanish asymptotically: 
\begin{equation}\label{eq:regression_weak_consistency_training}
  \lim_{N\to\infty} \Exp_{\C{D}_N} \left( \frac{1}{N}\sum_{n=1}^N \abs{f(X_n) - \hat f(X_n;\C{D}_N)}^2 \right) = 0.
\end{equation}
However, property \eref{eq:regression_weak_consistency_training} does not necessarily follow from \eref{eq:regression_weak_consistency}.\footnote{Here 
is a counterexample. Suppose that regression method $\hat f$ satisfies properties \eref{eq:regression_weak_consistency_training} and \eref{eq:regression_weak_consistency} and that $X$ has bounded density. Given a smooth function $\phi : \RN \to \RN$ with support $\subset [-1,1]$, $0 \le \phi(x) \le 1$, and $\phi(0) = 1$,
we now construct a modified sequence of estimated regression functions $\tilde f(\,\cdot\,; \C{D}_N)$ that is defined as follows:
$$\tilde f(x; \C{D}_N) := \hat f(x; \C{D}_N) + \sum_{i=1}^N \phi\left( \frac{x - X_i}{\Delta_i^{(N)}} \right)$$
where the $\Delta_i^{(N)}$ should be chosen such that $\Delta_i^{(N)} \le N^{-2}$ and that the intervals $[X_i-\Delta_i^{(N)},X_i+\Delta_i^{(N)}]$
are disjoint. Then, we have that
$$\lim_{N\to\infty} \Exp_{\C{D}_N} \Big( \Exp_X \abs{f(X) - \tilde f(X;\C{D}_N)}^2 \Big) = 0$$
but on the other hand
$$\lim_{N\to\infty} \Exp_{\C{D}_N} \frac{1}{N}\sum_{n=1}^N \abs{f(X_n) - \tilde f(X_n;\C{D}_N)}^2 = 1.$$}
%One would expect that asymptotic results on the training error would actually be easier to obtain than results on generalization error, but this question seems to be less well studied in the existing literature.
One would expect that asymptotic results on the training error would actually be easier to obtain than results on generalization error.
One result that we found in the literature is \citep[Lemma~5]{Kpotufe++2014}, which states that a certain box kernel regression method satisfies
\eref{eq:regression_weak_consistency_training} under certain assumptions on the distribution $\Prb_{X,Y}$. The reason that we bring
this up at this point is that property \eref{eq:regression_weak_consistency_training} allows to prove consistency even when one uses
the same data for both regression and independence testing (see also Lemma~\ref{lem:suitable}).

From now on, we will always consider the following setting. Let $(\tilde X_1, \tilde Y_1)$, $(\tilde X_2, \tilde Y_2)$, $\dots$ be i.i.d.\ according to some joint distribution $\Prb_{X,Y}$.
We distinguish two different scenarios: 
\begin{itemize}
  \item ``Data splitting'': using half of the data for training, and the other half of the data for testing. In particular, we define
    $X_n := \tilde X_{2n-1}$, $Y_n := \tilde Y_{2n-1}$, $X_n' := \tilde X_{2n}$ and $Y_n' := \tilde Y_{2n}$ for $n=1,2,\dots$.
  \item ``Data recycling'': using the same data both for regression and for testing. In particular, we define
    $X_n := \tilde X_n$, $Y_n := \tilde Y_n$, $X_n' := \tilde X_n$ and $Y_n' := \tilde Y_n$ for $n=1,2,\dots$.
\end{itemize}
In both scenarios, for $N=1,2,\dots$, we define a sequence of training data sets $\C{D}_N := \{(X_n,Y_n)\}_{n=1}^N$ (for the regression) and a sequence of test data sets $\C{D}_N' := \{(X_n',Y_n')\}_{n=1}^N$ (for testing independence of residuals).
Note that in the data recycling scenario, training and test data are identical, whereas in the data splitting scenario, training and test data are independent.
%\subsubsection{Data splitting}
%Suppose we split our data into two different subsets: a training set $\C{D}_N = \{(x_n,y_n)\}_{n=1}^N$ and a test set $\C{D}_N' = \{(x_n',y_n')\}_{n=1}^N$.

Define a random variable (the ``residual'')
\begin{equation}\label{eq:true_residual}
  E := Y - f(X) = Y - \Exp(Y \given X),
\end{equation}
and its vector-valued versions on the test data:
\begin{equation}\label{eq:true_residuals}
%e(\C{D}_N') := \big(y_1' - f(x_1'), \dots, y_N' - f(x_N')\big) \in \RN^N,% =: (\epsilon_{1;N}', \dots, \epsilon_{N;N}'),
  \res := \big(Y_1' - f(X_1'), \dots, Y_N' - f(X_N')\big),
\end{equation}
called the \textbf{true residuals}. 
Using a regression method, we obtain an estimate $\hat f(x; \C{D}_N)$ for the regression function $f(x) = \Exp(Y \given X=x)$ from the training data $\C{D}_N$. 
%We then define an estimate of $E$ as: 
%\begin{equation}\label{eq:pred_residual}
%\Jonas{\hat E := \hat E(\C{D}_N) :=} \hat E(;\C{D}_N) := Y - \hat m(X; \C{D}_N)
%\end{equation}
%\Jonas{Brauchen wir dieses Definition ueberhaupt?}
%and its vector-valued version on the test data:
We then define an estimate of the vector-valued version of $E$ on the test data:
\begin{equation}\label{eq:pred_residuals}
%\hat e(\C{D}_N'; \C{D}_N) := \big(y_1' - \hat f(x_1'; \C{D}_N), \dots, y_N' - \hat f(x_N'; \C{D}_N)\big) \in \RN^N,% =: (\hat\epsilon_{1;N}', \dots, \hat\epsilon_{N;N}')
  \hatres := \big(Y_1' - \hat f(X_1'; \C{D}_N), \dots, Y_N' - \hat f(X_N'; \C{D}_N)\big),
  %\B{y}' - \hat f(\B{x}'; \C{D}_N)
\end{equation}
called the \textbf{predicted residuals}.

\begin{definition}\label{def:suitable}
We call the regression method \textbf{suitable} for regressing $Y$ on $X$ if the mean squared error between true and predicted residuals vanishes asymptotically in expectation:
\begin{equation}\label{eq:regression_suitable}
  \lim_{N\to\infty} \Exp_{\C{D}_N,\C{D}_N'} \left( \frac{1}{N} \norm{\hatres - \res}^2 \right) = 0.
\end{equation}
\end{definition}
Here, the expectation is taken over both training data $\C{D}_N$ and test data $\C{D}_N'$.
\begin{lemma}\label{lem:suitable}
In the data splitting case, any regression method that is weakly consistent for $\Prb_{X,Y}$ is suitable. 
In the data recycling case, any regression method satisfying property \eref{eq:regression_weak_consistency_training} is suitable.
\end{lemma}
\begin{proof}
Simply rewriting:
  \begin{equation*}\begin{split}
    & \lim_{N\to\infty} \Exp \left( \frac{1}{N} \norm{\hatres - \res}^2 \right) = \\
%  & \lim_{N\to\infty} \Exp \left( \frac{1}{N} \sum_{n=1}^N \abs{\hat \epsilon_{n;N}' - \epsilon_{n;N}'}^2 \right) = \\
  & \lim_{N\to\infty} \Exp \left( \frac{1}{N} \sum_{n=1}^N \abs{\big(Y_n' - \hat f(X_n'; \C{D}_N)\big) - \big(Y_n' - f(X_n')\big)}^2 \right) = \\
  & \lim_{N\to\infty} \Exp_{\C{D}_N,\C{D}_N'} \left( \frac{1}{N} \sum_{n=1}^N \abs{\hat f(X_n'; \C{D}_N) - f(X_n')}^2 \right)
  \end{split}\end{equation*}
Therefore, \eref{eq:regression_suitable} reduces to \eref{eq:regression_weak_consistency} in the data splitting scenario (where each $X_n'$ is
an independent copy of $X$), and reduces to \eref{eq:regression_weak_consistency_training} in the data recycling scenario (where $X_n' = X_n$).
%Note that the error in the estimated residuals equals (minus) the error in the estimated regression function:
%$$\hat \epsilon_{n;N}' - \epsilon_{n;N}' = f(x_n') - \hat f(x_n'; \C{D}_N), \qquad n=1,\dots,N.$$
%In vector notation:
%$$\hat E(\C{D}_N';\C{D}_N) - E(\C{D}_N') = \big(f(x_n') - \hat f(x_n'; \C{D}_N)\big)_{n=1}^N$$
\end{proof}
In particular, if $\Exp (X^2) < \infty$ and $\Exp (Y^2) < \infty$, any weakly universally consistent regression method is suitable both for regressing $X$ on $Y$ and $Y$ on $X$ in the data splitting scenario.
%\Jonas{very clear now!! i like it quite a bit better than the last version!}

\subsection{Consistency of \texttt{ANM-HSIC}}\label{sec:ANM_HSIC_consistency}

%\todo{Shall we indeed refer to Algorithm~\ref{alg:bivariateANM} with score \eref{eq:ANM_score_HSIC_fixed} as RESIT? The dilemma is that although that really captures the spirit of RESIT, the small technical differences
%with other/earlier versions of RESIT (data splitting, HSIC value instead of $p$-value) are currently needed for the consistency proof. It may be misleading to gloss over these details from a mathematical point of view, but
%from a more high-level perspective, we would like to ``identify'' all these RESIT versions.}
%\Jonas{I would still call it RESIT and say that it is slightly different from the originally proposed version.}
%\Joris{OK.}
%\Jonas{Also, is it necessary to show Alg. 3? Can we not refer to Alg. 1 and specify the score? Would save us a page!}
%\Joris{I agree. This is some leftover: previously, Algorithm 1 did not allow for data splitting; also, I didn't have a name for the score yet.}
We can now prove our main result, stating that the empirical HSIC calculated from the test set inputs and the predicted residuals on the test set 
(using the regression function estimated from the training set) converges in probability to the population HSIC of the true inputs and the true residuals:
%\footnote{\todo{Samory pointed out a potential problem in the proof: $\hat \epsilon_i'$ depends on $x_i'$ (but not on the other $x_j'$). This means that all $\hat \epsilon_i'$ together will always be dependent on all $x_i'$, i.e., we expect that HSIC will not find an independence, even if the true $\epsilon_i'$ are indeed independent of the $x_i'$. I reread my proof carefully and I don't think that this issue is actually relevant for this proof, as I do not use or show that $\hat \epsilon_{:N}'$ is independent of $x_{:N}'$ (I only use that $\epsilon_{:N}'$ is independent of $x_{:N}'$). If anybody checks the proof for validity, please pay particular attention to this issue.}}
\begin{theorem}\label{theo:consistency}
%$$\lim_{N\to\infty} \Exp_{\C{D}_N} \Exp_{\C{D}_N'}(\empHSIC(X,\hat E^{(N)};\C{D}_N')) = \lim_{N\to\infty} \Exp_{\C{D}_N'}(\empHSIC(X,E;\C{D}_N')) = \HSIC(X,E)$$
Let $X, Y \in \RN$ be two random variables with joint distribution $\Prb_{X,Y}$. Let $k,l:\RN\times\RN\to\RN$ be two bounded non-negative kernels
and assume that $l$ is Lipschitz continuous. 
%Let $(x_1,y_1),(x_2,y_2),\dots$ and $(x_1',y_1'),(x_2',y_2'),\dots$ be distributed i.i.d.\ according to $p(X,Y)$,
%and define training data sets $\C{D}_N = \{(x_n,y_n)\}_{n=1}^N$ and test data sets $\C{D}_N' = \{(x_n',y_n')\}_{n=1}^N$.
Suppose we are given sequences of training data sets $\C{D}_N$ and test data sets $\C{D}_N'$ (in either the data splitting or the data recycling scenario described above).
Suppose we use a suitable regression procedure (c.f. Lemma~\ref{lem:suitable}), 
to obtain a sequence $\hat f(x; \C{D}_N)$ of estimates of the regression function $\Exp(Y \given X=x)$ from the training data.
Defining the true residual $E$ by \eref{eq:true_residual}, and the predicted residuals $\hatres$ on the test data as in \eref{eq:pred_residuals}, then, for $N \to \infty$:
%$$\empHSIC(X,\hat E^{(N)};\C{D}_N') \xto{P} \HSIC(X,E).$$
$$\empHSIC_{k,l}(\B{X}_{\dots N}',\hatres) \xto{P} \HSIC_{k,l}(X,E).$$
\end{theorem}
\begin{proof}
%Defining the predicted residuals on the test data as $\hat \epsilon_n' := y_n' - \hat f(x_n'; \C{D}_N)$, 
We start by applying Lemma \ref{lemm:HSICbound}:
\begin{equation*}
  \abs{\HSICa - \HSICb}^2
%  & = \abs{\frac{1}{(N-1)^2} \sum_{i,j=1}^N \bar K_{ij}' l(\hat \epsilon_i',\hat \epsilon_j') - \frac{1}{(N-1)^2} \sum_{i,j=1}^N \bar K_{ij}' l(\epsilon_i',\epsilon_j')} \\
%  & \le \frac{N^2}{(N-1)^2} \frac{\sqrt{2} \lambda C}{\sqrt{N}} \norm{\hat \epsilon_{:N}' - \epsilon_{:N}'} 
% \le 4 \frac{\sqrt{2} \lambda C}{\sqrt{N}} \norm{\hat \epsilon_{:N}' - \epsilon_{:N}'}
  \le \left(\frac{32 \lambda C}{\sqrt{N}}\right)^2 \norm{\hatres - \res}^2
\end{equation*}
where $\lambda$ and $C$ are constants. From the suitability of the regression method, \eref{eq:regression_suitable},
%$$\lim_{N\to\infty} \Exp \left( \frac{1}{N} \norm{\hatres - \res}^2 \right) = 0.$$
it therefore follows that
\begin{equation*}
  \lim_{N\to\infty} \Exp_{\C{D}_N, \C{D}_N'} \abs{\HSICa - \HSICb}^2 = 0,
\end{equation*}
%Therefore:
%\begin{equation*}\begin{split}
%  & \Exp_{\C{D}_N} \Exp_{\C{D}_N'} \abs{\HSICa - \HSICb}^2 \\
%  & \le 32\lambda^2 C^2 \ \Exp_{\C{D}_N} \Exp_{\C{D}_N'} \frac{1}{N} \norm{\hat \epsilon_{:N}' - \epsilon_{:N}'}^2.
%\end{split}\end{equation*}
%Weak universal consistency of the regression implies that
%$$\lim_{N\to\infty} \Exp_{\C{D}_N} \Exp_{\C{D}_N'} \frac{1}{N} \norm{\epsilon - \hat \epsilon^{(N)}}^2 = 0$$
%Recalling \eref{eq:weak_consistency}, which followed from the weak universal consistency of the regression method:
%$$\lim_{N\to\infty} \Exp_{\C{D}_N} \Exp_{\C{D}_N'} \frac{1}{N} \norm{\hat \epsilon_{:N}' - \epsilon_{:N}'}^2 = 0,$$
%we conclude that:
%\begin{equation*}
%  \lim_{N\to\infty} \Exp_{\C{D}_N} \Exp_{\C{D}_N'} \abs{\HSICa - \HSICb}^2 = 0,
%\end{equation*}
i.e.,
$$\HSICa - \HSICb \xto{L_2} 0.$$
As convergence in $L_2$ implies convergence in probability (see, e.g., \citep{Wasserman2004}),
$$\HSICa - \HSICb \xto{P} 0.$$
From the consistency of the empirical HSIC, Corollary \ref{coro:HSICconsistency}:
$$\HSICb \xto{P} \HSICc.$$
Hence, by taking sums (see e.g., \citep[Theorem 5.5]{Wasserman2004}), we arrive at the desired statement.
% (using the first property of Theorem \ref{theo:convergence_properties}):
%$$\HSICa \xto{P} \HSICc.$$
\end{proof}

%We would like to show that the procedure in \cite{HoyerJanzingMooijPetersSchoelkopf_NIPS_08} for
%testing whether $\Prb_{X,Y}$ satisfies an additive noise model is consistent. We therefore need to specify
%an appropriate threshold or decision criterion for the independence test. If we know beforehand that
%either $\Prb_{X,Y}$ satisfies an additive noise model $X \to Y$, or $Y \to X$, but not both, then
%a simple solution is to compare the two HSIC values. This leads to Algorithm~\ref{alg:bivariateANM} with score \eref{eq:ANM_score_HSIC_fixed}.
We are now ready to show that Algorithm~\ref{alg:bivariateANM} with score \eref{eq:ANM_score_HSIC_fixed} (which is the special case $k=l$) is consistent.

\begin{corollary}\label{coro:bivariate_ANM_HSIC_consistency}
Let $X, Y$ be two real-valued random variables with joint distribution $\Prb_{X,Y}$ that either satisfies an additive noise model $X \to Y$, or $Y \to X$, but not both.
Suppose we are given sequences of training data sets $\C{D}_N$ and test data sets $\C{D}_N'$ (in either the data splitting or the data recycling scenario).
Let $k,l:\RN\times\RN\to\RN$ be two bounded non-negative Lipschitz-continuous kernels such that their product $k \cdot l$ is characteristic. If the regression procedure
used in Algorithm \ref{alg:bivariateANM} is suitable for both $\Prb_{X,Y}$ and $\Prb_{Y,X}$, then Algorithm \ref{alg:bivariateANM} with score \eref{eq:ANM_score_HSIC_fixed} is a consistent procedure for
estimating the direction of the additive noise model.
%\Jonas{we might need finite variances...}\Joris{Could you remember me why we need that?}
\end{corollary}
\begin{proof}
Define ``population residuals'' $E_Y := Y - \Exp(Y \given X)$ and $E_X := X - \Exp(X \given Y)$.
Note that $\Prb_{X,Y}$ satisfies a bivariate additive noise model $X \to Y$ if and only if $E_Y \indep X$
(c.f.\ Lemma~\ref{lem:ANMtest}). Further, by Lemma~\ref{lemm:HSICzero}, we have $\HSIC_{k,l}(X,E_Y) = 0$ 
if and only if $X \indep E_Y$.
Similarly, $\Prb_{X,Y}$ satisfies a bivariate additive noise model $Y \to X$ if and only if $\HSIC_{l,k}(Y,E_X) = 0$.
%Indeed, suppose that $\Prb_{X,Y}$ is induced by $(p_X, p_U, f)$, say
%$Y = f(X) + U$ with $X \indep U$, $X \sim p_X$, $U \sim p_U$.
%Then $\Exp(Y \given X = x) = f(x) + \nu$, with $\nu = \Exp (U)$.
%Therefore, $E_Y = Y - \Exp(Y \given X) = Y - (f(X) + \nu) = U - \nu$ is independent of $X$. 
%Conversely, if $E_Y$ is independent of $X$, $\Prb_{X,Y}$ is induced by the bivariate additive noise model 
%%$\langle p(X), p(Y - \Exp(Y \given X) - \Exp(Y - \Exp(Y \given X))), \Exp(Y \given X) + \Exp(Y - \Exp(Y \given X)) \rangle$.
%$(p_X, p_{E_Y}, x \mapsto \Exp(Y \given X = x))$.

Now, by Theorem \ref{theo:consistency},
$$\hat C_{X\to Y} := \empHSIC_{k,l}(\B{X}_{\dots N}',\hat{\B{E}}_Y(\C{D}_N';\C{D}_N)) \xto{P} \HSIC_{k,l}(X,E_Y),$$
%$$C_{X\to Y} := \empHSIC_{k,l}(X,\hat E_Y; \C{D}_N') \xto{P} \HSIC_{k,l}(X,E_Y),$$
and similarly 
$$\hat C_{Y\to X} := \empHSIC_{l,k}(\B{Y}_{\dots N}',\hat{\B{E}}_X(\C{D}_N';\C{D}_N)) \xto{P} \HSIC_{l,k}(Y,E_X),$$
%$$C_{Y\to X} := \empHSIC_{l,k}(Y,\hat E_X; \C{D}_N') \xto{P} \HSIC_{l,k}(Y,E_X).$$
where the predicted residuals are defined by
\begin{align}\label{eq:pred_residuals_XY}
  \hat{\B{E}}_Y(\C{D}_N';\C{D}_N) & := \big(Y_1' - \hat f_Y(X_1'; \C{D}_N), \dots, Y_N' - \hat f_Y(X_N'; \C{D}_N)\big), \\
  \hat{\B{E}}_X(\C{D}_N';\C{D}_N) & := \big(X_1' - \hat f_X(Y_1'; \C{D}_N), \dots, X_N' - \hat f_X(Y_N'; \C{D}_N)\big).
\end{align}
with estimates $\hat f_Y(x; \C{D}_N), \hat f_X(y; \C{D}_N)$ of the regression functions $\Exp(Y \given X=x), \Exp(X \given Y=y)$ from the training data $\C{D}_N$.
%\Jonas{see comment above}
%where we defined
%\begin{align}\label{eq:pred_residuals_XY}
%  \hat e_Y(\C{D}_N'; \C{D}_N) & := \big(y_1' - \hat f_Y(x_1'; \C{D}_N), \dots, y_N' - \hat f_Y(x_N'; \C{D}_N)\big) \in \RN^N, \\
%  \hat e_X(\C{D}_N'; \C{D}_N) & := \big(x_1' - \hat f_X(y_1'; \C{D}_N), \dots, x_N' - \hat f_X(y_N'; \C{D}_N)\big) \in \RN^N,
%\end{align}

Because $\Prb_{X,Y}$ satisfies an additive noise model only in one of the two directions, this
implies that either $\HSIC_{k,l}(X,E_Y) = 0$ and $\HSIC_{l,k}(Y,E_X) > 0$ (corresponding with $X \to Y$),
or $\HSIC_{k,l}(X,E_Y) > 0$ and $\HSIC_{l,k}(Y,E_X) = 0$ (corresponding with $Y \to X$).
Therefore the test procedure is consistent.
\end{proof}
%\Jonas{I would leave it this way (no extensions). I would not care too much about the discrepancy between theory and practice!}

%\subsection{Possible extensions}
%
%This subsection contains some thoughts regarding possible extensions.

\section{Relationship between scores \eref{eq:ANM_score_Gauss} and \eref{eq:ANM_score_FriedmanNachman}}\label{sec:gp} %%%%%%%%%%%%%%%%%%%%%%%%%%%%%%%%%%%%%%%%%%%%%%%%%%%%%%%%%%%%%%%%%%%%%%%%%%%%%%%%%

For the special case of an additive noise model $X \to Y$, the empirical-Bayes score proposed in \citet{FriedmanNachman2000} is given in
\eref{eq:ANM_score_FriedmanNachman}:
$$\hat C_{X \to Y} = \min_{\mu, \tau^2, \B{\theta}, \sigma^2} \left( -\log \C{N}(\B{x} \given \mu \B{1}, \tau^2 \B{I}) - \log \C{N}(\B{y} \given \B{0}, \B{K}_{\B{\theta}}(\B{x}) + \sigma^2 \B{I}) \right).$$
It is a sum of the negative log likelihood of a Gaussian model for the inputs:
\begin{equation}\label{eq:Gauss_input}\begin{split}
  & \min_{\mu, \tau^2} \big( -\log \C{N}(\B{x} \given \mu \B{1}, \tau^2 \B{I}) \big) \\
  & = \min_{\mu, \tau^2} \left( \frac{N}{2} \log (2\pi\tau^2) + \frac{1}{2 \tau^2} \sum_{i=1}^N (x_i - \mu)^2 \right) \\
  & = \frac{N}{2} \log (2\pi e) + \frac{N}{2} \log \left( \frac{1}{N} \sum_{i=1}^N (x_i - \bar x)^2 \right)
\end{split}\end{equation}
with $\bar x := \frac{1}{N} \sum_{i=1}^N x_i$, and the negative log marginal likelihood of a GP model for the outputs, given the inputs:
\begin{equation}\label{eq:Gauss_conditional}\begin{split}
  & \min_{\B{\theta}, \sigma^2} \big( - \log \C{N}(\B{y} \given 0, \B{K}_{\B{\theta}}(\B{x}) + \sigma^2 \B{I}) \big) \\
  & = \min_{\B{\theta}, \sigma^2} \left( \frac{N}{2} \log (2\pi) + \frac{1}{2} \log |\det (\B{K}_{\B{\theta}}(\B{x}) + \sigma^2 \B{I})| + \B{y}^T (\B{K}_{\B{\theta}}(\B{x}) + \sigma^2 \B{I})^{-1} \B{y} \right).
\end{split}\end{equation}
Note that \eref{eq:Gauss_input} is an empirical estimator of the entropy of a Gaussian with variance $\Var(X)$, up to a factor $N$:
$$H(X) = \frac{1}{2} \log(2\pi e) + \frac{1}{2} \log \Var(X).$$
We will show that \eref{eq:Gauss_conditional} is closely related to an empirical estimator of the entropy of 
the residuals $Y - \Exp(Y \given X)$:
$$H(Y - \Exp(Y \given X)) = \frac{1}{2} \log(2\pi e) + \frac{1}{2} \log \Var(Y - \Exp(Y \given X)).$$
This means that the score \eref{eq:ANM_score_FriedmanNachman} considered by \citet{FriedmanNachman2000} is closely
related to the Gaussian score \eref{eq:ANM_score_Gauss} for $X \to Y$:
$$\hat C_{X\to Y} = \log \Var(X) + \log \Var(Y-\hat f_Y(X)).$$

%For the $H(X)$ part, we get the following average negative log-likelihood when fitting a Gaussian $\C{N}(\mu,\sigma^2)$:
%$$\frac{1}{2} \log (2\pi e) + \frac{1}{2} \log\left(\frac{1}{N} \sum_{i=1}^N (x_i - \bar x)^2\right)$$
%with $\bar x := \frac{1}{N} \sum_{i=1}^N x_i$.

The following Lemma shows that standard Gaussian Process regression can be interpreted as a penalized maximum likelihood optimization. 
\begin{lemma}\label{lemm:gp_as_pml}
  Let $\B{K}_{\B{\theta}}(\B{x})$ be the kernel matrix (abbreviated as $\B{K}$) and define a negative penalized log-likelihood as:
%$$\C{L}(\B{f}; \B{y}, K) := \C{N}(\B{f} | \B{0}, \B{K}) \C{N}(\B{y} - \B{f} | \B{0}, \sigma^2 \B{I}) \abs{\det \big(\B{K} - \B{K} (\B{K} + \sigma^2 \B{I})^{-1} \B{K}\big)}^{1/2} (2\pi)^{N/2}$$
%The negative log-likelihood can also be written as:
\begin{equation}\label{eq:gp_as_pml}
  -\log \C{L}(\B{f}, \sigma^2; \B{y}, \B{K}) := \underbrace{\frac{N}{2} \log(2\pi \sigma^2) + \frac{1}{2 \sigma^2} \sum_{i=1}^N (y_i - f_i)^2}_{Likelihood} + \underbrace{\frac{1}{2} \B{f}^T \B{K}^{-1} \B{f} + \frac{1}{2} \log |\det \big(\B{I} + \sigma^{-2} \B{K}\big)|}_{Penalty}.
\end{equation}
Minimizing with respect to $\B{f}$ yields a minimum at:
\begin{equation}\label{eq:gp_as_pml_minimum}
  \B{\hat f}_{\sigma,\B{\theta}} = \argmin_{\B{f}} \left( -\log \C{L}(\B{f}, \sigma^2; \B{y}, \B{K}_{\B{\theta}}) \right) = \B{K}_{\B{\theta}}(\B{K}_{\B{\theta}} + \sigma^2 \B{I})^{-1} \B{y},
\end{equation}
and the value at the minimum is given by:
\begin{equation}\label{eq:gp_as_pml_value}
  \min_{\B{f}} \left( -\log \C{L}(\B{f}, \sigma^2; \B{y}, \B{K}_{\B{\theta}}) \right) = -\log \C{L}(\B{\hat f}_{\sigma,\B{\theta}}, \sigma^2; \B{y}, \B{K}_{\B{\theta}}) = -\log \C{N}(\B{y} \given \B{0}, \B{K}_{\B{\theta}} + \sigma^2 \B{I}).
\end{equation}
\end{lemma}
\begin{proof}
Because $\B{B} (\B{A}^{-1} + \B{B}^{-1}) \B{A} = \B{A} + \B{B}$ for invertible (equally-sized square) matrices $\B{A}, \B{B}$, the following identity holds:
$$(\B{A}^{-1} + \B{B}^{-1})^{-1} = \B{A} (\B{A} + \B{B})^{-1} \B{B}.$$
Substituting $\B{A} = \B{K}$ and $\B{B} = \sigma^2 \B{I}$, we obtain directly that:
\begin{equation}\label{eq:gp_as_pml_id3}
  (\B{K}^{-1} + \sigma^{-2} \B{I})^{-1} = \B{K} (\B{K} + \sigma^2 \B{I})^{-1} \sigma^2.
\end{equation}
By taking log-determinants, it also follows that:
\begin{equation}\label{eq:gp_as_pml_id1}
  \log |\det \B{K}| + \log |\det (\B{K}^{-1} + \sigma^{-2} \B{I})| = \log |\det \big(\B{I} + \sigma^{-2}\B{K}\big)|.
\end{equation}
Therefore, we can rewrite \eref{eq:gp_as_pml} as follows:
\begin{equation}\label{eq:gp_as_pml_id4}
\C{L}(\B{f}, \sigma^2; \B{y}, \B{K}) = \C{N}(\B{f} | \B{0}, \B{K}) \C{N}(\B{y} - \B{f} | \B{0}, \sigma^2 \B{I}) \abs{\det \big(\B{K}^{-1} + \sigma^{-2} \B{I}\big)}^{-1/2} (2\pi)^{N/2}.
\end{equation}

Equation (A.7) in \citep{RasmussenWilliams2006} for the product of two Gaussians states that
$$\C{N}(\B{x}|\B{a},\B{A}) \C{N}(\B{x}|\B{b},\B{B}) = \C{N}(\B{a}|\B{b},\B{A}+\B{B}) \C{N}(\B{x}|\B{c},\B{C}),$$
where $\B{C} = (\B{A}^{-1} + \B{B}^{-1})^{-1}$ and $\B{c} = \B{C}(\B{A}^{-1}\B{a} + \B{B}^{-1}\B{b})$.
Substituting $\B{x} = \B{f}$, $\B{a} = \B{0}$, $\B{A} = \B{K}$, $\B{b} = \B{y}$, and $\B{B} = \sigma^2\B{I}$, and using \eref{eq:gp_as_pml_id3}, this gives:
\begin{equation}\label{eq:gp_as_pml_id2}\begin{split}
  \C{N}(\B{f} | \B{0}, \B{K}) \C{N}(\B{y} - \B{f} | \B{0}, \sigma^2 \B{I}) = \C{N}(\B{y} \given \B{0}, \B{K} + \sigma^2 \B{I}) \C{N}(\B{f} | \B{\hat f}, (\B{K}^{-1} + \sigma^{-2} \B{I})^{-1}),
\end{split}\end{equation}
where
$$\B{\hat f}_{\sigma,\B{\theta}} := \B{K}_{\B{\theta}}(\B{K}_{\B{\theta}} + \sigma^2 \B{I})^{-1} \B{y}.$$ %= \B{K} \B{\hat \alpha}.$$
Therefore, we can rewrite \eref{eq:gp_as_pml_id4} as:
$$\C{L}(\B{f}, \sigma^2; \B{y}, \B{K}) = \C{N}(\B{y} \given \B{0}, \B{K} + \sigma^2 \B{I}) \C{N}(\B{f} | \B{\hat f}, (\B{K}^{-1} + \sigma^{-2} \B{I})^{-1}) \abs{\det \big(\B{K}^{-1} + \sigma^{-2} \B{I}\big)}^{-1/2} (2\pi)^{N/2}.$$

%$$-\log \C{N}(\B{f} | \B{0}, \B{K}) -\log \C{N}(\B{y} - \B{f} | \B{0}, \sigma^2 \B{I}) = 2 \frac{N}{2}\log(2\pi) + \frac{1}{2}\log|\B{K}| + \frac{1}{2}\log|\sigma^2 \B{I}| + \frac{1}{2} \B{f}^T \B{K}^{-1} \B{f} + \frac{1}{2 \sigma^2} (\B{y} - \B{f})^T (\B{y} - \B{f})$$
%so
%\begin{equation*}\begin{split}
%  & -\log \C{N}(\B{f} | \B{0}, \B{K}) -\log \C{N}(\B{y} - \B{f} | \B{0}, \sigma^2 \B{I}) \\
%  & = -\log \C{L}(\B{f}, \sigma^2; \B{y}, \B{K}) + \frac{N}{2}\log(2\pi) + \frac{1}{2}\log|\B{K}| - \frac{1}{2} \log |\det \big(\B{I} + \sigma^{-2} \B{K}\big)| \\
%  & = -\log \C{L}(\B{f}, \sigma^2; \B{y}, \B{K}) + \frac{N}{2}\log(2\pi) - \frac{1}{2} \log |\det \big(\B{K}^{-1} + \sigma^{-2} \B{I}\big)|
%\end{split}\end{equation*}
It is now obvious that the penalized likelihood is maximized for $\B{f} = \B{\hat f}_{\sigma,\B{\theta}}$ (for fixed hyperparameters $\sigma,\B{\theta}$) and that at the maximum, it has the value 
$$\C{L}(\B{\hat f}_{\sigma,\B{\theta}}, \sigma^2; \B{y}, \B{K}_{\B{\theta}}) = \C{N}(\B{y} \given \B{0}, \B{K}_{\B{\theta}} + \sigma^2 \B{I}).$$
\end{proof}
Note that the estimated function \eref{eq:gp_as_pml_minimum} is identical to the mean posterior GP,
and the value \eref{eq:gp_as_pml_value} is identical to the negative logarithm of the marginal likelihood (evidence) of the data according to the GP model \citep{RasmussenWilliams2006}.

Making use of Lemma~\ref{lemm:gp_as_pml}, the conditional part \eref{eq:Gauss_conditional} in score \eref{eq:ANM_score_FriedmanNachman} can be rewritten as:
\begin{equation*}\begin{split}
  &\min_{\sigma^2,\B{\theta}} \left(-\log \C{N}(\B{y} \given \B{0}, \B{K}_{\B{\theta}} + \sigma^2 \B{I})\right) \\
  %  & = \min_{\sigma^2, \B{\theta}, \B{f}} \left( -\log \C{L}(\B{f}, \sigma^2; \B{y},\B{K}_{\B{\theta}}) \right) \\
  %  & = \min_{\sigma^2, \B{\theta}, \B{f}} \left( \frac{N}{2} \log(2\pi \sigma^2) + \frac{1}{2 \sigma^2} \sum_{i=1}^N (y_i - f_i)^2 + \frac{1}{2} \B{f}^T \B{K}_{\B{\theta}}^{-1} \B{f} + \frac{1}{2} \log |\det \big(\B{I} + \sigma^{-2} \B{K}_{\B{\theta}}\big)| \right) \\
  & = \min_{\sigma^2, \B{\theta}} \left( \frac{N}{2} \log(2\pi \sigma^2) + \frac{1}{2 \sigma^2} \sum_{i=1}^N (y_i - (\B{\hat f}_{\sigma,\B{\theta}})_i)^2 + \frac{1}{2} \B{\hat f}_{\sigma,\B{\theta}}^T \B{K}_{\B{\theta}}^{-1} \B{\hat f}_{\sigma,\B{\theta}} + \frac{1}{2} \log |\det \big(\B{I} + \sigma^{-2} \B{K}_{\B{\theta}}\big)| \right) \\
  & = \underbrace{\frac{N}{2} \log(2\pi \hat\sigma^2) + \frac{1}{2 \hat\sigma^2} \sum_{i=1}^N (y_i - (\B{\hat f})_i)^2}_{\text{Likelihood term}} + \underbrace{\frac{1}{2} \B{\hat f}^T \B{K}_{\B{\hat \theta}}^{-1} \B{\hat f} + \frac{1}{2} \log |\det \big(\B{I} + \hat\sigma^{-2} \B{K}_{\B{\hat \theta}}\big)|}_{\text{Complexity penalty}},
\end{split}\end{equation*}
where $\B{\hat f} := \B{\hat f}_{\hat \sigma, \B{\hat \theta}}$ for the minimizing $(\hat \sigma, \B{\hat \theta})$.
If the complexity penalty is small compared to the likelihood term around the optimal values $(\hat \sigma, \B{\hat \theta})$, we can approximate:
\begin{equation*}\begin{split}
  & \min_{\sigma^2,\B{\theta}} \left(-\log \C{N}(\B{y} \given \B{0}, \B{K}_{\B{\theta}} + \sigma^2 \B{I})\right) \\
  & \approx \frac{N}{2} \log(2\pi \hat\sigma^2) + \frac{1}{2 \hat\sigma^2} \sum_{i=1}^N (y_i - \hat f_i)^2 \\
  & \approx \min_{\sigma^2} \left( \frac{N}{2} \log(2\pi\sigma^2) + \frac{1}{2 \sigma^2} \sum_{i=1}^N (y_i - \hat f_i)^2 \right) \\
  & = \frac{N}{2} \log(2\pi e) + \frac{N}{2} \log \left( \frac{1}{N} \sum_{i=1}^N (y_i - \hat f_i)^2 \right).
\end{split}\end{equation*}
This shows that there is a close relationship between the two scores
\eref{eq:ANM_score_FriedmanNachman} and \eref{eq:ANM_score_Gauss}.

\section{Details on the simulated data}\label{sec:sim} %%%%%%%%%%%%%%%%%%%%%%%%%%%%%%%%%%%%%%%%%%%%%%%%%%%%%%%%%%%%%%%%%%%%%%%%%%%%%%%%%%%%%%%%%%%%%%%%%%%%%%%%%%%%%%%%%%%%%%%%%%%%%%%%%%%%%%%%%%%%%%%%

%\begin{algorithm}
%\caption{Simulate data}
%\label{alg:simulate_data}
%\begin{algorithmic}
%  \INPUT $\sigma_{X}$, $\sigma_{Z}$, $\sigma_{E}$, $\tau$
%\STATE $S \leftarrow \{1,\dots,d\}$
%\FOR{$j=d$ {\bfseries downto} $1$}
%  \FORALL{$i \in S$}
%    \STATE $\hat\epsilon_i \leftarrow \residuals{X_{S \setminus \{i\}}}{X_i}$
%    \STATE $p_i \leftarrow \indeptest{X_{S \setminus \{i\}}}{\hat\epsilon_i}$
%  \ENDFOR
%  \STATE $i^* \leftarrow \argmax p_i$
%  \IF{$p_{i^*} < \alpha$}
%    \STATE {\bfseries return} no consistent DAGs
%  \ENDIF
%  \STATE $\sigma_j \leftarrow i^*$
%  \STATE $S \leftarrow S \setminus \{i^*\}$
%\ENDFOR
%\FOR{$j=1$ {\bfseries to} $d$}
%  \STATE $i \leftarrow \sigma_j$
%  \STATE $\parents{i} \leftarrow \{\sigma_1,\dots,\sigma_{j-1}\}$
%    \FOR{$k=1$ {\bfseries to} $j-1$}
%       \STATE $\hat\epsilon_i \leftarrow \residuals{X_{\parents{i}\setminus\{\sigma_k\}}}{X_i}$
%       \IF{$\indeptest{X_{\parents{i}}}{\hat\epsilon_i} \ge \alpha$}
%         \STATE $\parents{i} \leftarrow \parents{i} \setminus \{\sigma_k\}$ 
%       \ENDIF
%    \ENDFOR
%\ENDFOR
%\OUTPUT parent sets $(\parents{i})_{i\in\vars}$
%\end{algorithmic}
%\end{algorithm}

% function monoincf = randmonoinc(X,sigma_in,sigma_out,sigma_noise)
\subsection{Sampling from a random density}
We first describe how we sample from a random density. First, we sample $\B{X} \in \RN^N$ from a standard-normal distribution:
$$\B{X} \sim \C{N}(\B{0}_N,\B{I}_N)$$
and define $\overrightarrow{\B{X}}$ to be the vector that is obtained by sorting $\B{X}$ in ascending order.
Then, we sample a realization $\B{F}$ of a Gaussian Process with inputs $\overrightarrow{\B{X}}$, 
using a kernel with hyperparameters $\B{\theta}$ and white noise with standard deviation $\sigma$:
$$\B{F} \sim \C{N}(\B{0}, \B{K}_{\B{\theta}}(\overrightarrow{\B{X}}) + \sigma^2 \B{I}),$$
where $\B{K}_{\B{\theta}}(\overrightarrow{\B{X}})$ is the Gram matrix for $\overrightarrow{\B{X}}$ using
kernel $k_{\B{\theta}}$.
We use the trapezoidal rule to calculate the cumulative integral of the function $e^F : \RN \to \RN$ that
linearly interpolates the points $(\overrightarrow{\B{X}}, \exp(\B{F}))$. In this way, we obtain a vector
$\B{G} \in \RN^N$ where each element $G_i$ corresponds to $\int_{\overrightarrow{X}_1}^{\overrightarrow{X}_i} e^{F}(x) \,dx$.
As covariance function, we used the Gaussian kernel:
\begin{equation}\label{eq:kernel}
  k_{\B{\theta}}(\B{x},\B{x}') = \exp \left( -\sum_{i=1}^D \frac{(x_i - x_i')^2}{\theta_i^2} \right).
\end{equation}
We will denote this whole sampling procedure by:
$$\B{G} \sim \C{RD}(\B{\theta},\sigma).$$

\subsection{Sampling cause-effect pairs}
We simulate cause-effect pairs as follows. First, we sample three noise variables:
%\begin{align*}
%  W_{E_X} &\sim \C{U}(0.1,1) & \B{E}_X &\sim \C{RD}(W_{E_X}, 1, 10^{-3})\\
%  W_{E_Y} &\sim \C{U}(0.1,1) & \B{E}_Y &\sim \C{RD}(W_{E_X}, 1, 10^{-3})\\
%  W_{E_Z} &\sim \C{U}(0.1,1) & \B{E}_Z &\sim \C{RD}(W_{E_Z}, 1, 10^{-3})
%\end{align*}
\begin{align*}
  W_{E_X} &\sim \Gamma(a_{W_{E_X}}, b_{W_{E_X}}) & \B{E}_X &\sim \C{RD}(W_{E_X}, \tau)\\
  W_{E_Y} &\sim \Gamma(a_{W_{E_Y}}, b_{W_{E_Y}}) & \B{E}_Y &\sim \C{RD}(W_{E_X}, \tau)\\
  W_{E_Z} &\sim \Gamma(a_{W_{E_Z}}, b_{W_{E_Z}}) & \B{E}_Z &\sim \C{RD}(W_{E_Z}, \tau)
\end{align*}
where each noise variable has a random characteristic length scale. We then standardize each noise sample $\B{E}_X$, $\B{E}_Y$ and $\B{E}_Z$.

If there is no confounder, we sample $\B{X}$ from a GP with inputs $\B{E}_X$:
\begin{align*}
%  \gamma_{E_X} & \sim \C{N}(0,1), \qquad \sigma_{E_X} := 1 + e^{\gamma_{E_X}} \\
  S_{E_X} & \sim \Gamma(a_{S_{E_X}},b_{S_{E_X}}) \\
  \B{X} & \sim \C{N}(\B{0}, \B{K}_{S_{E_X}}(\B{E}_X) + \tau^2 \B{I})
\end{align*}
and then we standardize $\B{X}$. Then, we sample $\B{Y}$ from a GP with inputs $(\B{X},\B{E}_Y) \in \RN^{N\times 2}$:
% Y = randgp([X,EY,EZ],[(\exp(randn(1,1))+1)*sigma_cause,(\exp(randn(1,1))+1)*sigma_noise,(\exp(randn(1,numZ))+1)*sigma_confounder],1,1e-4);
\begin{align*}
%  \gamma_X & \sim \C{N}(0,1), \qquad \sigma_X := \sigma_{cause} (1 + e^{\gamma_X}) \\
%  \gamma_{E_Y} & \sim \C{N}(0,1), \qquad \sigma_{E_Y} := \sigma_{noise} (1 + e^{\gamma_{E_Y}}) \\
  S_{X} & \sim \Gamma(a_{S_{X}},b_{S_{X}}) \\
  S_{E_Y} & \sim \Gamma(a_{S_{E_Y}},b_{S_{E_Y}}) \\
  \B{Y} & \sim \C{N}(\B{0}, \B{K}_{(S_X,S_{E_Y})}((\B{X},\B{E}_Y)) + \tau^2 \B{I})
\end{align*}
and then we standardize $\B{Y}$.

If there is a confounder, we sample $\B{X}$ from a GP with inputs $(\B{E}_X, \B{E}_Z) \in \RN^{N\times 2}$:
\begin{align*}
%  \gamma_{E_X} & \sim \C{N}(0,1), \qquad \sigma_{E_X} := 1 + e^{\gamma_{E_X}} \\
%  \gamma_{E_Z} & \sim \C{N}(0,1), \qquad \sigma_{E_Z} := \sigma_{confounder} (1 + e^{\gamma_{E_Z}}) \\
  S_{E_X} & \sim \Gamma(a_{S_{E_X}},b_{S_{E_X}}) \\
  S_{E_Z} & \sim \Gamma(a_{S_{E_Z}},b_{S_{E_Z}}) \\
\B{X} & \sim \C{N}(\B{0}, \B{K}_{(S_{E_X},S_{E_Z})}((\B{E}_X,\B{E}_Z)) + \tau^2 \B{I})
\end{align*}
and then we standardize $\B{X}$.
Then, we sample $\B{Y}$ from a GP with inputs $(\B{X},\B{E}_Y,\B{E}_Z) \in \RN^{N\times 3}$:
% Y = randgp([X,EY,EZ],[(\exp(randn(1,1))+1)*sigma_cause,(\exp(randn(1,1))+1)*sigma_noise,(\exp(randn(1,numZ))+1)*sigma_confounder],1,1e-4);
\begin{align*}
%  \gamma_X & \sim \C{N}(0,1), \qquad \sigma_X := \sigma_{cause} (1 + e^{\gamma_X}) \\
%  \gamma_{E_Y} & \sim \C{N}(0,1), \qquad \sigma_{E_Y} := \sigma_{noise} (1 + e^{\gamma_{E_Y}}) \\
%  \gamma_{E_Z} & \sim \C{N}(0,1), \qquad \sigma_{E_Z} := \sigma_{confounder} (1 + e^{\gamma_{E_Z}}) \\
  S_{X} & \sim \Gamma(a_{S_{X}},b_{S_{X}}) \\
  S_{E_Y} & \sim \Gamma(a_{S_{E_Y}},b_{S_{E_Y}}) \\
  S_{E_Z} & \sim \Gamma(a_{S_{E_Z}},b_{S_{E_Z}}) \\
  \B{Y} & \sim \C{N}(\B{0}, \B{K}_{(S_X,S_{E_Y},S_{E_Z})}((\B{X},\B{E}_Y,\B{E}_Z)) + \tau^2 \B{I})
\end{align*}
and then we standardize $\B{Y}$.

Finally, we add measurement noise:
\begin{align*}
%  \beta_X & \sim \C{N}(0,1), \qquad \sigma_{\epsilon_X} := \tau (1 + e^\beta_X) \\
%  \beta_Y & \sim \C{N}(0,1), \qquad \sigma_{\epsilon_Y} := \tau (1 + e^\beta_Y) \\
  S_{M_X} & \sim \Gamma(a_{S_{M_X}},b_{S_{M_X}}) \\
  \B{M}_X & \sim \C{N}(\B{0}, S_{M_X}^2 \B{I}) \\
  \B{X} & \leftarrow \B{X} + \B{M}_X \\
  S_{M_Y} & \sim \Gamma(a_{S_{M_Y}},b_{S_{M_Y}}) \\
  \B{M}_Y & \sim \C{N}(\B{0}, S_{M_Y}^2 \B{I}) \\
  \B{Y} & \leftarrow \B{Y} + \B{M}_Y
\end{align*}

We considered the four scenarios in Table~\ref{tab:sim_scenarios}: \texttt{SIM}, a scenario
without confounders; \texttt{SIM-c}, a similar scenario but with one confounder; 
\texttt{SIM-ln}, a scenario with low noise levels (for which we expect IGCI to perform
well); \texttt{SIM-G}, a scenario with a distribution of $X$ that is almost Gaussian.
We used $N = 1000$ samples for each pair, and simulated 100 cause-effect pairs for each scenario.

% function simulate_pairs2(N,datadir,numpairs,confprob,output,seed,AB_EX,AB_EY,AB_EZ,AB_sEX,AB_sEY,AB_sEZ,AB_sX,AB_UX,AB_UY,jitter)
%       SIMDIR='sim'
%       matlab -nodisplay -nodesktop -r "simulate_pairs2(1000,'$SIMDIR',100,[1],'files',123,[5,0.1],[5,0.1],[5,0.1],[2,1.5],[2,1.5*10],[2,1.5*10],[2,1.5*2],[2,0.1],[2,0.1],1e-4); exit"
%       SIMDIR='simconf'
%       matlab -nodisplay -nodesktop -r "simulate_pairs2(1000,'$SIMDIR',100,[0 1],'files',123,[5,0.1],[5,0.1],[5,0.1],[2,1.5],[2,1.5*10],[2,1.5*10],[2,1.5*2],[2,0.1],[2,0.1],1e-4); exit"
%       SIMDIR='simlownoise'
%       matlab -nodisplay -nodesktop -r "simulate_pairs2(1000,'$SIMDIR',100,[1],'files',123,[5,0.1],[5,0.1],[5,0.1],[2,1.5],[2,1.5*200],[2,1.5*10],[2,1.5*2],[2,0.01],[2,0.01],1e-4); exit"
%       SIMDIR='simgauss'
%       matlab -nodisplay -nodesktop -r "simulate_pairs2(1000,'$SIMDIR',100,[1],'files',123,[1e6,1e-3],[5,0.1],[5,0.1],[1e6,1e-3],[2,1.5*10],[2,1.5*10],[2,1.5*2],[2,0.1],[2,0.1],1e-4); exit"

\begin{table}
\caption{\label{tab:sim_scenarios}Parameter settings used to simulate cause-effect pairs for four scenarios. The common parameters for the four scenarios are:
$\tau = 10^{-4}$, $(a_{Wn{E_Y}}, b_{W_{E_Y}}) = (5,0.1)$, $(a_{W_{E_Z}}, b_{W_{E_Z}}) = (5,0.1)$, $(a_{S_{E_Z}},b_{S_{E_Z}}) = (2,15)$, $(a_{S_{X}},b_{S_{X}}) = (2,15)$.}
\small
\centerline{\begin{tabular}{lcccccc}
% name &                    \# confounders & $AB_{EX}$ &    $AB_{sEX}$ &   $AB_{sEY}$ &    $AB_{UX}$ &  $AB_{UY}$ \\
  name &                    \# confounders & $(a_{W_{E_X}}, b_{W_{E_X}})$ &    $(a_{S_{E_X}},b_{S_{E_X}})$ &   $(a_{S_{E_Y}},b_{S_{E_Y}})$ &    $(a_{S_{M_X}},b_{S_{M_X}})$ &  $(a_{S_{M_Y}},b_{S_{M_Y}})$ \\
  \hline
  \texttt{SIM} &            0 &              $(5,0.1)$ &        $(2,1.5)$ &        $(2,15)$ &  $(2,0.1)$ &  $(2,0.1)$ \\
  \texttt{SIM-c} &          1 &              $(5,0.1)$ &        $(2,1.5)$ &        $(2,15)$ &  $(2,0.1)$ &  $(2,0.1)$ \\
  \texttt{SIM-ln} &         0 &              $(5,0.1)$ &        $(2,1.5)$ &        $(2,300)$ & $(2,0.01)$ & $(2,0.01)$ \\
  \texttt{SIM-G} &          0 &              $(10^6,10^{-3})$ & $(10^6,10^{-3})$ & $(2,15)$ &  $(2,0.1)$ &  $(2,0.1)$ \\
\end{tabular}}
\end{table}

\section{Description of the {\tt CauseEffectPairs} benchmark set}\label{sec:dataset}

%\todo{Please add necessary citations to dataset.bib if dataset providers request this. 
%Please also update the corresponding pair????\_des.txt files in the webdav directory (for
%example in case of dead links or incomplete descriptions).
%The idea is that one first describes the dataset, including which features we selected, and
%then for each pair argues what would be the ground truth (see e.g., D1). Sometimes, however,
%if there are many similar pairs in a dataset, one can combine these (see e.g., D8).
%Please also update tables 1 and 2 and choose a consistent name of the datasets and the pairs
%(i.e., the names in tables 1 and 2 should correspond to those in the (sub)subsection headers).
%Please don't refer to variables as $X$ and $Y$ (like we did in the pair????\_des.txt files), instead refer to
%the variables by their name (nobody is interested in whether we called something $X$ rather than
%$Y$ or vice versa, nor whether it is the first or the second variable: this can be found in Table 2),
%also in the subsubsection headers. Describe any confounders, selection biases, and feedbacks that you can think of.}

The {\tt CauseEffectPairs} benchmark set described here is an extension of the
collection of the eight data sets that formed the {\tt CauseEffectPairs} task 
in the \emph{Causality Challenge \#2: Pot-Luck}
competition \citep{MooijJanzing_JMLR_10} that was performed as part of the
NIPS 2008 Workshop on Causality \citep{NIPSCausalityChallenge2008}.\footnote{The introduction of this 
section and the descriptions of these 8 pairs are heavily based on \citep{MooijJanzing_JMLR_10}.}
Here we describe version \version\ of the {\tt CauseEffectPairs} benchmark, which consists of \nrpairs\ ``cause-effect pairs'', each one
consisting of samples of a pair of statistically dependent
random variables, where one variable is known to cause the other
one. The task is to identify for each pair which of the two variables is the cause and which
one the effect, using the observed samples only. The data are publicly available at \citep{MooijJanzingSchoelkopf2014}.

The data sets were selected such
that we expect common agreement on the ground truth. For example, the first pair
consists of measurements of altitude and mean annual temperature of more than 300 weather
stations in Germany. It should be obvious that altitude causes temperature rather than the other way around. Even though part of the
statistical dependences may also be due to hidden common causes and selection bias,
we expect that there is a significant cause-effect relation between the two
variables in each pair, based on our understanding of the data
generating process.

The best way to decide upon the ground truth of the causal relationships in the
systems that generated the data would be by performing interventions on one of
the variables and observing whether the intervention changes the distribution
of the other variable.  Unfortunately, these interventions cannot be performed in
practice for many of the existing pairs because the original
data-generating system is no longer available, or because of other practical
reasons. Therefore, we have selected data sets in which the causal
direction should be clear from the meanings of the variables and the way in
which the data were generated. Unfortunately, for many data sets that are publicly
available, it is not always clearly documented exactly how the variables are defined
and measured.

In selecting the cause-effect pair data sets, we applied the following criteria:
\begin{itemize}
\item The minimum number of samples per pair should be a few hundred;
\item The variables should have values in $\RN^d$ for some $d=1,2,3,\dots$;
\item There should be a significant cause--effect relationship between the
two variables;
\item The direction of the causal relationship should be known or obvious from
the meaning of the variables.
\end{itemize}
%The second criterion was motivated by the fact that the case of one discrete and one binary variable 
%allows the application of methods that have already been proposed earlier  
%\citep{ComleyDowe2003,SunJanzingSchoelkopf2006}; hence, we considered purely continuous data 
%as the more novel challenge.
Version \version\ of the {\tt CauseEffectPairs} collection consists of \nrpairs\ pairs satisfying
these criteria, taken from \nrdatasets\ different data sets from different domains. We refer to these
pairs as {\tt pair0001}, \dots, {\tt pair00\nrpairs}. Table~\ref{tab:CEP_pairs} gives an overview of
the cause-effect pairs. In the following subsections, we describe the cause-effect pairs in detail, 
and motivate our decisions on the causal relationships present in the pairs. We provide a scatter plot
for each pair, where the horizontal axis corresponds with the cause, and the vertical axis with the
effect. For completeness, we describe all the pairs in the data set, including those
that have been described before in \citep{MooijJanzing_JMLR_10}.

\bigskip
{\tiny\begin{longtable}{llllcl}
\caption{\label{tab:CEP_pairs} Overview of the pairs in version \version\ of the \CEP\ benchmark.}\\
%\begin{tabular}{llllll}
\hline
Pair & Variabele 1 & Variable 2 & Dataset & Ground Truth & Weight \\
\hline
{\tt pair0001} & Altitude & Temperature & \hyperlink{sec:D1}{D1} & $\rightarrow$ & 1/6 \\
{\tt pair0002} & Altitude & Precipitation & \hyperlink{sec:D1}{D1} & $\rightarrow$ & 1/6 \\
{\tt pair0003} & Longitude & Temperature & \hyperlink{sec:D1}{D1} & $\rightarrow$ & 1/6 \\
{\tt pair0004} & Altitude & Sunshine hours & \hyperlink{sec:D1}{D1} & $\rightarrow$ & 1/6 \\
\hline
{\tt pair0005} & Age & Length & \hyperlink{sec:D2}{D2} & $\rightarrow$ & 1/7 \\
{\tt pair0006} & Age & Shell weight & \hyperlink{sec:D2}{D2} & $\rightarrow$ & 1/7 \\
{\tt pair0007} & Age & Diameter & \hyperlink{sec:D2}{D2} & $\rightarrow$ & 1/7 \\
{\tt pair0008} & Age & Height & \hyperlink{sec:D2}{D2} & $\rightarrow$ & 1/7 \\
{\tt pair0009} & Age & Whole weight & \hyperlink{sec:D2}{D2} & $\rightarrow$ & 1/7 \\
{\tt pair0010} & Age & Shucked weight & \hyperlink{sec:D2}{D2} & $\rightarrow$ & 1/7 \\
{\tt pair0011} & Age & Viscera weight & \hyperlink{sec:D2}{D2} & $\rightarrow$ & 1/7 \\
\hline
{\tt pair0012} & Age & Wage per hour & \hyperlink{sec:D3}{D3} & $\rightarrow$ & 1/2 \\
\hline
{\tt pair0013} & Displacement & Fuel consumption & \hyperlink{sec:D4}{D4} & $\rightarrow$ & 1/4 \\
{\tt pair0014} & Horse power & Fuel consumption & \hyperlink{sec:D4}{D4} & $\rightarrow$ & 1/4 \\
{\tt pair0015} & Weight & Fuel consumption & \hyperlink{sec:D4}{D4} & $\rightarrow$ & 1/4 \\
{\tt pair0016} & Horsepower & Acceleration & \hyperlink{sec:D4}{D4} & $\rightarrow$ & 1/4 \\
\hline
{\tt pair0017} & Age & Dividends from stocks & \hyperlink{sec:D3}{D3} & $\rightarrow$ & 1/2 \\
\hline
{\tt pair0018} & Age & Concentration GAG & \hyperlink{sec:D5}{D5} & $\rightarrow$ & 1 \\
\hline
{\tt pair0019} & Current duration & Next interval & \hyperlink{sec:D6}{D6} & $\rightarrow$ & 1 \\
\hline
{\tt pair0020} & Latitude & Temperature & \hyperlink{sec:D1}{D1} & $\rightarrow$ & 1/6 \\
{\tt pair0021} & Longitude & Precipitation & \hyperlink{sec:D1}{D1} & $\rightarrow$ & 1/6 \\
\hline
{\tt pair0022} & Age & Height & \hyperlink{sec:D7}{D7} & $\rightarrow$ & 1/3 \\
{\tt pair0023} & Age & Weight & \hyperlink{sec:D7}{D7} & $\rightarrow$ & 1/3 \\
{\tt pair0024} & Age & Heart rate & \hyperlink{sec:D7}{D7} & $\rightarrow$ & 1/3 \\
\hline
{\tt pair0025} & Cement & Compressive strength & \hyperlink{sec:D8}{D8} & $\rightarrow$ & 1/8 \\
{\tt pair0026} & Blast furnace slag & Compressive strength & \hyperlink{sec:D8}{D8} & $\rightarrow$ & 1/8 \\
{\tt pair0027} & Fly ash & Compressive strength & \hyperlink{sec:D8}{D8} & $\rightarrow$ & 1/8 \\
{\tt pair0028} & Water & Compressive strength & \hyperlink{sec:D8}{D8} & $\rightarrow$ & 1/8 \\
{\tt pair0029} & Superplasticizer & Compressive strength & \hyperlink{sec:D8}{D8} & $\rightarrow$ & 1/8 \\
{\tt pair0030} & Coarse aggregate & Compressive strength & \hyperlink{sec:D8}{D8} & $\rightarrow$ & 1/8 \\
{\tt pair0031} & Fine aggregate & Compressive strength & \hyperlink{sec:D8}{D8} & $\rightarrow$ & 1/8 \\
{\tt pair0032} & Age & Compressive strength & \hyperlink{sec:D8}{D8} & $\rightarrow$ & 1/8 \\
\hline
{\tt pair0033} & Alcohol consumption & Mean corpuscular volume & \hyperlink{sec:D9}{D9} & $\rightarrow$ & 1/5 \\
{\tt pair0034} & Alcohol consumption & Alkaline phosphotase & \hyperlink{sec:D9}{D9} & $\rightarrow$ & 1/5 \\
{\tt pair0035} & Alcohol consumption & Alanine aminotransferase & \hyperlink{sec:D9}{D9} & $\rightarrow$ & 1/5 \\
{\tt pair0036} & Alcohol consumption & Aspartate aminotransferase & \hyperlink{sec:D9}{D9} & $\rightarrow$ & 1/5 \\
{\tt pair0037} & Alcohol consumption & Gamma-glutamyl transpeptidase & \hyperlink{sec:D9}{D9} & $\rightarrow$ & 1/5 \\
\hline
{\tt pair0038} & Age & Body mass index & \hyperlink{sec:D10}{D10} & $\rightarrow$ & 1/4 \\
{\tt pair0039} & Age & Serum insulin & \hyperlink{sec:D10}{D10} & $\rightarrow$ & 1/4 \\
{\tt pair0040} & Age & Diastolic blood pressure & \hyperlink{sec:D10}{D10} & $\rightarrow$ & 1/4 \\
{\tt pair0041} & Age & Plasma glucose concentration & \hyperlink{sec:D10}{D10} & $\rightarrow$ & 1/4 \\
\hline
{\tt pair0042} & Day of the year & Temperature & \hyperlink{sec:D11}{D11} & $\rightarrow$ & 1/2 \\
\hline
{\tt pair0043} & Temperature at $t$ & Temperature  at $t+1$ & \hyperlink{sec:D12}{D12} & $\rightarrow$ & 1/4 \\
{\tt pair0044} & Surface pressure  at $t$ & Surface pressure  at $t+1$& \hyperlink{sec:D12}{D12} & $\rightarrow$ & 1/4 \\
{\tt pair0045} & Sea level pressure  at $t$ & Sea level pressure at $t+1$ & \hyperlink{sec:D12}{D12} & $\rightarrow$ & 1/4 \\
{\tt pair0046} & Relative humidity  at $t$ & Relative humidity at $t+1$ & \hyperlink{sec:D12}{D12} & $\rightarrow$ & 1/4 \\
\hline
{\tt pair0047} & Number of cars & Type of day & \hyperlink{sec:D13}{D13} & $\leftarrow$ & 1 \\
\hline
{\tt pair0048} & Indoor temperature & Outdoor temperature & \hyperlink{sec:D14}{D14} & $\leftarrow$ & 1 \\
\hline
{\tt pair0049} & Ozone concentration & Temperature & \hyperlink{sec:D15}{D15} & $\leftarrow$ & 1/3 \\
{\tt pair0050} & Ozone concentration & Temperature & \hyperlink{sec:D15}{D15} & $\leftarrow$ & 1/3 \\
{\tt pair0051} & Ozone concentration & Temperature & \hyperlink{sec:D15}{D15} & $\leftarrow$ & 1/3 \\
\hline
{\tt pair0052} & (Temp, Press, SLP, Rh) & (Temp, Press, SLP, RH) & \hyperlink{sec:D12}{D12} & $\leftarrow$ & 0 \\
\hline
{\tt pair0053} & Ozone concentration & (Wind speed, Radiation, Temperature) & \hyperlink{sec:D16}{D16} & $\leftarrow$ & 0 \\
\hline
{\tt pair0054} & (Displacement, Horsepower, Weight) & (Fuel consumption, Acceleration) & \hyperlink{sec:D4}{D4} & $\rightarrow$ & 0 \\
\hline
{\tt pair0055} & Ozone concentration (16-dim.) & Radiation (16-dim.) & \hyperlink{sec:D15}{D15} & $\leftarrow$ & 0 \\
\hline
{\tt pair0056} & Female life expectancy, 2000--2005 & Latitude of capital & \hyperlink{sec:D17}{D17} & $\leftarrow$ & 1/12 \\
{\tt pair0057} & Female life expectancy, 1995--2000 & Latitude of capital & \hyperlink{sec:D17}{D17} & $\leftarrow$ & 1/12 \\
{\tt pair0058} & Female life expectancy, 1990--1995 & Latitude of capital & \hyperlink{sec:D17}{D17} & $\leftarrow$ & 1/12 \\
{\tt pair0059} & Female life expectancy, 1985--1990 & Latitude of capital & \hyperlink{sec:D17}{D17} & $\leftarrow$ & 1/12 \\
{\tt pair0060} & Male life expectancy, 2000--2005 & Latitude of capital & \hyperlink{sec:D17}{D17} & $\leftarrow$ & 1/12 \\
{\tt pair0061} & Male life expectancy, 1995--2000 & Latitude of capital & \hyperlink{sec:D17}{D17} & $\leftarrow$ & 1/12 \\
{\tt pair0062} & Male life expectancy, 1990--1995 & Latitude of capital & \hyperlink{sec:D17}{D17} & $\leftarrow$ & 1/12 \\
{\tt pair0063} & Male life expectancy, 1985--1990 & Latitude of capital & \hyperlink{sec:D17}{D17} & $\leftarrow$ & 1/12 \\
{\tt pair0064} & Drinking water access & Infant mortality & \hyperlink{sec:D17}{D17} & $\rightarrow$ & 1/12 \\
\hline
{\tt pair0065} & Stock return of Hang Seng Bank & Stock return of HSBC Hldgs & \hyperlink{sec:D18}{D18} & $\rightarrow$ & 1/3 \\
{\tt pair0066} & Stock return of Hutchison & Stock return of Cheung kong & \hyperlink{sec:D18}{D18} & $\rightarrow$ & 1/3 \\
{\tt pair0067} & Stock return of Cheung kong & Stock return of Sun Hung Kai Prop.\ & \hyperlink{sec:D18}{D18} & $\rightarrow$ & 1/3 \\
\hline
{\tt pair0068} & Bytes sent & Open http connections & \hyperlink{sec:D19}{D19} & $\leftarrow$ & 1 \\
\hline
{\tt pair0069} & Inside temperature & Outside temperature & \hyperlink{sec:D20}{D20} & $\leftarrow$ & 1 \\
\hline
{\tt pair0070} & Parameter & Answer & \hyperlink{sec:D21}{D21} & $\rightarrow$ & 1 \\
\hline
{\tt pair0071} & Symptoms (6-dim.) & Classification of disease (2-dim.) & \hyperlink{sec:D22}{D22} & $\rightarrow$ & 0 \\
\hline
{\tt pair0072} & Sunspots & Global mean temperature & \hyperlink{sec:D23}{D23} & $\rightarrow$ & 1 \\
\hline
{\tt pair0073} & CO$_2$ emissions & Energy use & \hyperlink{sec:D17}{D17} & $\leftarrow$ & 1/12 \\
{\tt pair0074} & GNI per capita & Life expectancy & \hyperlink{sec:D17}{D17} & $\rightarrow$ & 1/12 \\
{\tt pair0075} & Under-5 mortality rate & GNI per capita & \hyperlink{sec:D17}{D17} & $\leftarrow$ & 1/12 \\
\hline
{\tt pair0076} & Population growth & Food consumption growth & \hyperlink{sec:D24}{D24} & $\rightarrow$ & 1 \\
\hline
{\tt pair0077} & Temperature & Solar radiation & \hyperlink{sec:D11}{D11} & $\leftarrow$ & 1/2 \\
\hline
{\tt pair0078} & PPFD & Net Ecosystem Productivity & \hyperlink{sec:D25}{D25} & $\rightarrow$ & 1/3 \\
{\tt pair0079} & Net Ecosystem Productivity & Diffuse PPFD & \hyperlink{sec:D25}{D25} & $\leftarrow$ & 1/3 \\
{\tt pair0080} & Net Ecosystem Productivity & Direct PPFD & \hyperlink{sec:D25}{D25} & $\leftarrow$ & 1/3 \\
\hline
{\tt pair0081} & Temperature & Local CO$_2$ flux, BE-Bra & \hyperlink{sec:D26}{D26} & $\rightarrow$ & 1/3 \\
{\tt pair0082} & Temperature & Local CO$_2$ flux, DE-Har & \hyperlink{sec:D26}{D26} & $\rightarrow$ & 1/3 \\
{\tt pair0083} & Temperature & Local CO$_2$ flux, US-PFa & \hyperlink{sec:D26}{D26} & $\rightarrow$ & 1/3 \\
\hline
{\tt pair0084} & Employment & Population & \hyperlink{sec:D27}{D27} & $\leftarrow$ & 1 \\
\hline
{\tt pair0085} & Time of measurement & Protein content of milk & \hyperlink{sec:D28}{D28} & $\rightarrow$ & 1 \\
\hline
{\tt pair0086} & Size of apartment & Monthly rent & \hyperlink{sec:D29}{D29} & $\rightarrow$ & 1 \\
\hline
{\tt pair0087} & Temperature & Total snow & \hyperlink{sec:D30}{D30} & $\rightarrow$ & 1 \\
\hline
{\tt pair0088} & Age & Relative spinal bone mineral density & \hyperlink{sec:D31}{D31} & $\rightarrow$ & 1 \\
\hline
{\tt pair0089} & Root decomposition in Oct & Root decomposition in Apr & \hyperlink{sec:D32}{D32} & $\leftarrow$ & 1/4 \\
{\tt pair0090} & Root decomposition in Oct & Root decomposition in Apr & \hyperlink{sec:D32}{D32} & $\leftarrow$ & 1/4 \\
{\tt pair0091} & Clay content in soil & soil moisture & \hyperlink{sec:D32}{D32} & $\rightarrow$ & 1/4 \\
{\tt pair0092} & Organic carbon in soil & Clay content in soil & \hyperlink{sec:D32}{D32} & $\leftarrow$ & 1/4 \\
\hline
{\tt pair0093} & Precipitation & Runoff & \hyperlink{sec:D33}{D33} & $\rightarrow$ & 1 \\
\hline
{\tt pair0094} & Hour of the day & Temperature & \hyperlink{sec:D34}{D34} & $\rightarrow$ & 1/3 \\
{\tt pair0095} & Hour of the day & Electricity consumption& \hyperlink{sec:D34}{D34} & $\rightarrow$ & 1/3 \\
{\tt pair0096} & Temperature & Electricity consumption& \hyperlink{sec:D34}{D34} & $\rightarrow$ & 1/3 \\
\hline
{\tt pair0097} & Initial speed & Final speed & \hyperlink{sec:D35}{D35} & $\rightarrow$ & 1/2 \\
{\tt pair0098} & Initial speed & Final speed & \hyperlink{sec:D35}{D35} & $\rightarrow$ & 1/2 \\
\hline
{\tt pair0099} & Language test score & Social-economic status family & \hyperlink{sec:D36}{D36} & $\leftarrow$ & 1 \\
\hline
{\tt pair0100} & Cycle time of CPU & Performance & \hyperlink{sec:D37}{D37} & $\rightarrow$ & 1 \\
%\end{tabular}
\end{longtable}}

\hypertarget{sec:D1}{\subsection*{D1: DWD}}

%Source: \url{http://www.dwd.de} \\
%Citations: \citep{DWD} \\
%Pairs: \{1, 2, 3, 4, 20, 21\} \\
%Samples: 349 \\

The DWD climate data were provided by the Deutscher Wetterdienst (DWD). We downloaded
the data from \url{http://www.dwd.de} and merged several of the original data
sets to obtain data for 349 weather stations in Germany, selecting only those
weather stations without missing data. After merging the data sets, we
selected the following six variables: altitude, latitude, longitude, and annual
mean values (over the years 1961--1990) of sunshine duration, temperature and
precipitation. We converted the latitude and longitude variables from
sexagesimal to decimal notation. Out of these six variables, we selected six
different pairs with ``obvious'' causal relationships: altitude--temperature ({\tt pair0001}),
altitude--precipitation ({\tt pair0002}), longitude--temperature ({\tt pair0003}), altitude--sunshine hours ({\tt pair0004}),
latitute--temperature ({\tt pair0020}), and longitude--precipitation ({\tt pair0021}).

\begin{figure}[h!]
\centerline{\small\begin{tabular}{ccc}
  \includegraphics[width=0.2\textwidth]{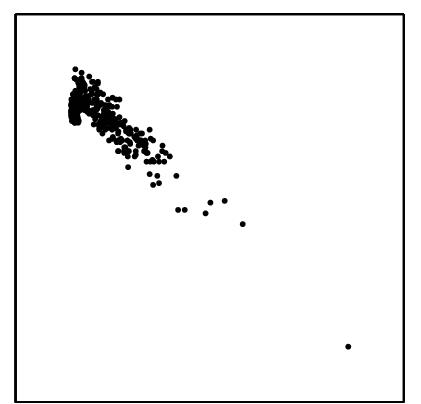}&
  \includegraphics[width=0.2\textwidth]{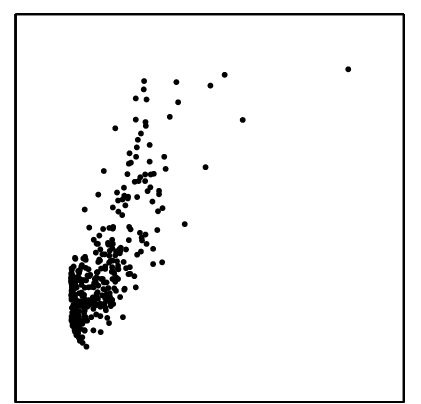}&
  \includegraphics[width=0.2\textwidth]{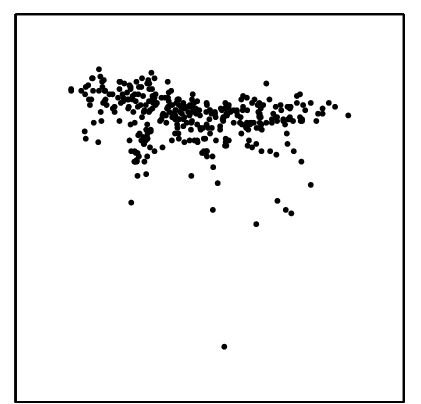}\\
  {\tt pair0001} & {\tt pair0002} & {\tt pair0003}\\
  \\
  \includegraphics[width=0.2\textwidth]{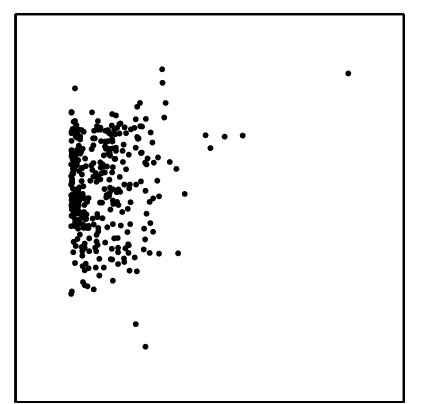}&
  \includegraphics[width=0.2\textwidth]{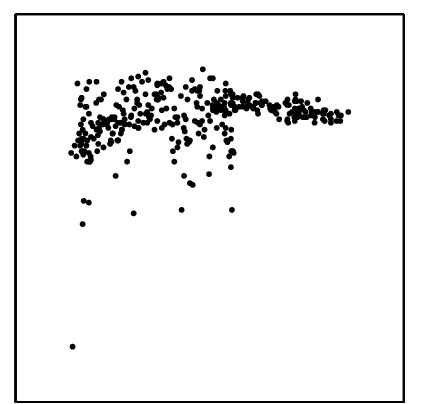}&
  \includegraphics[width=0.2\textwidth]{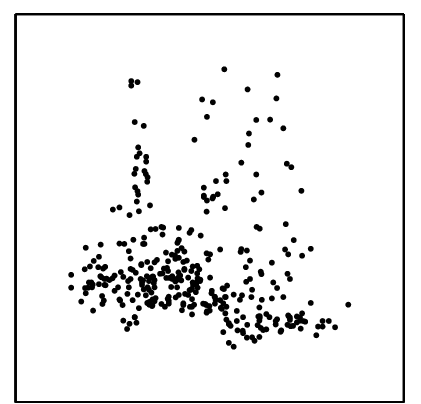}\\
  {\tt pair0004} & {\tt pair0020} & {\tt pair0021}\\
\end{tabular}}
\caption{\label{fig:D1}Scatter plots of pairs from D1.
{\tt pair0001}: altitude \causes temperature,
{\tt pair0002}: altitude \causes precipitation,
{\tt pair0003}: longitude \causes temperature,
{\tt pair0004}: altitude \causes sunshine hours,
{\tt pair0020}: latitude \causes temperature,
{\tt pair0021}: longitude \causes precipitation.}
\end{figure}

\subsubsection*{{\tt pair0001}: Altitude \causes temperature}

As an elementary fact of meteorology, places with higher altitude tend to be
colder than those that are closer to sea level (roughly 1 centigrade per 100
meter). There is no doubt that altitude is the cause and
temperature the effect: one could easily think of an intervention where the
thermometer is lifted (e.g., by using a balloon) to measure the temperature at
a higher point of the same longitude and latitude. On the other hand, heating
or cooling a location usually does not change its altitude (except perhaps if the 
location happens to be the space enclosed by a hot air balloon, but let us assume
that the thermometers used to gather this data were fixed to the ground).
The altitudes in the DWD data set range from 0\,m to 2960\,m, which is 
sufficiently large to detect significant statistical dependences.

One potential confounder is latitude, since all mountains are
in the south and far from the sea, which is also an important factor for the
local climate. The places with the highest average temperatures are therefore
those with low altitude but lying far in the south. Hence this
confounder should induce positive correlations between altitude and temperature
as opposed to the negative correlation between altitude and temperature that is
observed empirically. This suggests that the direct causal
relation between altitude and temperature dominates over the confounder. 

\subsubsection*{{\tt pair0002}: Altitude \causes precipitation}

It is known that altitude is also an important factor for precipitation since
rain often occurs when air is forced to rise over a mountain range and the air
becomes oversaturated with water due to the lower temperature (orographic
rainfall). This effect defines an indirect causal influence of altitude on
precipitation via temperature. These causal relations are, however, less simple
than the causal influence from altitude to temperature because gradients of the
altitude with respect to the main direction of the wind are more relevant than
the altitude itself. A hypothetical intervention that would allow us to validate
the causal relation could be to build artificial mountains and observe orographic
rainfall.

\subsubsection*{{\tt pair0003}: Longitude \causes temperature}

To detect the causal relation between longitude and temperature, a hypothetical
intervention could be to move a thermometer between West and East. Even if one
would adjust for altitude and latitude, it is unlikely that temperature would
remain the same since the climate in the West is more oceanic and less
continental than in the East of Germany. Therefore, longitude causes
temperature.

\subsubsection*{{\tt pair0004}: Altitude \causes sunshine hours}

Sunshine duration and altitude are slightly positively correlated. 
Possible explanations are that higher weather stations are sometimes
above low-hanging clouds.  Cities in valleys,
especially if they are close to rivers or lakes,  typically have more misty
days. Moving a sunshine sensor above the  clouds clearly increases the sunshine
duration whereas installing an artificial sun  would not change the altitude.
The causal influence from altitude to sunshine  duration can be confounded, for
instance, by the fact that there is a simple statistical dependence between
altitude and longitude in Germany as explained earlier.

\subsubsection*{{\tt pair0020}: Latitude \causes temperature}

Moving a thermometer towards the equator will generally result in an increased
mean annual temperature. Changing the temperature, on the other hand,
does not necessarily result in a north-south movement of the
thermometer. The obvious ground truth of latitude causing temperature might be 
somewhat ``confounded'' by longitude, in combination with the selection bias that arises
from only including weather stations in Germany.

\subsubsection*{{\tt pair0021}: Longitude \causes precipitation}

As the climate in the West is more oceanic and less continental than in the
East of Germany, we expect there to be a relationship between longitude and
precipitation. Changing longitude by moving in East-West direction may therefore
change precipitation, even if one would adjust for altitude and latitude.
On the other hand, making it rain locally
(e.g., by cloud seeding) will not result in a change in longitude.

\hypertarget{sec:D2}{\subsection*{D2: Abalone}}
  
%Source: \url{https://archive.ics.uci.edu/ml/datasets/Abalone} \\
%Citations: \citep{UCI_ML_repository,NashSellersTalbotCawthornFord94} \\
%Pairs: 5,6,7,8,9,10,11 \\
%Samples: 4177

The {\tt Abalone} data set \citep{NashSellersTalbotCawthornFord94}
in the UCI Machine Learning Repository \citep{UCI_ML_repository}
contains 4177 measurements of several variables concerning the sea
snail \emph{Abalone}. We downloaded the data from \url{https://archive.ics.uci.edu/ml/datasets/Abalone}.
The original data set contains the nine variables sex, length, diameter,
height, whole weight, shucked weight, viscera weight, shell weight and number
of rings. The number of rings in the shell is directly related to the age of
the snail: adding 1.5 to the number of rings gives the age in years.
Of these variables, we selected six pairs with obvious cause-effect 
relationships: age--length ({\tt pair0005}), age--shell weight ({\tt pair0006}),
age--diameter ({\tt pair0007}), age--height ({\tt pair0008}), age--whole weight ({\tt pair0009}),
age--shucked weight ({\tt pair0010}), age--viscera weight ({\tt pair0011}).

\begin{figure}[h!]
\centerline{\small\begin{tabular}{cccc}
  \includegraphics[width=0.2\textwidth]{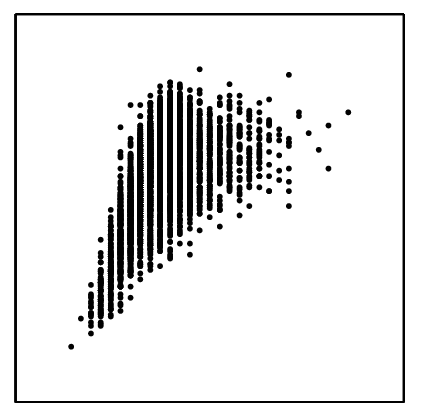}&
  \includegraphics[width=0.2\textwidth]{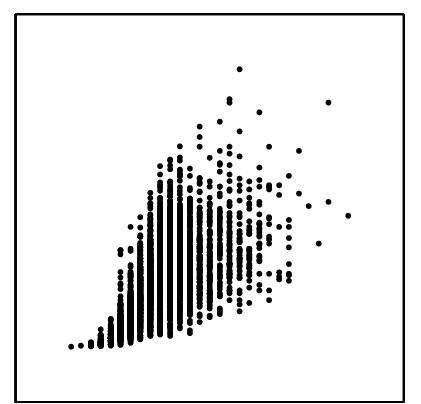}&
  \includegraphics[width=0.2\textwidth]{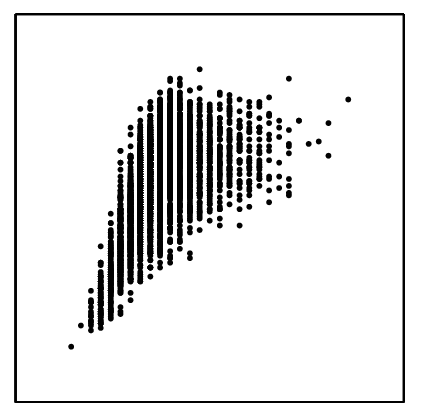}&
  \includegraphics[width=0.2\textwidth]{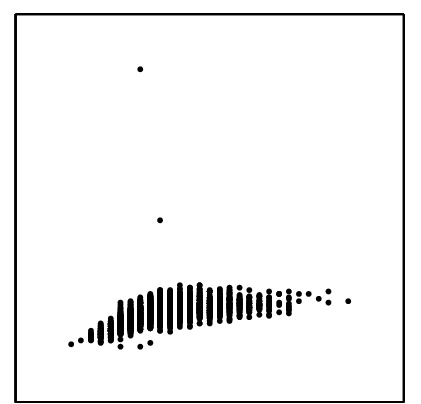}\\
  {\tt pair0005} & {\tt pair0006} & {\tt pair0007} & {\tt pair0008}\\
  \\
  \includegraphics[width=0.2\textwidth]{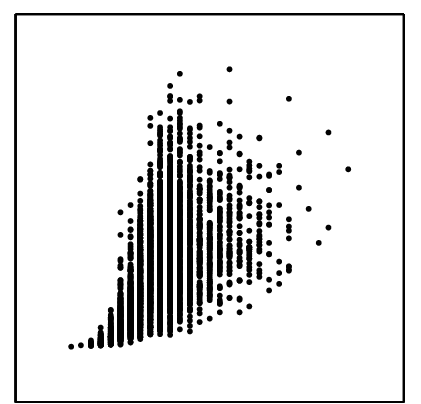}&
  \includegraphics[width=0.2\textwidth]{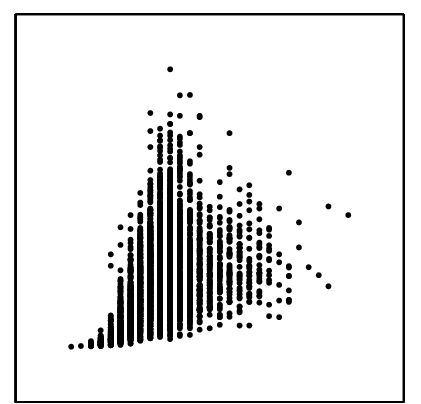}&
  \includegraphics[width=0.2\textwidth]{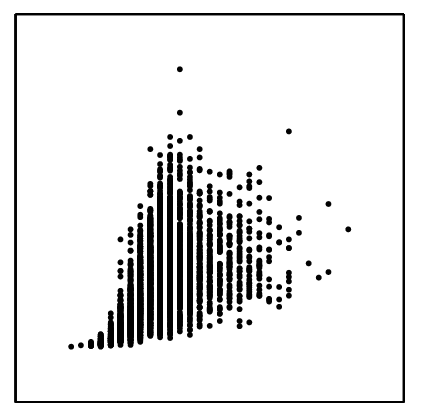}\\
  {\tt pair0009} & {\tt pair0010} & {\tt pair0011}\\
\end{tabular}}
\caption{\label{fig:D2}Scatter plots of pairs from D2.
{\tt pair0005}: age \causes length, {\tt pair0006}: age \causes shell weight,
{\tt pair0007}: age \causes diameter, {\tt pair0008}: age \causes height,
{\tt pair0009}: age \causes whole weight, {\tt pair0010}: age \causes shucked weight, {\tt pair0011}: age \causes viscera weight.}
\end{figure}

\subsubsection*{{\tt pair0005}--{\tt pair0011}: Age \causes \{length, shell weight, diameter, height, whole/shucked/viscera weight\}}

For the variable ``age'' it is not obvious what a reasonable intervention would
be since there is no possibility to change the time. However, waiting and
observing how variables change over time can be considered as equivalent to the hypothetical intervention on age
(provided that the relevant background conditions do not change too much).
Clearly, this ``intervention'' would change the probability distribution of
the length, whereas changing the length of snails (by surgery) would not change
the distribution of age (assuming that the surgery does not take years). Regardless of the difficulties of defining
interventions, we expect common agreement on the ground truth: age causes all the
other variables related to length, diameter height and weight. 

There is one subtlety that has to do with how age is measured for these
shells: this is done by counting the rings. For the variable
``number of rings'' however, changing the length of the snail may actually
change the number of rings. We here presume that all snails have undergone
their natural growing process so that the number of rings is a good proxy
for the variable age.

\hypertarget{sec:D3}{\subsection*{D3: Census Income KDD}}

%Source: \url{https://archive.ics.uci.edu/ml/datasets/Census-Income+(KDD)} \\
%Citations: \citep{UCI_ML_repository,Census} \\
%Pairs: 12,17 \\
%Samples: 5000 \\

The {\tt Census Income (KDD)} data set \citep{Census} in the UCI Machine
Learning Repository \citep{UCI_ML_repository} has been extracted from 
the 1984 and 1985 U.S.\ Census studies. We downloaded the data from \url{https://archive.ics.uci.edu/ml/datasets/Census-Income+(KDD)}.
We have selected the following variables: {\tt AAGE} (age), {\tt AHRSPAY} (wage per hour)
and {\tt DIVVAL} (dividends from stocks). 

\begin{figure}[h!]
\centerline{\small\begin{tabular}{cc}
  \includegraphics[width=0.2\textwidth]{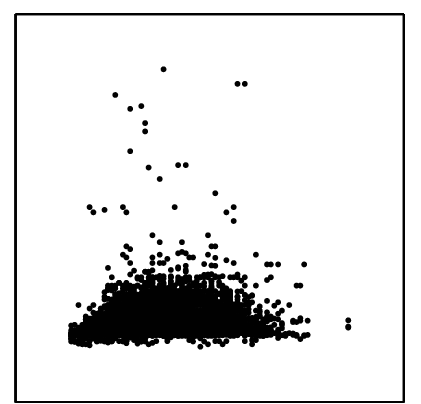}&
  \includegraphics[width=0.2\textwidth]{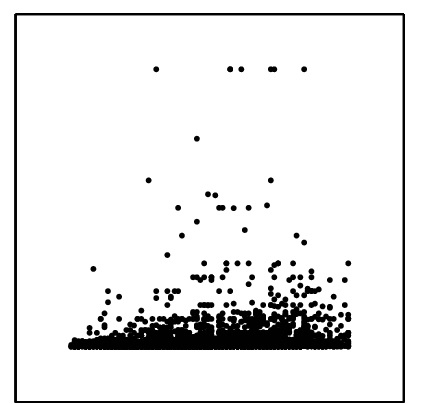}\\
  {\tt pair0012} & {\tt pair0017} 
\end{tabular}}
\caption{\label{fig:D3}Scatter plots of pairs from D3.
{\tt pair0012}: age \causes wage per hour, {\tt pair0017}: age \causes dividends from stocks.}
\end{figure}

\subsubsection*{{\tt pair0012}: Age \causes wage per hour}

We only used the first 5000 instances for which wage per hour was not equal to zero.
The data clearly shows an increase of wage up to about 45 and a decrease for higher
age.

As already argued for the {\tt Abalone} data, interventions on the variable
``age'' are difficult to define. 
Compared to the discussion in the context of the {\tt Abalone} data set, it seems more problematic to consider
waiting as a reasonable ``intervention'' here, since the relevant (economical) background conditions 
change rapidly compared  to the length of the human life: If someone's salary is higher than the salary of a 20 year younger colleague {\it because} of his/her longer job experience, we cannot conclude that  the younger colleague will earn the same 
money 20 years later as the older colleague earns now. Possibly, the factory or even  the branch of industry he/she was working in does
not exist any more and his/her job experience is no longer appreciated.
However, we know that employees sometimes indeed do get a higher income because of their longer job experience.
Pretending longer job experience by 
a fake certificate of employment would be a possible
intervention. 
On the  other hand, changing the wage per hour is an intervention that is easy
to imagine (though difficult for us to perform) and this  would certainly not
change the age. 

\subsubsection*{{\tt pair0017}: Age \causes dividends from stocks}

We only used the first 5000 instances for which dividends from stocks was not
equal to zero.  Similar considerations apply as for age vs.\ wage per hour.
Doing an intervention on age is not practical, but companies could
theoretically intervene on the dividends from stocks, and that would not result
in a change of age, obviously. On the other hand, age influences income, and
thereby over time, the amount of money that people can invest in stocks, and
thereby, the amount of dividends they earn from stocks. This causal relation is
a very indirect one, though, and the dependence between age and dividends from
stock is less pronounced than that between age and wage per hour.

\hypertarget{sec:D4}{\subsection*{D4: Auto-MPG}}
  
%Source: \url{http://archive.ics.uci.edu/ml/datasets/Auto+MPG} \\
%Citations: \citep{UCI_ML_repository,StatLib} \\
%Pairs: 13,14,15,16,54 \\
%Samples: 392 \\

The {\tt Auto-MPG} data set in the UCI Machine Learning Repository \citep{UCI_ML_repository} 
concerns city-cycle fuel consumption in miles per gallon
(MPG), i.e., the number of miles a car can drive on one gallon of gasoline,
and contains several other attributes, like displacement,
horsepower, weight, and acceleration. The original dataset comes from the StatLib
library \citep{StatLib} and was used in the 1983 American Statistical Association 
Exposition. We downloaded the data from \url{http://archive.ics.uci.edu/ml/datasets/Auto+MPG}
and selected only instances without missing data, thereby obtaining 392 samples.

\begin{figure}[h!]
\centerline{\small\begin{tabular}{cccc}
  \includegraphics[width=0.2\textwidth]{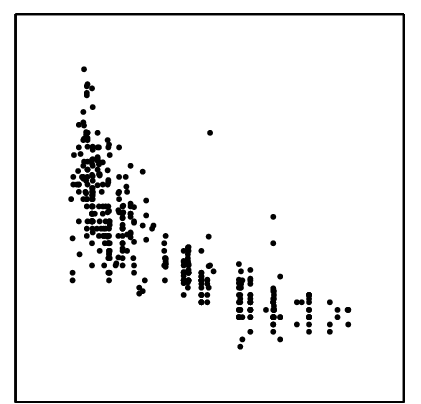}&
  \includegraphics[width=0.2\textwidth]{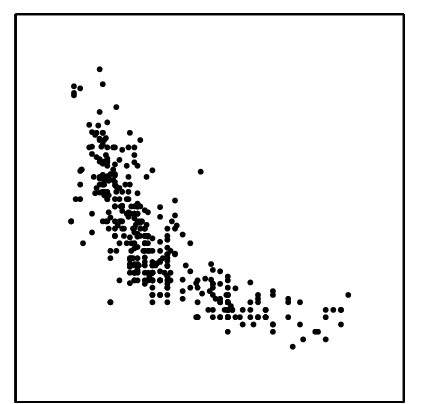}&
  \includegraphics[width=0.2\textwidth]{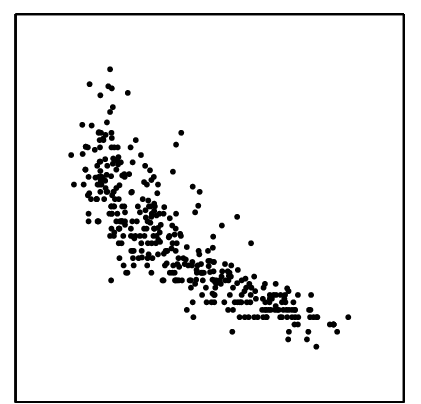}&
  \includegraphics[width=0.2\textwidth]{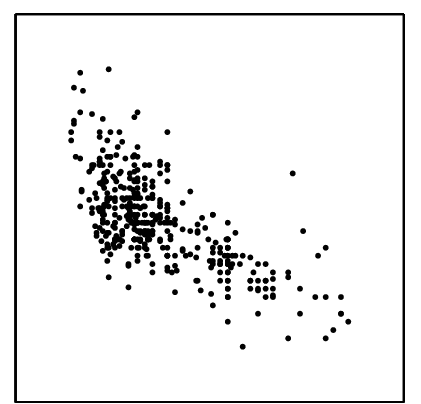}\\
  {\tt pair0013} & {\tt pair0014} & {\tt pair0015} & {\tt pair0016}
\end{tabular}}
\caption{\label{fig:D4}Scatter plots of pairs from D4.
{\tt pair0013}: displacement \causes fuel consumption,
{\tt pair0014}: horse power \causes fuel consumption,
{\tt pair0015}: weight \causes fuel consumption,
{\tt pair0016}: horsepower \causes acceleration,
{\tt pair0054}: (displacement,horsepower,weight) \causes (MPG,acceleration)}
\end{figure}

\subsubsection*{{\tt pair0013}: Displacement \causes fuel consumption}
Displacement is the total volume of air/fuel mixture an engine can draw in during 
one complete engine cycle. %(see also \url{http://en.wikipedia.org/wiki/Engine_displacement})
The larger the discplacement, the more fuel the engine can consume with every turn.
Intervening on displacement (e.g., by increasing the cylinder bore) changes the fuel consumption. 
Changing the fuel consumption (e.g., by increasing the weight of the car, or changing
its air resistance, or by using another gear) will not change the displacement, though.

\subsubsection*{{\tt pair0014}: Horse power \causes fuel consumption}
Horse power measures the amount of power an engine can deliver. There are various ways to define horsepower and different standards to measure horse power of vehicles.
%From \url{http://en.wikipedia.org/wiki/Horse_power}: ``The power of an engine may be measured or estimated at several points in the transmission of the power from its generation to its application. A number of names are used for the power developed at various stages in this process, but none is a clear indicator of either the measurement system or definition used.''
In general, though, it should be obvious that fuel consumption depends on various factors, including horse power. Changing horsepower (e.g., by adding more cylinders to an engine, or adding a second engine to the car) would lead to a change in fuel consumption. On the other
hand, changing fuel consumption does not necessarily change horse power.

\subsubsection*{{\tt pair0015}: Weight \causes fuel consumption}
There is a strong selection bias here, as car designers use a more powerful motor (with higher fuel consumption) for a heavier car.
Nevertheless, the causal relationship between weight and fuel consumption should be obvious: if we intervene on weight, then fuel consumption will change, but not necessarily vice versa.

\subsubsection*{{\tt pair0016}: Horsepower \causes acceleration}
Horsepower is one one of the factors that cause acceleration. Other factors are wheel size, the gear used, and air resistance.
%However, note that when a car is designed, horsepower is chosen with the goal of being able to achieve a certain maximum acceleration,
%given for example the weight of a car. Indeed, it does not make sense to put a small engine into a big truck.
Intervening on acceleration does not necessarily change horsepower.

\subsubsection*{{\tt pair0054}: (Displacement,horsepower,weight) \causes (MPG,acceleration)}
This pair consists of two multivariate variables that are combinations of the variables we have considered before.
The multivariate variable consisting of the three components displacement, horsepower and weight can be considered to cause the multivariate variable comprised of fuel consumption and acceleration.

\hypertarget{sec:D5}{\subsection*{D5: GAGurine}}

%Source: R MASS package \\
%Citations: \citep{VenablesRipley2002} \\
%Pairs: 18 \\
%Samples: 314 \\

This data concerns the concentration of the chemical compound Glycosaminoglycan (GAG) in the urine of 314 children aged from zero to seventeen years. This is the \texttt{GAGurine} data set supplied with the \texttt{MASS} package of the computing language \texttt{R} \citep{VenablesRipley2002}.

\begin{figure}[h!]
\centerline{\small\begin{tabular}{c}
  \includegraphics[width=0.2\textwidth]{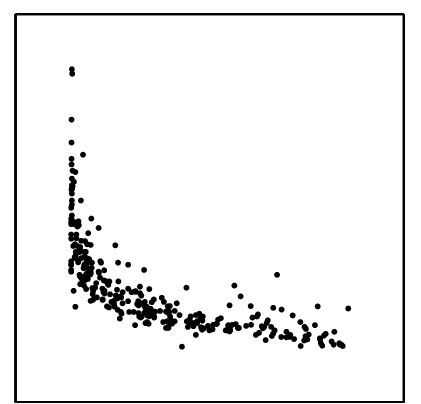}\\
  {\tt pair0018}
\end{tabular}}
\caption{\label{fig:D5}Scatter plots of pairs from D5.
{\tt pair0018}: age \causes concentration GAG.}
\end{figure}

\subsubsection*{{\tt pair0018}: Age \causes concentration GAG}

Obviously, GAG concentration does not cause age, but it could be the other way around, considering the strong dependence between the two variables.

\hypertarget{sec:D6}{\subsection*{D6: Old Faithful}}

%Source: R MASS package \\
%Citations: \citep{VenablesRipley2002,AzzaliniBowman1990} \\
%Pairs: 19 \\
%Samples: 194 \\

This is the \texttt{geyser}
data set supplied with the \texttt{MASS} package of the computing language \texttt{R} \citep{VenablesRipley2002}.
It is originally described in \citep{AzzaliniBowman1990} and contains
data about the duration of an eruption and the time interval between subsequent
eruptions of the Old Faithful geyser in Yellowstone National Park, USA. The data 
consists of 194 samples and was collected in a single continuous measurement from August 1 to August 15, 1985. 

\begin{figure}[h!]
\centerline{\small\begin{tabular}{c}
  \includegraphics[width=0.2\textwidth]{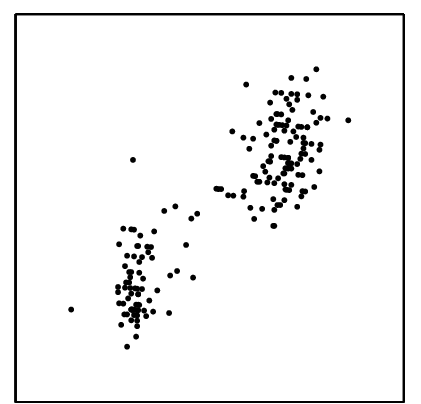}\\
  {\tt pair0019}
\end{tabular}}
\caption{\label{fig:D6}Scatter plots of pairs from D6.
{\tt pair0019}: current duration \causes next interval.}
\end{figure}

\subsubsection*{{\tt pair0019}: Current duration \causes next interval}

The chronological ordering of events implicates that the time interval between 
the current and the next eruption is an effect of the duration of the current 
eruption.

\hypertarget{sec:D7}{\subsection*{D7: Arrhythmia}}

%Source: \url{https://archive.ics.uci.edu/ml/datasets/Arrhythmia} \\
%Citations: \citep{UCI_ML_repository,GuvenirAcarDemirozCekin1997} \\
%Pairs: 22,23,24 \\
%Samples: 452 \\

%\citep{GuvenirAcarDemirozCekin1997}
The {\tt Arrhythmia} dataset \citep{GuvenirAcarDemirozCekin1997} from the UCI Machine Learning Repository \citep{UCI_ML_repository}
concerns cardiac arrhythmia. It consists of 452 patient records and contains many different variables. 
We downloaded the data from \url{https://archive.ics.uci.edu/ml/datasets/Arrhythmia} and
only used the variables for which the causal relationships should be evident.
We removed two instances from the dataset, corresponding with patient lengths of 680 and 780 cm, respectively.
%"It was used to test algorithms that distinguish between the presence and types of cardiac arrhythmia and classify it."

\begin{figure}[h!]
\centerline{\small\begin{tabular}{ccc}
  \includegraphics[width=0.2\textwidth]{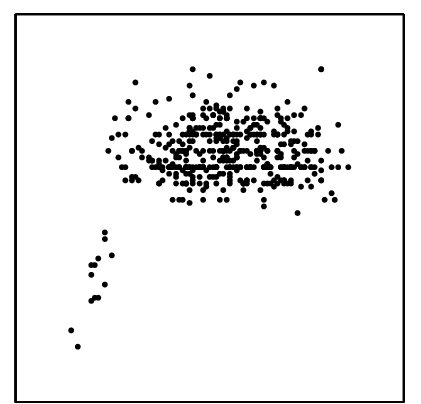}&
  \includegraphics[width=0.2\textwidth]{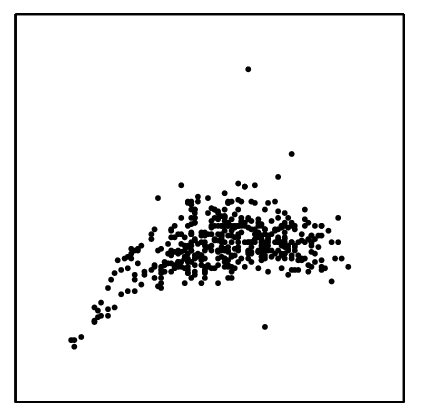}&
  \includegraphics[width=0.2\textwidth]{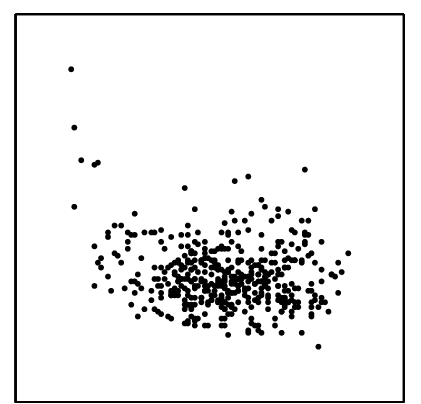}\\
  {\tt pair0022} & {\tt pair0023} & {\tt pair0024}
\end{tabular}}
\caption{\label{fig:D7}Scatter plots of pairs from D7.
{\tt pair0022}: age \causes height, 
{\tt pair0023}: age \causes weight,
{\tt pair0024}: age \causes heart rate.}
\end{figure}

\subsubsection*{{\tt pair0022}--{\tt pair0024}: Age \causes \{height, weight, heart rate\}}
As discussed before, ``interventions'' on age (for example, waiting a few years) may affect height of persons. On the other hand, we know that height does not cause age.
The same holds for age and weight and for age and heart rate. It is important to note here that age is simply measured in years since the birth of a person. Indeed, weight,
height and also heart rate might influence ``biological aging'', the gradual deterioration of function of the human body.

\hypertarget{sec:D8}{\subsection*{D8: Concrete Compressive Strength}}

%Source: \url{https://archive.ics.uci.edu/ml/datasets/Concrete+Compressive+Strength} \\
%Citations: \citep{UCI_ML_repository,Yeh1998} \\
%Pairs: 25,26,27,28,29,30,31,32 \\
%Samples: 1030 \\

This data set, available at the UCI Machine Learning Repository \citep{UCI_ML_repository},
concerns a systematic study \citep{Yeh1998} regarding concrete compressive strength as a function of ingredients and age. Citing \citep{Yeh1998}:
``High-performance concrete (HPC) is a new terminology used in the concrete construction industry. In addition to the three basic ingredients in conventional concrete, i.e., Portland cement, fine and coarse aggregates, and water, the making of HPC needs to incorporate supplementary cementitious materials, such as fly ash and blast furnace slag, and chemical admixture, such as superplasticizer 1 and 2. 
Several studies independently have shown that concrete strength development is determined not only by the water-to-cement ratio, but that it also is influenced by the content of other concrete ingredients.''
Compressive strength is measured in units of MPa, age in days, and the other variables are measured in kilograms per cubic metre of concrete mixture. 
The dataset was downloaded from \url{https://archive.ics.uci.edu/ml/datasets/Concrete+Compressive+Strength} and contains 1030 measurements.

\begin{figure}[h!]
\centerline{\small\begin{tabular}{cccc}
  \includegraphics[width=0.2\textwidth]{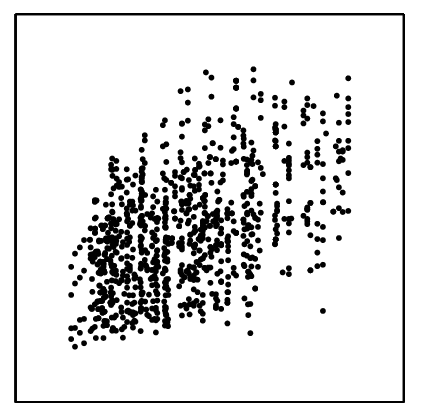}&
  \includegraphics[width=0.2\textwidth]{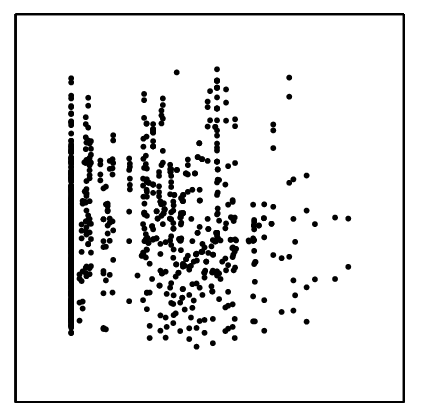}&
  \includegraphics[width=0.2\textwidth]{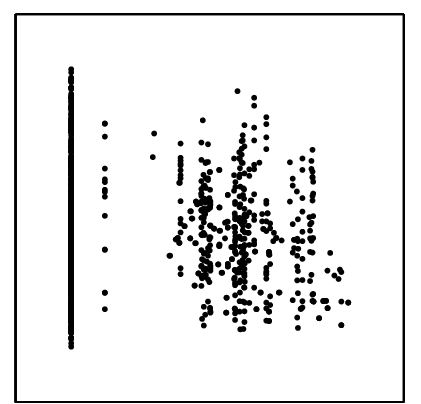}&
  \includegraphics[width=0.2\textwidth]{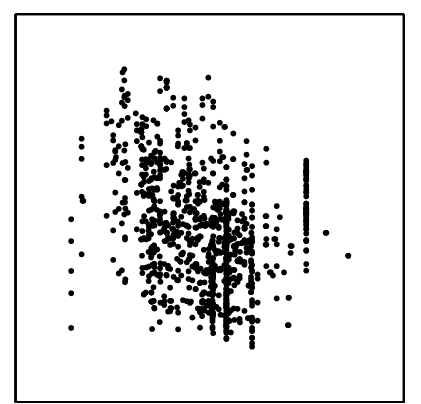}\\
  {\tt pair0025} & {\tt pair0026} & {\tt pair0027} & {\tt pair0028}\\
  \\
  \includegraphics[width=0.2\textwidth]{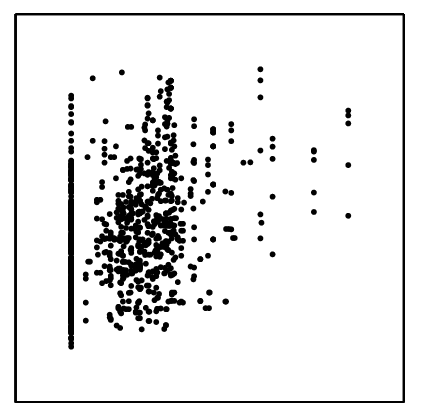}&
  \includegraphics[width=0.2\textwidth]{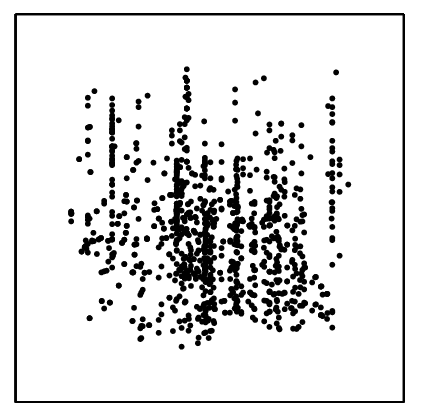}&
  \includegraphics[width=0.2\textwidth]{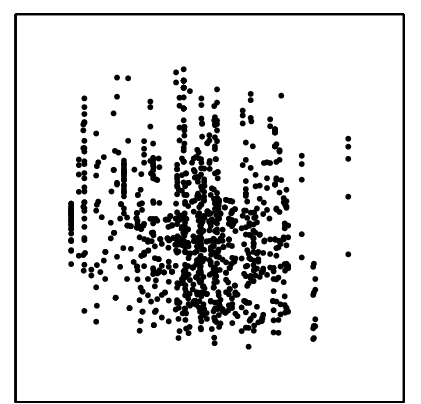}&
  \includegraphics[width=0.2\textwidth]{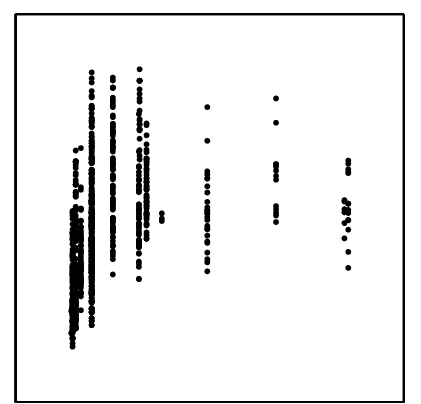}\\
  {\tt pair0029} & {\tt pair0030} & {\tt pair0031} & {\tt pair0032}
\end{tabular}}
\caption{\label{fig:D8}Scatter plots of pairs from D8.
{\tt pair0025}: cement  \causes compressive strength,
{\tt pair0026}: blast furnace slag  \causes compressive strength,
{\tt pair0027}: fly ash \causes compressive strength,
{\tt pair0028}: water \causes compressive strength,
{\tt pair0029}: superplasticizer \causes compressive strength,
{\tt pair0030}: coarse aggregate \causes compressive strength,
{\tt pair0031}: fine aggregate \causes compressive strength,
{\tt pair0032}: age \causes compressive strength.}
\end{figure}

\subsubsection*{{\tt pair0025}--{\tt pair0032}: \{Cement, blast furnace slag, fly ash, water, superplasticizer, coarse aggregate, fine aggregate, age\} \causes Compressive Strength}

It should be obvious that compressive strength is the effect, and the other variables are its causes.
Note, however, that in practice one cannot easily intervene on the mixture components without simultaneously changing the other mixture components. For example, if one adds more water to the mixture, then as a result, all other components will decrease, as they are measured in kilograms per cubic metre of concrete mixture. Nevertheless, we expect that we can see these interventions as reasonable approximations of ``perfect interventions'' on a single variable.

\hypertarget{sec:D9}{\subsection*{D9: Liver Disorders}}

%Source: \url{https://archive.ics.uci.edu/ml/datasets/Liver+Disorders} \\
%Citations: \citep{UCI_ML_repository} \\
%Pairs: 33,34,35,36,37 \\
%Samples: 345 \\

This data set, available at the UCI Machine Learning Repository
\citep{UCI_ML_repository}, was collected by BUPA Medical Research Ltd. It
consists of several blood test results, which are all thought to be indicative
for liver disorders that may arise from excessive alcohol consumption.  Each of the
345 instances constitutes the record of a single male individual. Daily alcohol consumption
is measured in number of half-pint equivalents of alcoholic beverages drunk per
day. The blood test results are mean corpuscular volume (MCV), alkaline
phosphotase (ALP), alanine aminotransferase (ALT), aspartate aminotransferase
(AST), and gamma-glutamyl transpeptidase (GGT).
The data is available at \url{https://archive.ics.uci.edu/ml/datasets/Liver+Disorders}.

\begin{figure}[h!]
\centerline{\small\begin{tabular}{ccc}
  \includegraphics[width=0.2\textwidth]{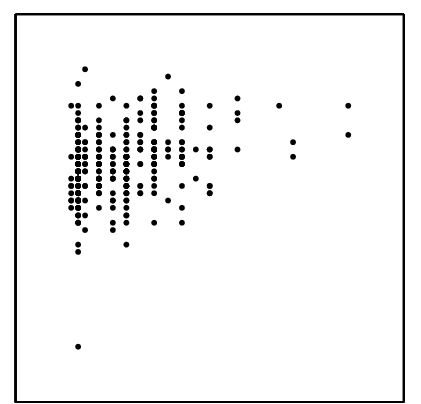}&
  \includegraphics[width=0.2\textwidth]{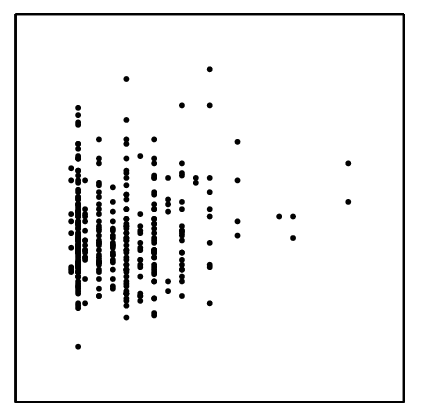}&
  \includegraphics[width=0.2\textwidth]{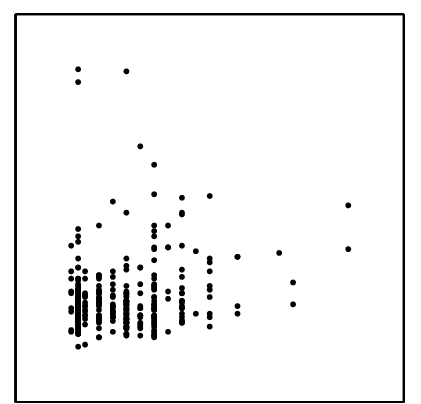}\\
  {\tt pair0033} & {\tt pair0034} & {\tt pair0035} \\
  \\
  \includegraphics[width=0.2\textwidth]{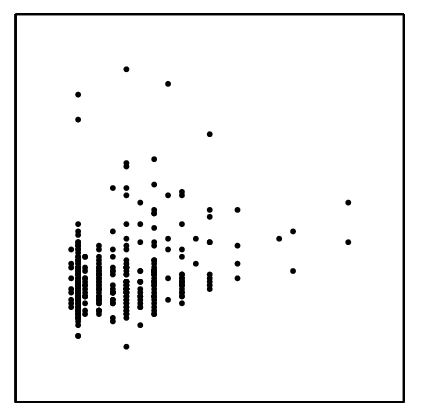}&
  \includegraphics[width=0.2\textwidth]{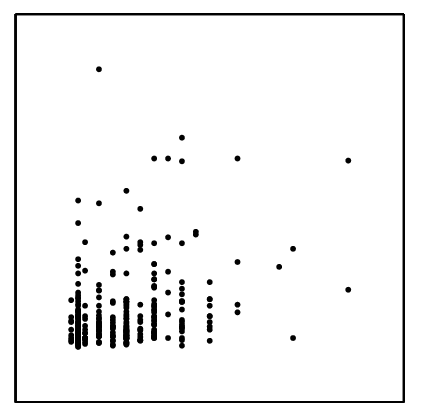}& \\
  {\tt pair0036} & {\tt pair0037} & 
\end{tabular}}
\caption{\label{fig:D9}Scatter plots of pairs from D9.
{\tt pair0033}: alcohol consumption \causes mean corpuscular volume,
{\tt pair0034}: alcohol consumption \causes alkaline phosphotase,
{\tt pair0035}: alcohol consumption \causes alanine aminotransferase,
{\tt pair0036}: alcohol consumption \causes aspartate aminotransferase,
{\tt pair0037}: alcohol consumption \causes gamma-glutamyl transpeptidase.}
\end{figure}

Although one would expect that daily alcohol consumption is the
cause, and the blood test results are the effects, this is not necessarily the case.
Indeed, citing \citep{BaynesDominiczak1999}: ``[...] increased plasma concentrations of acetaldehyde after the ingestion of alcohol [...] causes the individual to experience unpleasant flushing and sweating, which discourages alcohol abuse. Disulfiram, a drug that inhibits ALDH, also leads to these symptoms when alcohol is taken, and may be given to reinforce abstinence from alcohol.'' This means that \emph{a priori}, a reverse causation of the chemical whose concentration is measured in one of these blood tests on daily alcohol consumption cannot be excluded with certainty. Nevertheless, we consider this to be unlikely, as the
medical literature describes how these particular blood tests can be used to diagnose
liver disorders, but we did not find any evidence that these chemicals can be used to
\emph{treat} excessive alcohol consumption.

% Many tests (aminotransferases, alkaline phosphotase) detect liver cell damage

%It appears that drinks>5 is some sort of a selector on this database. See the PC/BEAGLE User's Guide for more information.

%Van \url{http://uwbloedserieus.nl/aanvraagformulier.php?id=250}:

\subsubsection*{{\tt pair0033}: Alcohol Consumption \causes Mean Corpuscular Volume}
The mean corpuscular volume (MCV) is the average volume of a red blood cell. An elevated MCV has been associated with
alcoholism \citep{TonnesenHejbergFrobeniusAndersen1986}, but there are many other factors also associated with MCV.

\subsubsection*{{\tt pair0034}: Alcohol Consumption \causes Alkaline Phosphotase}
Alkaline phosphatase (ALP) is an enzyme that is predominantly abundant in liver cells, but is also present in bone and placental tissue. Elevated ALP levels in blood can be due to many different liver diseases and also bone diseases, but also occur during pregnancy \citep{HarrisonsV2}.
%From \url{http://en.wikipedia.org/wiki/Liver_function_tests}:
%Alkaline phosphatase (ALP) is an enzyme in the cells lining the biliary ducts of the liver. ALP levels in plasma rise with large bile duct obstruction, intrahepatic cholestasis, or infiltrative diseases of the liver. ALP is also present in bone and placental tissue, so it is higher in growing children (as their bones are being remodelled) and elderly patients with Paget's disease. In the third trimester of pregnancy, ALP is about two to three times higher.

%Alkalische fosfatase: dit is een enzym dat zich voornamelijk in de cellen van de lever bevindt. Met name bij aandoeningen aan de galwegen, zoals bijvoorbeeld bij een afsluiting van de galwegen door galstenen, komt er veel alkalische fosfatase in het bloed terecht.\\

\subsubsection*{{\tt pair0035}: Alcohol Consumption \causes Alanine Aminotransferase}
%From \url{http://en.wikipedia.org/wiki/Alanine_transaminase}: ``Alanine Aminotransferase (ALT) is found in plasma and in various bodily tissues, but is most commonly associated with the liver. It is commonly measured clinically as a part of a diagnostic evaluation of hepatocellular injury, to determine liver health. Significantly elevated levels of ALT (SGPT) often suggest the existence of other medical problems such as viral hepatitis, diabetes, congestive heart failure, liver damage, bile duct problems, infectious mononucleosis, or myopathy, so ALT is commonly used as a way of screening for liver problems. Elevated ALT may also be caused by dietary choline deficiency. However, elevated levels of ALT do not automatically mean that medical problems exist. Fluctuation of ALT levels is normal over the course of the day, and they can also increase in response to strenuous physical exercise.''

Alanine Aminotransferase (ALT) is an enzyme that is found primarily in the liver cells. It is released into the blood in greater amounts when there is damage to the liver cells, for example due to a viral hepatitis or bile duct problems. ALT levels are often normal in alcoholic liver disease \citep{HarrisonsV2}.

\subsubsection*{{\tt pair0036}: Alcohol Consumption \causes Aspartate Aminotransferase}
%From \url{http://en.wikipedia.org/wiki/Liver_function_tests}:
%``AST, also called serum glutamic oxaloacetic transaminase or aspartate aminotransferase, is similar to ALT in that it is another enzyme associated with liver parenchymal cells. It is raised in acute liver damage, but is also present in red blood cells, and cardiac and skeletal muscle, so is not specific to the liver. The ratio of AST to ALT is sometimes useful in differentiating between causes of liver damage.[6][7] Elevated AST levels are not specific for liver damage, and AST has also been used as a cardiac marker.''
%From \url{http://en.wikipedia.org/wiki/Aspartate_transaminase}: ``Aspartate Aminotransferase (AST) is found in the liver, heart, skeletal muscle, kidneys, brain, and red blood cells, and it is commonly measured clinically as a marker for liver health.''

%ASAT: dit is een enzym dat zich in de cellen van de lever, het hart en andere spieren bevindt. Bij schade aan deze cellen komt het enzym in het bloed terecht.\\

Aspartate aminotransferase (AST) is an enzyme that is found in the liver, but also in many other bodily tissues, for example the heart and skeletal muscles. Similar to ALT, the AST levels raise in acute liver damage. Elevated AST levels are not specific to the liver, but can also be caused by other diseases, for example by pancreatitis. An AST:ALT ratio of more than 3:1 is highly suggestive of alcoholic liver disease \citep{HarrisonsV2}.

\subsubsection*{{\tt pair0037}: Alcohol Consumption \causes Gamma-Glutamyl Transpeptidase}
%From \url{http://en.wikipedia.org/wiki/Liver_function_tests}:
%``Although reasonably specific to the liver and a more sensitive marker for cholestatic damage than ALP, Gamma Glutamyl Transpeptidase (GGT) may be elevated with even minor, subclinical levels of liver dysfunction. It can also be helpful in identifying the cause of an isolated elevation in ALP (GGT is raised in chronic alcohol toxicity).''

Gamma-Glutamyl Transpeptidase (GGT) GGT is another enzyme that is primarily found in liver cells. It is rarely elevated in conditions other than liver disease. High GGT levels have been associated with alcohol use \citep{HarrisonsV2}.

%Bijbel: Some have advocated the use of GGT to identify patients with occult alcohol use. Its lack of specificity makes its use in this context questionable.  GGT is rarely elevated in conditions other than liver disase.

%GGT: dit is een enzym dat zich voornamelijk in de cellen van de lever bevindt. Met name bij overmatig alcoholgebruik en aandoeningen aan de galwegen, zoals bijvoorbeeld bij een afsluiting van de galwegen door galstenen, komt er veel GGT in het bloed terecht.\\

\hypertarget{sec:D10}{\subsection*{D10: Pima Indians Diabetes}}

%Source: \url{https://archive.ics.uci.edu/ml/datasets/Pima+Indians+Diabetes} \\
%Citations: \citep{UCI_ML_repository} \\
%Pairs: 38,39,40,41 \\
%Samples: 768 \\

This data set, available at the UCI Machine Learning Repository \citep{UCI_ML_repository},
was collected by the National Institute of Diabetes and Digestive and Kidney Diseases
in the USA to forecast the onset of diabetes mellitus in a high risk population of Pima Indians
near Phoenix, Arizona. Cases in this data set were selected according to several criteria, in
particular being female, at least 21 years of age and of Pima Indian heritage.
This means that there could be selection bias on age.
%"i. The subject was female.
%ii. The subject was $\ge 21$ year of age at the time of the index examination.  An index examination refers to the study that was chosen for use in this model.  It does not necessarily correspond to the chronologically first examination for this subject.
%iii.  Only one examination was selected per subject.  That examination was one that revealed a nondiabetic GTT and met one of the following two criteria: a.  Diabetes was diagnosed within five years of the examination, OR b.  A GTT performed five or more years later failed to reveal diabetes mellitus.
%iv.  If diabetes occurred within one year of an examination, that examination was excluded from the study to remove from the forecasting model those cases that were potentially easier to forecast.  In 75\% of the excluded examinations, DM was diagnosed within six months."

We downloaded the data from \url{https://archive.ics.uci.edu/ml/datasets/Pima+Indians+Diabetes}.
We only selected the instances with nonzero values, as it seems likely that zero values encode
missing data. This yielded 768 samples.
%** UPDATE: Until 02/28/2011 this web page indicated that there were no missing values in the dataset. As pointed out by a repository user, this cannot be true: there are zeros in places where they are biologically impossible, such as the blood pressure attribute. It seems very likely that zero values encode missing data. However, since the dataset donors made no such statement we encourage you to use your best judgement and state your assumptions.

\begin{figure}[h!]\small
\centerline{\small\begin{tabular}{cccc}
  \includegraphics[width=0.2\textwidth]{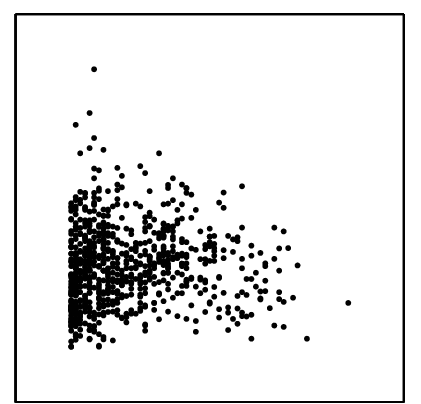}&
  \includegraphics[width=0.2\textwidth]{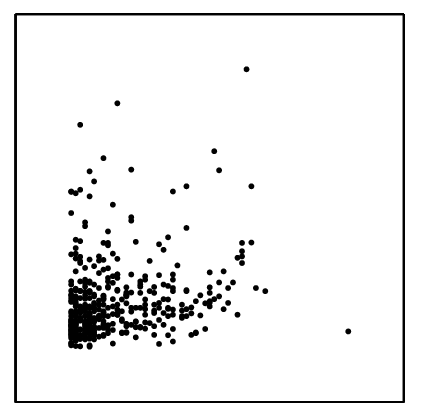}&
  \includegraphics[width=0.2\textwidth]{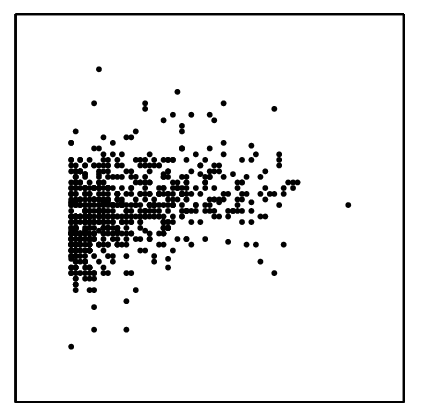}&
  \includegraphics[width=0.2\textwidth]{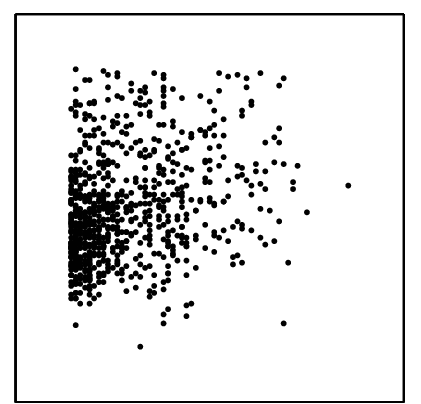}\\
  {\tt pair0038} & {\tt pair0039} & {\tt pair0040} & {\tt pair0041}
\end{tabular}}
\caption{\label{fig:D10}Scatter plots of pairs from D10.
{\tt pair0038}: age \causes body mass index,
{\tt pair0039}: age \causes serum insulin,
{\tt pair0040}: age \causes diastolic blood pressure,
{\tt pair0041}: age \causes plasma glucose concentration.}
\end{figure}

\subsubsection*{{\tt pair0038}: Age \causes Body Mass Index}
Body mass index (BMI) is defined as the ratio between weight (kg) and the square of height (m). Obviously, age is not caused by body mass index, but as age is a cause of both height and weight, age causes BMI.

\subsubsection*{{\tt pair0039}: Age \causes Serum Insulin}
2-Hour serum insulin ($\mu$U/ml), measured 2 hours after the ingestion of a standard dose of glucose, in an oral glucose tolerance test. We can exclude that serum insulin causes age, and there could be an effect of age on serum insulin. Another explanation for the observed dependence could be the selection bias.

\subsubsection*{{\tt pair0040}: Age \causes Diastolic Blood Pressure}
Diastolic blood pressure (mm Hg). It seems obvious that blood pressure does not cause age. The other causal direction seems plausible, but again, an alternative explanation for the dependence could be selection bias.

\subsubsection*{{\tt pair0041}: Age \causes Plasma Glucose Concentration}
Plasma glucose concentration, measured 2 hours after the ingestion of a
standard dose of glucose, in an oral glucose tolerance test. Similar reasoning as before: we do not believe that plasma glucose concentration causes ages, but it could be the other way around, and there may be selection bias.

\hypertarget{sec:D11}{\subsection*{D11: B.~Janzing's meteo data}}

%Source: Bernward Janzing \\
%Citations: \\
%Pairs: 42,77 \\
%Samples: 9162 \\

This data set is from a private weather station, owned by Bernward Janzing, located in Furtwangen (Black Forest), Germany
at an altitude of 956 m.   
The measurements include temperature, precipitation, and snow height  (since 1979),
as well as solar radiation (since 1986).
The data have been archived by Bernward Janzing,
statistical evaluations have been published in \citep{Wetterjanzing}, monthly summaries
of the weather are published in local newspapers since 1981.

\begin{figure}[h!]
\centerline{\small\begin{tabular}{ccc}
  \includegraphics[width=0.2\textwidth]{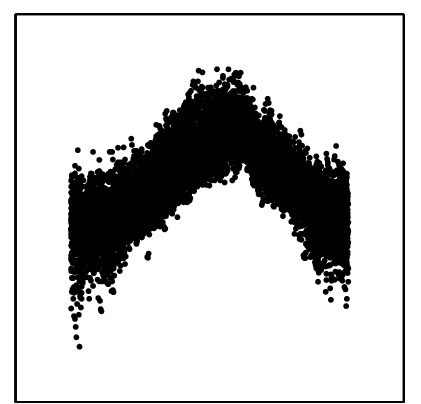}&
  \includegraphics[width=0.2\textwidth]{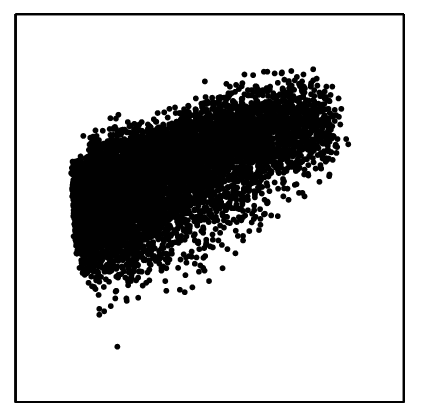}\\
  {\tt pair0042} & {\tt pair0077}
\end{tabular}}
\caption{\label{fig:D11}Scatter plots of pairs from D11.
{\tt pair0042}: day of the year \causes temperature,
{\tt pair0077}: solar radiation \causes temperature.}
\end{figure}

\subsubsection*{{\tt pair0042}: Day of the year \causes Temperature}

This data set shows the dependence between season and temperature over 25 years plus one month, namely the time range
 01/01/1979--01/31/2004. It consists of 9162 measurements.

One variable is the day of the year, represented by an integer 
from $1$ to $365$ (or $366$ for leap years). The information about the year has been dropped.
$Y$ is the mean temperature of the respective day, calculated according to the following definition:
\[
T_{mean}:= \frac{T_{morning} + T_{midday} +2 T_{evening}}{4},
\]
where morning, midday, and evening are measured at 7:00 am, 14:00 pm, and 21:00 pm (MEZ), respectivley (without daylight saving time). Double counting of the evening value 
is official standard of the German authority ``Deutscher Wetterdienst''.
It has been defined at a time where no electronic data loggers were available and
 thermometers had to be read out by humans. Weighting the evening value twice has been considered a useful heuristics to account for the missing values at night.

We consider day of the year as the cause, since it can be seen as expressing the angular position
on its orbit around the sun. Although true interventions are infeasible, it is commonly agreed that 
changing the position of the earth would  
result in temperature changes at a fixed location due to
the different solar incidence angle.

\subsubsection*{{\tt pair0077}: Solar radiation \causes Temperature}

This data set shows the relation between solar radiation and temperature
over 23 years, namely the interval 01/01/1986--12/31/2008. It consists of 8401 measurements.

Solar radiation is measured per area in $\mathrm{W}/\mathrm{m}^2$ averaged over one day on a horizontal surface. Temperature is the averaged daily, as in {\tt pair0042}. The original data has been processed by us to extract the common time interval. 
We assume that radiation causes temperature. 
 High solar radiation increases the temperature of the air already at a scale of hours. Interventions are easy to implement: Creating artificial shade on a large enough surface would decrease the air temperature. On longer time scales there might also be an influence from temperature to radiation via the generation of clouds through evaporation in more humid environments. This should, however, not play a role for daily averages.

\hypertarget{sec:D12}{\subsection*{D12: NCEP-NCAR Reanalysis}}

%Source: \url{http://www.esrl.noaa.gov/psd/data/gridded/data.ncep.reanalysis.surface.html} \\
%Citations: \citep{Kalnay1996} \\
%Pairs: 43,44,45,46,52 \\
%Samples: 10369 \\

This data set, available from the NOAA (National Oceanic and Atmospheric Administration) Earth System Research Laboratory website at \url{http://www.esrl.noaa.gov/psd/data/gridded/data.ncep.reanalysis.surface.html}, is a subset of a reanalysis data set, incorporating observations and numerical weather prediction model output from 1948 to date \citep{Kalnay1996}. The reanalysis data set was produced by the National Center for Environmental Prediction (NCEP) and the National Center for Atmospheric Research (NCAR). Reanalysis data products aim for a realistic representation of all relevant climatological variables on a spatiotemporal grid. We collected four variables from a global grid of 144 $\times$ 73 cells: air temperature (in K, {\tt pair0043}), surface pressure (in Pascal, {\tt pair0044}), sea level pressure (in Pascal, {\tt pair0045}) and relative humidity (in $\%$, {\tt pair0045}) on two consecutive days, day 50 and day 51 of the year 2000 (i.e., Feb 19th and 20th). Each data pair consists of $144 \times 73 - 143 = 10369$ data points, distributed across the globe. 143 data points were subtracted because at the north pole values are repeated across all longitudes. 
\begin{figure}[h!]\small
\centerline{\small\begin{tabular}{cccc}
  \includegraphics[width=0.2\textwidth]{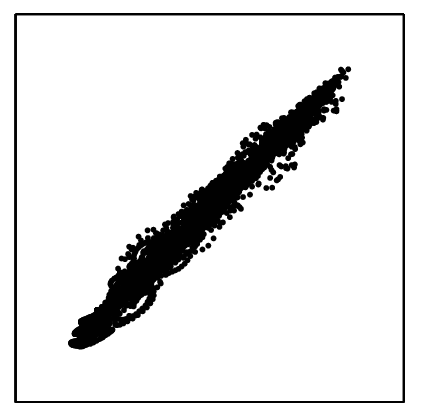}&
  \includegraphics[width=0.2\textwidth]{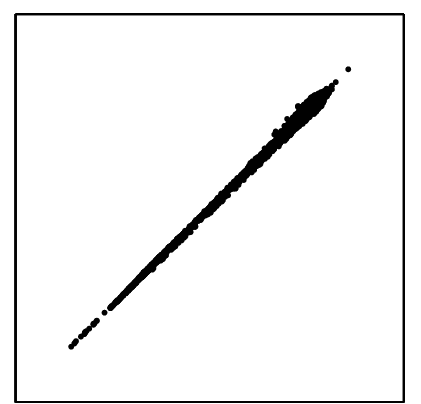}&
  \includegraphics[width=0.2\textwidth]{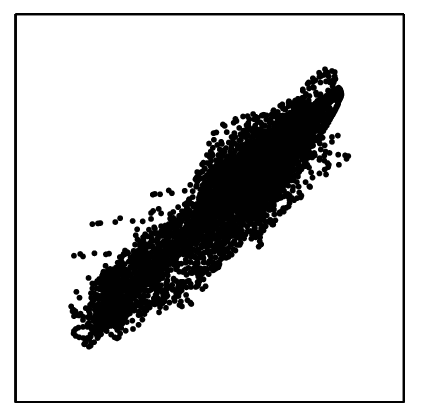}&
  \includegraphics[width=0.2\textwidth]{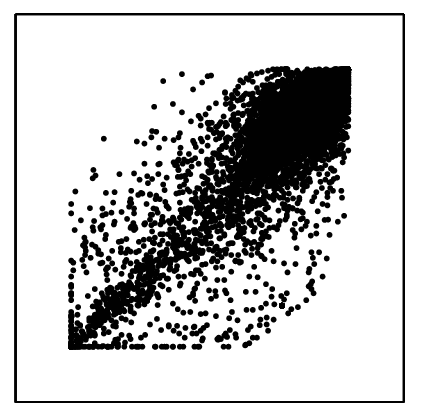}\\
  {\tt pair0043} & {\tt pair0044} & {\tt pair0045} & {\tt pair0046}
\end{tabular}}
\caption{\label{fig:D12}Scatter plots of pairs from D12.
{\tt pair0043}: temperature at \emph{t} \causes temperature at \emph{t}+1,
{\tt pair0044}: surface oressure at \emph{t} \causes surface pressure at \emph{t}+1,
{\tt pair0045}: sea level oressure at \emph{t} \causes sea level pressure at \emph{t}+1,
{\tt pair0046}: relative humidity at \emph{t} \causes relative humidity at \emph{t}+1,
{\tt pair0052}: (temp, press, slp, rh) at \emph{t} \causes (temp, press, slp, rh) at \emph{t+1}.}
\end{figure}

Each data point is the daily average over an area that covers 2.5$^\circ$ $\times$ 2.5$^\circ$ (approximately 250 km $\times$ 250 km at the equator). Because causal influence cannot propagate backwards in time,  temperature, pressure and humidity in a certain area are partly affected by their value the day before in the same area.

\subsubsection*{{\tt pair0043}: Temperature at \emph{t} \causes Temperature at \emph{t}+1}

Due to heat storage, mean daily air temperature near surface at any day largely impact daily air temperature at the following day. We assume there is no causation backwards in time, hence the correlation between temperatures at two consecutive days must be driven by confounders (such as large-scale weather patterns) or a causal influence from the first day to the second.

\subsubsection*{{\tt pair0044}: Surface Pressure at \emph{t} \causes Surface Pressure at \emph{t}+1}

Pressure patterns near the earth's surface are mostly driven by large-scale weather patterns. However, large-scale weather patterns are also driven by local pressure gradients and hence, some of the correlation between surface pressure at two consecutive days stems from a direct causal link between the first and the second day, as we assume there is no causation in time.

\subsubsection*{{\tt pair0045}: Sea Level Pressure at \emph{t} \causes Sea Level Pressure at \emph{t}+1}

Similar reasoning as in {\tt pair0044}. 

\subsubsection*{{\tt pair0046}: Relative Humidity at \emph{t} \causes Relative Humidity at \emph{t}+1}

Humidity of the air at one day affects the humidity of the following day because if no air movement takes place and no drying or moistening occurs, it will approximately stay the same.
Furthermore, as reasoned above, because there is no causation backwards in time, relative humidity at day $t+1$ cannot affect humidity at day $t$. Note that relative humidity has values between 0 and 100. Values can be saturated in very humid places such as tropical rainforest and approach 0 in deserts. For this reason, the scatter plot looks as if the data were clipped.

\subsubsection*{{\tt pair0052}: (Temp, Press, SLP, RH) at \emph{t} \causes (Temp, Press, SLP, RH) at \emph{t+1}}

The pairs {\tt pair0043}--{\tt pair0046} were combined to a 4-dimensional vector. From the reasoning above it follows that the vector of temperature, near surface pressure, sea level pressure and relative humidity at day $t$ has a causal influence on the vector of the same variables at time $t+1$.

\hypertarget{sec:D13}{\subsection*{D13: Traffic}}

%Source: \url{http://www.b30-oberschwaben.de/html/tabelle.html} \\
%Citations: \\
%Pairs: 47 \\
%Samples: 254 \\

This dataset has been extracted from \url{http://www.b30-oberschwaben.de/html/tabelle.html}, a website containing various kinds of information about the national highway B30. 
This is a road in the federal state Baden-W\"urttemberg, Germany, which provides an important connection of the region around Ulm (in the North) with the Lake Constance region (in the South).
After extraction, the data set contains 254 samples.

\begin{figure}[h!]
\centerline{\small\begin{tabular}{c}
  \includegraphics[width=0.2\textwidth]{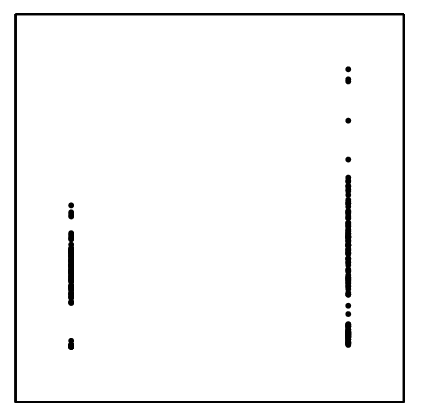}\\
  {\tt pair0047}
\end{tabular}}
\caption{\label{fig:D13}Scatter plots of pairs from D13.
{\tt pair0047}: type of day \causes number of cars.}
\end{figure}

\subsubsection*{{\tt pair0047}: Type of Day \causes Number of Cars}

One variable is the number of cars per day, the other denotes the type of the respective day, with ``1'' indicating Sundays and holidays and ``2'' indicating working days.    
The type of day causes the number of cars per day. Indeed,
introducing an additional holiday by a political decision would certainly change the amount of traffic on that day, while changing the amount of traffic by instructing a large number of drivers to drive or not to drive at a certain day would certainly not change the type of that day.

\hypertarget{sec:D14}{\subsection*{D14: Hipel \& McLeod}}

%Source: \url{http://www.stats.uwo.ca/faculty/mcleod/epubs/mhsets/readme-mhsets.html} \\
%Citations: \citep{Hipel1994} \\
%Pairs: 48 \\
%Samples: 168 \\

This dataset contains 168 measurements of indoor and outdoor temperatures. 
It was taken from a book by \citet{Hipel1994} and can be downloaded from 
\url{http://www.stats.uwo.ca/faculty/mcleod/epubs/mhsets/readme-mhsets.html}. 

\begin{figure}[h!]
\centerline{\small\begin{tabular}{c}
  \includegraphics[width=0.2\textwidth]{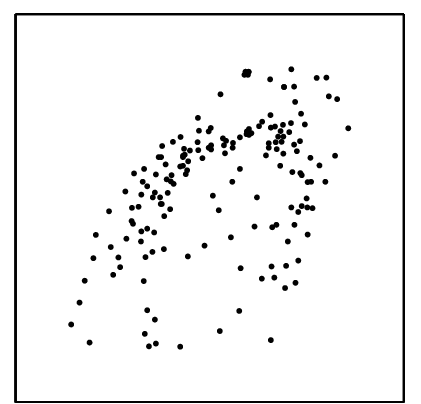}\\
  {\tt pair0048}
\end{tabular}}
\caption{\label{fig:D14}Scatter plots of pairs from D14.
{\tt pair0048}: outdoor temperature \causes indoor temperature.}
\end{figure}

\subsubsection*{{\tt pair0048}: Outdoor temperature \causes Indoor temperature}
Outdoor temperatures can have a strong impact on indoor temperatures, in particular when indoor temperatures are not adjusted by air conditioning or heating. Contrarily, indoor temperatures will have little or no effect on outdoor temperatures, because the outside environment has a much larger heat capacity.

\hypertarget{sec:D15}{\subsection*{D15: Bafu}}

%Source: \url{http://www.bafu.admin.ch/luft/luftbelastung/blick_zurueck/datenabfrage/index.html} \\
%Citations: \\
%Pairs: 49,50,51,55 \\
%Samples: 365 \\

This data set deals with the relationship between daily ozone concentration in the air 
and temperature. It was downloaded from \url{http://www.bafu.admin.ch/luft/luftbelastung/blick_zurueck/datenabfrage/index.html}.
Lower atmosphere ozone (O$_3$) is a secondary pollutant that is produced by the photochemical oxidation of carbon monoxide (CO), methane (CH$_4$), and non-methane volatile organic compounds (NMVOCs) by OH in the presence of nitrogen oxides (NO$_x$, NO + NO$_2$) \citep{Rasmussen2012}.
It is known that ozone concentration strongly correlates with surface temperature \citep{Bloomer2009}. Several explanations are given in the literature \citep[see e.g.,][]{Rasmussen2012}. Without going into details of the complex underlying chemical processes, we mention that the crucial chemical reactions are stronger at higher temperatures. For instance, 
isoprene emissions of plants increase with increasing temperature and isoprene can play a similar role in the generation of O$_3$ as NO$_x$ \citep{Rasmussen2012}.
Apart from this, air pollution may be influenced indirectly by temperature, e.g., via increasing traffic at `good' weather conditions or an increased occurrence rate of wildfires. 
All these explanations state a causal path from temperature to ozone. Note that
the phenomenon of ozone pollution in the lower atmosphere discussed here should not 
be confused with the `ozone hole', which is a lack of ozone in the higher atmosphere. Close to the surface, ozone concentration does not have an impact on temperatures.
For all three data sets, ozone is measured in $\mu$g/m$^3$ and temperature in $^\circ$C.
%Apart from the fact that such a causal path from ozone to temperature
%refers to ozone in the higher atmosphere, this causal relation refers to
%a time scale of climate research instead of the time scale of days considered here. 

\begin{figure}[h!]
\centerline{\small\begin{tabular}{ccc}
  \includegraphics[width=0.2\textwidth]{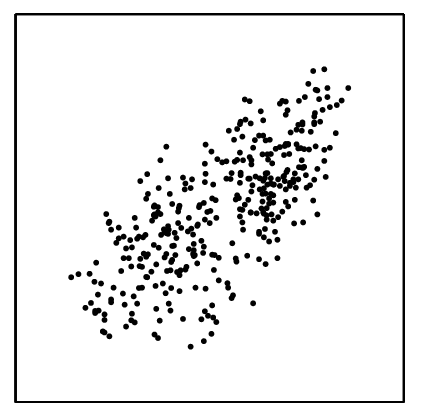}&
  \includegraphics[width=0.2\textwidth]{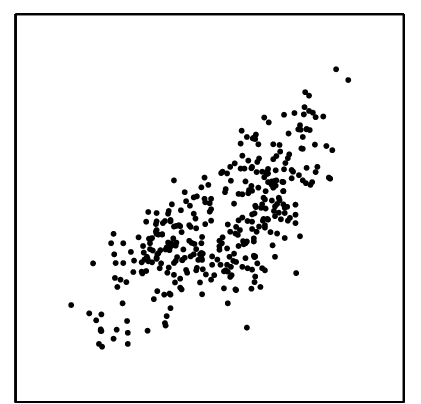}&
  \includegraphics[width=0.2\textwidth]{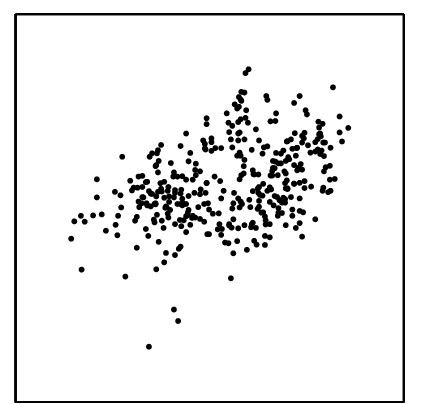}\\
  {\tt pair0049} & {\tt pair0050} & {\tt pair0051}
\end{tabular}}
\caption{\label{fig:D15}Scatter plots of pairs from D15.
{\tt pair0049}: temperature \causes ozone concentration,
{\tt pair0050}: temperature \causes ozone concentration,
{\tt pair0051}: temperature \causes ozone concentration,
{\tt pair0055}: radiation \causes ozone concentration.}
\end{figure}

\subsubsection*{{\tt pair0049}: Temperature \causes Ozone concentration}

365 daily mean values of ozone and temperature of year 2009 in Lausanne-C\'esar-Roux, Switzerland.

\subsubsection*{{\tt pair0050}: Temperature \causes Ozone concentration}

365 daily mean values of ozone and temperature of year 2009 in Chaumont, Switzerland.

\subsubsection*{{\tt pair0051}: Temperature \causes Ozone concentration}

365 daily mean values of ozone and temperature of year 2009 in Davos-See, Switzerland.

\subsubsection*{{\tt pair0055}: Radiation \causes Ozone concentration}

72 daily mean values of ozone concentrations and radiation in the last 83 days of 2009 at 16 different places in Switzerland (11 days were deleted due to missing data). Solar radiation and surface ozone concentration are correlated \citep{Feister1991}. The deposition of ozone is driven by complex micrometeorological
processes including wind direction, air temperature, and global radiation \citep{Stockwell1997}. For instance, solar radiation affects the height of the planetary boundary layer and cloud formation and thus indirectly influences ozone concentrations. In contrast, global radiation is not driven by ozone concentrations close to the surface.

Ozone is given in $\mu$g/m$^3$, radiation in W/m$^2$. 
The 16 different places are: 
1: Bern-Bollwerk, 2: Magadino-Cadenazzo, 3: Lausanne-C\'esar-Roux, 4: Payerne, 5: Lugano-Universita, 6: Taenikon, 7: Zuerich-Kaserne, 8: Laegeren, 9: Basel-Binningen,
10: Chaumont, 11: Duebendorf, 12: Rigi-Seebodenalp, 13: Haerkingen, 14: Davos-See,
15: Sion-A\'eroport, 16: Jungfraujoch.

\hypertarget{sec:D16}{\subsection*{D16: Environmental}}
  
%Source: \url{http://www.mathe.tu-freiberg.de/Stoyan/umwdat.html} \\
%Citations: \citep{Stoyan1997} \\
%Pairs: 53 \\
%Samples: 989 \\

We downloaded ozone concentration, wind speed, radiation and temperature from 
\url{http://www.mathe.tu-freiberg.de/Stoyan/umwdat.html}, discussed in \citet{Stoyan1997}. The data consist of 989 daily values over the time period from 05/01/1989 to 10/31/1994 observed in Heilbronn, Germany. 

%\begin{figure}[h!]
%\centerline{\small\begin{tabular}{c}
%  \includegraphics[width=0.2\textwidth]{pair0053}\\
%  {\tt pair0053}
%\end{tabular}}
%\caption{\label{fig:D16}Scatter plots of pairs from D16.
%\subsubsection*{{\tt pair0053}: (wind speed, radiation, temperature) \causes Ozone concentration}
%\end{figure}

\subsubsection*{{\tt pair0053}: (wind speed, radiation, temperature) \causes Ozone concentration}

As we have argued above in Section~\hyperlink{sec:D15}{D15}, wind direction (and speed), air temperature, and global radiation influence local ozone concentrations. Wind can influence ozone concentrations for example in the following way. No wind will keep the the concentration of ozone in a given air parcel constant if no lateral or vertical sources or sinks are prevalent. In contrast, winds can move and disperse and hence mix air with different ozone concentrations. Ozone concentration is given in $\mu$g/m$^3$, wind speed in m/s, global radiation in W/m$^2$ and temperature in $^\circ$C.

\hypertarget{sec:D17}{\subsection*{D17: UNdata}}

%Source: \url{http://data.un.org} \\
%Citations: \\
%Pairs: 56,57,58,59,60,61,62,63,64,73,74,75 \\
%Samples: 192 \\

The following data were taken from the ``UNdata'' database of the United Nations Statistics Division at \url{http://data.un.org}.

\begin{figure}[h!]
\centerline{\small\begin{tabular}{cccc}
  \includegraphics[width=0.2\textwidth]{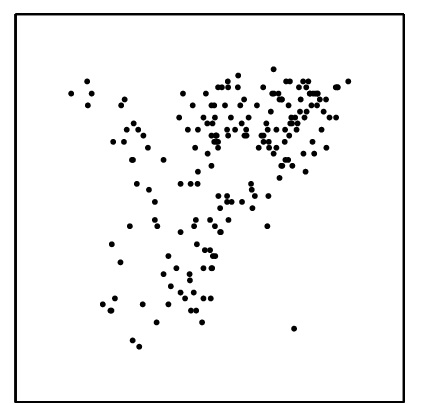}&
  \includegraphics[width=0.2\textwidth]{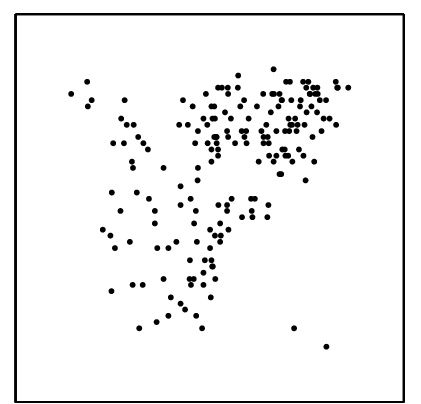}&
  \includegraphics[width=0.2\textwidth]{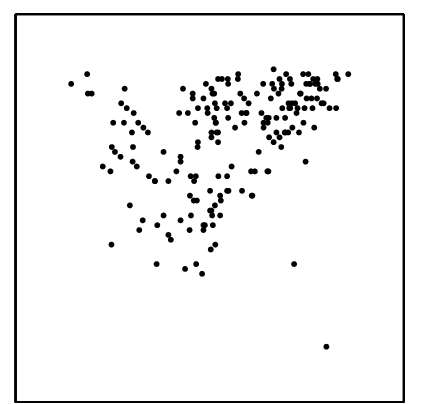}&
  \includegraphics[width=0.2\textwidth]{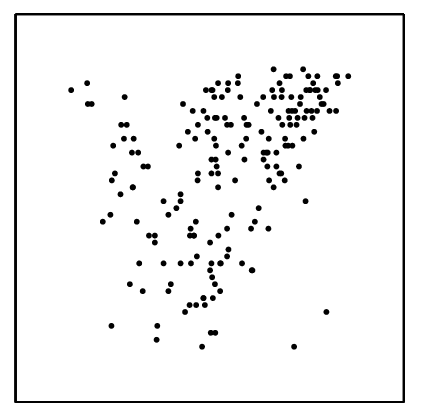}\\
  {\tt pair0056} & {\tt pair0057} & {\tt pair0058} & {\tt pair0059}\\
  \\
  \includegraphics[width=0.2\textwidth]{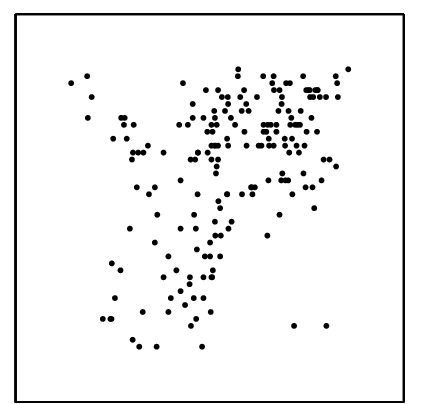}&
  \includegraphics[width=0.2\textwidth]{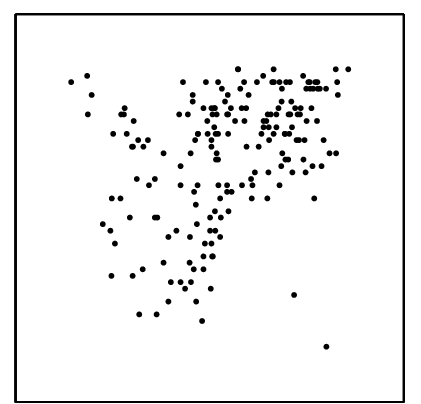}&
  \includegraphics[width=0.2\textwidth]{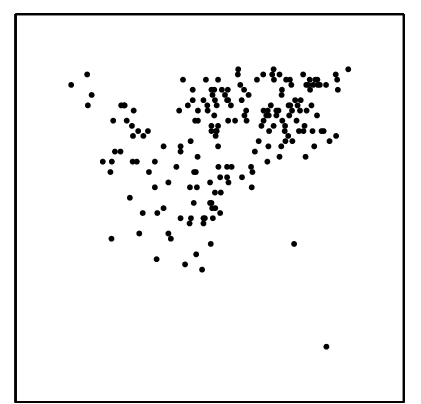}&
  \includegraphics[width=0.2\textwidth]{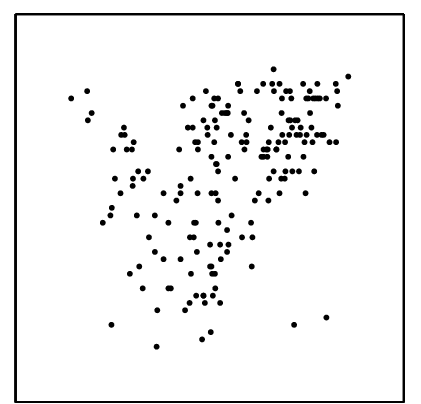}\\
  {\tt pair0060} & {\tt pair0061} & {\tt pair0062} & {\tt pair0063}\\
  \\
  \includegraphics[width=0.2\textwidth]{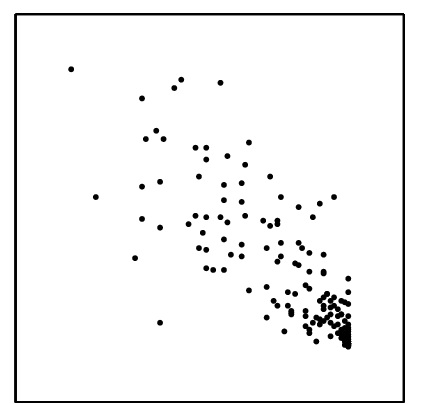}&
  \includegraphics[width=0.2\textwidth]{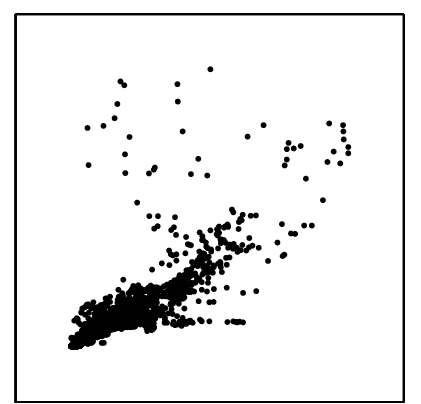}&
  \includegraphics[width=0.2\textwidth]{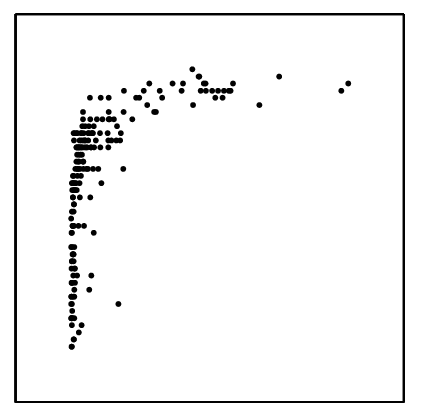}&
  \includegraphics[width=0.2\textwidth]{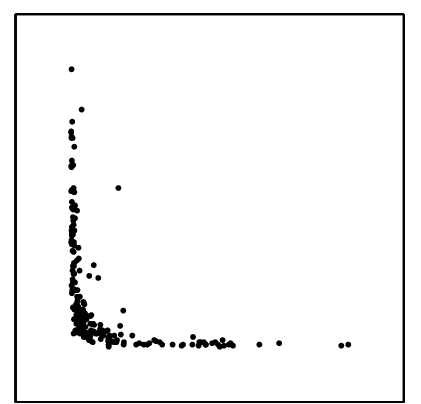}\\
  {\tt pair0064} & {\tt pair0073} & {\tt pair0074} & {\tt pair0075}
\end{tabular}}
\caption{\label{fig:D17}Scatter plots of pairs from D17.
{\tt pair0056}--{\tt pair0059}: latitude of capital \causes female life expectancy,
{\tt pair0060}--{\tt pair0063}: latitude of capital \causes male life expectancy,
{\tt pair0064}: drinking water access \causes infant mortality,
{\tt pair0073}: energy use \causes CO$_2$ emissions,
{\tt pair0074}: GNI per capita \causes life expectancy,
{\tt pair0075}: GNI per capita \causes under-5 mortality rate.}
\end{figure}

\subsubsection*{{\tt pair0056}--{\tt pair0059}: Latitude of Capital \causes Female Life Expectancy}

Pairs {\tt pair0056}--{\tt pair0059} consist of female life expectancy (in years) at birth versus latitude of the country's capital, for various
countries (China, Russia and Canada were removed). The four pairs correspond with measurements over the periods 2000--2005, 1995--2000,
1990--1995, 1985--1990, respectively. The data were downloaded from \url{http://data.un.org/Data.aspx?d=GenderStat&f=inID%3a37}.

The location of a country (encoded in the latitude of its capital) has an influence on how poor or rich a country is, hence affecting the quality of the health care system and ultimately life expectancy. This influence could stem from abundance of natural resources within the country's borders or the influence neighboring countries have on its economic welfare.
Furthermore, the latitude can influence life expectancy via climatic factors. For instance, life expectancy might be smaller if a country frequently experiences climatic extremes. 
In contrast, it is clear that life expectancy does not have any effect on latitude.

%\subsubsection*{{\tt pair0057}: Latitude of capital \causes Life expectancy }
%As {\tt pair0056}, life expectancy for female, 1995-2000.
%\subsubsection*{{\tt pair0058}: Latitude of capital \causes Life expectancy }
%As {\tt pair0056}, life expectancy for female, 1990-1995.
%\subsubsection*{{\tt pair0059}: Latitude of capital \causes Life expectancy }
%As {\tt pair0056}, life expectancy for female, 1985-1990.  

\subsubsection*{{\tt pair0060}--{\tt pair0063}: Latitude of Capital \causes Male Life Expectancy}

%\subsubsection*{{\tt pair0060}: Latitude of capital \causes Life expectancy }
%As {\tt pair0056}, life expectancy for male, 2000-2005.
%\subsubsection*{{\tt pair0061}: Latitude of capital \causes Life expectancy }
%As {\tt pair0056}, life expectancy for male, 1995-2000.
%\subsubsection*{{\tt pair0062}: Latitude of capital \causes Life expectancy }
%As {\tt pair0056}, life expectancy for male, 1990-1995.  
%\subsubsection*{{\tt pair0063}: Latitude of capital \causes Life expectancy }
%As {\tt pair0056}, life expectancy for male, 1985-1990.

Pairs {\tt pair0060}--{\tt pair0063} are similar, but concern male life expectancy.
The same reasoning as for female life expectancy applies here.

\subsubsection*{{\tt pair0064}: Drinking Water Access \causes Infant Mortality}

Here, one variable describes the percentage of population with sustainable access to improved drinking water sources in 2006, whereas the other variable denotes
the infant mortality rate (per 1000 live births) for both sexes. The data were downloaded from \url{http://data.un.org/Data.aspx?d=WHO&f=inID%3aMBD10} 
and \url{http://data.un.org/Data.aspx?d=WHO&f=inID%3aRF03}, respectively, and consist of 163 samples. 

Clean drinking water is a primary requirement for health, in particular for infants \citep{Esrey1991}.
Changing the percentage of people with access to clean water will directly change the mortality rate of infants, since infants are particularly susceptible to diseases \citep{Lee1997}.  
There may be some feedback, because if infant mortality is high in a poor country, development aid may be directed towards increasing the access to clean drinking water.

\subsubsection*{{\tt pair0073}: Energy Use \causes CO$_2$ Emissions}

This data set contains energy use (in kg of oil equivalent per capita) and CO$_2$ emission data from 152 countries between 1960 and 2005, yielding together 5084 samples.
Considering the current energy mix across the world, the use of energy clearly results in CO$_2$ emissions (although in varying amounts across energy sources). Contrarily, a hypothetical change in CO$_2$ emissions will not affect the energy use of a country on the short term. On the longer term, if CO$_2$ emissions increase, this may cause energy use to decrease because of fear for climate change.

\subsubsection*{{\tt pair0074}: GNI per capita \causes Life Expectancy}

We collected the Gross National Income (GNI, in USD) per capita and the life expectancy at birth (in years) for 194 different countries.
GNI can be seen as an index of wealth of a country. In general, richer countries have a better health care system than poor countries an thus can take better care of their citizens when they are ill. Reversely, we believe that the life expectancy of humans has a smaller impact on how wealthy a country is than vice versa.

\subsubsection*{{\tt pair0075}: GNI per capita \causes Under-5 Mortality Rate}

Here we collected the Gross National Income (GNI, in USD) per capita and the under-5 mortality rate (deaths per 1000 live births) for 205 different countries.
The reasoning is similar as in {\tt pair0074}. GNI as an index of wealth influences the quality of the health care system, which in turn determines whether young children will or will not die from minor diseases.
As children typically do not contribute much to GNI per capita, we do not expect the reverse causal relation to be very strong.

\hypertarget{sec:D18}{\subsection*{D18: Yahoo database}}

%Source: \url{http://finance.yahoo.com} \\
%Citations: \\
%Pairs: 65,66,67 \\
%Samples: 1331 \\

These data denote stock return values and were downloaded from \url{http://finance.yahoo.com}. We collected 1331 samples from the following stocks between January 4th, 2000 and June 17, 2005:
Hang Seng Bank (0011.HK), HSBC Hldgs (0005.HK), Hutchison (0013.HK), Cheung kong (0001.HK), and Sun Hung Kai Prop.\ (0016.HK). Subsequently, the following preprocessing was applied, which is common in financial data processing:
\begin{enumerate}
\item Extract the dividend/split adjusted closing price data from the Yahoo Finance data base.
\item For the few days when the price is not available, we use simple linear interpolation to estimate the price.  %Consequently the two time series are aligned.
\item For each stock, denote the closing price on day $t$ by $P_t$, and the corresponding return is calculated as $X_t = (P_t-P_{t-1 }) / P_{t-1}$.
\end{enumerate}

\begin{figure}[h!]
\centerline{\small\begin{tabular}{ccc}
  \includegraphics[width=0.2\textwidth]{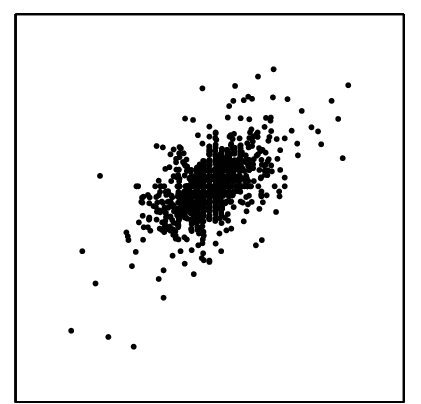}&
  \includegraphics[width=0.2\textwidth]{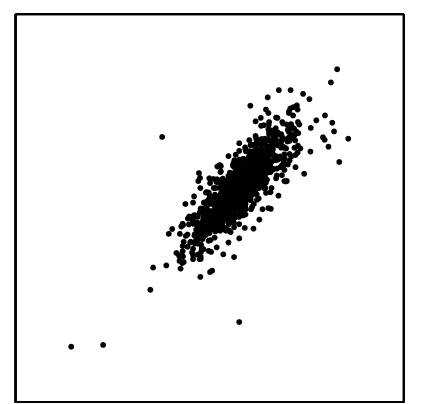}&
  \includegraphics[width=0.2\textwidth]{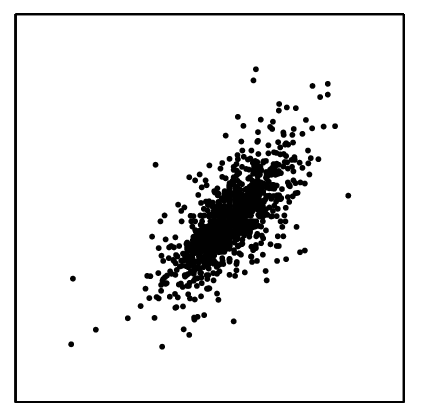}\\
  {\tt pair0065} & {\tt pair0066} & {\tt pair0067}
\end{tabular}}
\caption{\label{fig:D18}Scatter plots of pairs from D18.
{\tt pair0065}: Stock Return of Hang Seng Bank \causes Stock Return of HSBC Hldgs,
{\tt pair0066}: Stock Return of Hutchison \causes Stock Return of Cheung kong,
{\tt pair0067}: Stock Return of Cheung kong \causes Stock Return of Sun Hung Kai Prop.}
\end{figure}

\subsubsection*{{\tt pair0065}: Stock Return of Hang Seng Bank \causes Stock Return of HSBC Hldgs}

HSBC owns 60$\%$ of Hang Seng Bank. Consequently, if stock returns of Hang Seng Bank change, this should have an influence on stock returns of HSBC Hldgs, whereas causation in the other direction would be expected to be less strong.

\subsubsection*{{\tt pair0066}: Stock Return of Hutchison \causes Stock Return of Cheung kong}

Cheung kong owns about 50$\%$ of Hutchison. Same reasoning as in {\tt pair0065}.

\subsubsection*{{\tt pair0067}: Stock Return of Cheung kong \causes Stock Return of Sun Hung Kai Prop.}

Sun Hung Kai Prop. is a typical stock in the Hang Seng Property subindex, and is believed to depend on other major stocks, including Cheung kong.

\hypertarget{sec:D19}{\subsection*{D19: Internet traffic data}}

%Source: P.\ Daniusis \\
%Citations: \\
%Pairs: 68 \\
%Samples: 1291 \\

This dataset has been created from the log-files of a http-server of  the Max Planck Institute
for Intelligent Systems in T\"ubingen, Germany. 
The variable Internet connections counts the number of times an internal website of the institute has been accessed
during a time interval of $1$ minute (more precisely, it counts the number of URL requests).
Requests for non-existing websites are not counted.
The variable Byte transferred counts the total number of bytes sent  
for all those accesses during the same time interval. 
The values $(x_1,y_1),\dots,(x_{498},y_{498})$ refer to $498$ time intervals. To avoid too strong dependence between the measurements, the time intervals are not adjacent but have a distance of $20$ minutes.

\begin{figure}[h!]
\centerline{\small\begin{tabular}{c}
  \includegraphics[width=0.2\textwidth]{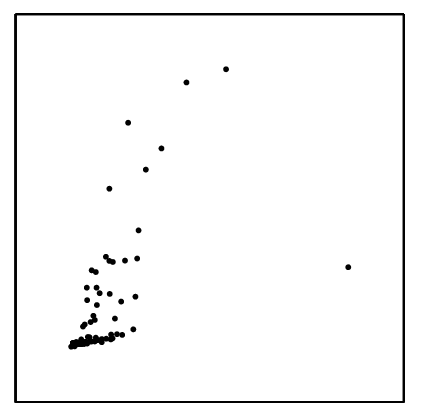}\\
  {\tt pair0068}
\end{tabular}}
\caption{\label{fig:D19}Scatter plots of pairs from D19. 
{\texttt pair0068}: internet connections \causes bytes transferred.}
\end{figure}

\subsubsection*{{\tt pair0068}: Internet connections \causes Bytes transferred}

%This dataset deals with the Internet traffic of a server of the University Vilnius, Lithuania, collected around 2010. It contains two time series $X_t$ and $Y_t$, where $Y_t$ denotes the 
%number of open http connections and $X_t$ the bytes sent  at minute $t$, logged over a time interval
%of $1291$ minutes. A connection stays alive some time after last user click, i.e., the closing of a connection is caused by 

Internet connections  causes Bytes transferred because an additional access of the website raises the transfer of data, while
transferring more data does not create an additional website access. 
Note that not every access yields data transfer because the website may still be cached. 
However, this fact does not spoil the causal relation, it only makes it less deterministic.

\hypertarget{sec:D20}{\subsection*{D20: Inside and outside temperature}}

%Source: Joris M.\ Mooij \\
%Citations: \\
%Pairs: 69 \\
%Samples: 16382 \\

This bivariate time-series data consists of measurements of inside room
temperature ($^\circ$C) and outside temperature ($^\circ$C),
where measurements were taken every 5 minutes for a period of about 56 days,
yielding a total of 16382 measurements.
The outside thermometer was located on a spot that was exposed to direct
sunlight, which explains the large fluctuations. The data were collected by
Joris M.\ Mooij.

\begin{figure}[h!]
\centerline{\small\begin{tabular}{c}
  \includegraphics[width=0.2\textwidth]{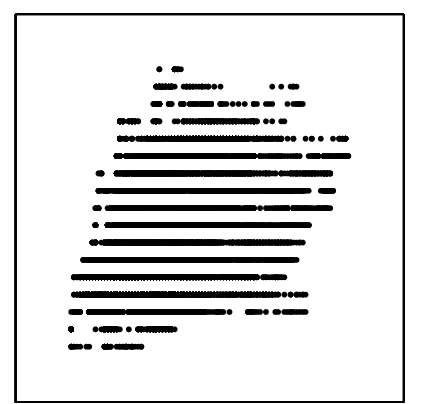}\\
  {\tt pair0069}
\end{tabular}}
\caption{\label{fig:D20}Scatter plots of pairs from D20.
{\tt pair0069}: outside temperature \causes inside temperature.}
\end{figure}

\subsubsection*{{\tt pair0069}: Outside Temperature \causes Inside Temperature}

Although there is a causal relationship in both directions, we
expect that the strongest effect is from outside temperature on inside temperature,
as the heat capacity of the inside of a house is much smaller than that of its
surroundings. See also the reasoning for {\tt pair0048}.

\hypertarget{sec:D21}{\subsection*{D21: Armann \& Buelthoff's data }}

%Source: \\
%Citations: \\
%Pairs: 70 \\
%Samples: 4499 \\

This dataset is taken from a psychological experiment that artificially generates images
of human faces that interpolate between male and female, taking real faces as basis \citep{Armann}.
The interpolation is done via principal component analysis after representing true face images
as vectors in an appropriate high-dimensional space.
Human subjects are instructed to label the faces as male or female.
The variable ``parameter'' runs between $0$ and $14$ and describes the transition from female to male. 
It is chosen by the experimenter. The binary variable ``answer'' indicates the answers `female' and 'male', respectively.
The dataset consists of 4499 samples.

\begin{figure}[h!]
\centerline{\small\begin{tabular}{c}
  \includegraphics[width=0.2\textwidth]{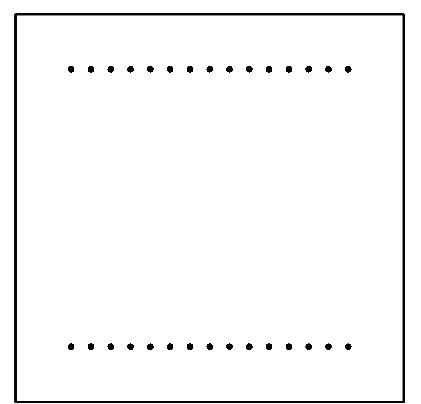}\\
  {\tt pair0070}
\end{tabular}}
\caption{\label{fig:D21}Scatter plots of pairs from D21.
{\tt pair0070}: parameter \causes answer.}
\end{figure}

\subsubsection*{{\tt pair0070}: Parameter \causes Answer}

Certainly parameter causes answer. We do not have to talk about {\it hypothetical} interventions. Instead, we have a true intervention, since ``parameter'' has been set by the experimenter.

\hypertarget{sec:D22}{\subsection*{D22: Acute Inflammations}}

%Source:  \url{https://archive.ics.uci.edu/ml/datasets/Acute+Inflammations}\\
%Citations: \citep{UCI_ML_repository,Czerniak2003} \\
%Pairs: 71 \\
%Samples: 120 \\

This data set, available at the UCI Machine Learning Repository \citep{UCI_ML_repository},
was collected in order to create a computer expert system that decides whether a patient suffers from 
two different diseases of urinary system \citep{Czerniak2003}. We downloaded it from \url{https://archive.ics.uci.edu/ml/datasets/Acute+Inflammations}.
The two possible diseases are acute inflammations of urinary bladder and acute nephritises of renal pelvis origin.
As it is also possible to chose none of those, the class variable takes values in $\{0,1\}^2$.
The decision is based on six symptoms: temperature of patient (e.g. $35.9$), 
occurrence of nausea (``yes'' or ``no''),
lumbar pain (``yes'' or ``no''),
urine pushing (``yes'' or ``no''), %(continuous need for urination)
micturition pains (``yes'' or ``no'') and
burning of urethra, itch, swelling of urethra outlet (``yes'' or ``no'').
These are grouped together in a six-dimensional vector ``symptoms''.

%\begin{figure}[h!]
%\centerline{\small\begin{tabular}{c}
%  \includegraphics[width=0.2\textwidth]{pair0071}\\
%  {\tt pair0071}
%\end{tabular}}
%\caption{\label{fig:D22}Scatter plots of pairs from D22.}
%\end{figure}

\subsubsection*{{\tt pair0071}: Symptoms \causes classification of disease}

One would think that the disease is causing the symptoms but this data set was created artificially. 
The description on the UCI homepage says: ``The data was created by a medical expert as a data set to test the expert system, which
will perform the presumptive diagnosis of two diseases of urinary system. (...) Each instance represents an potential patient.''
We thus consider the symptoms as the cause for the expert's decision.

\hypertarget{sec:D23}{\subsection*{D23: Sunspot data}}
  
%Source: \url{http://solarscience.msfc.nasa.gov/SunspotCycle.shtml} \\
%Citations: \citep{Hathaway2010} \\
%Pairs: 72 \\
%Samples: 1632 \\

The data set consists of 1632 monthly values between May 1874 and April 2010 and therefore contains $1632$ data points. The temperature data have been taken from  
\url{http://www.cru.uea.ac.uk/cru/data/temperature/}
and have been collected by Climatic Research Unit (University of East Anglia) in conjunction with the Hadley Centre (at the UK Met Office) \citep{Morice2012}. The temperature data is expressed in deviations from the 1961--90 mean global temperature of the Earth (i.e., monthly anomalies). 
The sunspot data \citep{Hathaway2010} are taken from the National Aeronautics and Space Administration and
were downloaded from \url{http://solarscience.msfc.nasa.gov/SunspotCycle.shtml}.
According to the description on that website, ``sunspot number is calculated by first counting the number of sunspot groups and then the number of individual sunspots.
The sunspot number is then given by the sum of the number of individual sunspots and ten times the number of groups. Since most sunspot groups have, on average, about ten spots, this formula for counting sunspots gives reliable numbers even when the observing conditions are less than ideal and small spots are hard to see.''
%{\color{red} ACHTUNG! Jonas downloaded the data again but there is an inconsistency:
%for May 1874, he found the value $44.6$ sunspots.}

\begin{figure}[h!]
\centerline{\small\begin{tabular}{c}
  \includegraphics[width=0.2\textwidth]{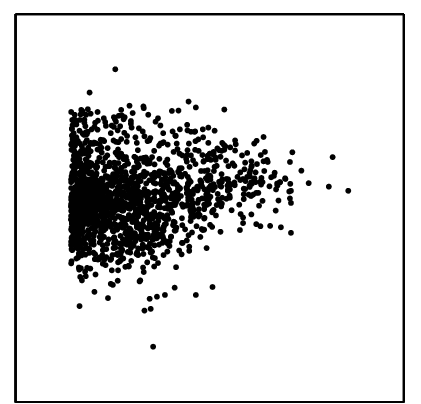}\\
  {\tt pair0072}
\end{tabular}}
\caption{\label{fig:D23}Scatter plots of pairs from D23.
{\tt pair0072}: sunspots \causes global mean temperature.}
\end{figure}

\subsubsection*{{\tt pair0072}: Sunspots \causes Global Mean Temperature}

Sunspots are phenomena that appear temporarily on the sun's surface. Although the causes of sunspots are not entirely understood, 
there is a significant dependence between the number of sunspots and the global mean temperature anomalies 
($p$-value for zero correlation is less than $10^{-4}$). There is evidence that the Earth's climate heats and cools as solar activity rises and falls \citep{Haigh2007}, and the sunspot number can be seen as a proxy for solar activity.
Also, we do not believe that the Earth's surface temperature (or changes of the Earth's atmosphere) has an influence on the activity of the sun. 
We therefore consider number of sunspots causing temperature as the ground truth. 
%\todo{Suppose we paint a few spots on the sun. Would it really change the temperature?}

\hypertarget{sec:D24}{\subsection*{D24: Food and Agriculture Organization of the UN}}

%Source: \url{http://www.fao.org/economic/ess/ess-fs/en/} \\
%Citations: \\
%Pairs: 76 \\
%Samples: 694 \\

The data set has been collected by Food and Agriculture Organization of the UN (\url{http://www.fao.org/economic/ess/ess-fs/en/}) and is accessible at \url{http://www.docstoc.com/docs/102679223/Food-consumption-and-population-growth---FAO}. It covers $174$ countries or areas during the period from 1990--92 to 1995--97 and the period from 1995--97 to 2000--02. As one entry is missing, this gives 347 data points.
We selected two variables: population growth and food consumption.
The first variable indicates the average annual rate of change of population (in \%), the second one describes
the average annual rate of change of total dietary consumption for total population (kcal/day) (also in \%).

\begin{figure}[h!]
\centerline{\small\begin{tabular}{c}
  \includegraphics[width=0.2\textwidth]{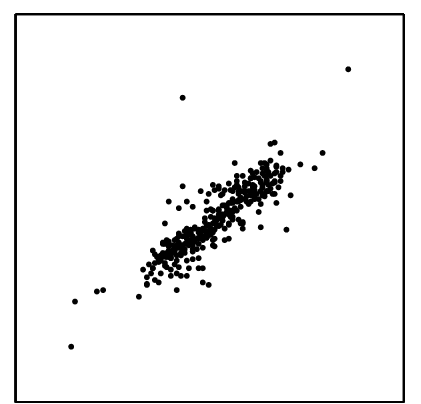}\\
  {\tt pair0076}
\end{tabular}}
\caption{\label{fig:D24}Scatter plots of pairs from D24.
{\tt pair0076}: population growth \causes food consumption growth.}
\end{figure}

\subsubsection*{{\tt pair0076}: Population Growth \causes Food Consumption Growth}

We regard population growth to cause food consumption growth, mainly because more people eat more. Both variables are most likely also confounded by the availability of food, driven for instance by advances in agriculture and subsequently increasing yields, but also by national and international conflicts, the global food market and other economic factors. However, for the short time period considered here, confounders which mainly influence the variables on a temporal scale can probably be neglected. Their might also be a causal link from food consumption growth to population growth, for instance one could imagine that if people are well fed, they also reproduce more. However, we assume this link only plays a minor role here.

\hypertarget{sec:D25}{\subsection*{D25: Light response data}}

%Source: A.~M.~Moffat \\
%Citations: \citep{Moffat2012} \\
%Pairs: 78,79,80 \\
%Samples: 721 \\

The filtered version of the light response data was obtained from \citet{Moffat2012}. It consists of 721 measurements of
Net Ecosystem Productivity (NEP) and three different measures of the Photosynthetic Photon Flux Density (PPFD):
the direct, diffuse, and total PPFD. NEP is a measure of the net $CO_2$ flux between the biosphere and the atmosphere, mainly driven by biotic activity.
It is defined as the photosynthetic carbon uptake minus the carbon release by respiration, and depends on the available light. NEP is measured in units of $\mu \mathrm{mol\,CO_2}\,\mathrm\,\mathrm{m}^{-2}\,\mathrm{s}^{-1}$.
PPFD measures light intensity in terms of photons that are available for photosynthesis, i.e., with wavelength between 400\,nm 
and 700\,nm (visible light). More precisely, PPFD is defined as the number of photons with wavelength of 400--700\,nm falling on 
a certain area per time interval, measured in units of $\mu \mathrm{mol\,photons}\,\mathrm{m}^{-2}\,\mathrm{s}^{-1}$. The total
PPFD is the sum of PPFDdif, which measures only diffusive photons, and PPFDdir, which measures only direct (solar light) photons.
The data was measured over several hectare of a forest in Hainich, Germany (site name DE-Hai, latitude: $51.08^\circ$N, longitude: $10.45^\circ$E), and is available from \url{http://fluxnet.ornl.gov}.

%PPFD(total) = PPFDdif + PPFDdir

\begin{figure}[h!]
\centerline{\small\begin{tabular}{ccc}
  \includegraphics[width=0.2\textwidth]{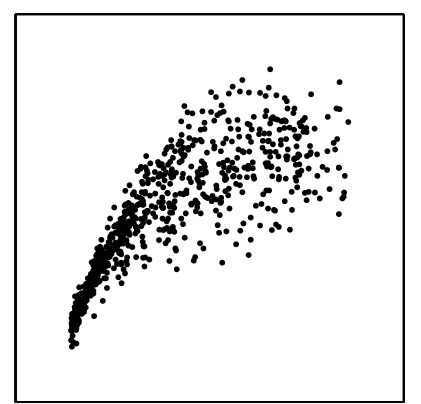}&
  \includegraphics[width=0.2\textwidth]{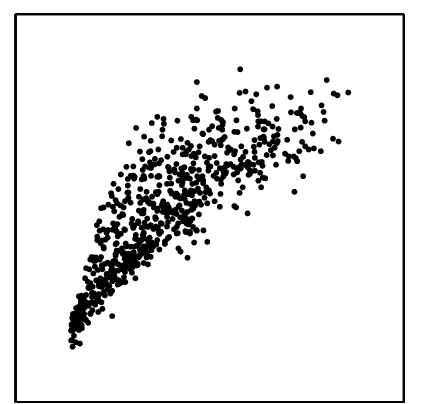}&
  \includegraphics[width=0.2\textwidth]{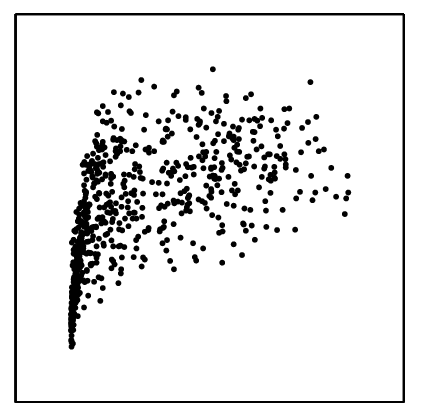}\\
  {\tt pair0078} & {\tt pair0079} & {\tt pair0080}
\end{tabular}}
\caption{\label{fig:D25}Scatter plots of pairs from D25.
{\tt pair0078}: PPFD \causes NEP, 
{\tt pair0079}: PPFDdif \causes NEP,
{\tt pair0080}: PPFDdir \causes NEP.}
\end{figure}

\subsubsection*{{\tt pair0078}--{\tt pair0080}: \{PPFD,PPFDdif,PPFDdir\} \causes NEP}
Net Ecosystem Productivity is known to be driven by both the direct and the diffuse Photosynthetic Photon Flux Density, and hence also by their sum, the total PPFD.

\hypertarget{sec:D26}{\subsection*{D26: FLUXNET}}

%Source: \url{http://fluxnet.ornl.gov} \\
%Citations: \citep{Baldocchi2001} \\
%Pairs: 81,82,83 \\
%Samples: 365 \\

The data set contains measurements of net $CO_2$ exchanges between atmosphere and biosphere aggregated
over night, and the corresponding temperature. It is taken from the FLUXNET network \citep{Baldocchi2001}, available at \url{http://fluxnet.ornl.gov} (see also Section~\hyperlink{sec:D25}{D25}).
The data have been collected at a 10 Hz rate and was aggregated to one value per day over one year (365 values) and at three different sites (BE-Bra, DE-Har, US-PFa). $CO_2$ exchange measurements typically have a footprint of about 1km$^2$.
The data set contains further information on the quality of the data (``1'' means that the value is credible, ``NaN'' means that the data point has been filled in).

\begin{figure}[h!]
\centerline{\small\begin{tabular}{ccc}
  \includegraphics[width=0.2\textwidth]{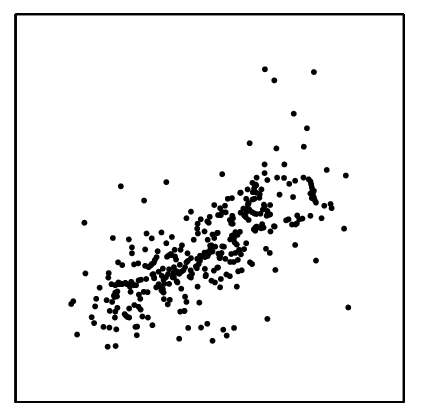}&
  \includegraphics[width=0.2\textwidth]{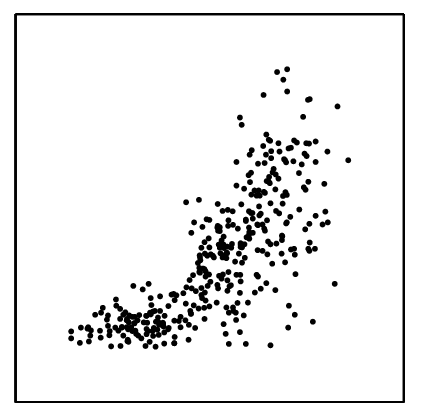}&
  \includegraphics[width=0.2\textwidth]{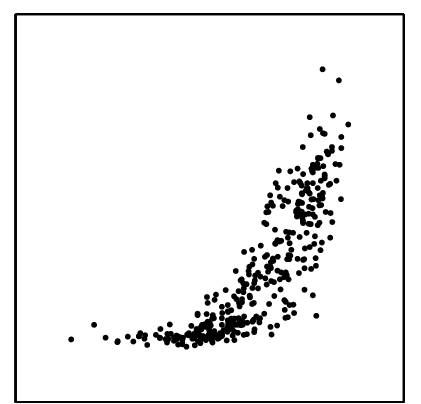}\\
  {\tt pair0081} & {\tt pair0082} & {\tt pair0083}
\end{tabular}}
\caption{\label{fig:D26}Scatter plots of pairs from D26.
{\tt pair0081} (BE-Bra): temperature \causes local $CO_2$ flux,
{\tt pair0082} (DE-Har): temperature \causes local $CO_2$ flux,
{\tt pair0083} (US-PFa): temperature \causes local $CO_2$ flux.}
\end{figure}

\subsubsection*{{\tt pair0081}--{\tt pair0083}: Temperature \causes Local $CO_2$ Flux}

Because of lack of sunlight, $CO_2$ exchange at night approximates ecosystem respiration (carbon release from the biosphere to the atmosphere), which is largely dependent on temperature \citep[e.g.][]{Mahecha2010}.
The $CO_2$ flux is mostly generated by microbial decomposition in soils and maintenance respiration from plants and does not have a direct effect on temperature. 
We thus consider temperature causing $CO_2$ flux as the ground truth. The three pairs {\tt pair0081}--{\tt pair0083} correspond with sites BE-Bra, DE-Har, US-PFa, respectively.

\hypertarget{sec:D27}{\subsection*{D27: US county-level growth data}}

%Source: \url{http://www.spatial-econometrics.com/data/contents.html} \\
%Citations: \citep{Wheeler2003} \\
%Pairs: 84 \\
%Samples: 3102 \\

The data set \citep{Wheeler2003}, available at \url{http://www.spatial-econometrics.com/data/contents.html}, contains both employment and population information for 3102 counties in the US in 1980.
We selected columns eight and nine in the file ``countyg.dat''.
Column eight contains the natural logarithm of the number of employed people, while
column nine contains the natural logarithm of the total number of people living 
in this county, and is therefore always larger than the number in column eight.

\begin{figure}[h!]
\centerline{\small\begin{tabular}{c}
  \includegraphics[width=0.2\textwidth]{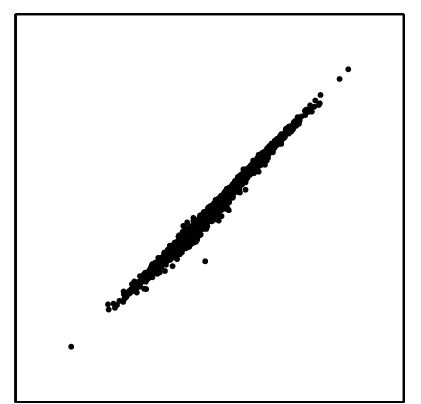}\\
  {\tt pair0084}
\end{tabular}}
\caption{\label{fig:D27}Scatter plots of pairs from D27.
{\tt pair0084}: population \causes employment.}
\end{figure}

\subsubsection*{{\tt pair0084}: Population \causes Employment}
It seems reasonable that the total population causes the employment and not vice versa. 
If we increase the number of people living in an area, this has a direct effect on the number of employed people.
We believe that the decision to move into an economically strong area is rather based on the employment rate rather than the absolute number of employed people. 
There might be an effect that the employment status influences the decision to get children
but we regard this effect to be less relevant.

\hypertarget{sec:D28}{\subsection*{D28: Milk protein trial}}

%Source: \\
%Citations: \citet{Verbyla1990} \\
%Pairs: 85 \\
%Samples: 994 \\

This data set is extracted from that for the milk protein trial used by \citet{Verbyla1990}. 
The original data set consists of assayed protein content of milk samples taken weekly from each of $79$ cows. 
The cows were randomly allocated to one of three diets: barley, mixed barley-lupins, and lupins, with $25$, $27$ and $27$ cows in the three groups, respectively. 
Measurements were taken for up to $19$ weeks but there were $38$ drop-outs from week $15$ onwards, corresponding to cows who stopped producing milk before the end of the experiment.
We removed the missing values (drop-outs) in the data set: we did not consider the measurements from week 15 onwards, which contain many drop-outs, and we discarded the cows with drop-outs before week 15. 
Finally, the data set contains 71 cows and 14 weeks, i.e., 994 samples in total. 
Furthermore, we re-organized the data set to see the relationship between the milk protein and the time to take the measurement.
We selected two variables: the time to take weekly measurements (from 1 to 14), and the protein content of the milk produced by each cow at that time.

\begin{figure}[h!]
\centerline{\small\begin{tabular}{c}
  \includegraphics[width=0.2\textwidth]{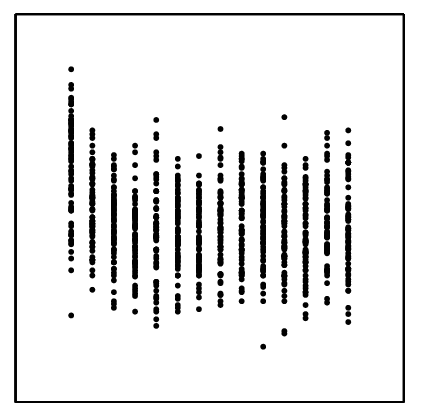}\\
  {\tt pair0085}
\end{tabular}}
\caption{\label{fig:D28}Scatter plots of pairs from D28.
{\tt pair0085}: time of measurement \causes protein content of milk.}
\end{figure}

\subsubsection*{{\tt pair0085}: Time of Measurement \causes Protein Content of Milk}
Clearly, the time of the measurement causes the protein content and not vice versa.
We do not consider the effect of the diets on the protein content.

\hypertarget{sec:D29}{\subsection*{D29: kamernet.nl data}}

%Source: \url{http://www.kamernet.nl} \\
%Citations: \\
%Pairs: 86 \\
%Samples: 666 \\

This data was collected by Joris M.\ Mooij from \url{http://www.kamernet.nl}, a Dutch website for matching supply and demand of rooms and appartments for students, in 2007.
The variables of interest are the size of the appartment or room (in $\mathrm{m}^2$) and the monthly rent in EUR. Two outliers (one with size
$0\,\mathrm{m}^2$, the other with rent of 1 EUR per month) were removed, after which 666 samples remained. 

\begin{figure}[h!]
\centerline{\small\begin{tabular}{c}
  \includegraphics[width=0.2\textwidth]{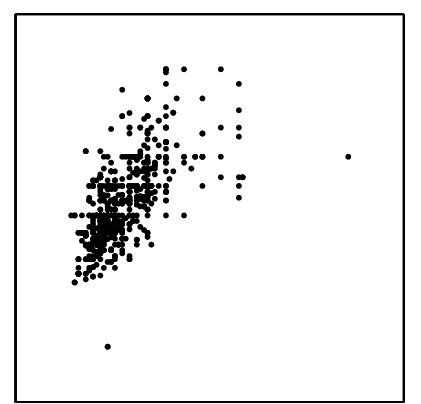}\\
  {\tt pair0086}
\end{tabular}}
\caption{\label{fig:D29}Scatter plots of pairs from D29.
{\tt pair0086}: size of apartment \causes monthly rent.}
\end{figure}

\subsubsection*{{\tt pair0086}: Size of apartment \causes Monthly rent}
Obviously, the size causes the rent, and not vice versa.

\hypertarget{sec:D30}{\subsection*{D30: Whistler daily snowfall}}

%Source: \url{http://www.mldata.org/repository/data/viewslug/whistler-daily-snowfall} 
%Citations: \\
%Pairs: 87 \\
%Samples: 7753 \\

The Whistler daily snowfall data is one of the data sets on \url{http://www.mldata.org}, and was originally obtained from
\url{http://www.climate.weatheroffice.ec.gc.ca/} (Whistler Roundhouse station, identifier 1108906). 
We downloaded it from \url{http://www.mldata.org/repository/data/viewslug/whistler-daily-snowfall}.
It concerns historical daily snowfall data in Whistler, BC, Canada, over the period July 1, 1972 to December 31, 2009. It was
measured at the top of the Whistler Gondola (Latitude: 50${}^\circ$04$'$04.000$''$\,N, Longitude: 122${}^\circ$56$'$50.000$''$\,W, Elevation: 1835\,m).
We selected two attributes, mean temperature (${}^\circ$C) and total snow (cm). The data consists of 7753 measurements of these two attributes.

\begin{figure}[h!]
\centerline{\small\begin{tabular}{c}
  \includegraphics[width=0.2\textwidth]{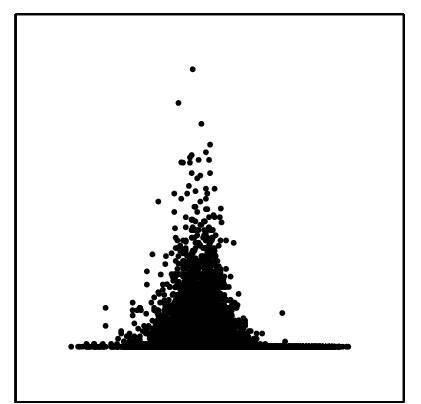}\\
  {\tt pair0087}
\end{tabular}}
\caption{\label{fig:D30}Scatter plots of pairs from D30.
{\tt pair0087}: temperature \causes total snow.}
\end{figure}

\subsubsection*{{\tt pair0087}: Temperature \causes Total Snow}

Common sense tells us that the mean temperature is one of the causes of the total amount of snow, although there may be a small feedback effect of
the amount of snow on temperature. Confounders are expected to be present (e.g., whether there are clouds).

\hypertarget{sec:D31}{\subsection*{D31: Bone Mineral Density}}

%Source: \url{http://www.mldata.org/repository/data/viewslug/whistler-daily-snowfall} 
%Citations: \\
%Pairs: 87 \\
%Samples: 7753 \\

This dataset comes from the {\tt R} package {\tt ElemStatLearn}, and contains
measurements of the age and the relative change of the bone mineral density of
261 adolescents. Each value is the difference in the spinal bone mineral
density taken on two consecutive visits, divided by the average. The age is the
average age over the two visits. We preprocessed the data by taking only the
first measurement for each adolescent, as each adolescent has 1--3 measurements.

\begin{figure}[h!]
\centerline{\small\begin{tabular}{c}
  \includegraphics[width=0.2\textwidth]{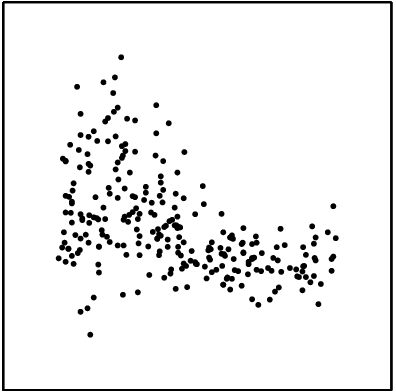}\\
  {\tt pair0088}
\end{tabular}}
\caption{\label{fig:D31}Scatter plots of pairs from D31.
{\tt pair0088}: age \causes relative bone mineral density.}
\end{figure}

\subsubsection*{{\tt pair0088}: Age \causes Bone mineral density}

Age must be the cause, bone mineral density the effect.

\hypertarget{sec:D32}{\subsection*{D32: Soil properties}}

%Source: \url{http://www.mldata.org/repository/data/viewslug/whistler-daily-snowfall} 
%Citations: \citep{Solly2014}
%Pairs: 89-92 \\
%Samples: max 150 \\

These data were collected within the Biodiversity Exploratories project, see \url{http://www.biodiversity-exploratories.de}. 
We used dataset 14686 (soil texture) and 16666 (root decomposition). 
With the goal to study fine root decomposition rates, \citet{Solly2014} placed litterbags containing fine roots in 150 forest and 150 grassland sites
along a climate gradient across Germany. Besides the decomposition rates, a range of other relevant variables
were measured, including soil properties such as clay content, soil organic carbon content and soil moisture. We deleted sites
with missing values and separated grasslands and forests.

\begin{figure}[h!]
\centerline{\small\begin{tabular}{cccc}
  \includegraphics[width=0.2\textwidth]{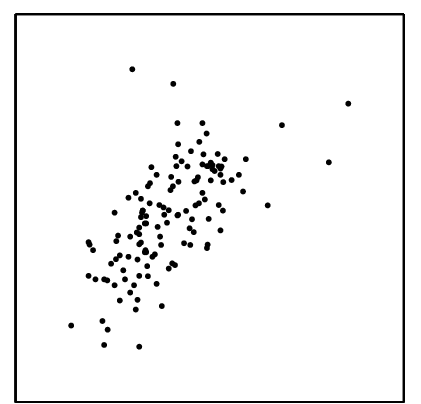}&
  \includegraphics[width=0.2\textwidth]{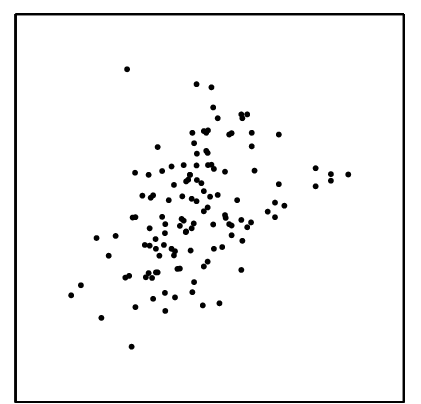}&
  \includegraphics[width=0.2\textwidth]{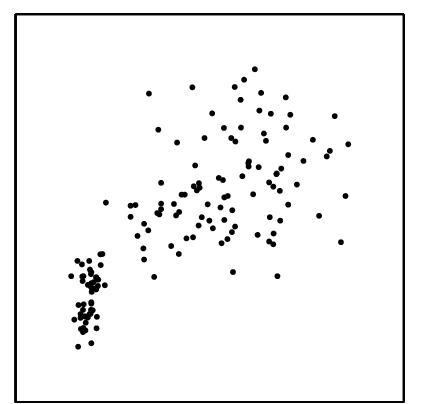}&
  \includegraphics[width=0.2\textwidth]{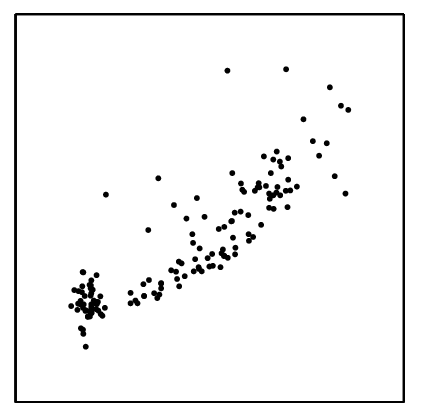}\\
  {\tt pair0089} & {\tt pair0090} & {\tt pair0091} & {\tt pair0092}
\end{tabular}}
\caption{\label{fig:D32}Scatter plots of pairs from D32.
{\tt pair0089}: Root decomposition in April \causes Root decomposition in October (Forests),
{\tt pair0090}: Root decomposition in April \causes Root decomposition in October (Grasslands),
{\tt pair0091}: Clay content in soil \causes Soil moisture (forests),
{\tt pair0092}: Clay content in soil \causes Organic carbon content (forests).}
\end{figure}

\subsubsection*{{\tt pair0089}--{\tt pair0090}: Root decomposition in April \causes Root decomposition in October}

Root decomposition happens monotonously in time. Hence the amount decomposed in April directly affects the amount decomposed in October in the same year.

\subsubsection*{{\tt pair0091}: Clay content in soil \causes Soil moisture}

The amount of water that can be stored in soils depends on its texture. The clay content of a soil influences whether precipitation is stored longer in soils or runs off immediately. In contrast, it is clear that wetness of a soil does not affect its clay content.

\subsubsection*{{\tt pair0092}: Clay content in soil \causes Organic carbon content}

How much carbon an ecosystem stores in its soil depends on multiple factors, including the land cover type, climate and soil texture. Higher amounts of clay are favorable for storage of organic carbon \citep{Solly2014}. Soil organic carbon, on the other hand, does not alter the texture of a soil.

\hypertarget{sec:D33}{\subsection*{D33: Runoff data}}

%Source: \url{http://www.mldata.org/repository/data/viewslug/whistler-daily-snowfall} 
%Citations: \\
%Pairs: 87 \\
%Samples: 7753 \\

This dataset comes from the MOPEX data base (\url{http://www.nws.noaa.gov/ohd/mopex/mo_datasets.htm}
and can be downloaded directly from \url{ftp://hydrology.nws.noaa.gov/pub/gcip/mopex/US_Data/Us_438_Daily/}.
It contains precipitation and runoff data from over 400 river catchments in the USA on a daily resolution from 1948 to 2004.
We computed yearly averages of precipitation and runoff for each catchment.

\begin{figure}[h!]
\centerline{\small\begin{tabular}{c}
  \includegraphics[width=0.2\textwidth]{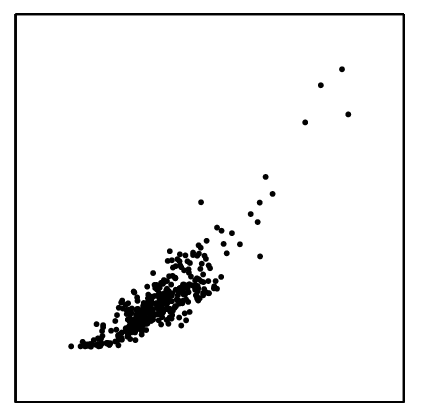}\\
  {\tt pair0093}
\end{tabular}}
\caption{\label{fig:D33}Scatter plots of pairs from D33.
{\tt pair0093}: Precipitation \causes Runoff.}
\end{figure}

\subsubsection*{{\tt pair0093}: Precipitation \causes Runoff}

Precipitation is by far the largest driver for runoff in a given river catchment. There might be a very small feedback from 
runoff that evaporates and generates new precipitation. This is, however, negligible if the catchment does not span over full continents.

\hypertarget{sec:D34}{\subsection*{D34: Electricity load}}

This data set comes from a regional energy distributor in Turkey. It contains three variables, the hour of the day, temperature in degree Celsius and electricity consumption (load) in MW per hour.
We are thank S. Armagan Tarim and Steve Prestwich for providing the data. 

\begin{figure}[h!]
\centerline{\small\begin{tabular}{ccc}
  \includegraphics[width=0.2\textwidth]{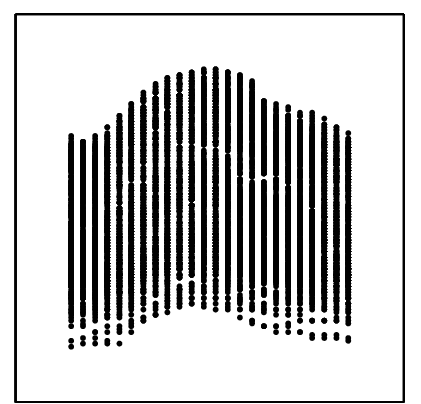}&
  \includegraphics[width=0.2\textwidth]{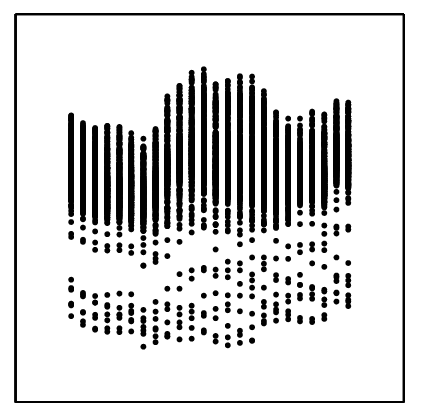}&
  \includegraphics[width=0.2\textwidth]{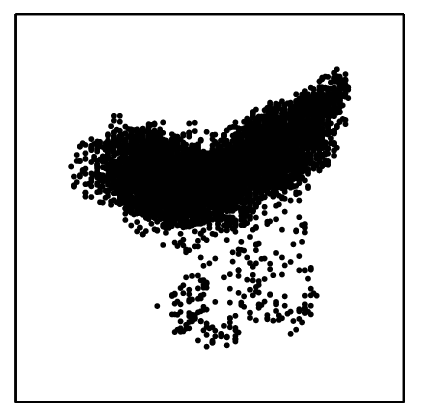} \\
  {\tt pair0094} & {\tt pair0095} & {\tt pair0096}
\end{tabular}}
\caption{\label{fig:D34}Scatter plots of pairs from D34.
{\tt pair0094}: Hour of the day \causes Temperature,
{\tt pair0095}: Hour of the day \causes Electricity consumption,
{\tt pair0096}: Temperature \causes Electricity consumption.}
\end{figure}

\subsubsection*{{\tt pair0094}: Hour of the day \causes Temperature}

We consider hour of the day as the cause, since it can be seen as expressing the angular position
of the sun. Although true interventions are unfeasible, it is commonly agreed that 
changing the position of the sun would result in temperature changes at a fixed location due to
the different solar incidence angle. 

\subsubsection*{{\tt pair0095}: Hour of the day \causes Electricity consumption}

The hour of the day constrains in many ways what people do and thus also their use of electricity. 
Consequently, we consider hour of the day as cause and electricity consumption as effect.

\subsubsection*{{\tt pair0096}: Temperature \causes Electricity consumption}

Changes in temperature can prompt people to use certain electric devices, e.g., an electric heating when it is gets very cold or the usage of a fan or air conditioning when it gets very hot. Furthermore, certain machines such as computers have to be cooled more if temperatures rise. Hence we consider temperature as cause and electricity consumption as effect. 

\hypertarget{sec:D35}{\subsection*{D35: Ball track}}

The data has been recorded by D.\ Janzing using a ball track that has been
equipped with two pairs of light barriers. The first pair measures the initial
speed and the second pair the speed of a ball at some later position of the
track. The units are arbitrary and differ for both measurements
since they are obtained by inverting the time the ball needed to pass the
distance between two light barriers of one pair.

\begin{figure}[h!]
\centerline{\small\begin{tabular}{cc}
  \includegraphics[width=0.2\textwidth]{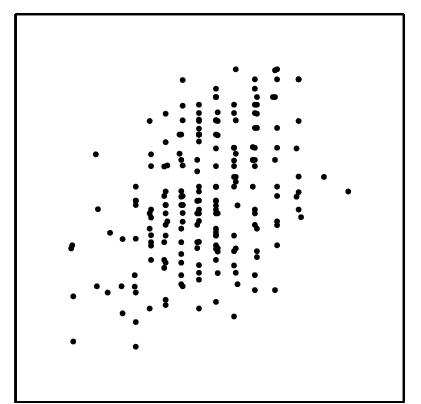}&
  \includegraphics[width=0.2\textwidth]{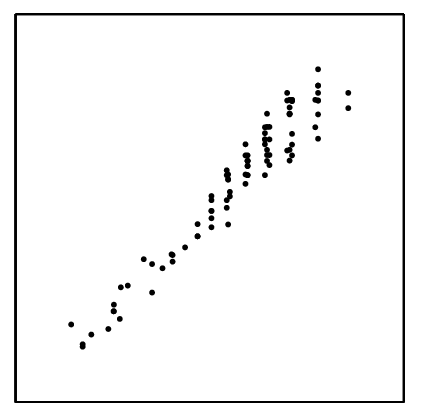} \\
  {\tt pair0097} & {\tt pair0098}
\end{tabular}}
\caption{\label{fig:D35}Scatter plots of pairs from D35.
{\tt pair0097}: Initial speed \causes Final speed,
{\tt pair0098}: Initial speed \causes Final speed.}
\end{figure}

The initial part of the track has large slope. The initial speed is strongly
determined by the exact position where the ball is put on the track. For part
of the runs, the position of the ball has been chosen by D.\ Janzing, the other
part by a 4-year old child. This should avoid that the variation of the initial
position is done in a too systematic way.

Two similar experiments have been performed, using different ball track setups. 
For {\tt pair0098} the ball track had a longer acceleration zone than for {\tt pair0097}, 
which allows for larger variations in initial speed.

\subsubsection*{{\tt pair0097}: Initial speed \causes Final speed}

These data consists of 202 measurements.
Obviously, the initial speed of the ball causes the final speed.

\subsubsection*{{\tt pair0098}: Initial speed \causes Final speed}

These data consist of 94 measurements.
Again, the initial speed of the ball causes the final speed.

\hypertarget{sec:D36}{\subsection*{D36: nlschools}}

This is dataset {\tt nlschools} from the {\tt R} package {\tt MASS}. The data
were used by \citet{SnijdersBosker1999} as a running example and are about a
study of 2287 eigth-grade pupils (aged about 11) in 132 classes in 131 schools
in the Netherlands. We used two variables: {\tt lang}, a language test score,
and {\tt SES}, the social-economic status of the pupil's family.

\begin{figure}[h!]
\centerline{\small\begin{tabular}{c}
  \includegraphics[width=0.2\textwidth]{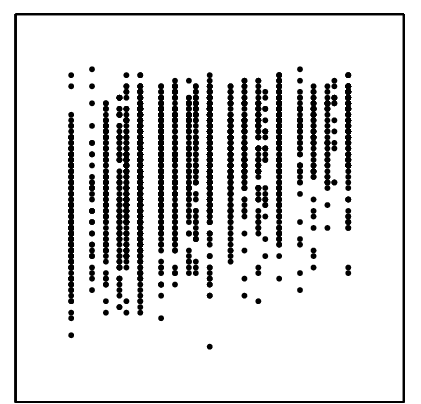}\\
  {\tt pair0099}
\end{tabular}}
\caption{\label{fig:D36}Scatter plots of pairs from D36.
{\tt pair0099}: Social-economic status of family \causes Language test score.}
\end{figure}

\subsubsection*{{\tt pair0099}: Social-economic status of family \causes Language test score}

We consider the social-economic status of the pupil's family to be the cause of the
language test score of the pupil. However, note that selection bias may be present
via the choice of the schools to include in the study.

\hypertarget{sec:D37}{\subsection*{D37: cpus}}

This is dataset {\tt cpus} from the {\tt R} package {\tt MASS}, and concerns
characteristics of 209 CPUs \citep{EinDorFeldmesser1987}. We used two variables:
{\tt syct}, cycle time in nanoseconds, and {\tt perf}, the published performance
on a benchmark mix relative to an IBM 370/158-3, and took the logarithms of the original values.

\begin{figure}[h!]
\centerline{\small\begin{tabular}{c}
  \includegraphics[width=0.2\textwidth]{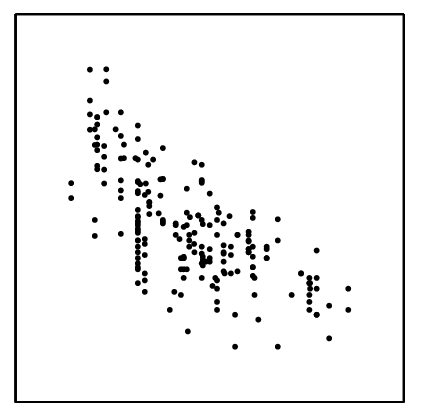}\\
  {\tt pair0100}
\end{tabular}}
\caption{\label{fig:D37}Scatter plots of pairs from D37.
{\tt pair0100}: CPU cycle time \causes Performance.}
\end{figure}

\subsubsection*{{\tt pair0100}: CPU cycle time \causes Performance}

It should be obvious that CPU cycle time causes its performance.

\section{Computation time}\label{sec:computation_times} %%%%%%%%%%%%%%%%%%%%%%%%%%%%%%%%%%%%%%%%%%%%%%%%%%%%%%%%%%%%%%%%%%%%%%%%%%%%%%%%%%%%%%%%%%%%%%%%%%%%%%%%%%%%%%%%%%%%%%%%%%%%%%%%%%%%%%%%%%%%%%%%

We report the total computation time for each benchmark set and for each of our
implementations of various methods in Figures~\ref{fig:time_ANM_all_data} and
\ref{fig:time_IGCI_uniform_all_data}. We used a machine with \texttt{Intel Xeon
CPU E5-2680 v2 @ 2.80GHz} processors, 40 cores, and 125 GB of RAM. The measured
computation time measures the total time spent (i.e., the sum of the
computation times of individual cores). We did not spend much effort on
optimizing the implementations, so the reported computation times should be seen
as upper bounds on what is achievable. We only report results for the
unperturbed data, as the preprocessing does not affect computation time
significantly. 

In general, for the ANM implementations, most time is taken by the Gaussian
Process regression. The HSIC test and entropy estimators are relatively quick
compared to that. A notable outlier is \texttt{ANM-MML} which spends much time
on estimating the MML of the marginal distribution using the algorithm
by \cite{FigueiredoJain2002}. IGCI implementations
are much faster than ANM (about two orders of magnitude in our setting), as
non-parametric regression is not required. One notable outlier for the IGCI
implementations is \texttt{IGCI-ent-PSD}, which shows that the 
\texttt{ent-PSD} estimator is slower than the other entropy estimators in the 
\texttt{ITE} toolbox. Interestingly, this is also the only non-parameteric entropy
estimator that turned out to be robust to perturbations of the data.

\begin{figure}[p]
\centering
\includegraphics[scale=0.9]{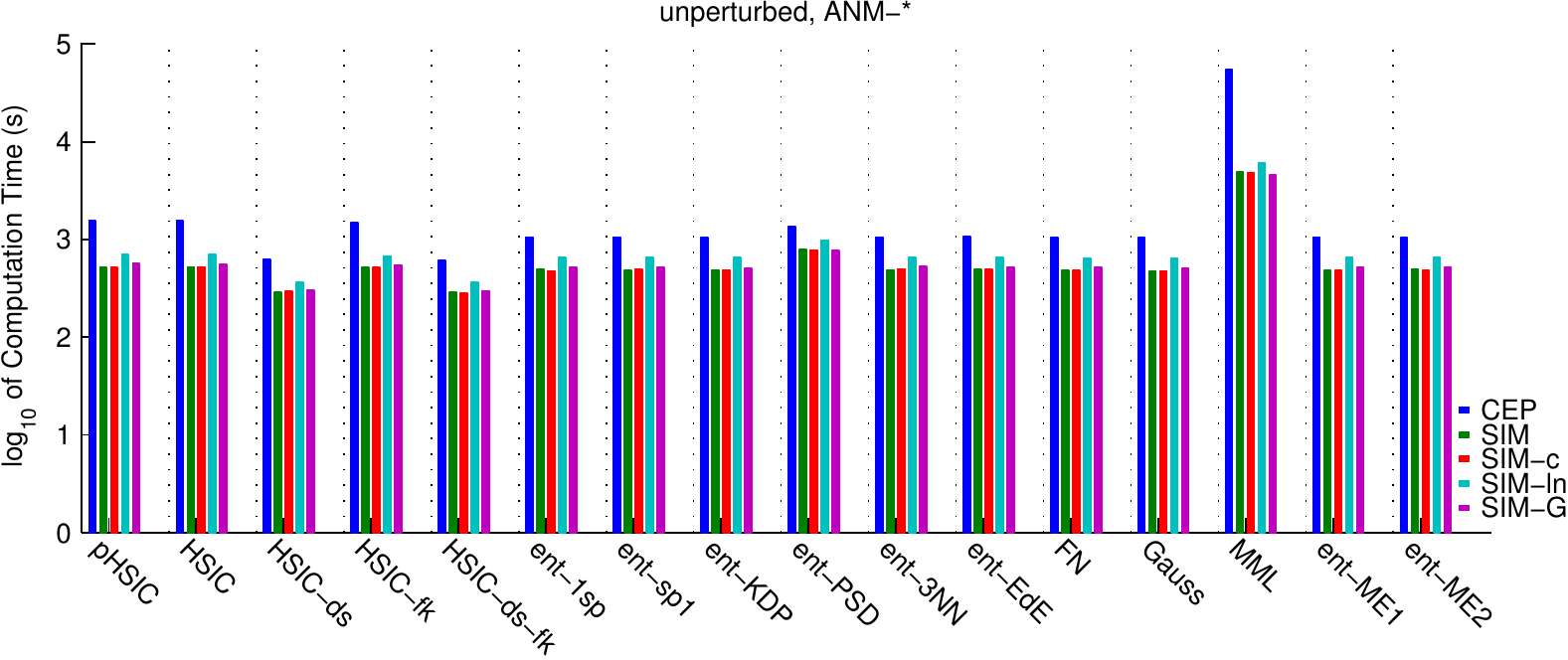}
\caption{\label{fig:time_ANM_all_data}Computation times of various ANM methods on different (unperturbed) data sets. For the variants of the spacing estimator, only the results for \texttt{sp1} are shown, as results for \texttt{sp2},\dots,\texttt{sp6} were similar.}
\end{figure}

\begin{figure}[p]
\centering
\includegraphics[scale=0.9]{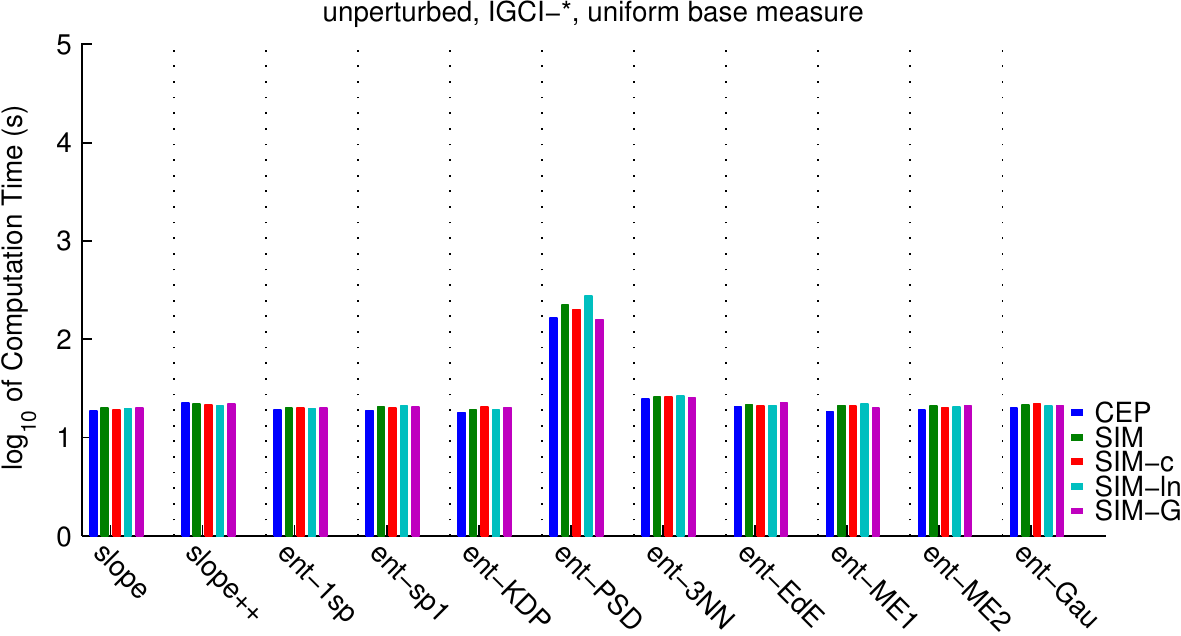}
\caption{\label{fig:time_IGCI_uniform_all_data}Computation times of various IGCI methods on different (unperturbed) data sets. For the variants of the spacing estimator, only the results for \texttt{sp1} are shown, as results for \texttt{sp2},\dots,\texttt{sp6} were similar. We only show results for the uniform base measure as those for the Gaussian base measure are similar.}
\end{figure}

%\begin{figure}[p]
%\centering
%\includegraphics[scale=0.9]{IGCI_all_N=Inf_Gaussian_time.pdf}
%\caption{\label{fig:time_IGCI_Gaussian_all_data}Computation times of various IGCI methods on different (unperturbed) data sets. For the variants of the spacing estimator, only the results for \texttt{sp1} are shown, as results for \texttt{sp2},\dots,\texttt{sp6} were similar.}
%\end{figure}

\section*{Acknowledgements} %%%%%%%%%%%%%%%%%%%%%%%%%%%%%%%%%%%%%%%%%%%%%%%%%%%%%%%%%%%%%%%%%%%%%%%%%%%%%%%%%%%%%%%%%%%%%%%%%%%%%%%%%%%%%%%%%%%%%%%%%%%%%%%%%%%%%%%%%%%%%%%%%%%%%%%%%%%%%%%%%%%%%%%%%%%

JMM was supported by NWO, the Netherlands Organization for Scientific Research (VIDI grant 639.072.410). 
JP received funding from the People Programme (Marie Curie Actions) of the European Union's Seventh Framework Programme (FP7/2007-2013) under REA grant agreement no 326496.
The authors thank Stefan Harmeling for fruitful discussions and providing the code to create Figure~\ref{fig:allpairs_CEP}. We also thank S.\ Armagan Tarim and Steve Prestwich for contributing cause-effect pairs \texttt{pair0094}, \texttt{pair0095}, and \texttt{pair0096}. Finally, we thank several anonymous reviewers for their comments that helped us to improve the drafts.

% REFERENCES:
%References follow the acknowledgments. Use unnumbered third level heading for
%the references. Any choice of citation style is acceptable as long as you are
%consistent. It is permissible to reduce the font size to `small' (9-point) 
%when listing the references.

%\bibliographystyle{unsrt} %apalike, chicago, unsrt, plain
%\bibliographystyle{plainnat}
%\renewcommand{\refname}{{\subsection*{References}}}

\small
\bibliography{dataset,article,consistency}
\end{document}